\documentclass[11pt,a4paper]{article}
\usepackage{a4wide}
\usepackage{hyperref}
\usepackage{breakurl}

\usepackage[dvipdfmx]{graphicx}
\usepackage{epsfig} % less modern
\usepackage{subfigure} 
\usepackage{algorithm}
\usepackage{algorithmic}

\usepackage{comment}
\usepackage{url}

\usepackage{graphicx} % to include figur\usepackage{kasai}es
\usepackage{epsfig}
\usepackage{subfigure}

\usepackage{cite}

\def\vec#1{\mbox{\boldmath $#1$}}
\def\mat#1{\mbox{\bf #1}}
\usepackage{additional_def}

\newcommand{\gradf}{{\rm grad} f}
\newcommand{\hessf}{{\rm Hess} f}
\newcommand{\gradh}{{\rm grad} h}

\usepackage{amsmath}
\usepackage{amsthm}
\newtheorem{Def}{Definition}[section]
\newtheorem{Thm}[Def]{Theorem}
\newtheorem{Lem}[Def]{Lemma}
\newtheorem{Prop}[Def]{Proposition}
\newtheorem{assumption}{Assumption}
\newtheorem{Cor}[Def]{Corollary}
\newtheorem{Rmk}[Def]{Remark}

\newcommand{\changeHK}[1]{#1}
\newcommand{\changeHKK}[1]{#1}
\newcommand{\changeHS}[1]{#1}
\newcommand{\changeHSS}[1]{#1}
\newcommand{\changeBM}[1]{#1}

\newcommand{\note}[1]{#1}

\title{Riemannian stochastic quasi-Newton algorithm with \\variance reduction and its convergence analysis}
\date{\today}

\author{Hiroyuki Kasai\thanks{Graduate School of Informatics and Engineering, The University of Electro-Communications, Tokyo, Japan ({\tt kasai@is.uec.ac.jp}).} \and Hiroyuki Sato\thanks{Department of Information and Computer Technology, Tokyo University of Science, Tokyo, Japan ({\tt hsato@rs.tus.ac.jp}).} \and Bamdev Mishra\thanks{Core Machine Learning Team, Amazon.com, Bangalore, India ({\tt bamdevm@amazon.com}.)}}

\begin{document}

\maketitle

\begin{abstract}
Stochastic variance reduction algorithms have recently become popular for minimizing the average of a large, but finite number of loss functions. The present paper proposes a Riemannian stochastic quasi-Newton algorithm with variance reduction (R-SQN-VR). The key challenges of averaging, adding, and subtracting multiple gradients are addressed with notions of retraction and vector transport. We present \changeHK{convergence analyses of R-SQN-VR on both \changeHK{non-convex and retraction-convex} functions under retraction and vector transport \changeHK{operators}.} The proposed algorithm is evaluated on the Karcher mean computation on the symmetric positive-definite manifold and the low-rank matrix completion on the Grassmann manifold. In all cases, the proposed algorithm outperforms the \changeHK{state-of-the-art Riemannian batch and stochastic gradient algorithms}. 
\end{abstract}

\section{Introduction}
\label{Sec:Introduction}
Let $f: \mathcal{M} \rightarrow \mathbb{R}$ be a smooth real-valued function on a Riemannian manifold $\mathcal{M}$ \cite{Absil_OptAlgMatManifold_2008}. The problem under consideration in the present paper is the minimization of the {\it expected risk} of $f$ for a given model variable $w \in \mathcal{M}$ taken with respect to the distribution of $z$\changeHK{, i.e.,} 
%\begin{eqnarray*}
	%\label{Eq:ExpectedRisk}
	$\min_{w \in \mathcal{M}} f(w)$, 
%\end{eqnarray*}
where $f(w) = \mathbb{E}_{z}[f(w; z)] = \int f(w; z) dP(z)$ and $z$ is a random seed representing a single sample or set of samples. When given a set of realizations $\{ z_{[n]}\}_{n=1}^N$ of $z$, we define the loss incurred by the parameter vector $w$ with respect to the $n$-th sample as 
%\begin{eqnarray*}
	%\label{Eq:}
	$f_n(w)  :=  f(w;z_{[n]})$,
%\end{eqnarray*}
and then the {\it empirical risk} is \changeHK{defined} as the average of the sample losses:
\begin{eqnarray}
	\label{Eq:EmpiricalRisk}
	\min_{w \in \mathcal{M}}  \left\{ f(w) := \frac{1}{N} \sum_{n=1}^N f_n(w) \right\}, 
\end{eqnarray}

where $N$ is the total number of the elements.
%A general loss minimization problem is defined as $\min_{w} f(w)$, where $f(w):= \frac{1}{N} \sum_{n=1}^N f_n(w)$, $w$ is the model variable, $N$ is the number of samples, and $f_n(w)$ is the loss incurred on $n$-th sample. 
This problem has many applications  that include, to name a few, principal component analysis (PCA) and the subspace tracking problem \cite{Balzano_arXiv_2010_s} on the Grassmann manifold.
%, which is the set of $r$-dimensional linear subspaces in $\mathbb{R}^d$. 
The low-rank matrix/tensor completion problem is a promising example \changeHK{of} the manifold of fixed-rank matrices/tensors \cite{Mishra_ICDC_2014_s,Kasai_ICML_2016_s}. The linear regression problem is also defined on the manifold of the fixed-rank matrices \cite{Meyer_ICML_2011}. 
%The independent component analysis (ICA) is an example defined on the oblique manifold.

{\it Riemannian gradient \changeHS{descent}} requires the {\it Riemannian full gradient} estimation, i.e., $\gradf(w)=\sum_{n=1}^N \gradf_n(w)$, \changeHK{for} every iteration, where $\gradf_n(w)$ is the {\it Riemannian stochastic gradient}  of $f_n(w)$ on the Riemannian manifold $\mathcal{M}$ for $n$-th sample. This \changeHK{estimation} is computationally heavy when $N$ is \changeHK{extremely} large. A popular alternative is {\it Riemannian stochastic gradient descent} (R-SGD) that extends {\it stochastic gradient descent} (SGD) in the Euclidean space \cite{Bonnabel_IEEETAC_2013_s}. \changeHK{Because} this uses only one $\gradf_n(w)$, the complexity per iteration \changeHK{is independent of} $N$. However, \changeHK{similarly} to SGD \cite{Robbins_MathStat_1951}, 
% R-SGD requires {\it decaying step-size} sequences \changeHS{which} start with a large step-size and decrease with iterations to guarantee its convergence, but this causes a slow convergence.
\changeHK{R-SGD suffers from a slow convergence due to a {\it decaying step-size} sequence.}
{\it Variance reduction} (VR) methods have been \changeHK{proposed recently} to accelerate the convergence of SGD in the Euclidean \changeHS{space} \cite{Johnson_NIPS_2013_s, Roux_NIPS_2012_s,Shalev_JMLR_2013_s,Defazio_NIPS_2014_s,Reddi_ICML_2016_s}. 
One distinguished feature is to calculate a full gradient estimation periodically, and \changeHK{to} re-use it to reduce the variance of noisy stochastic gradient. However, \changeHK{because} \changeHK{all previously described} algorithms are {\it first-order} algorithms, their convergence speed can be slow \changeHK{because of} 
their poor curvature approximations in {\it ill-conditioned} problems \changeHK{as seen in Section \ref{Sec:Numerical_comparisons}}. One promising approach is {\it second-order} algorithms such as stochastic {\it quasi-Newton} (QN) methods using Hessian evaluations \cite{Schraudolph_AISTATS_2007_s,Mokhtari_IEEETranSigPro_2014,Wang_arXiv_2016,Byrd_SIOPT_2016}. They achieve faster convergence \changeHK{by} exploiting curvature information of \changeHK{the objective function} $f$. Furthermore, addressing these two acceleration techniques, \cite{Moritz_AISTATS_2016_s} and \cite{Kolte_OPT_2015} propose a hybrid algorithm of the stochastic QN method accompanied with the VR method. 

\changeHK{Examining the} Riemannian \changeHS{manifolds} \changeHK{again}, many challenges on the QN method have been \changeHK{addressed} in deterministic settings \cite{Gabay_JOTA_1982,Ring_SIAMJO_2012,Huang_SIOPT_2015}. The VR method in the \changeBM{Euclidean space} has \changeHK{also been} extended \changeBM{to} Riemannian manifolds, so\changeHK{-}called R-SVRG \cite{Sato_arXiv_2017,Zhang_NIPS_2016}. Nevertheless, the second-order stochastic algorithm with the VR method \changeBM{has} not been explored thoroughly for \changeHK{the} problem (\ref{Eq:EmpiricalRisk}). \changeBM{To this end,} we propose a Riemannian stochastic QN method based on  L-BFGS and the VR method. 

\changeBM{Our contributions are \changeHK{four}-fold;}
(i) we propose a novel (and to the best of our knowledge, the first) Riemannian limited-memory \changeHK{QN} algorithm with a \changeHK{VR} method.
(ii) Our convergence analysis deals with \changeHK{both {\it non-convex} and ({\it strongly}) {\it retraction-convex} functions.  In this paper, $f$ is said to be strongly retraction-convex when $f$ is (strongly) convex along a curve on $\mathcal{M}$ defined by a retraction $R$ (\changeHKK{Assumption} \ref{Assump:3}) while the other functions are called as non-convex functions.}
%This \changeHK{analysis} is typical for batch algorithms on Riemannian manifolds \cite{Absil_OptAlgMatManifold_2008}. 
%For the latter, we specifically deal with a local convergence rate analysis based on the assumption imposed only in a local neighborhood around a minimum, which are milder and natural \cite{Absil_OptAlgMatManifold_2008}. 
(iv) \changeHK{The proposed algorithm and its analyses are considered under} \changeHK{computationally efficient} {\it retraction} and {\it vector transport} operations instead of the more restrictive {\it exponential mapping} and {\it parallel translation} operations. 
\changeHK{This is more challenging than R-SVRG \cite{Zhang_NIPS_2016}, but gives us a big advantage other than computational efficiency, i.e., wider kinds of applicable manifolds. For example, while \cite{Zhang_NIPS_2016} cannot be applied to the Stiefel and fixed-rank manifolds because these manifolds do not have closed form expressions for parallel translation, our analyses and algorithm can be directly applied to them.}

The specific features of the algorithms are two\changeHK{-}\changeHK{fold};
(i) we update \changeHK{the} curvature pair of the QN method every outer loop by exploiting full gradient estimations in the VR method, and thereby capture more precise and stabler curvature information. This avoids additional sweeping of samples required in the Euclidean \changeHK{stochastic} QN \cite{Byrd_SIOPT_2016}\changeHK{,} additional gradient estimations required in the Euclidean \changeHS{online BFGS (oBFGS)} \cite{Mokhtari_IEEETranSigPro_2014,Schraudolph_AISTATS_2007_s,Mokhtari_JMLR_2015_s}\changeHK{, or additional sub-sampling of Hessian \cite{Byrd_SIOPT_2016,Moritz_AISTATS_2016_s}}.
(ii) \changeHK{Compared} with a simple Riemannian extension of the QN method, \changeHK{a} noteworthy advantage of its combination with the VR method is that, as revealed below, frequent transportations of \changeHK{curvature} information between different tangent spaces, which \changeHK{are} inextricable in such a simple Riemannian extension, can be \changeHK{drastically reduced}. This is \changeHK{a} special benefit of the Riemannian hybrid algorithm, which does not exist in the Euclidean case \cite{Moritz_AISTATS_2016_s,Kolte_OPT_2015}. More specifically, the calculation\changeHK{s} of curvature information and the {\it second-order modified Riemannian stochastic gradient} are \changeHK{performed uniformly} on the tangent space of the outer loop.

The paper is organized as follows. Section 2 \changeHK{presents} details \changeHK{of} our proposed R-SQN-VR. Section 3 presents the convergence analyses. In \changeHS{Section} 4, numerical comparisons with R-SGD and R-SVRG on two problems are provided with results suggesting the superior performances of R-SQN-VR. The proposed R-SQN-VR is implemented in the Matlab toolbox Manopt \cite{Boumal_Manopt_2014_s}. The concrete proofs of theorems and additional experiments are provided as supplementary material.

\section{Riemannian stochastic quasi-Newton algorithm with variance reduction (R-SQN-VR)}
\label{Sec:Algorithm}

%{\small
\begin{table}[htbp]
\hrule height 0.1mm depth 0.2mm width 140mm
{\bf Algorithm 1} Riemannian stochastic quasi-Newton with variance reduction (R-SQN-VR).
\hrule height 0.1mm depth 0.1mm width 140mm
\begin{algorithmic}[1]
\REQUIRE{Update frequency $m_k$, step-size $\alpha_t^k>0$, memory size $L$, \changeHK{number of epochs K,} and cautious update threshold $\epsilon$. }
\STATE{Initialize $\tilde{w}^0$, and calculate the Riemannian full gradient $\gradf(\tilde{w}^{0})$}.
\FOR{$k=\changeHK{0},\ldots, \changeHK{K-1}$} 
\STATE{Store $w_0^k = \tilde{w}^{k}$.}
	\FOR{$t=0,1,2, \ldots, m_{k}-1$} 
	\STATE{Choose $i_t^k \in \{1, \ldots, N\}$ uniformly at random.}
	\STATE{Calculate the tangent vector $\tilde{\eta}_t^k$ from $\tilde{w}^{k}$ to $w_{t}^{k}$ by $\tilde{\eta}_t^k =R^{-1}_{\tilde{w}^{k}}(w_{t}^{k})$.}	
	\IF{$k>1$}
		%\STATE{\% Stochastic quasi-Newton with gradient variance reduction routine.}
		%\STATE{Calculate the tangent vector $\tilde{\eta}_t^k$ from $\tilde{w}^{k}$ to $w_{t}^{k}$ by $\tilde{\eta}_t^k =R^{-1}_{\tilde{w}^{k}}(w_{t}^{k})$.}
		\STATE{Transport the stochastic gradient $\gradf_{i_t^k}(w_{t}^{k})$ to $T_{\tilde{w}^{k}}\mathcal{M}$ by $(\mathcal{T}_{\tilde{\eta}_t^k})^{-1}\gradf_{i_t^k}(w_{t}^{k})$.}
		\STATE{Calculate $\tilde{\xi}_t^k$ as
	{$\tilde{\xi}_t^k = (\mathcal{T}_{\tilde{\eta}_t^k})^{-1}\gradf_{i_t^k}(w_{t}^{k}) - \changeHK{(} \gradf_{i_t^k}(\tilde{w}^{k}) -\gradf(\tilde{w}^{k}) \changeHK{)}$}.}
		\STATE{Calculate $\tilde{\mathcal{H}}^k_t\tilde{\xi}_t^k$, transport $\tilde{\mathcal{H}}^k_t\tilde{\xi}_t^k$ back to $T_{w_{t}^{k}}\mathcal{M}$ by $\mathcal{T}_{\tilde{\eta}_t^k} \tilde{\mathcal{H}}^k_t\tilde{\xi}_t^k$, and obtain $\mathcal{H}^k_t \xi_t^k$.}
		\STATE{Update $w_{t+1}^k$ from $w_{t}^k$ as $w_{t+1}^k = R_{\scriptsize w_{t}^k}(- \alpha^k_t \mathcal{H}^k_t \xi_t^k )$.}	
	%	}
	\ELSE
		%\STATE{\% Normal SVRG routine.}
		%\STATE{Calculate the tangent vector $\tilde{\zeta}_t^k$ from $\tilde{w}^{k}$ to $w_{t}^{k}$ {by $\zeta_{\tilde{w}^{k-1}} = R^{-1}_{\tilde{w}^{k}}(w_{t}^k)$}.}
		\STATE{Calculate $\xi_t^k$ 
		%by transporting $\gradf(\tilde{w}^{k-1})$ and $\gradf_{i_t^k}(\tilde{w}^{k-1})$ along $\zeta_{\tilde{w}^{k-1}}$ 
		as {$\xi_t^k = \gradf_{i_t^k}(w_{t}^{k}) -  \mathcal{T}_{\tilde{\eta}_t^k} (\gradf_{i_t^k}(\tilde{w}^{k})  - \gradf(\tilde{w}^{k}))$}.}
		\STATE{Update $w_{t+1}^k$ from $w_{t}^k$ as {$w_{t+1}^k = R_{\scriptsize w_{t}^k}(- \alpha^k_t \xi_t^k )$}.}	
	\ENDIF
	\ENDFOR
	\STATE{{Option I}: $\tilde{w}^{k+1}=g_{m_k}(w_{1}^k,\ldots,w_{m_k}^k)$ (or $\tilde{w}^{k+1}=w_t^k$ for randomly chosen $t \in \{1, \ldots, m_k\}$).}	
	\STATE{{Option II}: $\tilde{w}^{k+1}=w^k_{m_k}$.}
	%
	%
	%
	%\STATE{\% Update and store pair $(s,y)$}
	\STATE{Calculate the Riemannian full gradient $\gradf(\tilde{w}^{k+1})$}.	
	\STATE{Calculate the tangent vector $\eta_k$ from $\tilde{w}^{k}$ to $\tilde{w}^{k+1}$ {by $\eta_k = R^{-1}_{\tilde{w}^{k}}(\tilde{w}^{k+1})$}.}	
	\STATE{Compute $s^{k+1}_k = \mathcal{T}_{\eta_k}\changeHK{\eta_k}$, and 
	$y^{k+1}_k = \kappa_k^{-1} \gradf(\tilde{w}^{k+1}) - \mathcal{T}_{\eta_k} \gradf(\tilde{w}^{k})$ where $\kappa_k = \| \eta_k \|\changeHK{_{\tilde{w}^k}} / \| \mathcal{T}_{R_{\eta_k}} \eta_k \|\changeHK{_{\tilde{w}^k}}$.}		
	\IF{$\langle y_k^{k+1}, s_k^{k+1}\rangle_{\tilde{w}^{k+1}} \geq \changeHK{\epsilon \| s_k^{k+1} \|^2_{\tilde{w}^{k+1}}}$}
	\STATE{Discard pair $\changeHK{(}s^{k}_{k-L}, y^{k}_{k-L}\changeHK{)}$ when $k>L$, and store pair $(s^{k+1}_k, y^{k+1}_k)$.}
	\ENDIF	
	\STATE{Transport $\{(s^{k}_{\changeHK{j}}, y^{k}_{\changeHK{j}})\}_{\changeHK{j}=k-\tau+1}^{\changeHK{k-1}} \in T_{\tilde{w}^{k}}\mathcal{M}$ to $\{(s^{k+1}_{\changeHK{j}}, y^{k+1}_{\changeHK{j}})\}_{\changeHK{j}=k-\tau+1}^{\changeHK{k-1}} \in T_{\tilde{w}^{k+1}}\mathcal{M}$ by $\mathcal{T}_{\eta_k}$%, where $\tau={\rm min}(k,L)$
	.}	
\ENDFOR
\STATE{\changeHK{Option III: output $w_{\rm sol}=\tilde{w}^{K}$}}
\STATE{\changeHK{Option IV: output $w_{\rm sol}=w_t^k$ for randomly chosen $t \in \{1, \ldots, m_k\}$ and $k \in \{1, \ldots, K\}$}.}
\end{algorithmic}
\hrule height 0.1mm depth 0.1mm width 140mm
\end{table}

We assume that the manifold $\mathcal{M}$ is endowed with a Riemannian metric structure, i.e., a smooth inner product $\langle \cdot, \cdot \rangle_w$ of tangent vectors is associated with the tangent space $T_w \mathcal{M}$ for all $w \in \mathcal{M}$ \cite{Absil_OptAlgMatManifold_2008}. The {\it norm} $\| \cdot\|_w$ of a tangent vector is the norm associated with the Riemannian metric. 
%The metric and the norm are defined on the tangent space $T_w \mathcal{M}$. 
The metric structure allows a systematic framework for optimization over manifolds. Conceptually, the constrained optimization problem (\ref{Eq:EmpiricalRisk}) is translated into an {\it unconstrained} problem over $\mathcal{M}$. 
%Consequently, notions such as the Riemannian gradient,
%%(first\changeHK{-}order derivatives of an objective function)
%tangent space, and moving along a search direction have well-known expressions for a number of manifolds.

\subsection{R-SGD and R-SVRG}
\noindent 
{\bf R-SGD:}
Given a starting point $w_{\changeHK{0}} \in \mathcal{M}$, R-SGD produces a sequence 
%$(w_t)_{t\geq \changeHK{0}}$ 
\changeHK{$\{w_t\}$}
in $\mathcal{M}$ that converges to a first\changeHK{-}order critical point of (\ref{Eq:EmpiricalRisk}). Specifically, \changeHK{it} updates $w$ as
\begin{eqnarray*}
	w_{t+1}  &=&  R_{w_{t}}(-\alpha_t \gradf_n(w_t, z_t)),
\end{eqnarray*}	
where $\alpha_t$ is the step-size, and \changeHK{where} $\gradf_n(w_t, z_t)$ is a Riemannian stochastic gradient\changeHS{,} which is a tangent vector at $w_t \in \mathcal{M}$. $\gradf_n(w_t, z_t)$ represents an unbiased estimator of the Riemannian full gradient $\gradf(w_t)$, and the expectation of $\gradf_n(w_t, z_t)$ over the choices of $z_t$ \changeHK{is} $\gradf(w_t)$, i.e., 
$\mathbb{E}_{z_t}[\gradf_n(w_t, z_t)]
%=\int g(w_t, z_t) dP( z_t)
=\gradf(w_t)$. 
The update moves from $w_t$ in the direction $-\gradf_n(w_t, z_t)$ with a step-size $\alpha_t$ while \changeHK{remaining} on $\mathcal{M}$. 
%along the geodesic from the current iterate position $w_t$ in the stochastic direction $-g(w_t, z_t)$ with a step-size $\alpha_t$. This update is a straightforward extension of the standard gradient update in the Euclidean case. 
This mapping, denoted as $R_{w} :T_{w}\mathcal{M} \rightarrow \mathcal{M}: \zeta_w \mapsto R_{w} \changeHS{(}\zeta_w\changeHS{)}$, is called  {\it retraction} at $w$\changeHS{, which} \changeHK{maps} the tangent bundle $T_w\mathcal{M}$ \changeHK{onto} $\mathcal{M}$ with a local rigidity condition that preserves gradients at $w$. {\it Exponential mapping} ${\rm Exp}$ is an instance of the retraction.

\noindent 
{\bf R-SVRG:}
R-SVRG has double loops where \changeHK{a} $k$-th outer loop, called {\it epoch}, has $m_k$ inner iterations. R-SVRG keeps $\tilde{w}^{k} \in \mathcal{M}$ after $m_{k-1}$ inner iterations of $(k\!-\!1)$-th epoch, and computes the full Riemannian gradient $\gradf (\tilde{w}^{k})$ only for this stored $\tilde{w}^{k}$. \changeHK{It} also computes the Riemannian stochastic gradient $\gradf_{i_t^{k}}(\tilde{w}^{k})$ \changeHK{for} $i_t^{k}$-th sample. Then, picking $i_t^{k}$-th sample for each $t$-th inner iteration of $k$-th epoch at $w_{t}^{k}$, we calculate $\xi_t^{k}$, i.e., by modifying $\gradf_{i_t^k}(w_{t}^{k})$ using both $\gradf (\tilde{w}^{k})$ and $\gradf_{i_t^{k}}(\tilde{w}^{k})$. \changeHK{Because} they belong to different tangent spaces, a simple addition of them \changeHK{is} not well-defined because Riemannian manifolds are not vector \changeHS{spaces}. Therefore, after $\gradf_{i_t^{k}}(\tilde{w}^{k})$ and $\gradf(\tilde{w}^{k})$ are transported to $T_{\scriptsize w_{t}^{k}}\mathcal{M}$ \changeHK{by $\mathcal{T}_{\tilde{\eta}_t^k}$}, 
%then they are ready to be added to $\gradf_{i_t^{k}}(w_{t-1}^{k})$ on $T_{\scriptsize w_{t-1}^{k}}\mathcal{M}$. Consequently, 
$\xi_t^{k}$ is set as
\begin{eqnarray*}
\xi_t^{k} &=& \gradf_{i_t^{k}}(w_{t}^{k})- \mathcal{T}_{\tilde{\eta}_t^k}(\gradf_{i_t^{k}}(\tilde{w}^{k}) \changeHK{-} \gradf(\tilde{w}^{k})), 
\end{eqnarray*}
where $\mathcal{T}$ represents {\it vector transport} from $\tilde{w}^{k}$ to $w_{t}^{k}$, and $\tilde{\eta}_t^k \in T_{\tilde{w}^{k}} \mathcal{M}$ satisfies $R_{\tilde{w}^{k}}(\tilde{\eta}_t^k) = w^k_t$. 
The vector transport $\mathcal{T}: T\mathcal{M} \oplus T\mathcal{M} \rightarrow T\mathcal{M}, (\eta_w, \xi_w) \mapsto \mathcal{T}_{\eta_w} \xi_w$ is associated with retraction $R$ and all $\xi_w, \zeta_w \in \mathcal{T}_w \mathcal{M}$. It holds that (i) $\mathcal{T}_{\eta_w} \xi_w \in \mathcal{T}_{R(\eta_w)}\mathcal{M}$, (ii) $\mathcal{T}_{0_w} \xi_w =\xi_w$, \changeHS{and} (iii) $\mathcal{T}_{\eta_w}$ is a linear map. {\it Parallel translation} $P$ is an instance of the vector transport. \changeHK{Consequently,} the final update is defined as $w_{t+1}^{k} =  R_{\scriptsize w_{t}^{k}}( - \alpha^k_t \xi_t^{k})$.

\subsection{Proposed R-SQN-VR}
We propose a Riemannian stochastic \changeHK{QN} method accompanied with a \changeHK{VR} method (R-SQN-VR). A straightforward extension is to update the modified stochastic gradient $\xi_t^{k}$ by premultiplying a linear {\it inverse Hessian approximation operator} $\mathcal{H}^k_t$ at $w_t^k$ as
\begin{eqnarray*}
	%\label{Eq:UpdateFormula}
	w^{k}_{t+1} &=& R_{w^{k}_t}(-\alpha^k_t \mathcal{H}^k_t \xi_t^{k}),
\end{eqnarray*}
where $\mathcal{H}^k_t := \mathcal{T}_{\tilde{\eta}_t^k} \circ \tilde{\mathcal{H}}^k \circ (\mathcal{T}_{\tilde{\eta}_t^k})^{-1}$ by denoting the inverse Hessian approximation at $\tilde{w}^k$ \changeHK{simply as} $\tilde{\mathcal{H}}^k$. 
\changeHS{Here, $\mathcal{T}$} is an \changeHS{isometric} vector transport explained in Section \ref{Sec:ConvergenceAnalysis}. 
$\mathcal{H}^k_t$ should be positive definite\changeHS{, i.e.,} $\mathcal{H}^k_t \succ 0$ \changeHS{and} is close to the Hessian of \changeHK{$f$, i.e.}, $\hessf(w^k_t)$. It \changeHK{is noteworthy} that $\tilde{\mathcal{H}}^k$ is calculated only every outer \changeHK{epoch}, and remains to be used for $\mathcal{H}^k_t$ throughout the corresponding $k$-th epoch.

\noindent 
{\bf Curvature pair $(s^{k\changeHK{+1}}_k,y^{k\changeHK{+1}}_k)$:} This paper particularly addresses the operator $\tilde{\mathcal{H}}^k$ used in L-BFGS intended for a large-scale data. Thus, let $s_k^{k+1}$ and $y_k^{k+1}$ be the variable variation and the gradient variation \changeHK{at $T_{\tilde{w}^{k+1}}\mathcal{M}$}, respectively, where the superscript  \changeHK{expresses explicitly} that they belong to $T_{\tilde{w}^{k+1}}\mathcal{M}$. \changeHK{It should be noted that the curvature pair $(s^{k\changeHK{+1}}_k,y^{k\changeHK{+1}}_k)$ is calculated at the new $T_{\tilde{w}^{k+1}}\mathcal{M}$ just after $k$-th epoch finished.} \changeHK{Furthermore,} \changeHK{after the epoch index $k$ is incremented,} the curvature pair \changeHK{must} be \changeHK{used} only at $T_{\tilde{w}^{k}}\mathcal{M}$ because \changeHK{the calculation of} $\tilde{\mathcal{H}}^k$ \changeHK{is performed} only at $T_{\tilde{w}^{k}}\mathcal{M}$.

The variable variation $s^{k+1}_k$ is calculated from the difference between $\tilde{w}^{\changeHK{k+1}}$ and $\tilde{w}^{\changeHK{k}}$. This is represented by the tangent vector $\eta_{k}$ from $\tilde{w}^{k}$ to $\tilde{w}^{k+1}$, which is calculated \changeHK{using} the inverse of \changeHK{the} retraction $R^{-1}_{\tilde{w}^{k}}(\tilde{w}^{k+1})$. Since $\eta_{k}$ belongs to the $T_{\tilde{w}^{k}}\mathcal{M}$, transporting this onto $T_{\tilde{w}^{k+1}}\mathcal{M}$ yields
\begin{eqnarray}
	\label{Eq:Variable_Variation}
	s^{k+1}_k &=& \mathcal{T}_{\eta_{k}} \eta_{k} \ \ (= \ \mathcal{T}_{\eta_{k}}R^{-1}_{\tilde{w}^{k}}(\tilde{w}^{k+1})).
\end{eqnarray}	
The gradient variation $y^{k+1}_k$ is calculated from the difference between the new  full gradient $\gradf(\tilde{w}^{k+1}) \in T_{\tilde{w}^{k+1}}\mathcal{M}$ and the previous, but transported $\mathcal{T}_{\eta_{k}} \gradf(\tilde{w}^{k}) \in T_{\tilde{w}^{k}}\mathcal{M}$ \cite{Huang_SIOPT_2015} \changeHK{as}
\begin{eqnarray}
	\label{Eq:Gradient_Variation_Org}
		y^{k+1}_k &=& \kappa_k^{-1} \gradf(\tilde{w}^{k+1}) - \mathcal{T}_{\eta_{k}} \gradf(\tilde{w}^{k}),
\end{eqnarray}
where $\kappa_k > 0$ is explained in Section \ref{Sec:ConvergenceAnalysis}.

\noindent
{\bf Inverse Hessian approximation operator $\tilde{\mathcal{H}}^k$:}
% \changeHK{
$\tilde{\mathcal{H}}^k$ is \changeHK{calculated using the past curvature pairs. More specifically,} $\tilde{\mathcal{H}}^k$ is updated as 
$\tilde{\mathcal{H}}^{k+1}  =  (\check{\mathcal{V}}^k)^{\flat} \check{\mathcal{H}}_{k}\check{\mathcal{V}}^k  + \rho_k s_k s_k^{\flat}$,  
where 
%\begin{eqnarray}
$\check{\mathcal{H}}_{k} = \mathcal{T}_{\eta_k} \circ \tilde{\mathcal{H}}^{k} \circ \mathcal{T}^{-1}_{\eta_k},\ \ \rho_k = 1/\langle y_k, s_k \rangle, \check{\mathcal{V}}^k = \text{id} - \rho_k y_k s_{k}^{\flat}$ with identity mapping \text{id} \cite{Huang_SIOPT_2015}. \changeHK{Therein,} $a^{\flat}$ denotes the flat of $a \in T_w \mathcal{M}$, i.e., $a^{\flat}: T_w \mathcal{M} \rightarrow \mathbb{R}: v \rightarrow \langle a,v \rangle_{w}$.
Thus\changeHK{,} $\tilde{\mathcal{H}}^k$ depends on $\tilde{\mathcal{H}}^{k-1}$ and $(s_{k-1},y_{k-1})$, and similarly $\tilde{\mathcal{H}}^{k-1}$ depends on $\tilde{\mathcal{H}}^{k-2}$ and $(s_{k-2},y_{k-2})$. Proceeding recursively, $\tilde{\mathcal{H}}^k$ is a function of the initial $\tilde{\mathcal{H}}^0$ and all previous $k$ curvature pairs $\{(s_{j},y_{j})\}_{j=0}^{\changeHK{k-1}}$. Meanwhile, L-BFGS restricts use \changeHK{to} the \changeHK{most recent} $L$ \changeHK{pairs} $\{(s_{j},y_{j})\}_{j=k-L}^{\changeHK{k-1}}$ since $(s_{j},y_{j})$ with $j< k-L$ are likely to have little curvature information. \changeHS{Based} on this idea, L-BFGS performs $L$ updates by the initial $\tilde{\mathcal{H}}^{0}$. 
We use the $k$ pairs $\{(s_{j},y_{j})\}_{j=0}^{\changeHK{k-1}}$ when $k<L$.

Now, we consider the final calculation of $\tilde{\mathcal{H}}^{k}$ \changeHK{used} for $\mathcal{H}_t^k$ in the inner iterations of $k$-th outer \changeHK{epoch} using the $L$ most recent curvature pairs. Here, since this calculation is executed at $T_{\tilde{w}^k}\mathcal{M}$ and a Riemannian manifold is in general not \changeHS{a} vector space, all the $L$ curvature pairs must be located at $T_{\tilde{w}^k}\mathcal{M}$. To this end, just after the curvature pair is calculated in (\ref{Eq:Variable_Variation}) and (\ref{Eq:Gradient_Variation_Org}), the past $(L-1)$ \changeHK{pairs} of $\{(s^{k}_j,y^{k}_j)\}_{j=k-L+1}^{\changeHK{k-1}} \in T_{\tilde{w}^{k}}\mathcal{M}$ are transported into $T_{\tilde{w}^{k+1}}\mathcal{M}$ by the same vector transport $\mathcal{T}_{\eta_k}$ used when calculating $s_k^{k+1}$ and $y_k^{k+1}$. It should be emphasized that this transport is \changeHK{necessary} only \changeHK{for} every outer \changeHK{epoch} instead of every inner loop, and results in  drastic reduction of computational complexity in comparison with the straightforward extension of the Euclidean stochastic L-BFGS \cite{Mokhtari_JMLR_2015_s} into the manifold setting.
Consequently, the update is defined as 
\begin{eqnarray*}
%	\label{Eq:H_Update}
\tilde{\mathcal{H}}^{k} &=&
	((\check{\mathcal{V}}^k_{k-1})^{\changeHK{\flat}} \cdots (\check{\mathcal{V}}^k_{k-L})^{\changeHK{\flat}})  \check{\mathcal{H}}^k_0 (\check{\mathcal{V}}^k_{k-L} \cdots  \check{\mathcal{V}}^k_{k-1}) \nonumber \\
	&&+ \ \rho_{k-2} (\check{\mathcal{V}}^k_{k-1})^{\changeHK{\flat}}  s^{k}_{k-2} (s^{k}_{k-2})^{\flat} (\check{\mathcal{V}}^{k}_{k-1}) + \ \rho_{k-1} s^{k}_{k-1} (s^{k}_{k-1})^{\flat}, 
\end{eqnarray*}
where $\check{\mathcal{V}}^k_{j}={\rm id} - \rho_j y_j^{k} (s_j^k)^{\flat}$, and $\check{\mathcal{H}}^k_0$ is the initial inverse Hessian approximation. \changeHKK{${\rm id}$ is the identity mapping.} 
\changeHK{Because} $\check{\mathcal{H}}^k_0$ is not necessarily $\check{\mathcal{H}}^{k-L}$, and \changeHK{because it} \changeHK{is} any positive definite self-adjoint operator, we use $\check{\mathcal{H}}^k_0=\langle s^{k}_{k-1},y^{k}_{k-1} \rangle_{\changeHK{\tilde{w}^k}} / \langle y^{k}_{k-1},y^{k}_{k-1}\rangle_{\changeHK{\tilde{w}^k}} {\rm id}$ similar to the Euclidean case. The practical update of $\tilde{\mathcal{H}}^{k}$ uses {\it two-loop recursion} algorithm
%\cite[Section 7.2]{Nocedal_NumericalOptBook_2006}  
\cite{Nocedal_NumericalOptBook_2006} 
in Algorithm A.1 \changeHK{of} the supplementary material. 
%}

\noindent
{\bf Cautious update:} Euclidean L-BFGS fails on non-convex problems because the Hessian approximation \changeHK{has} eigenvalues that are \changeHK{away from zero} and \changeHK{are not} uniformly bounded above. To circumvent this issue, {\it cautious update} has been proposed in the Euclidean \changeHS{space} \cite{Li_SIOPT_2001}. By following this, we skip the update of the curvature pair when the following condition is not satisfied;
\begin{eqnarray}
	\label{Eq:Cautious_Update}
	\langle y_k^{k+1}, s_k^{k+1}\rangle_{\tilde{w}^{k+1}}
	& \geq & 
	%\epsilon \| \gradf(\tilde{w}^{k+1})\|_{\tilde{w}^{k+1}},
	\changeHK{\epsilon} {\| s_k^{k+1} \|^2_{\tilde{w}^{k+1}}},
\end{eqnarray}
where $\epsilon > 0$ is a predefined constant parameter. \changeHK{According to this update, the positive definiteness of $\tilde{\mathcal{H}^k}$ is guaranteed \changeHK{as far as} $\tilde{\mathcal{H}}^{k-1}$ is positive definite.
% as seen in the proof of Proposition \ref{Proposition:HessianOperatorBounds} in the supplementary material.
}

\noindent
{\bf Second-order modified stochastic gradient \changeHK{$\mathcal{H}^k_t \xi_t^{k}$}:} 
%The main feature of our proposal is to complete the calculation of the modified stochastic gradient and the second-order Hessian-vector product on the tangent space of the outer loop. 
R-SVRG transports $\gradf(\tilde{w}^{k})$ and $\gradf_{i_t^k}(\tilde{w}^{k})$ at $T_{\tilde{w}^{k}}\mathcal{M}$ into $T_{w^{k}_{t}}\mathcal{M}$ \changeHK{to} add them \changeHK{to} $\gradf_{i_t^k}(w_{t}^{k})$ at $T_{w^{k}_{t}}\mathcal{M}$. \changeHK{If} we follow the same strategy, we \changeHK{must} also transport $L$ pairs of $\{(s^{k}_j,y^{k}_j)\}_{j=\changeHK{k-L}}^{\changeHK{k-1}} \in T_{\tilde{w}^{k}}\mathcal{M}$ into the current $T_{w^{k}_{t}}\mathcal{M}$ \changeHK{at} every inner iteration. Addressing this problem and the fact that both the full gradient and the curvature pairs belong to the same tangent space $T_{\tilde{w}^{k}}\mathcal{M}$, we transport $\gradf_{i_t^k}(w_{t}^{k})$ from $T_{\changeHS{w^{k}_{t}}}\mathcal{M}$ into $T_{\tilde{w}^{k}}\mathcal{M}$, and complete all the calculations on $T_{\tilde{w}^{k}}\mathcal{M}$. More specifically, after transporting $\gradf_{i_t^k}(w_{t}^{k})$ as $(\mathcal{T}_{\tilde{\eta}_t^k})^{-1} \gradf_{i_t^k}(w_{t}^{k})$ from $w_{t}^{k}$ to $\tilde{w}^{k}$ using $\tilde{\eta}_t^k(=R^{-1}_{\tilde{w}^{k}}(w_t^k))$, the modified stochastic gradient $\tilde{\xi}_t^k  \in T_{\tilde{w}^{k}}\mathcal{M}$ is computed as
\begin{eqnarray*}
	%\tilde{\xi}_t^k =   \mathcal{T}_{S_{\zeta_{w_{t}^{k}}}}\gradf_{i_t^k}(w_{t}^{k})
	\tilde{\xi}_t^k &=& (\mathcal{T}_{\tilde{\eta}_t^k})^{-1} \gradf_{i_t^k}(w_{t}^{k})
	-\ (\gradf_{i_t^k}(\tilde{w}^{k}) -\gradf(\tilde{w}^{k})).
\end{eqnarray*}
After calculating $\tilde{\mathcal{H}}^k_t \tilde{\xi}_t^k \in T_{\changeHK{\tilde{w}^k}}\mathcal{M}$ \changeHK{using} the two-loop recursion algorithm, we obtain $\mathcal{H}^k_t \xi_t^k \in T_{w_{t}^{k}}\mathcal{M}$ by transporting $\tilde{\mathcal{H}}^k_t \tilde{\xi}_t^k$ to $T_{w_{t}^{k}}\mathcal{M}$ by $\mathcal{T}_{\tilde{\eta}_t^k} \tilde{\mathcal{H}}^k_t\tilde{\xi}_t^k$. \changeHK{Finally,} we update $w_{t+1}^k$ from $w_{t}^k$ as $w^{k}_{t+1} = R_{w^{k}_t}(-\alpha^k_t \mathcal{H}^k_t \xi_t^{k})$.
\changeHS{It should be noted that, although $-\xi_t^k$ is not generally guaranteed as a descent direction, $\mathbb{E}_{i_t^k}[-\xi_t^k]=-\gradf(w^k_{t})$ is a descent direction. \changeHK{Furthermore, }the positive definiteness of $\mathcal{H}^k_t$ yields that $-\mathcal{H}_t^k\xi_t^k$ is an average descent direction due to $\mathbb{E}_{i_t^k}[-\mathcal{H}^k_t\xi_t^k]=-\mathcal{H}_t^k \gradf(w^k_{t})$.}

%\noindent
%\changeHK{{\bf Computational complexity analysis:} TBA.}

%
%
%
%
%
\section{Convergence analysis}
\label{Sec:ConvergenceAnalysis}

This section presents \changeHK{convergence analyses on both non-convex and retraction-convex functions under retraction and vector transport operations.}. The concrete proofs are in the supplementary file. 

\begin{assumption}
\label{Assump:1}
We assume below \cite{Huang_SIOPT_2015};

(1.1)  The objective function $f$ and its components $f_1, \ldots, f_N$ are twice continuously differentiable. 

(1.2)  For a sequence $\{w_t^k\}$ generated by Algorithm 1, there exists a compact and connected set $K \subset \mathcal{M}$ such that $w_t^k \in K$ for all $k, t \ge 0$.
Also, for each $k\ge 1$, there exists a totally retractive neighborhood $\Theta_{\changeHK{k}}$ of $\tilde{w}^{\changeHK{k}}$ such that $w_t^k$ stays in $\Theta_{\changeHK{k}}$ for any $t\ge 0$,
where the $\rho$-totally \changeHSS{retractive} neighborhood $\Theta$ of $w$ is a set such that for all $z \in \Theta$, $\Theta \subset R_z(\mathbb B(0_z,\rho))$, and $R_z(\cdot)$ is a diffeomorphism on $\mathbb B(0_z,\rho)$, which is the ball in $T_w \mathcal{M}$ with center $0_z$ and radius $\rho$, where $0_z$ is the zero vector in $T_z\mathcal{M}$.
Furthermore, suppose that there exists $I>0$ such that $\inf_{k \ge 1} \{\sup_{z \in \Theta_{\changeHK{k}}} \| R_{\tilde{w}^{\changeHK{k}}}^{-1}(z)\|_{\tilde{w}^{\changeHK{k}}}\} \ge I$.

 \changeHK{(1.3) 
 \changeHSS{The sequence $\{w_t^k\}$ continuously remains in $\rho$-totally retractive neighborhood $\Theta$ of critical point $w^*$ and} $f$ is retraction-smooth with respect to \changeHSS{retraction} $R$ in $\Theta$. Here, 
$f$ is said to be retraction-smooth in $\Theta$ if $\changeHK{f(R_w(t\eta_w))}$ for all $w \in \mathcal{M}$, i.e., there exists a constant $0< \Lambda$ such that 
$\frac{d^2 \changeHK{f(R_w(t\eta_w))}}{dt^2}  \leq  \Lambda$, 
for all $w \in \Theta$, all $\| \eta_w \|_w=1$, and all $t$ such that $R_w(\tau \eta_w) \in \Theta$ for all $\tau \in [0,t]$.}

(1.4)  The vector transport $\mathcal{T}$ is isometric \changeHK{on $\mathcal{M}$}. It satisfies
$\langle \mathcal{T}_{\changeHK{\xi_w}} \eta_w, \mathcal{T}_{\changeHK{\xi_w}} \zeta_w \rangle_{\changeHK{R_w(\xi_w)}} = \langle \eta_w, \zeta_w \rangle_w$ \changeHK{for any $w \in \mathcal{M}$ and $\changeHK{\xi_w,}\eta_w, \zeta_w \in T_w\mathcal{M}$}.	

(1.5)  There exists a constant $c_0$ such that the vector transport $\mathcal{T}$ satisfies the following conditions for all $w,z \in \mathcal{U}$, which is \changeHSS{some} neighborhood of \changeHSS{an arbitrary point} $\bar{w} \in \mathcal{M}$:
$
\| \mathcal{T}_{\eta_{\changeHK{w}}}-\mathcal{T}_{R_{\eta_{\changeHK{w}}}}\| \le c_0 \| \eta_{\changeHK{w}} \|_{\changeHK{w}},  
\| \mathcal{T}^{-1}_{\eta_{\changeHK{w}}}-\mathcal{T}^{-1}_{R_{\eta_{\changeHK{w}}}}\| \le c_0 \| \eta_{\changeHK{w}}\|_{\changeHK{w}}, \nonumber
$
where $\mathcal{T}_R$ denotes the differentiated retraction, i.e.,
$\mathcal{T}_{R_{\changeHSS{\zeta_w}}}\xi_w = {\rm D} R\changeHK{_w}(\changeHSS{\zeta_w}) [\xi_w]$ \changeHSS{with} $\xi_w \in T_w \mathcal{M}$, and $\eta_w = R_w^{-1}(z)$.

%(1.6) The vector transport $\mathcal{T}$ satisfies the locking condition\changeHK{,} which is defined as
%\begin{eqnarray}
%\label{Eq:locking_condition}
%\mathcal{T}_{\eta_w}\xi_w = \kappa \mathcal{T}_{R_{\eta_w}} \xi_w, \ \text{where}\  \kappa = \frac{\| \xi_w \|_w}{\| \mathcal{T}_{R_{\changeHK{\eta}_w}} \xi_w\|_{\changeHK{R_w(\eta_w)}}},
%\end{eqnarray}
%for all $\eta_w, \xi_w \in T_w \mathcal{M}$ and all $w \in \mathcal{M}$. 	

\changeHK{\changeHKK{(1.6)} Riemannian stochastic gradient is bounded as $\mathbb{E}_{i_t^k}[\|  \gradf_{i_t^k}(w^k_t)\|_{w^k_t}^2]<C^2$ as \cite{Mokhtari_IEEETranSigPro_2014,Wang_arXiv_2016,Byrd_SIOPT_2016}.}
\end{assumption}

%The following proposition bounds the eigenvalues of $\mathcal{H}^k_t$. 
%\begin{Prop}[Eigenvalue bounds of  $\mathcal{H}^k_t$]
%\label{Proposition:HessianOperatorBounds}
%Consider the operator $\tilde{\mathcal{H}}^k\changeHK{:= \mathcal{T}_{\tilde{\eta}_t^k} \circ \tilde{\mathcal{H}}^k \circ (\mathcal{T}_{\tilde{\eta}_t^k})^{-1}}$\changeHK{, where $\tilde{\mathcal{H}}^k$ is} defined by the recursion in (\ref{Eq:H_Update}). Define the constant $0< \gamma < \Gamma < \infty$. If Assumption \ref{Assump:1} holds, the eigenvalues of $\mathcal{H}^k_t$ is bounded by $\gamma$ and $\Gamma$ for all $k \geq 1, t \geq 1$, i.e., 
%\begin{eqnarray}
%\label{Eq:InverseHessianOperatorBounds}
% \gamma {\rm id} \ \preceq \ \mathcal{H}^k_t \ \preceq \ \Gamma{\rm id}.
%\end{eqnarray}
%\end{Prop}

\noindent 
{\bf Essential inequalities.} We briefly summarize essential inequalities. They are detailed in the supplementary material. 
For all $w, z \in \mathcal{U}$, which is a neighborhood of $\bar{w}$, the difference between the parallel translation and the vector transport is given with a constant $\theta$ as (Lemma \ref{Lemma:VecParaDiff})
\begin{eqnarray}
\label{Ineq:2}
\|\mathcal{T}_\eta \xi-P_{\eta}\xi\|_z & \le & \theta \| \xi\|_w \| \eta\|_w,
\end{eqnarray}
where $\xi, \eta \in T_w\mathcal{M}$ and $R_w(\eta)=z$. 
Similarly, as for the difference between the exponential mapping and the retraction, there exist $\tau_1 > 0$, $\tau_2 > 0$ for all $w \in \mathcal{U}$ and all small length of $\xi \in T_w \mathcal{M}$ such that (Lemma \ref{Lemma:retraction_dist})
\begin{eqnarray}
\label{Ineq:3}
\tau_1 {\rm dist} (w, R_w(\xi))  \le \|\xi\|_w \le  \tau_2 {\rm dist} (w, R_w(\xi)).
\end{eqnarray}
Then, the variance of $\xi_t^k$ is upper bounded by (Lemma \ref{AppenLem:UpperBoundVariance})
\begin{eqnarray}
\label{Ineq:4}
	\mathbb{E}_{i_t^k}[\| \xi_t^k \|_{w_{t}^k}^2] &\leq &
4(\beta^2+\tau_2^2C^2\theta^2)(7({\rm dist}(w_{t}^k,w^*))^2  + 4({\rm dist}(\tilde{w}^{k},w^*))^2),
\end{eqnarray}
where \changeHK{$C$ is the constant of Assumption \ref{Assump:1}}, $\beta$ is a Lipschitz constant, and $\theta$ is the constant in (\ref{Ineq:2}).
\changeHK{Finally, there exist $0 < \gamma < \Gamma$ such that (Proposition \ref{AppProposition:HessianOperatorBoundsNonConvex} for non-convex functions and Proposition \ref{AppProposition:HessianOperatorBoundsConvex} for retraction-convex functions)}
\begin{eqnarray}
\label{Eq:HessianOperatorBoundsNonConvexConvex}
 \gamma {\rm id} \ \preceq \ \mathcal{H}^k_t \ \preceq \ \Gamma {\rm id},
\end{eqnarray}
\changeHK{where the $A \preceq B$ with $A, B \in \mathbb{R}^{n\times n}$ means that $B-A$ is positive semidefinite. }

Now, we first present a global convergence analysis to a critical point starting from any initialization point, which is common in a non-convex setting
with additional \changeHK{but mild} assumptions;
%. We first state \changeHK{the additional} assumptions particularly required for this \changeHK{analysis}.
%
\begin{assumption}
\label{Assump:2}
%We assume the following assumptions; 
%
%(2.1) $f$ is bounded below by a scalar $f_{\rm inf}$. 
%
%(2.2) Since $\Theta$ is compact, all continuous functions on  $\Theta$ can be bounded. Therefore, there exists $S > 0$ such that for all $w \in \Theta$ and $n \in N$, we have $\| \gradf(w) \|_w \le S {\rm \  \ and\  \ } \| \gradf_n(w) \|_w \le S$.
%
%(3.3) The step-size sequence $\{ \alpha^k_t \}$ \changeHK{satisfies}
%$\sum \alpha^k_t \ =\  \infty$ and $\sum (\alpha^k_t)^2 \ <\  \infty$.
\changeHK{We assume that $f$ is bounded below by a scalar $f_{\rm inf}$, and a decaying step-size sequence $\{ \alpha^k_t \}$ \changeHK{satisfies}
$\sum \alpha^k_t=\infty$ and $\sum (\alpha^k_t)^2 < \infty$. Additionally, since $\Theta$ is compact, all continuous functions on  $\Theta$ can be bounded. Therefore, there exists $S > 0$ such that for all $w \in \Theta$ and $n \in N$, we have $\| \gradf(w) \|_w \le S {\rm \  \ and\  \ } \| \gradf_n(w) \|_w \le S$.}
\end{assumption}

%Now, we present the global convergence result below.

\begin{Thm}[Global convergence analysis \changeHK{on non-convex functions}]
\label{Thm:GlobalConvAnalysisNonConvex}
\changeHK{Let $\mathcal{M}$ be a Riemannian manifold and $w^* \in \mathcal{M}$ be a non-degenerate local minimizer of $f$.}
Consider Algorithm 1 and suppose Assumptions \ref{Assump:1} and \ref{Assump:2}, and that the mapping $w \changeHS{\mapsto} \| \gradf(w) \|_w^2$ has the positive real number that the largest eigenvalue of its Riemannian Hessian is bounded for all $w \in \mathcal{M}$. Then, we have
	$\lim_{k \rightarrow \infty} \mathbb{E} [\| \gradf(w_t^k) \|^2_{w_t^k}]  =  0$.
\end{Thm}

\changeHK{We next present a global convergence rate analysis. This requires an strict selection of a fixed step size satisfying the condition below, but, instead, provides a convergence {\it rate} under it.}
\changeHK{
\begin{Thm}[Global convergence rate analysis on non-convex functions]
\label{Thm:GlobalConvRateAnalysisNonConvex}
Let $\mathcal{M}$ be a Riemannian manifold and $w^* \in \mathcal{M}$ be a non-degenerate local minimizer of $f$. Consider Algorithm 1 with Option II and IV, and suppose Assumption 1. 
Let the constants  $\theta$ in (\ref{Ineq:2}), $\tau_1$ and $\tau_2$ in (\ref{Ineq:3}), and $\beta$, and $C$ in (\ref{Ineq:4}).  
$\Lambda$ is the constant Assumption \ref{Assump:1}.3, and $\gamma{\changeHK{_{nc}}}$ and $\Gamma{\changeHK{_{nc}}}$ are the constants $\gamma$ and $\Gamma$ in (\ref{Eq:HessianOperatorBoundsNonConvexConvex}).
Set $\nu = \frac{\sqrt{\beta^2+\tau^2_2 C^2 \theta^2}\Gamma_{nc} \changeHKK{\tau_1}}{N^{a_1/2} }\zeta^{1-a_2}$ and 
$\alpha_t^k=\alpha = \frac{\mu_0\changeHK{\tau_1}}{\sqrt{\beta^2+\tau^2_2 C^2 \theta^2} N^{a_1} \Gamma_{nc} \zeta^{a_2}}$, where $0< a_1<1$, and $0 < a_2 <  \note{2}$. 
Given sufficiently small $\mu_0 \in (0,1)$, suppose that $\varrho > 0$ is chosen such that 
$
	\frac{\sqrt{\beta^2+\tau^2_2 C^2 \theta^2}}{\Lambda \Gamma_{nc}} \gamma\changeHK{_{nc}} \left(1 - \frac{\varrho \Gamma_{nc}}{\mu_0 \gamma\changeHK{_{nc}}\changeHKK{\tau_1}} \right) 
	 >  \frac{2\mu_0  (e-1) }{\zeta^{2-a_2}\changeHKK{\tau_1}} + \frac{\mu_0\changeHKK{\tau_1}}{N^{a_1}\zeta^{a_2}} + \frac{4\mu_0^2(e-1)}{N^{\frac{3a_1}{2}}\zeta^{a_2}\changeHKK{(2\tau_1+1)}} 
$
holds. Set $m = \lfloor  \frac{N^{3a_1/2}}{5 \mu_0 \zeta^{1-a_2} \changeHKK{\tau_1(2\tau_1+1)}} \rfloor$ and $T=mK$. Then, we have
\begin{eqnarray}
	\mathbb{E}[\| \gradf (w_{\rm sol}) \|^2] & \leq & \frac{\sqrt{\beta^2+\tau^2_2 C^2 \theta^2} N^{a_1} \zeta^{a_2} [f(w^0) - f(w^*)]}{T \varrho}.
\end{eqnarray}
\end{Thm}
}
%
%
%
%
%
%\subsection{\changeHK{Convergence rate analysis on retraction-convex function}}

\changeHK{The total number of gradient evaluations is $\mathcal{O}(N^{a_1}/\epsilon)$ to obtain an $\epsilon$-solution. The proof is given by extending those of \cite{Reddi_ICML_2016_s, Zhang_NIPS_2016,Wang_arXiv_2016}. }

\changeHK{
As a final analysis, we present a local convergence rate in neighborhood of a local minimum by introducing additionally a {\it local} assumption for retraction-convexity below. This is also very common and standard in manifold optimization. 
}
 \changeHK{
\begin{assumption}
\label{Assump:3}
We assume that the objective function $f$ is strongly retraction-convex with respect to $R$ in $\Theta$. Here, 
$f$ is said to be strongly retraction-convex in $\Theta$ if $\changeHK{f(R_w(t\eta_w))}$ for all $w \in \mathcal{M}$ and $\eta_w \in T_w \mathcal{M}$ is strongly convex, i.e., there exists a constant $0< \lambda$ such that 
$\lambda  \ \leq\  \frac{d^2 \changeHK{f(R_w(t\eta_w))}}{dt^2}$, 
for all $w \in \Theta$, all $\| \eta_w \|_w=1$, and all $t$ such that $R_w(\tau \eta_w) \in \Theta$ for all $\tau \in [0,t]$.	
\changeHKK{
Additionally, the vector transport $\mathcal{T}$ satisfies the locking condition\changeHK{,} which is defined as
\begin{eqnarray}
\label{Eq:locking_condition}
\mathcal{T}_{\eta_w}\xi_w = \kappa \mathcal{T}_{R_{\eta_w}} \xi_w, \ \text{where}\  \kappa = \frac{\| \xi_w \|_w}{\| \mathcal{T}_{R_{\changeHK{\eta}_w}} \xi_w\|_{\changeHK{R_w(\eta_w)}}},
\end{eqnarray}
for all $\eta_w, \xi_w \in T_w \mathcal{M}$ and all $w \in \mathcal{M}$. 
}
\end{assumption}
}
\changeHK{
It should be noted that, if we extend this local assumption to the entire manifold, as R-SVRG \cite{Zhang_NIPS_2016}, our rate below directly results in the global rate. However, such a global assumption is fairly restrictive in terms of what cost functions and manifolds can be considered, and hence, the standard manifold literature mostly focuses on local rate analysis. For example, R-SVRG \cite{Sato_arXiv_2017} does not show a global rate on retraction-convex functions. 
}
%
%Now, we are ready to give a local convergence rate as;
%
\begin{Thm}[Local convergence rate analysis \changeHK{on retraction-convex functions}]
\label{Thm:LocalConvergenceConvex}
Let $\mathcal{M}$ be a Riemannian manifold and $w^* \in \mathcal{M}$ be a non-degenerate local minimizer of $f$. Suppose Assumption \ref{Assump:1} holds. $\Lambda$ and $\lambda$ are constants in  Assumption \changeHK{\ref{Assump:1} and \ref{Assump:3}}, respectively.
Let the constants  $\theta$ in (\ref{Ineq:2}), $\tau_1$ and $\tau_2$ in (\ref{Ineq:3}), and $\beta$, and $C$ in (\ref{Ineq:4}). 
$\gamma\changeHK{_c}$ and $\changeHK{\Gamma_{c}}$ are the constants in \changeHK{(\ref{Eq:HessianOperatorBoundsNonConvexConvex})}.
Let $\alpha$ be a positive number satisfying $\lambda \tau^2_1 > 2\alpha(  \lambda^2 \tau^2_1 - \changeHS{14 \alpha \Lambda \changeHK{\Gamma_c^2}} (\beta^2+\tau_2^2 C^2 \theta^2))$ \changeHS{and $\gamma\changeHK{_c} \lambda^2 \tau^2_1 > 14 \alpha \Lambda \changeHK{\Gamma_c^2} (\beta^2+\tau_2^2 C^2\theta^2)$}. 
It then follows that for any sequence $\{\tilde{w}^k\}$ generated by Algorithm 1 \note{with Option I under} a fixed step size $\alpha_t^k:=\alpha$ and $m_k:=m$ converging to $w^*$, there exists $\changeHK{0<K_{th}<K}$ such that for all $k>\changeHK{K_{th}}$,
\begin{equation}
\label{Eq:LocalRate}
\begin{array}{l}
\mathbb{E}[ ({\rm dist}(\tilde{w}^{k+1},w^*))^2] \le \displaystyle{\frac{2(\Lambda \tau^2_2 + 16 m \alpha^2 \Lambda \changeHK{\Gamma_{c}^2}(\beta^2+\tau_2^2C^2 \theta^2)}{m\alpha( \gamma\changeHK{_c} \lambda^2 \tau^2_1 - 14 \alpha \Lambda \changeHK{\Gamma_{c}^2} (\beta^2+\tau_2^2 C^2\theta^2)) }   \mathbb{E}[({\rm dist}(\tilde{w}^{k},w^*))^2).}
\end{array}
\end{equation}
\end{Thm}

%It should also be noted that, \changeHK{the} parallel translation \changeHK{$P$} as \changeHK{the} vector transport \changeHK{can} lead to a smaller $\Gamma$ and a larger $\gamma$ in (\ref{Eq:HessianOperatorBoundsNonConvexConvex}). As a result, this produces a smaller coefficient above, and \changeHK{can} result in a faster local convergence rate. 
\changeHK{The proof structure is different from that of \cite{Sato_arXiv_2017,Zhang_NIPS_2016} due to the way of bounding of $\mathbb{E}[\| \xi_t^k \|^2]$ and the existence of $\mathcal{H}_t^k$. Additionally, comparing (\ref{Eq:LocalRate}) with that of R-SVRG \cite{Sato_arXiv_2017,Zhang_NIPS_2016}, we notice the rate degradation. To the best of our knowledge, no theoretical rate result that is better than or equals to that of SVRG \cite{Johnson_NIPS_2013_s} has been also given in the Euclidean SQN-VR \cite{Moritz_AISTATS_2016_s}. Thus, this issue is a common area of research in both the Euclidean and Riemannian settings to further improve the theoretical rate. However, it should be emphasized that R-SQN-VR shows much better performances than R-SVRG, especially on a ill-conditioned problem, as shown later in Figure \ref{fig:PerformanceEvalatons}.}

\section{Numerical comparisons}
\label{Sec:Numerical_comparisons}

This section compares R-SQN-VR with R-SGD with \changeHK{a} decaying step-size sequence and R-SVRG with a fixed step size. 
The decaying step-size sequence is $\alpha_k = \alpha(1+ \alpha \varsigma \lfloor k/m_k \rfloor)^{-1}$, where $\lfloor \cdot \rfloor$ denotes the floor function.   
%\changeHK{which are selected from} multiple choices.  
\changeHK{As references, }we also compare them with \changeHK{two Riemannian batch methods, i.e.,} R-SD, which is the steepest descent algorithm  \changeHK{on Riemannian manifolds} with backtracking line search 
\cite{Absil_OptAlgMatManifold_2008}, and \changeHK{R-L-BFGS, which is the Riemannian L-BFGS with strong wolfe condition \cite{Ring_SIAMJO_2012,Yuana_ICCS_2016_s}.}
%\cite{Absil_OptAlgMatManifold_2008}. 
All experiments are \changeHK{executed} in Matlab on a 4.0 GHz Intel Core i7 \changeHK{PC} with 16 GB RAM, and \changeHK{are stopped when the gradient norm \changeHK{gets} below $10^{-8}$ or when they reach a predefined maximum iteration.}
All results except R-SD and \changeHK{R-L-BFGS} are the {\it best-tuned} results
\changeHK{from multiple choices of step sizes $\alpha$ and a fixed $\varsigma$ = $10^{-3}$}. 
This paper addresses the Karcher mean computation problem of symmetric positive-definite (SPD) manifold, and  the low-rank matrix completion (MC) problem on the Grassmann manifold. \changeHK{The details of the problems and manifolds are in the supplementary file.}

\begin{figure*}[t]
%\vspace*{-0.5cm}
\begin{center}
\begin{minipage}{0.320\hsize}
\begin{center}
\includegraphics[width=1.1\hsize]{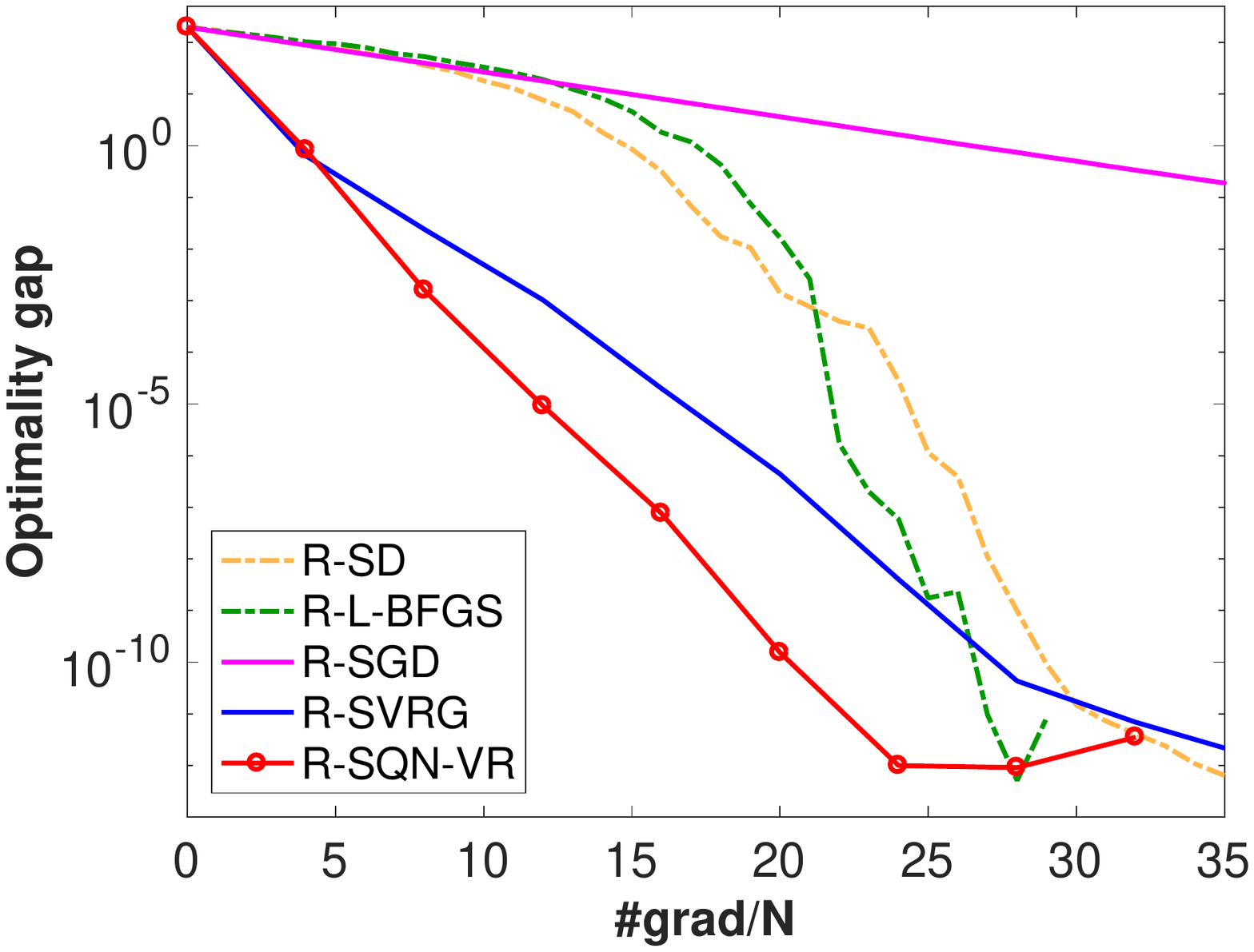}\\
\vspace*{-1.5cm}
{\scriptsize(a) {\bf Case KM-1: small size.}}
\end{center}
\end{minipage}
%\vspace*{0.1cm}
\begin{minipage}{0.320\hsize}
\begin{center}
\includegraphics[width=1.1\hsize]{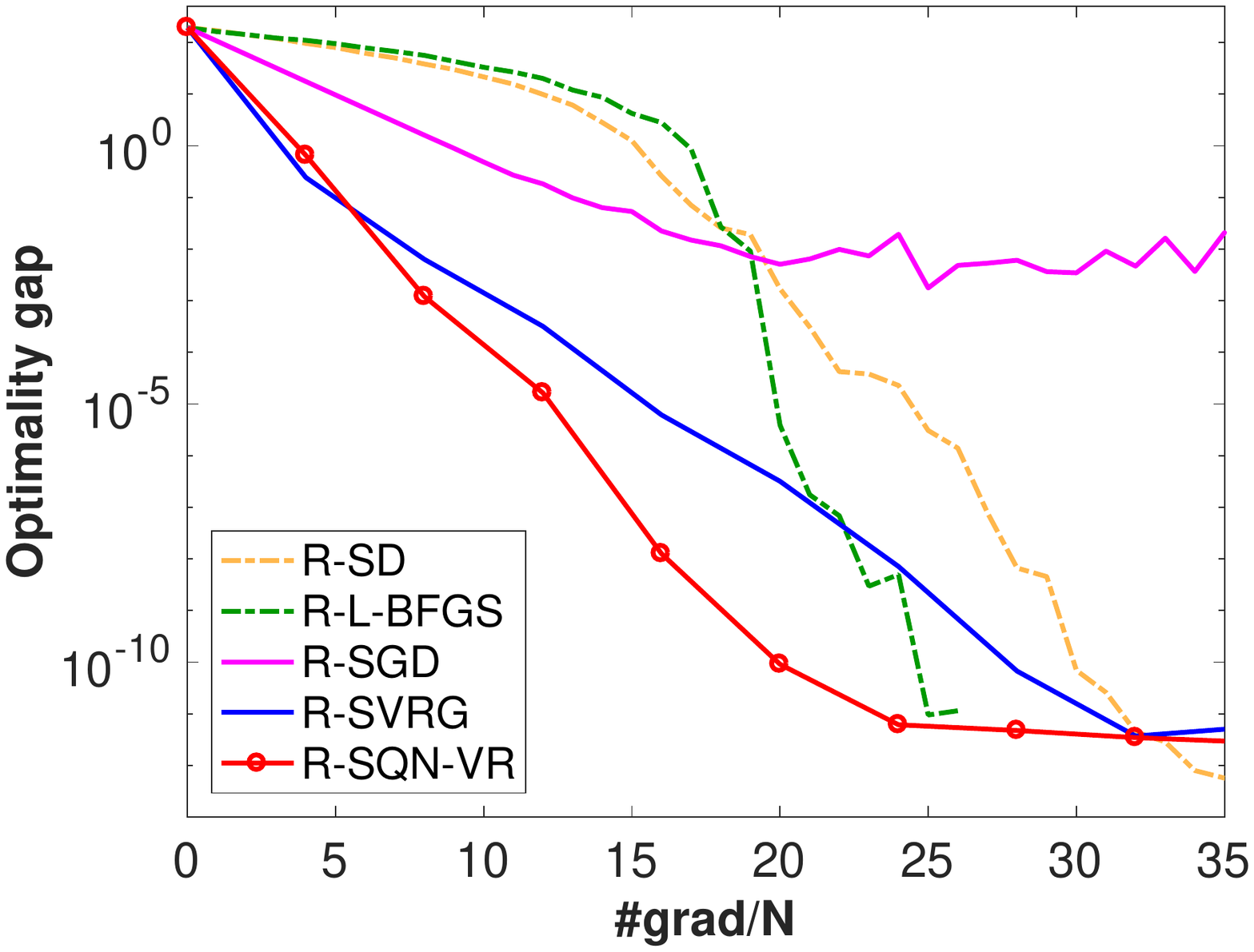}\\
\vspace*{-1.5cm}
{\scriptsize(b) {\bf Case KM-2: large size.}}
\end{center}
\end{minipage}
%\vspace*{0.1cm}
\begin{minipage}{0.320\hsize}
\begin{center}
%\vspace*{0.1cm}
\includegraphics[width=1.1\hsize]{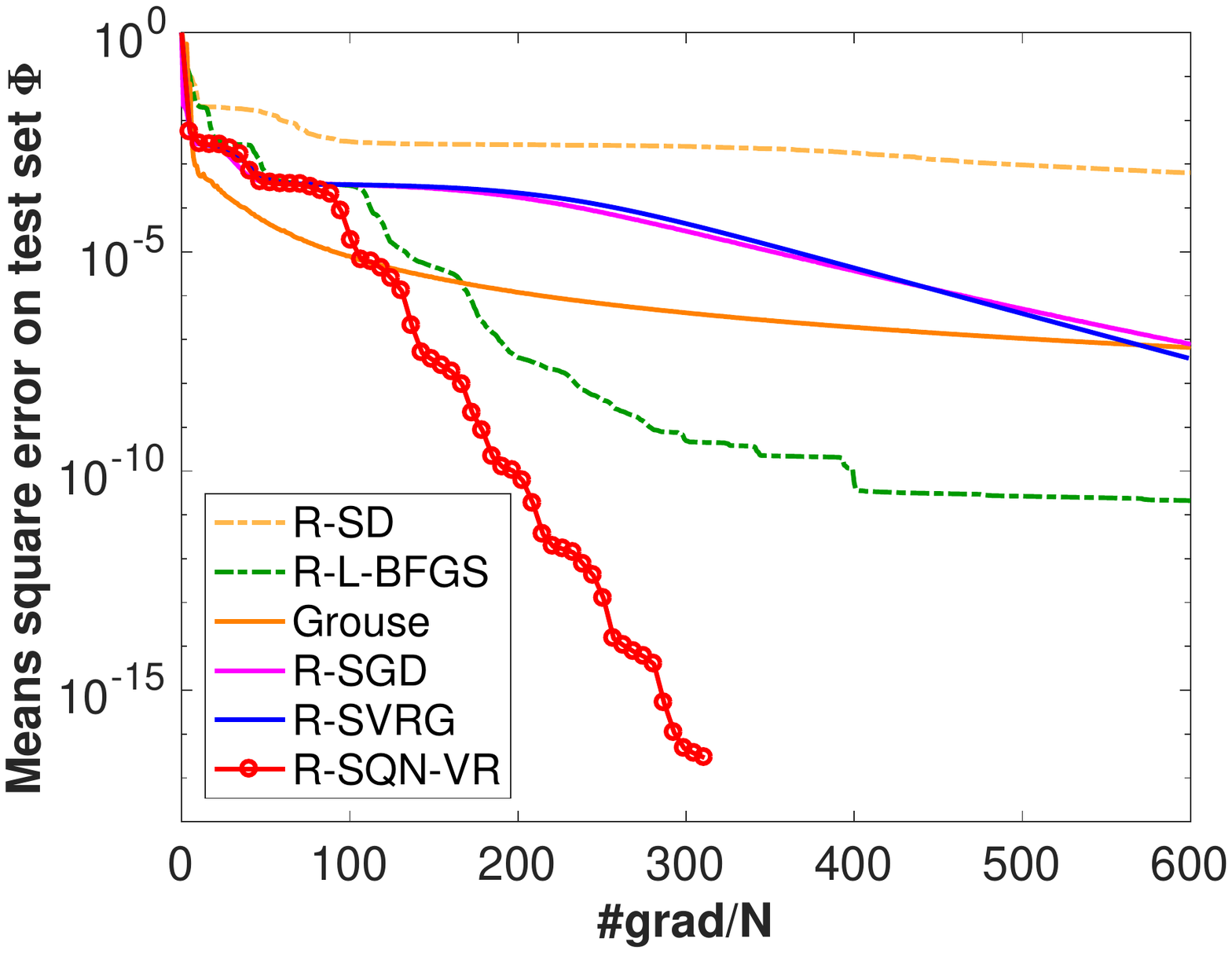}\\
\vspace*{-1.5cm}
{\scriptsize(c) {\bf Case MC-S1: baseline.}}
%\vspace*{0.6cm}
\end{center}
\end{minipage}\\
\vspace*{-1cm}
%%%%%%
\begin{minipage}{0.320\hsize}
\begin{center}
\includegraphics[width=1.1\hsize]{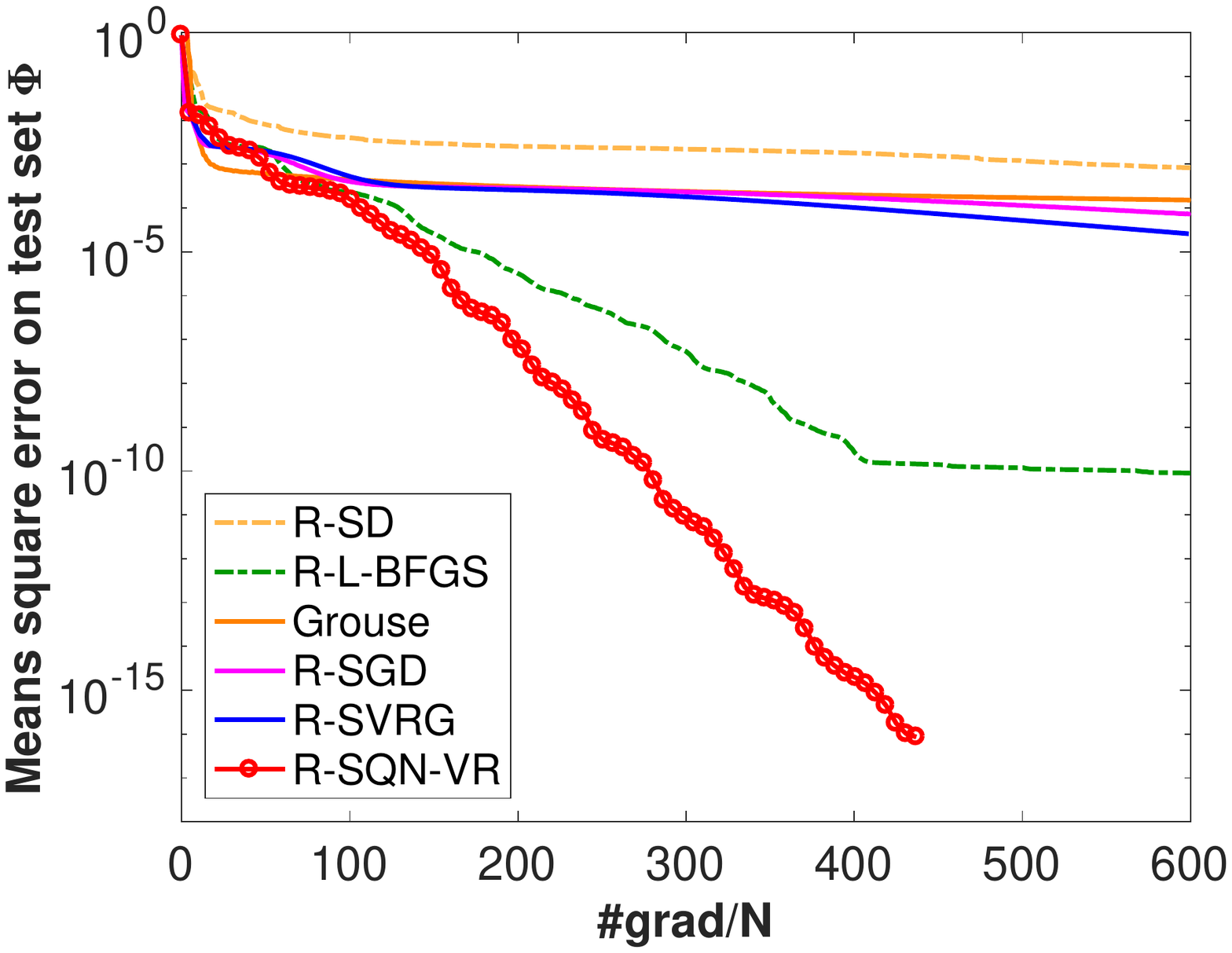}\\
\vspace*{-1.5cm}
{\scriptsize (d) {\bf Case MC-S2: low sampling.}}
\end{center}
\end{minipage}
%\vspace*{0.1cm}
\begin{minipage}{0.320\hsize}
\begin{center}
\includegraphics[width=1.1\hsize]{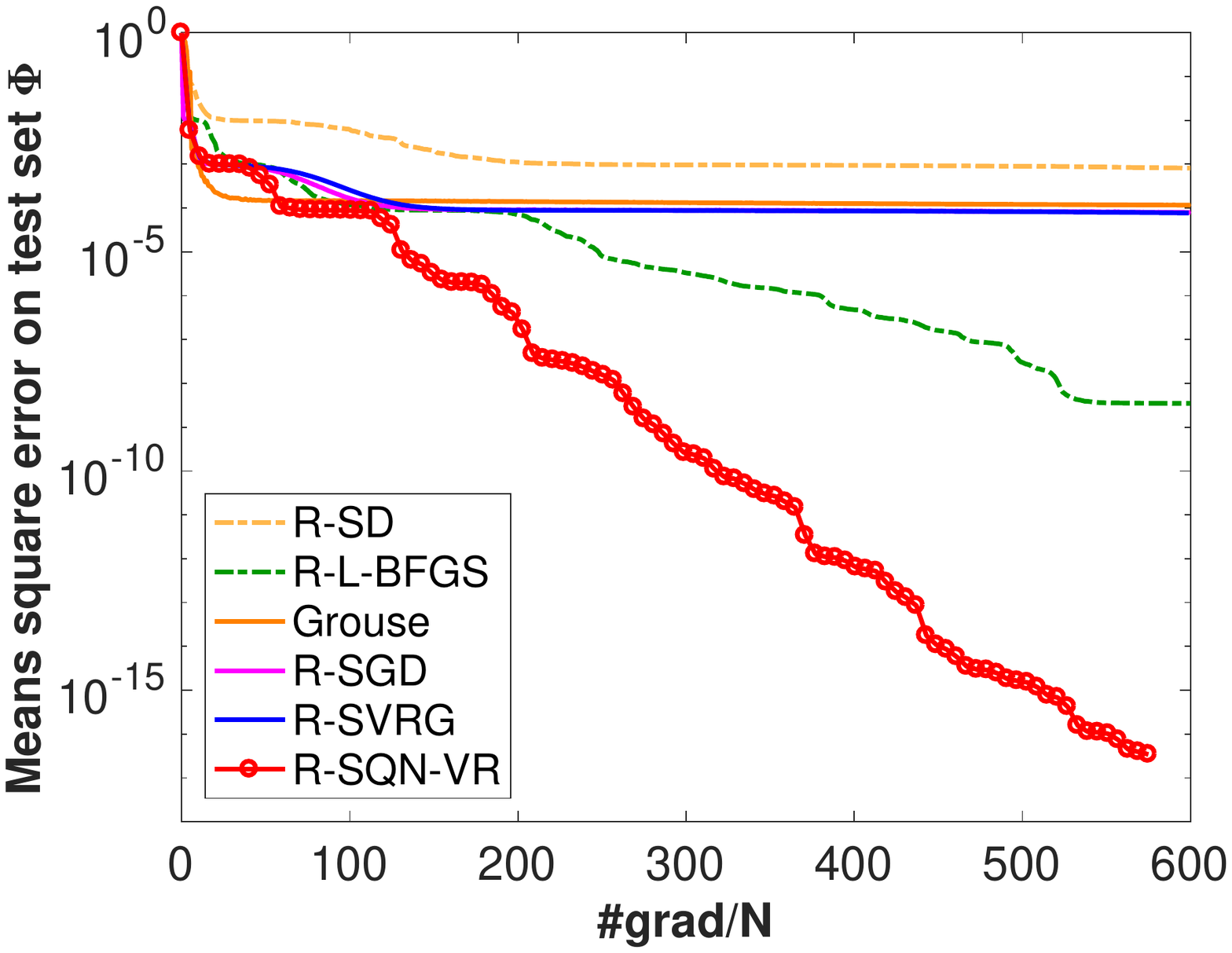}\\
\vspace*{-1.5cm}
{\scriptsize (e) {\bf Case MC-S3: ill-conditioning.}}
\end{center}
\end{minipage}
\vspace*{-1cm}
\begin{minipage}{0.320\hsize}
\begin{center}
\includegraphics[width=1.1\hsize]{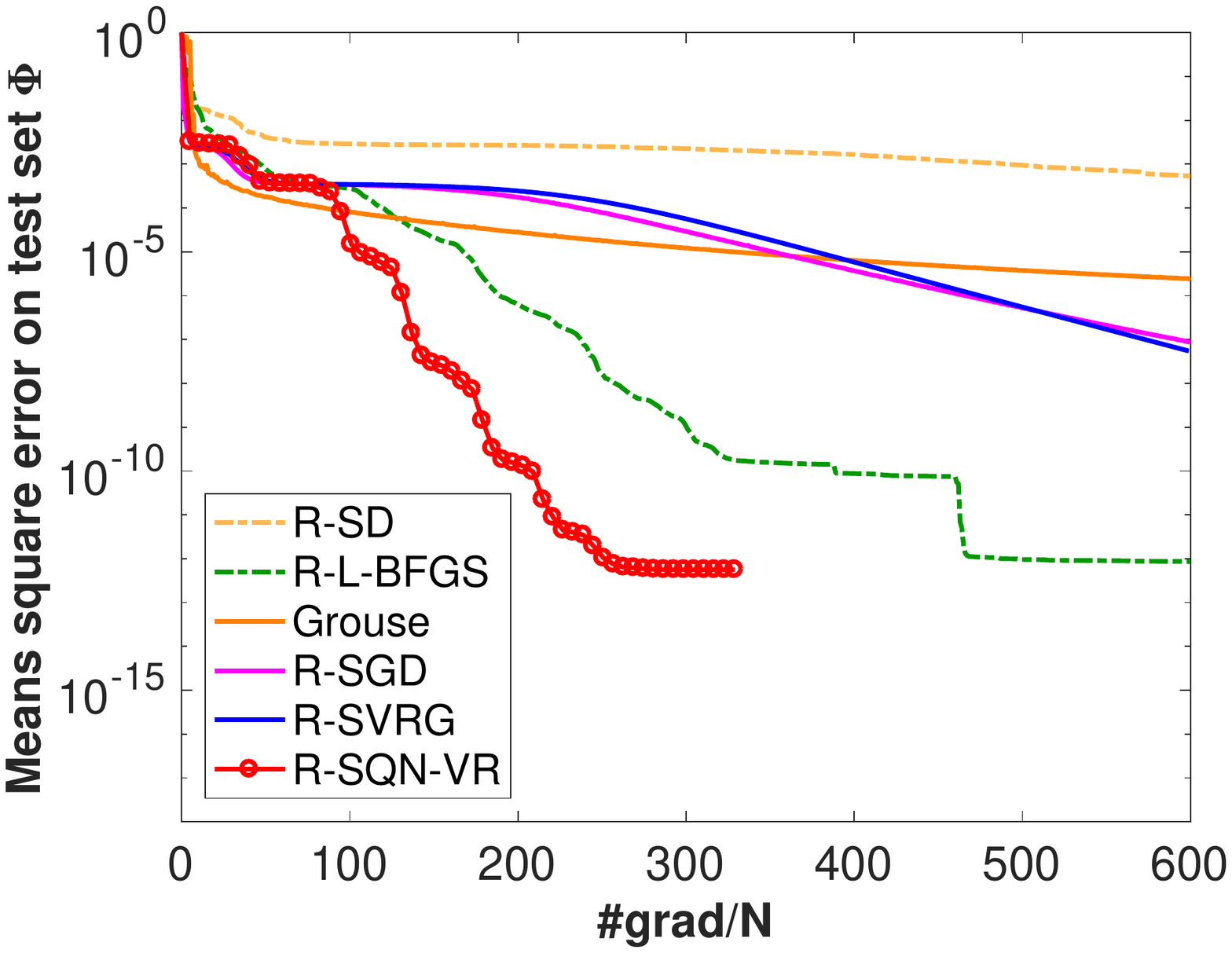}\\
\vspace*{-1.5cm}
{\scriptsize (f) {\bf Case MC-S4: noisy data.}}
\end{center}
\end{minipage}\\
\vspace*{0.3cm}
%%%%%%%%%%
\begin{minipage}{0.320\hsize}
\begin{center}
\includegraphics[width=1.1\hsize]{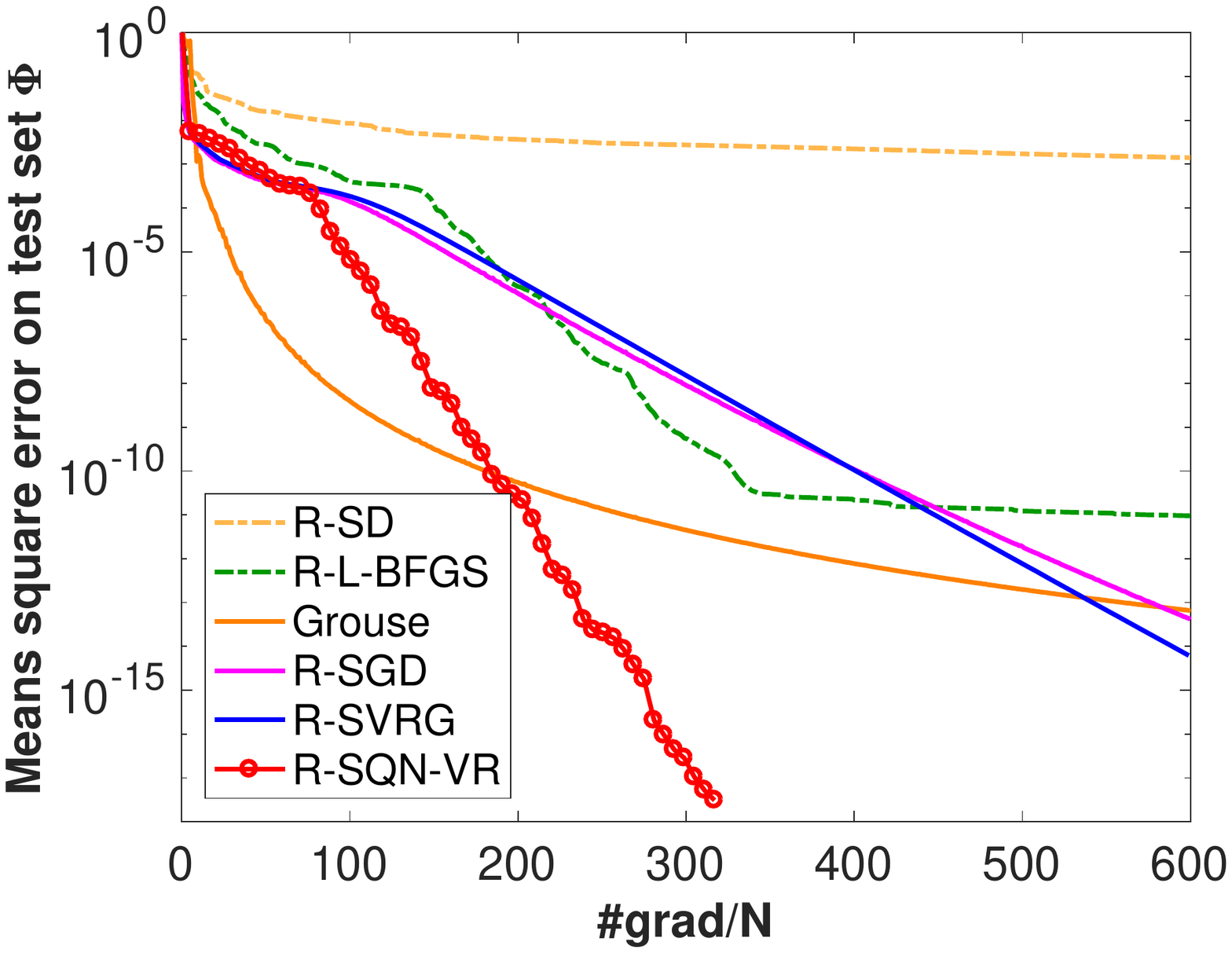}\\
\vspace*{-1.5cm}
{\scriptsize (g) {\bf Case MC-S5: higher rank.}}
\end{center}
\end{minipage}
\begin{minipage}{0.320\hsize}
\begin{center}
\includegraphics[width=1.1\hsize]{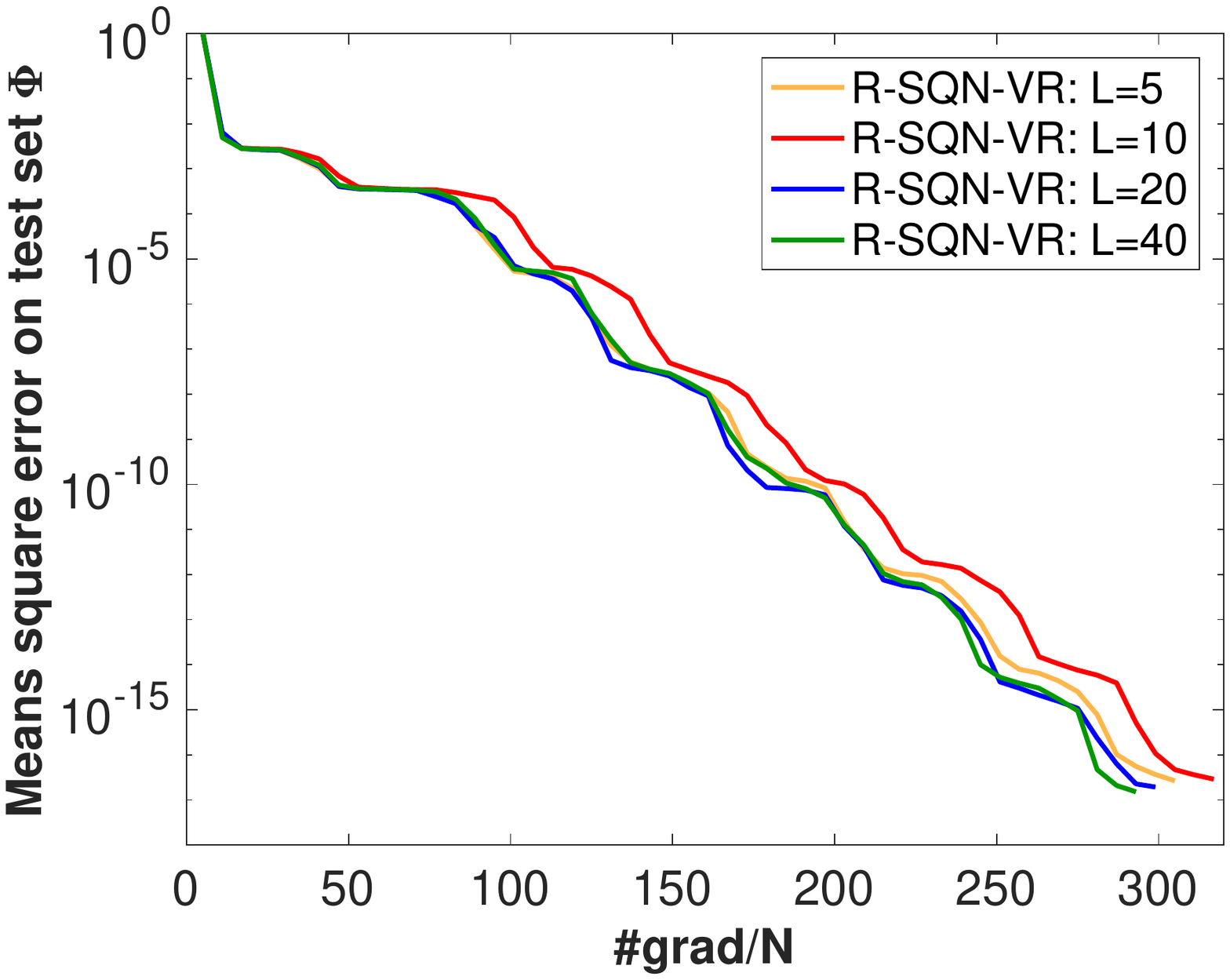}\\
\vspace*{-1.5cm}
{\scriptsize (h) {\bf Case MC-S6: memory sizes.}}
\end{center}
\end{minipage}
%\vspace*{3cm}
\begin{minipage}{0.320\hsize}
\begin{center}
\vspace*{0.3cm}
\includegraphics[width=1.1\hsize]{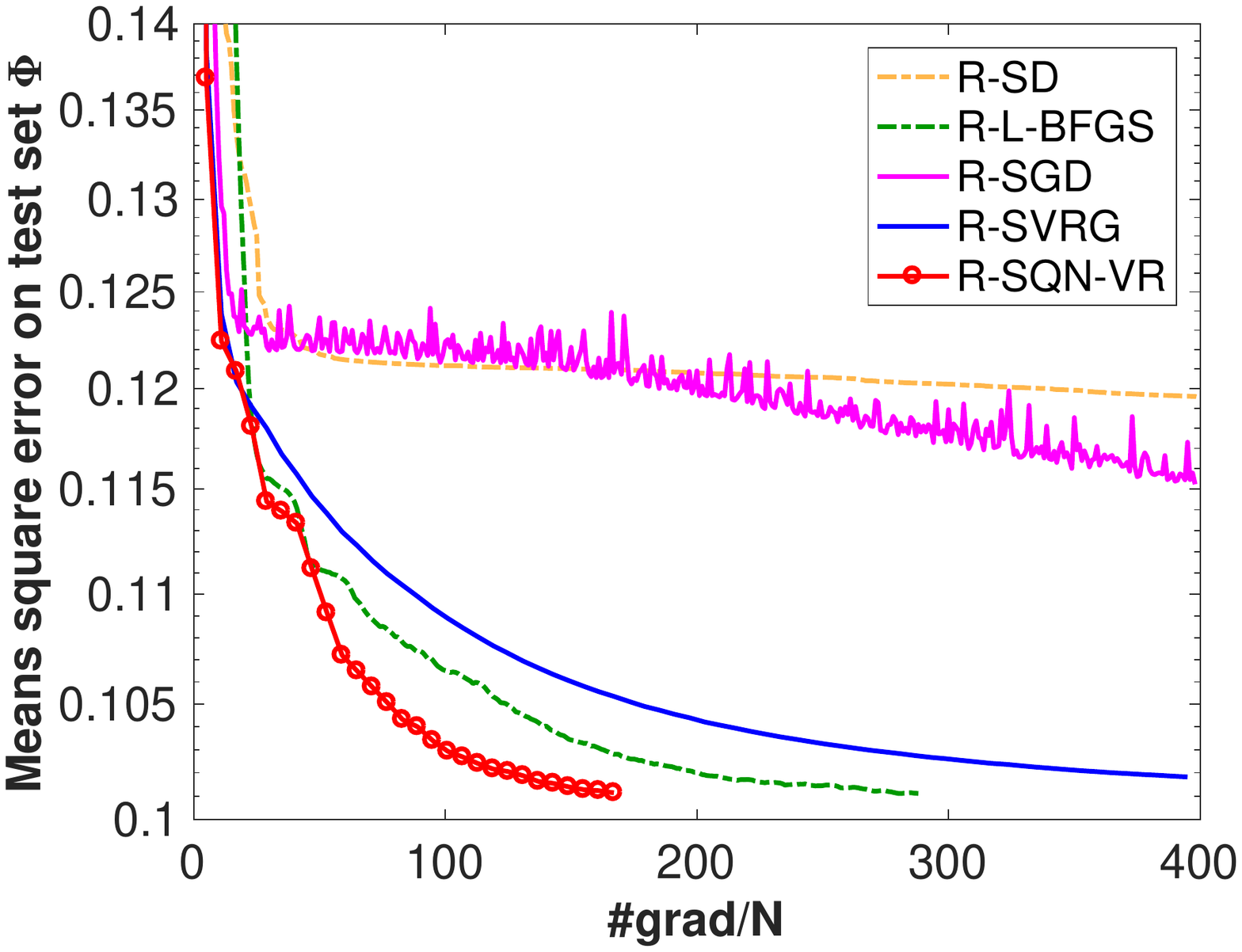}\\
\vspace*{-1.5cm}
{\scriptsize (i) {\bf Case MC-R: MovieLens-1M.}}
\vspace*{0.2cm}
\end{center}
\end{minipage}
\caption{Performance evaluations on Karcher mean (KM) problem and low-rank MC problem.}
\label{fig:PerformanceEvalatons}
\end{center}
\end{figure*}

\changeHK{{\bf Karcher mean problem on SPD manifold.} \changeHK{The first comparison is} the Karcher mean problem on SPD matrices \cite{Yuana_ICCS_2016_s}.} All experiments use the batch size fixed to 1 and $L=4$, \changeHK{and} are initialized randomly and are stopped when the number of iterations \changeHK{reaches} $10$ for R-SVRG and R-SQN-VR, and $60$ for others. $\alpha$ are \changeHK{tuned from} \changeHK{$\{10^{-5},\ldots, 10^{-1}\}$}. $m_k$ and the batch size are $3N$ and $1$, respectively. \changeHK{Figures \ref{fig:PerformanceEvalatons}(a) \changeHK{and (b)} show the results of the optimality gap when $N=500$ with $d=3$ ({\bf Case KM-1}) and the larger size case with $N=1500$ ({\bf Case KM-2}), respectively.} These results reveal that R-SQN-VR outperforms others. 

\changeHK{{\bf MC problem on Grassmann manifold.}} We first consider a synthetic dataset. The proposed algorithm is also compared with Grouse \cite{Balzano_arXiv_2010_s}, a state-of-the-art stochastic \changeHK{gradient} algorithm on the Grassmann manifold. Algorithms are initialized randomly as \cite{Kressner_BIT_2014_s}. $\alpha$ are \changeHK{tuned from} \changeHK{$\{10^{-3},5\times10^{-3}, \ldots, 10^{-2}, 5\times10^{-2}\}$} for R-SGD, R-SVRG and R-SQN-VR, and $\changeHK{\{1,10,100\}}$ for Grouse. We \changeHK{set explicitly the condition number\changeHK{, denoted as CN}, of the matrix, which represents} the ratio of the \changeHK{maximal and the minimal} singular values of the matrix. We also set the over-sampling ratio (OS) for the number of known entries. 
%An OS of $5$ \changeHK{means that we randomly and uniformly select known} entries of $5(N+d-r)r$ a priori out of the total $Nd$ entries. 
The Gaussian noise is also added with the noise level $\sigma$ as suggested in \cite{Kressner_BIT_2014_s}. $m_k$ and the batch size are set to $5N$ and \changeHK{50}, respectively. 
\changeHK{The maximum number of the outer iterations to stop is $100$ for R-SVRG and R-SQN-VR, and $100(m_k+1)$ for the others.}
This experiment evaluates 
%two combinations of vector transport and retraction, namely, the parallel translation and the exponential mapping, and 
the projection\changeHS{-}based vector transport and the QR-decomposition\changeHS{-}based retraction,
%While the former combination satisfies the \changeHK{necessary} conditions for the convergence, the latter 
\changeHK{which do not satisfy the locking condition}, but is computationally efficient. 
\changeHK{The baseline problem instance  ({\bf Case MC-S1}) is the case of $N\!=\!5000$, $d\!=\!\changeHK{200}$, \changeHK{rank $r=5$}, $L\!=\!\changeHK{10}$, $\changeHK{{\rm OS}=\changeHK{8}}$, $\sigma\changeHK{=}10^{-10}$ and $\changeHK{{\rm CN}=\changeHK{50}}$. Additionally, \changeHK{changing some parameters of those in {\bf Case MC-S1}}, we evaluate 
the lower-sampling case with ${\rm OS}\!=\!4$ ({\bf Case MC-S2}), 
the ill-conditioning case with ${\rm CN}\!=\!\changeHK{100}$ ({\bf Case MC-S3}), 
the higher noise case with $\sigma=10^{-6}$ ({\bf Case MC-S\changeHK{4}}), and 
the higher rank case with $r=10$ ({\bf Case MC-S\changeHK{5}}).
The results of the MSE on test set $\Phi$, which is different from the training set $\Omega$, are shown in Figures \ref{fig:PerformanceEvalatons}(c)-(h), respectively. This gives the prediction accuracy of missing elements.} From the figures, we confirm the superior performance of R-SQN-VR. 
{\bf Case MC-S6} \changeHK{for different memory sizes $L$ reveals that the larger size does not always show better results, which is also noticed in \cite{Wang_arXiv_2016}}. 
\changeHK{Finally,} we compare the algorithms on a real-world dataset, the MovieLens-1M dataset\footnote{\burl{http://grouplens.org/datasets/movielens/}}. \changeHK{It contains} a million ratings  \changeHK{for} $3952$ movies ($N$) \changeHK{of} $6040$ users ($d$). \changeHK{We} further randomly split \changeHK{this set} into 80/10/10 \changeHK{percent data out of the entire data as} train/validation/test partitions. $\alpha$ is chosen from $\{10^{-5},5\times10^{-5}, \ldots, 10^{-2}, 5\times10^{-2}\}$,  \changeHK{the batch size is $50$, $r=10$, and $L=10$}. 
%The QR decomposition\changeHS{-}based vector transport is used \changeHK{envisioning a larger size of datasets}. 
The algorithms are \changeHK{terminated} when the MSE on the validation set starts to increase or the number of the outer iteration reaches 100. Figure \ref{fig:PerformanceEvalatons}(i) shows the result except Grouse, which faces issues with convergence on this set ({\bf Case MC-R}). R-SQN-VR shows much faster convergences than others. 
%Table \ref{Tab:MovieLens} also shows the best result when $r$ is $10$ and $20$ for each algorithm with the lowest test MSEs when the algorithm stopped for five runs. This also shows that the proposed R-SQN-VR gives faster convergences and lower test MSE than other algorithms, especially when $r$ is larger. 

%\begin{center}
%\begin{table}[htbp]
%\caption{Test MSE on $\Phi$ and  \# of gradient (MovieLens-1M).}
%\label{Tab:MovieLens}
%\begin{center}
%\begin{tabular}{c|l|c|c}
%\hline
%\hline
%$r$ &  Algorithm & \#grad/$N$ &   MSE on $\Phi$ \\
%\hline
%10 & R-SG& $ 399.6  \pm 8.9_{-1} $ & $ 2.597056_{-4} \pm 3.2_{-6} $ \\
%\cline{2-4}
%& R-SGD & $375.0 \pm 1.8_{1} $ & $ 2.629046_{-4} \pm 1.0_{-6} $\\
%\cline{2-4}
%& R-SVRG & $205.4 \pm 4.1_{1}$ & $2.529259_{-4}  \pm 4.7_{-7} $\\
%\cline{2-4}
%& R-SQN-VR & $ {\bf 116.6 \pm 3.6_{1}} $ &  ${\bf 2.523758_{-4} \pm 3.4_{-7}}$\\
%\hline
%\hline
%20 & R-SG & $400  \pm 0$ & $ 1.950905_{-4} \pm 2.1_{-6}$\\
%\cline{2-4}
%& R-SGD & $382.0\pm 2.8$ & $1.975479_{-4} \pm 9.6_{-7} $\\
%\cline{2-4}
%& R-SVRG & $356.0\pm 5.5_{1} $ & $1.896678_{-4} \pm 4.8_{-7} $ \\
%\cline{2-4}
%& R-SQN-VR & $ {\bf 137.0  \pm 5.1_{1}}  $& ${\bf 1.888096_{-4} \pm  5.0_{-7}} $\\
%\hline
%\hline
%\multicolumn{4}{r}{ {\scriptsize The subscript $k$ indicates a scale of $10^k$.}}
%\end{tabular}
%\end{center}
%\end{table}
%\end{center}
%
%

%
%
%
%
%
\section{Conclusions}

%We have proposed a Riemannian stochastic quasi-Newton algorithm with variance reduction (R-SQN-VR). We particularly addressed retraction and vector transport. The proposed algorithm stems from the algorithm in Euclidean space, but is now extended to Riemannian manifolds. 
%%The central difficulty of averaging, adding, and subtracting multiple gradients on a Riemannian manifold is handled by exploiting vector transport and retraction. 
%\changeHK{We proved convergence analyses of R-SQN-VR on both non-convex and retraction-convex functions under retraction and vector transport operations.} Numerical comparisons suggested the superior performance of R-SQN-VR on various benchmarks.
\changeBM{
We have proposed a Riemannian stochastic quasi-Newton algorithm with variance reduction (R-SQN-VR) on manifolds that is well suited for \changeHK{finite-sum minimization problems}. We presented a rigorous convergence analysis for taking the Hessian approximation into a variance reduction stochastic setting on a manifold. Our proposed algorithm makes the explicit use of retraction and vector transport operators on manifolds, which makes the proposed algorithm appealing on a wider number of manifolds. The numerical comparisons show the benefits of our proposed algorithm on a number of applications.
}
\clearpage
\bibliographystyle{unsrt}
\bibliography{stochastic_online_learning,matrix_tensor_completion,optimization_general,manifold_optimization}

\clearpage

\appendix

\renewcommand\thefigure{A.\arabic{figure}}  
\setcounter{figure}{0} 

\renewcommand\thetable{A.\arabic{table}}  
\setcounter{table}{0} 

\renewcommand{\theequation}{A.\arabic{equation}}
\setcounter{equation}{0}

\renewcommand\thealgorithm{A.\arabic{algorithm}}  
\setcounter{algorithm}{0}

\section{Problems and manifolds in numerical comparison}

\changeHK{This section gives a brief \changeHK{explanation} of the problems and the manifolds that are \changeHK{evaluated} in the numerical comparisons in Section \ref{Sec:Numerical_comparisons}.}

\subsection{SPD manifold and Karcher mean problem}
{\bf SPD manifold $\mathcal{S}\changeHK{_{++}^d}$.} Let $\mathcal{S}\changeHK{_{++}^d}$ be the manifold of $\changeHK{d \times d}$ SPD matrices. \changeHS{If we endow $\mathcal{S}_{++}^d$ with the Riemannian metric defined by}
\begin{eqnarray*}
\langle \xi_{\scriptsize \mat{X}}, \eta_{\scriptsize \mat{X}} \rangle_{\scriptsize \mat{X}}
&=&\rm{trace}(\xi_{\scriptsize \mat{X}} \mat{X}^{-1} \eta_{\scriptsize \mat{X}} \mat{X}^{-1})
\end{eqnarray*}
at \changeHS{$\mat{X} \in \mathcal{S}_{++}^d$}, the SPD manifold $\mathcal{S}\changeHK{_{++}^d}$ becomes \changeHS{a} Riemannian manifold. 
\changeHS{The explicit formula for the exponential mapping is given by}
\begin{eqnarray*}
{\rm Exp}_{\scriptsize \mat{X}}(\xi_{\scriptsize \mat{X}}) &=& \mat{X}^{1/2} \exp(\mat{X}^{-1/2} \xi_{\scriptsize \mat{X}}\mat{X}^{-1/2}) \mat{X}^{1/2}
\end{eqnarray*}
for \changeHS{any} $\xi_{\scriptsize \mat{X}} \in T_{\scriptsize \mat{X}}\mathcal{S}\changeHK{_{++}^d}$ and $\mat{X} \in \mathcal{S}\changeHK{_{++}^d}$. On the other hand,  
$R_{\scriptsize \mat{X}}(\xi_{\scriptsize \mat{X}})=\mat{X}+\xi_{\scriptsize \mat{X}}+\frac{1}{2}\xi_{\scriptsize \mat{X}} \mat{X}^{-1} \xi_{\scriptsize \mat{X}}$ proposed \changeHK{in} \cite{JeurisVV_2012} is \changeHK{a retraction}, \changeHS{which} is symmetric positive-definite for all $\xi_{\scriptsize \mat{X}} \in T_{\scriptsize \mat{X}}\mathcal{S}\changeHK{_{++}^d}$ and $\mat{X} \in \mathcal{S}\changeHK{_{++}^d}$. \changeHK{The} parallel translation on $\mathcal{S}\changeHK{_{++}^d}$ \changeHK{along $\eta_{\scriptsize \mat{X}}$} is given by 
\begin{eqnarray*}
	P\changeHK{_{\eta_{\tiny \mat{X}}}}(\xi_{\scriptsize \mat{X}}) &=& \mat{X}^{1/2} \mat{Y} \mat{X}^{-1/2}\xi_{\scriptsize \mat{X}}\mat{X}^{-1/2} \mat{Y} \mat{X}^{1/2}, 
\end{eqnarray*}
where $\mat{Y}= \exp(\mat{X}^{-1/2}\eta_{\scriptsize \mat{X}}\mat{X}^{-1/2}/2\changeHK{)}.$ 
A \changeHK{more efficient} \changeHK{algorithm that constructs} an isometric vector transport is proposed based on a field of orthonormal tangent bases \cite{Yuana_ICCS_2016_s} \changeHK{while satisf\changeHS{y}ing} the locking condition (\ref{Eq:locking_condition}). We use it in this experiment, and the details are in \cite{Huang_SIOPT_2015, Yuana_ICCS_2016_s}. The logarithm map of $\mat{Y}$ at $\mat{X}$ is given by 
\begin{eqnarray*}
{\rm Log}_{\scriptsize \mat{X}}(\mat{Y}) &=& \mat{X}^{1/2} \log (\mat{X}^{-1/2}\mat{Y}\mat{X}^{-1/2})\mat{X}^{1/2} = \log(\mat{Y}\mat{X}^{-1})\mat{X}.
\end{eqnarray*}

{\bf Karcher mean problem on $\mathcal{S}\changeHK{_{++}^d}$.} The Karcher mean is introduced as a notion of {\it mean} on Riemannian manifolds by Karcher \cite{Karcher_1977_s}. It generalizes the notion of an  ``average''  on \changeHS{a} manifold. Given $N$ points on $\mathcal{S}\changeHK{_{++}^d}$ with matrix representations $\mat{Q}_1,\ldots,\mat{Q}_N$, the Karcher mean is defined as the solution to the problem
\begin{eqnarray*}
\label{Eq:KarcherMean}
{\displaystyle \min_{{\scriptsize \mat{X} \in \mathcal{S}\changeHK{_{++}^d}}}} &{\displaystyle \frac{1}{N} \sum_{n=1}^N ({\rm dist}(\mat{X}, \mat{Q}_n))^2},
\end{eqnarray*}
$\min_{{\scriptsize \mat{X} \in \mathcal{S}\changeHK{_{++}^d}}}  \frac{1}{N} \sum_{n=1}^N ({\rm dist}(\mat{X}, \mat{Q}_n))^2$, where ${\rm dist}(p, q)=\| \log (p^{-1/2}qp^{-1/2})\|_F$ \changeHK{represents} the distance along \changeHK{the corresponding geodesic} between the elements on $\mathcal{S}\changeHK{_{++}^d}$ \changeHS{with respect to} the affine-invariant metric. \changeHS{The gradient of the loss function is computed as} $\frac{2}{N} \sum_{n=1}^{N} -{\rm log}(\mat{Q}_n \mat{X}^{-1})\mat{X}$. The Karcher mean on $\mathcal{S}\changeHK{_{++}^d}$ is frequently used for computer vision problems, such as visual object categorization and pose categorization \cite{Jayasumana_IEEETPAMI_2015_s}. Since recursive calculations are needed with each visual image, stochastic gradient algorithms become an appealing choice for large datasets.

\subsection{Grassmann manifold and MC problem}
{\bf Grassmann manifold ${\rm Gr}(r,d)$.} \changeHS{A point on} the Grassmann manifold is \changeHS{an equivalence class represented} by a $d \times r$ orthogonal matrix \mat{U} with orthonormal columns, i.e., $\mat{U}^T\mat{U}=\mat{I}$. Two orthogonal matrices \changeHK{express} the same element on the Grassmann manifold if they are related by right multiplication of an $r\times r$ orthogonal matrix $\mat{O} \in \mathcal{O}(r)$. Equivalently, an element of ${\rm Gr}(r,d)$ is identified with a set of $d \times r$ orthogonal matrices $[\mat{U}]: =\{\mat{U}\mat{O} :\mat{O} \in \mathcal{O}(r)\}$. That is, ${\rm Gr}(r,d) :={\rm St}(r,d)/ \mathcal{O}(r)$, where ${\rm St}(r,d)$ is the {\it Stiefel manifold} that is the set of matrices of size $d \times r$ with orthonormal columns. The Grassmann manifold has the structure of a Riemannian quotient manifold 
\cite[Section~3.4]{Absil_OptAlgMatManifold_2008}.

The exponential mapping for the Grassmann manifold from $\mat{U}(0) := \mat{U} \in {\rm Gr}(r,d)$ in the direction of $\xi \in T_{\scriptsize \mat{U}(0)} {\rm Gr}(r,d)$ is given in a closed form as 
\cite[Section 5.4]{Absil_OptAlgMatManifold_2008} 
\begin{eqnarray*}
\mat{U} (t)  &=&  [\mat{U}(0)  \mat{V}\ \  \mat{W}] 
	\left[
    		\begin{array}{c}
      		\cos t \Sigma \\
      		\sin t \Sigma \\
    		\end{array}		
	\right]
    \mat{V}^T,
\end{eqnarray*}    	
where $\xi=\mat{W} \Sigma \mat{V}^T$ is the singular value decomposition (SVD) of $\xi$ \changeHK{with rank $r$}. The $\sin(\cdot)$ and $\cos(\cdot)$ operations \changeHK{are performed} only on the diagonal entries. The parallel translation of $\zeta \in T_{\scriptsize \mat{U}(0)} {\rm Gr}(r,d)$ on the Grassmann manifold along $\gamma(t)$ with $\dot \gamma(0) = \mat{W} \Sigma \mat{V}^T$ is given in a closed form by
\begin{eqnarray*}
\label{Eq:parallel_translation}
\zeta(t) & = & \left( [\mat{U}(0) \mat{V}\ \  \mat{W}] 
	\left[
    		\begin{array}{c}
      		-\sin t \Sigma \\
      		\cos t \Sigma \\
    		\end{array}
	\right]
    \mat{W}^T + (\mat{I} - \mat{W}\mat{W}^T)
    \right) \zeta.
\end{eqnarray*}	
%$
%\zeta(t)  =  ( [\mat{U}(0) \mat{V}\ \  \mat{W}] 
%	\left[
%      		-\sin t \Sigma;
%      		\cos t \Sigma 
%	\right]
%    \mat{W}^T + (\mat{I} - \mat{W}\mat{W}^T)
%    ) \zeta.
%$
%
The logarithm map of $\mat{U}(t)$ at $\mat{U}(0)$ on the Grassmann manifold is given by
\begin{eqnarray*}
\label{Eq:logarithm_map}
\xi   &=& {\rm Log}_{\scriptsize \mat{U}(0)}(\mat{U}(t)) \ = \ \mat{W} \arctan(\Sigma) \mat{V}^T,
\end{eqnarray*}	
where $\mat{W}\Sigma \mat{V}^T$ is the SVD of $(\mat{U}(t) - \mat{U}(0) \mat{U}(0)^T  \mat{U}(t))\allowbreak(\mat{U}(0)^T \mat{U}(t))^{-1}$ \changeHK{with rank $r$}. Furthermore, a popular retraction is 
\begin{eqnarray*}
R_{\scriptsize \mat{U}(0)}(\xi) &=&{\rm qf}(\mat{U}(0)+t\xi) \quad\quad\quad (= \mat{U}(t))
\end{eqnarray*}
which extracts the orthonormal factor based on QR decomposition, and a popular vector transport uses an orthogonal projection of $t\xi$ to the horizontal space at $\mat{U}(t)$, i.e., $\changeHS{(\mat{I} - \mat{U}(t)\mat{U}(t)^T)} t\xi$ \cite{Absil_OptAlgMatManifold_2008}.

{\bf Matrix completion problem.}
The matrix completion problem \changeHK{is} completing an incomplete matrix $\mat{X}$, say of size $d \times N$, from a small number of entries by assuming \changeHK{that the latent structure of the matrix is low-rank}. If $\Omega$ is the set of known indices in $\mat{X}$, the rank-$r$ matrix completion problem amounts to solving 
\begin{equation*}
\label{Eq:MC_batch}
\begin{array}{ll}
\min_{\scriptsize \mat{U},\mat{A}} \|\mathcal{P}_{\Omega}(\mat{UA}) - \mathcal{P}_{\Omega}(\mat X) \|_F^2,
\end{array}
\end{equation*}
where $\mat{U} \in \mathbb{R}^{d \times r}, \mat{A} \in \mathbb{R}^{r \times N}$, and the operator $\mathcal{P}_{\Omega}$ acts as $\mathcal{P}_{\Omega}(\mat{X}_{ij})=\mat{X}_{ij}$ if $(i,j) \in \Omega$ and $\mathcal{P}_{\Omega}(\mat{X}_{ij})=0$ otherwise.
%This is called the orthogonal sampling operator and is a mathematically convenient way to represent the subset of known entries. 
Partitioning $\mat{X} = [\vec{x}_1, \ldots, \vec{x}_n] $, the previous problem is equivalent to 
\begin{eqnarray*}
\label{Eq:MC}
\min_{{\scriptsize \mat{U}} \in \mathbb{R}^{d \times r},\ \vec{a}_n \in \mathbb{R}^{r}} 
\frac{1}{N} \sum_{n=1}^N \| \mathcal{P}_{\Omega_n}(\mat{U} \vec{a}_n) - \mathcal{P}_{\Omega_n}(\vec{x}_n) \|_2^2,
\end{eqnarray*}
where $\vec{x}_n \in \mathbb{R}^d$ and the operator $\mathcal{P}_{\Omega_n}$ is the sampling operator for the $n$-th column. Given \mat{U}, $\vec{a}_n$ admits a closed form solution. Consequently, the problem only depends on the column space of $\mat{U}$ and is on ${\rm Gr}(r,d)$ \cite{boumal15a}.

\section{Two-loop Hessian inverse updating algorithm}

The section summarizes the Riemannian two-loop Hessian inverse updating algorithm in Algorithm \ref{ApdAlg:H_Update}. This is an straightforward extension of that in the Euclidean space explained in \cite[Section 7.2]{Nocedal_NumericalOptBook_2006}.

\begin{algorithm}[htbp]
\caption{Hessian \changeHK{i}nverse \changeHK{u}pdating}
\label{ApdAlg:H_Update}
\begin{algorithmic}[1]
\REQUIRE{Pair-updating counter $t$, memory depth $\tau$, correction pairs $\{ s^{k}_u, y^{k}_u\}_{u=k-\tau}^{k-1}$, gradient $p$.}
\STATE{$p_0=p$.}
\STATE{$\mathcal{H}^0_k=\changeHK{\chi_k}{\rm id} = \frac{\langle s^{k}_{t},y^{k}_{t}\rangle}{\langle y^{k}_{t}, y^{k}_{t} \rangle} {\rm id}$.}
\FOR{$u=0,1,2,\ldots, \tau-1$} 
	\STATE{$\rho_{k-u} = 1/\langle s^{k}_{k-u-1},y^{k}_{k-u-1} \rangle$.}
	\STATE{$\alpha_u = \rho_{k-u-1} \langle s^{k}_{k-u-1}, p_u \rangle$.}
	\STATE{$p_{u+1} = p_{u} - \alpha_u y^{k}_{k-u-1}$.}
\ENDFOR
\STATE{$q_0 = \mathcal{H}^0_k p_{\tau}$.}
\FOR{$u=0,1,2,\ldots, \tau-1$} 
	\STATE{$\beta_u = \rho_{k-\tau+u} \langle y^{k}_{k-\tau+u}, q_u \rangle$.}
	\STATE{$q_{u+1} = q_u + (\alpha_{\tau-u-1}- \beta_u) s^{\changeHK{k}}_{k-\tau+u}$.}	
\ENDFOR
\STATE{$q = q_{\tau}$.}
\end{algorithmic}
\end{algorithm}

\clearpage
\section{Proofs of convergence analysis on non-convex functions}

\changeHK{This section presents the proof of the global convergence analysis on non-convex functions.} Hereinafter, we use $\mathbb{E}[\cdot]$ to \changeHK{express} expectation with respect to the joint distribution of all random variables. For example, \changeHK{$w_t$} is determined by the realizations of the independent random variables  \changeHK{$\{i_1,i_2,\ldots,i_{t-1}\}$}, the total expectation of  \changeHK{$f(w_t)$ for any $t \in \mathbb{N}$} can be taken as  \changeHK{$\mathbb{E}[f(w_t)]=\mathbb{E}_{i_1} \mathbb{E}_{i_2} \ldots \mathbb{E}_{i_{t-1}}[f(w_t)]$}. We also use  \changeHK{$\mathbb{E}_{i_t}[\cdot]$} to denote an expected value taken with respect to the distribution of the random variable  \changeHK{$i_t$}. 
\changeHK{In addition, we omit the subscript $\changeHK{\tilde{w}^k}$ for a Riemannian metric $\langle \cdot, \cdot \rangle_{\tilde{w}\changeHK{^k}}$ when the tangent space to be considered is clear.}

\subsection{Preliminary lemmas}

This subsection first states some preliminary lemmas. 

\changeHK{The literature \cite{Absil_OptAlgMatManifold_2008} generalizes} a Taylor's theorem to Riemannian manifolds. However, \changeHK{it addresses} the exponential mapping instead of \changeHK{the} retraction. Therefore, \cite{Huang_SIOPT_2015} apply\changeHK{s} Taylor's theorem on \changeHK{the retraction} \changeHK{by newly introducing} a function \changeHK{along} a curve on the manifold. \changeHK{Here, we denote} $f(R_{\changeHK{w_t^k}}(t \eta_k/\|\eta_k\|\changeHK{_{w^k_{t}}}))$ \changeHK{for} a twice continuously differentiable \changeHK{objective} function. From Taylor's theorem, we obtain below;

\begin{Lem}[In Lemma 3.2 in {\cite{Huang_SIOPT_2015}}]
\label{Lemma:HessianUpperBoundFunc}
\changeHK{Under} Assumptions \ref{Assump:1}.1, \ref{Assump:1}.2\changeHK{, and \ref{Assump:1}.3}, there exists $\Lambda$ such that
\begin{eqnarray}
\label{Eq:HessianUpperBoundFunc}
f(w^k_{t+1}) - f(w^k_{t}) & \leq & \langle  \gradf (w^k_{t}), \alpha^k_t \eta_k \rangle_{w^k_{t}} + \frac{1}{2} \Lambda (\alpha^k_t \| \eta_k \|\changeHK{_{w^k_{t}}})^2.
\end{eqnarray}
\end{Lem}
\begin{proof}
From Taylor's theorem, we have
\begin{eqnarray}
\label{}
& & f(w^k_{t+1}) - f(w^k_{t})  \nonumber \\
	%& = & m_k(\alpha^k_t \| \eta_k \|) - m_k(0) \nonumber \\
	& = & \changeHK{f(R_{\changeHK{w_t^k}}(\alpha_t^k \eta_k)) - f(R_{\changeHK{w_t^k}}(0))} \nonumber \\
	& = & \changeHSS{\frac{d}{d\tau}} \changeHK{f(R_{\changeHK{w_t^k}}(\changeHSS{\tau \eta_k/\|\eta_k\|_{w^k_t}}))}\changeHSS{\Big|_{\tau=0}\cdot} \alpha^k_t \| \eta_k \|\changeHK{_{w^k_{t}}}  + \frac{1}{2} \changeHSS{\frac{d^2}{d\tau^2}} \changeHK{f(R_{\changeHK{w_t^k}}(\changeHSS{\tau} \eta_k/\|\eta_k\|\changeHK{_{w^k_{t}}}))}\changeHSS{\Big|_{\tau=p}}\cdot (\alpha^k_t \| \eta_k \|\changeHK{_{w^k_{t}}} )^2 \nonumber \\
	& = & \langle \gradf (w^k_{t}), \alpha^k_t \eta_k \rangle_{w^k_{t}}  + \frac{1}{2} \changeHSS{\frac{d^2}{d\tau^2}} \changeHK{f(R_{\changeHK{w_t^k}}(\changeHSS{\tau} \eta_k/\|\eta_k\|\changeHK{_{w^k_{t}}}))}\changeHSS{\Big|_{\tau=p}}\cdot (\alpha^k_t \| \eta_k \|\changeHK{_{w^k_{t}}} )^2 \nonumber \\	
	& \leq & \langle \gradf (w^k_{t}), \alpha^k_t \eta_k \rangle_{w^k_{t}} + \frac{1}{2} \Lambda (\alpha^k_t \| \eta_k \|\changeHK{_{w^k_{t}}})^2,\nonumber 
\end{eqnarray}
where $0 \leq p \leq \alpha^k_t \| \eta_k \|\changeHK{_{w^k_{t}}}$, and \changeHSS{$\Lambda$ is the constant in Assumption \ref{Assump:1}.3}. This completes the proof. 
\end{proof}

\begin{Lem}
\label{Lemma:y_y_s_y_lower}
Suppose Assumption \ref{Assump:1} holds\changeHSS{. Then} there exist\changeHK{s} a constant $0 < \changeHK{\upsilon}$ for all $k$ such that
\begin{eqnarray}
\label{Eq:y_y_s_y_lower}
%\changeHK{\upsilon} \ \leq \  
\changeHK{\upsilon} & \leq & \frac{\langle y_k, y_k \rangle}{\langle s_k, y_k \rangle}.
\end{eqnarray}
\end{Lem}
\begin{proof}
\changeHSS{The claim for the case $s_k=0$ is obvious. Assume that $s_k \neq 0$ below.}
This is given by applying \changeHK{Cauchy-Schwarz} inequality to the condition (\ref{Eq:Cautious_Update}) recursively. More specifically, (\ref{Eq:Cautious_Update}) yields that
\begin{eqnarray}
\changeHK{\epsilon \| s_k \|^2} \ \leq \ \langle y_k, s_k \rangle \ \leq \| y_k \| \| s_k \|, \nonumber
\end{eqnarray}
and considering the most left and right terms, we obtain 
\begin{eqnarray}
\label{}
\| s_k \| & \leq & 
\changeHK{\frac{1}{\epsilon} \| y_k \|.}
\nonumber
\end{eqnarray}
Substituting this into the above equation yields  
\begin{eqnarray}
\label{}
\langle s_k, y_k \rangle \ \leq \ \| s_k \| \| y_k \| \ \leq \  
\changeHK{\frac{1}{\epsilon} \| y_k \|^2.}
\nonumber
\end{eqnarray}
Consequently, we obtain
\begin{eqnarray}
\label{}
\frac{\| y_k \|^2}{\langle s_k, y_k \rangle} &\geq & 
%\epsilon \|\gradf(\tilde{w}^{k})\|.
\changeHK{\epsilon}\quad \changeHK{(=\upsilon)}.
\nonumber
\end{eqnarray}
This completes the claim by denoting \changeHK{$\epsilon$} as $\changeHK{\upsilon}$. 
\end{proof}

\begin{Lem}
\label{Lemma:y_y_s_y_non_convex}
Suppose Assumption \ref{Assump:1} holds\changeHSS{. T}here exists a constant $0 < \changeHK{\Upsilon_{nc}}$\ for all $k$ such that
\begin{eqnarray}
\label{Eq:y_y_s_y_non_convex}
%\changeHK{\upsilon} \ \leq \  
\frac{\langle y_k, y_k \rangle}{\langle s_k, y_k \rangle} &  \leq & \changeHK{\Upsilon_{nc}}.
\end{eqnarray}
\end{Lem}
\begin{proof}
\changeHK{Most part of this proof is \changeHK{given in Lemma 3.9 of {\cite{Huang_SIOPT_2015}}. But, \cite{Huang_SIOPT_2015} uses the \changeHK{strongly retraction-convexity} assumption for the final part. Therefore, we include that of \cite{Huang_SIOPT_2015} for completeness, and describe the final part with a slight modifications. This proof is also included for the subsequent} analysis. }

Define $y_k^P = \gradf(\tilde{w}^{k+1}) - P_{\gamma_k}^{1 \leftarrow 0} \gradf(\tilde{w}^{k})$, where $\gamma_k(t)=R_{\tilde{w}^{k}}(t \eta_k)$, i.e., the retraction \changeHK{curve} \changeHS{connecting} $\changeHK{\tilde{w}^k}$ \changeHS{and} $\changeHK{\tilde{w}^{k+1}}$\changeHS{,} and $P_{\gamma_k}$ is the parallel translation along $\gamma_k(t)$. We have $\| P_{\gamma_k}^{1 \leftarrow 0} y_k^P - \bar{H}_k  \eta_k \| \leq b_0 \|  \eta_k \|^2 = b_0 \| s_k \|^2$, where $\bar{H}_k = \int_0^1 P_{\gamma_k}^{0 \leftarrow t} \hessf(\gamma_k(t)) P_{\gamma_k}^{t \leftarrow 0} dt$ and $b_0 > 0$. It follows that 
\begin{eqnarray}
	\label{Eq:y_s_bound_new}
	\| y_k \| 
	& \leq & \| y_k - y_k^P\| + \| y_k^P \|  \nonumber \\
	& =    & \| y_k - y_k^P\| + \| P_{\gamma_k}^{0 \leftarrow 1} y_k^P \| \nonumber \\
	& \leq &  \| y_k - y_k^P\| + \| P_{\gamma_k}^{0 \leftarrow 1} y_k^P - \bar{H}_k  \eta_k \| + \| \bar{H}_k  \eta_k  \| \nonumber \\		
	& \leq & \| \gradf(\tilde{w}^{k+1}) / \kappa_k - \mathcal{T}_{\eta_k} \gradf(\tilde{w}^{k}) - \gradf (\tilde{w}^{k+1}) + P_{\gamma_k}^{0 \leftarrow 1} \gradf(\tilde{w}^{k})\| \nonumber \\
	&& + \| \bar{H}_k \eta_k  \|  + b_0 \| s_k \|^2\nonumber \\
	& \leq & \| \gradf(\tilde{w}^{k+1}) / \kappa_k - \gradf (\tilde{w}^{k+1}) \| 
	+ \|  P_{\gamma_k}^{0 \leftarrow 1} gradf(\tilde{w}^{k}) - \mathcal{T}_{\eta_k} \gradf(\tilde{w}^{k}) \| \nonumber \\
	&&  + \| \bar{H}_k \eta_k  \|  + b_0 \| s_k \|^2\nonumber \\
	%& \leq & b_1 \| s_k \| \| \gradf(w_{k+1}) \| + b_2 \| s_k \| \| \gradf(w_{k}) \| + b_3 \| s_k \| +  b_0 \| s_k \|^2\nonumber \\				
	\label{Eq:y_s_bound_remark}
	& \leq & b_1 \| s_k \| \| \gradf (\tilde{w}^{k+1})\|  + b_2 \| s_k \| \| \gradf (\tilde{w}^{k})\| + b_3 \|s_k \| + b_0 \| s_k \|^2\\
	& \leq & b_4 \| s_k \|,\nonumber
\end{eqnarray}
where $b_1, b_2, b_3$, and $b_4 > 0$. 
\changeHK{Here, we directly obtain the following fact from (\ref{Eq:Cautious_Update}) in the cautious update as
\begin{eqnarray}
	\label{Eq:s2_ys_bound}
	\frac{\| s_k \|^2}{\langle y_k, s_k \rangle} &\leq & \frac{1}{\epsilon}. 
\end{eqnarray}
Therefore, we finally obtain the upper bound of $\frac{\langle y_k, y_k \rangle }{\langle s_k, y_k \rangle}$ from (\ref{Eq:y_s_bound_new}) and (\ref{Eq:s2_ys_bound}) as
\begin{eqnarray}
	\label{Eq:y2_ys_bound}
	\frac{\langle y_k, y_k \rangle }{\langle s_k, y_k \rangle} &=& \frac{\| s_k \|^2}{\langle s_k, y_k \rangle}\cdot \frac{\| y_k \|^2}{\| s_k \|^2} 
	\ \ \leq \ \ \frac{b_4^2}{\epsilon}\  \ \ (= \Upsilon_{nc}).
\end{eqnarray}	
}
\changeHK{Denoting $b_4^2/\epsilon$ as $\Upsilon_{nc}$,} this completes the proof. 
\end{proof}

\begin{Rmk}
From the proof of Lemma \ref{Lemma:y_y_s_y_non_convex}, if \changeHK{the} parallel translation \changeHK{is used} for vector transport, i.e., $\mathcal{T} = P$, the first two terms in (\ref{Eq:y_s_bound_remark}) are equal to zero, and the upper bound $\changeHK{\Upsilon_{nc}}$ in (\ref{Eq:y_y_s_y_non_convex}) can get smaller than that of the case in \changeHK{the} vector transport. 
\end{Rmk}

\subsection{\changeHK{Eigenvalue bounds of $\mathcal{H}_t^k$ on non-convex functions}}
\label{AppSec:BoundOfHNonConvex}

This \changeHK{sub}section presents Proposition \ref{AppProposition:HessianOperatorBoundsNonConvex}, which is an essential \changeHK{proposition} that bounds the eigenvalues of $\mathcal{H}_t^k$ at $w_t^k$, \changeHK{i.e.,}  $\mathcal{H}^k_t := \mathcal{T}_{\tilde{\eta}_t^k} \circ \tilde{\mathcal{H}}^k \circ (\mathcal{T}_{\tilde{\eta}_t^k})^{-1}$. \changeHK{To this end, we particularly use the {\it Hessian approximation operator} $\tilde{\mathcal{B}}^{k}=(\tilde{\mathcal{H}}^k)^{-1}$ as opposed to $\tilde{\mathcal{H}}^k$. Since mentioned in the algorithm description, we consider the curvature information for $\tilde{\mathcal{H}}^k$ at $\tilde{w}^k$, i.e., every outer \changeHK{epoch}, and reuse this $\tilde{\mathcal{H}}^k$ in the calculation of the \changeHK{second-order} modified stochastic gradient $\mathcal{H}^k_t\changeHK{\xi_t^k}$ at $w^k_t$. Thereby, the way of the proof consists of two steps as follows;
\begin{enumerate}
	\item We first address the bounds of $\tilde{\mathcal{H}}^k$ at $\tilde{w}^k$. The main task of the proof is to bound the Hessian operator $\tilde{\mathcal{B}}^{k}=(\tilde{\mathcal{H}}^k)^{-1}$.
	\item Next, we bound $\mathcal{H}_t^k$ at $w^k_t$ based on the bounds of $\tilde{\mathcal{H}}^k$ at $\tilde{w}^k$.
\end{enumerate}
}

It should be noted that, in this \changeHK{sub}section, the curvature pair $\changeHK{\{s_j^{k}, y_j^{k}\}_{j=k-L}^{\changeHK{k-1}}} \in T_{\changeHK{\tilde{w}^k}}\mathcal{M}$ is \changeHK{simply} notated as $\changeHK{\{s_j, y_j\}_{j=k-L}^{\changeHK{k-1}}}$. 

First, we attempt to bound ${\rm trace}(\hat{\tilde{\mathcal{B}}})$ in order to bound the eigenvalues of $\tilde{\mathcal{H}}^k$, where a {\it hat} denotes the coordinate expression of the operator. \changeHK{The basic structure of the proof follows stochastic L-BFGS methods in the Euclidean space, e.g., \cite{Mokhtari_JMLR_2015_s,Byrd_SIOPT_2016,Wang_arXiv_2016}. Nevertheless, some special treatments considering the Riemannian setting and the \changeHK{l}emmas earlier are required.} It should be noted that ${\rm trace}(\hat{\tilde{\mathcal{B}}})$ \changeHK{does} not depend on the chosen basis.
\begin{Lem}[Bounds of trace of $\tilde{\mathcal{B}}^k$]
\label{Lemma:TraceBound}
\changeHK{Consider the recursion of $\tilde{\mathcal{B}}^k_u$ as}
\begin{eqnarray}
	\label{Eq:HessianOperatorUpdate}
	\tilde{\mathcal{B}}^k_{u+1} & = & \check{\mathcal{B}}^k_u 
	- \frac{\check{\mathcal{B}}^k_u s_{k-\tau + u} (\check{\mathcal{B}}^k_u s_{k-\tau + u})^{\flat}}
	{(\check{\mathcal{B}}^k_u s_{k-\tau + u})^{\flat} s_{k-\tau + u}}
	+ \frac{y_{k-\tau + t} y^{\flat}_{k-\tau + u}}{y^{\flat}_{k-\tau + u} s_{k-\tau + u}},
\end{eqnarray}
where $\check{\mathcal{B}}^k_u= \mathcal{T}_{\eta_k} \tilde{\mathcal{B}}^k_u (\mathcal{T}_{\eta_k})^{-1}$ \changeHK{for $u=0, \ldots, \tau-1$}. 
\changeHK{The Hessian approximation at $k$-th outer epoch is $\tilde{\mathcal{B}}^k=\tilde{\mathcal{B}}^k_\tau$ when $u=\tau-1$. Then, consider the Hessian approximation $\tilde{\mathcal{B}}^k=\tilde{\mathcal{B}}^k_\tau$ \changeHK{in} (\ref{Eq:HessianOperatorUpdate}) with $\tilde{\mathcal{B}}^k_0=\gamma^{-1}_k {\rm id}$.} If Assumption \ref{Assump:1} holds, the ${\rm trace}(\hat{\tilde{\mathcal{B}}}^k)$ in a coordinate expression of $\tilde{\mathcal{B}}^k$ is uniformly upper bounded for all $k \geq 1$ \changeHSS{as}
\begin{eqnarray}
	\label{Eq:TraceBoundNonConvex}
	{\rm trace}(\changeHK{\hat{\tilde{\mathcal{B}}}^k}) & \leq & (\changeHK{M}+\tau) \changeHK{\Upsilon_{nc}}\changeHSS{,}
\end{eqnarray}
\changeHK{where $M$ is the dimension of $\mathcal{M}$.} 
Here, a hat expression represents the coordinate expression of an operator.
\end{Lem}

\begin{proof}
\changeHS{The proof can be completed parallel to the Euclidean case~\cite{Moritz_AISTATS_2016_s}.}
\changeHK{We use} a {\it hat} \changeHK{symbol in order} to \changeHK{represent} the coordinate expression of the operator $\changeHK{\tilde{\mathcal{B}}^k_{u+1}}$ and $\changeHK{\check{\mathcal{B}}^k_u}$ in update formula (\ref{Eq:HessianOperatorUpdate}). \changeHK{Because} $\changeHK{\mathcal{T}}$ is an \changeHSS{isometric} vector transport, $\mathcal{T}_{\eta_k}$ is invertible for all $k$\changeHK{. Accordingly, ${\rm trace}(\hat{\tilde{\mathcal{B}}}^k)$ and ${\rm det}(\hat{\tilde{\mathcal{B}}}^k)$ can be reformulated as} 
\begin{eqnarray}
	\label{Eq:TraceEquivalency}
	{\rm trace}(\hat{\check{\mathcal{B}}}^k) & = &{\rm trace}(\hat{\mathcal{T}}_{\eta_k} \hat{\tilde{\mathcal{B}}}^k \hat{\mathcal{T}}^{-1}_{\eta_k}) = {\rm trace}(\hat{\tilde{\mathcal{B}}}^k), \\ 
	\end{eqnarray}
We \changeHS{first} consider the trace lower bound of ${\rm trace}(\hat{\tilde{\mathcal{B}}}^k_\tau)$ from (\ref{Eq:TraceEquivalency}) and (\ref{Eq:HessianOperatorUpdate}) \changeHK{as}
\begin{eqnarray}
	{\rm trace}(\hat{\tilde{\mathcal{B}}}^k_{u+1}) & 
	= & {\rm trace}(\hat{\check{\mathcal{B}}}^k_u) 
	- \changeHS{\frac{\langle \hat{\check{\mathcal{B}}}^k_u \changeHSS{\hat{s}}_{k-\tau + u}, \hat{\check{\mathcal{B}}}^k_u \changeHSS{\hat{s}}_{k-\tau + u} \rangle}
	{\langle \hat{\check{\mathcal{B}}}^k_u \changeHSS{\hat{s}}_{k-\tau + u}, \changeHSS{\hat{s}}_{k-\tau + u} \rangle}}
	+ \changeHS{\frac{\langle y_{k-\tau + u}, y_{k-\tau + u} \rangle}{\langle y_{k-\tau + u}, s_{k-\tau + u} \rangle}} \nonumber\\
	%{\rm trace}(\hat{\mathcal{B}}_{k,u+1}) 
	& = & {\rm trace}(\hat{\tilde{\mathcal{B}}}^k_u) 
	- \frac{\| \hat{\check{\mathcal{B}}}^k_u \changeHSS{\hat{s}}_{k-\tau + u}\|^2}
	{\langle \hat{\check{\mathcal{B}}}^k_u \changeHSS{\hat{s}}_{k-\tau + u}, \changeHSS{\hat{s}}_{k-\tau + u} \rangle}
	+ \frac{\| y_{k-\tau + u}\|^2}{\langle y_{k-\tau + u}, s_{k-\tau + u} \rangle}\changeHSS{,} \nonumber
\end{eqnarray}
\changeHSS{where we use the same notation $\langle \cdot, \cdot\rangle$ for the inner product with respect to local coordinates corresponding to the Riemannian metric.}
\changeHK{Here, the positive definiteness of $\hat{\check{\mathcal{B}}}^k_u$ guarantees the negativity of the second term. Therefore, the bound of the third term yields} \changeHK{from} \changeHK{Lemma \ref{Lemma:y_y_s_y_non_convex}} as 
\changeHK{
\begin{eqnarray}
	{\rm trace}(\hat{\tilde{\mathcal{B}}}^k_{u+1}) & \leq & {\rm trace}(\hat{\tilde{\mathcal{B}}}^k_u) + \changeHK{\Upsilon_{nc}}.\nonumber
\end{eqnarray}
}
By \changeHK{calculating recursively this} for $u=0, \cdots, \tau-1$, we can conclude that 
\begin{eqnarray}
	{\rm trace}(\hat{\tilde{\mathcal{B}}}^k_u) & \leq & {\rm trace}(\hat{\tilde{\mathcal{B}}}^k_0) + \changeHK{u} \changeHK{\Upsilon_{nc}}.\nonumber
\end{eqnarray}
\changeHS{All that is left is to bound} ${\rm trace}(\hat{\tilde{\mathcal{B}}}^k_0)$. For this purpose, we consider the definition $\hat{\tilde{\mathcal{B}}}^k_0= {\rm id}/\changeHK{\chi_k}$\changeHSS{,} where, as a common choice in L-BFGS in the Euclidean, $\changeHK{\chi_k} = \frac{\langle s_k,y_k \rangle}{\langle y_k, y_k \rangle}$, and we obtain 
\begin{eqnarray}
	\label{Eq:B_k_0}
	{\rm trace}(\hat{\tilde{\mathcal{B}}}\changeHK{^k_0}) & = & {\rm trace}\left(\frac{\mat{I}}{\changeHK{\chi_k}}\right) \ = \ \frac{{\changeHK{M}}}{\changeHK{\chi_k}} \ = \ {\changeHK{M}} \frac{\langle y_k, y_k \rangle}{\langle s_k, y_k \rangle} \ \leq \ {\changeHK{M}} \changeHK{\Upsilon_{nc}}.\nonumber
\end{eqnarray}
Consequently, we obtain 
\begin{eqnarray}
	\label{Eq:B_k_u}
	{\rm trace}(\hat{\tilde{\mathcal{B}}}^k_u) & \leq & ({\changeHK{M}}+\changeHK{u}) \changeHK{\Upsilon_{nc}}. 
\end{eqnarray}
Plugging $u=\tau$ \changeHK{into (\ref{Eq:B_k_u})} yields the claim (\changeHSS{\ref{Eq:TraceBoundNonConvex}}). 
%For $k=0$, we have $\gamma_k=\gamma_0=1$ and (\ref{Eq:B_k_0}) reduces to ${\rm trace}(\hat{\tilde{\mathcal{B}}}^k_0)={\changeHK{M}}$ while (\ref{Eq:B_k_u}) results in ${\rm trace}(\hat{\tilde{\mathcal{B}}}^k_u) \leq (1+\changeHK{u}) \changeHK{\Upsilon_{nc}}$. 
%Furthermore, for $k < \tau$, we make $\hat{\mathcal{B}}^k=\hat{\tilde{\mathcal{B}}}^k_k$ instead of $\hat{\tilde{\mathcal{B}}}^{k}=\hat{\tilde{\mathcal{B}}}^k_\tau$. In this case, the bound in (\ref{Eq:B_k_u}) can be tightened to ${\rm trace}(\hat{\tilde{\mathcal{B}}}^k_u) \leq ({\changeHK{M}}+k) \changeHK{\Upsilon_{nc}}$. 

Thus, this  completes the proof. 
\end{proof}

Now we prove the main lemma for Proposition \ref{AppProposition:HessianOperatorBoundsNonConvex}.

\begin{Lem}[Bounds of $\tilde{\mathcal{H}}^k$]
\label{Lem:BoundsOfHkNonConvex}
%Consider the operator $\tilde{\mathcal{H}}^k$.
%\changeHS{defined} by the recursion in \changeHK{$\tilde{\mathcal{H}}^{k}$}. 
Suppose the constant $0< \gamma\changeHK{_{nc}} < \Gamma\changeHK{_{nc}} < \infty$. If Assumption \ref{Assump:1} holds, the eigenvalues of $\tilde{\mathcal{H}}^k$ is bounded by $\gamma\changeHK{_{nc}}$ and $\Gamma\changeHK{_{nc}}$ for all $k \geq 1$ \changeHK{as} 
\begin{eqnarray}
	%\label{Eq:InverseHessianOperatorBounds}
	\gamma\changeHK{_{nc}} {\rm id} \ \preceq \ \tilde{\mathcal{H}}^k \ \preceq \ \Gamma{\changeHK{_{nc}}}{\rm id}\changeHSS{,} \nonumber
\end{eqnarray}
\changeHSS{where $\gamma_{nc}$ and $\Gamma_{nc}$ are some positive constants.}
\end{Lem}

\begin{proof}
%Lemma \ref{Lemma:TraceDetBound} states that the trace and determinants of the operator $\hat{\tilde{\mathcal{B}}}\changeHK{^k}=\hat{\tilde{\mathcal{B}}}^k_u$ are bounded for all $k \geq 1$. 
\changeHS{We first state the lower bound part. The proof is obtained as parallel to the Euclidean case~\cite{Mokhtari_JMLR_2015_s}.}
\changeHK{The sum of its eigenvalues of $\hat{\tilde{\mathcal{B}}}\changeHK{^k}$ corresponds to the} bounds on the trace. 
\changeHK{Here, we denote} \changeHSS{$\pi_i$} \changeHK{as} the $i$-th largest eigenvalue of the operator matrix $\hat{\tilde{\mathcal{B}}}^k$ \changeHK{for $1 \leq i \leq M$}. From (\ref{Eq:TraceBoundNonConvex}) in Lemma \ref{Lemma:TraceBound}, the sum of the eigenvalues of $\hat{\tilde{\mathcal{B}}}\changeHK{^k}$ satisfies \changeHK{below}
\begin{eqnarray}
	\label{Eq:B_t_sum_bound}
	\sum_{i=1}^{\changeHK{M}} \changeHSS{\pi_i} & =  & {\rm trace}(\hat{\tilde{\mathcal{B}}}\changeHK{^k}) \ \leq \ ({\changeHK{M}}+\tau) \changeHK{\Upsilon_{nc}}.
\end{eqnarray}
\changeHK{Because} all the eigenvalues are positive \changeHK{due to the positive definiteness of} $\hat{\tilde{\mathcal{B}}}^k$, \changeHK{it is obvious that} every \changeHSS{eigenvalue} is less than the upper bound \changeHK{of the} sum of \changeHK{all of the eigenvalues}. Consequently, we \changeHK{obtain} $\lambda_i \le ({\changeHK{M}}+\tau) \changeHK{\Upsilon_{nc}}$ for all $i$, and finally obtain $\tilde{\mathcal{B}}^k \changeHS{\preceq} ({\changeHK{M}}+\tau) \changeHK{\Upsilon_{nc}} {\rm id}$. 
%\end{proof}
The bounds in \changeHK{(\ref{Eq:B_t_sum_bound})} implies that its inverse is the bound \changeHK{for} the eigenvalues of $\hat{\tilde{\mathcal{H}}}^k=(\hat{\tilde{\mathcal{B}}}^k)^{-1}$ as
\begin{eqnarray}
	\label{Eq:H_t_bound}
	\changeHK{(\gamma\changeHK{_{nc}}\changeHSS{{\rm id}}=)} \quad \frac{1}{({\changeHK{M}}+\tau) \changeHK{\Upsilon_{nc}}} {\rm id}&  \changeHS{\preceq} & \tilde{\mathcal{H}}^k .
\end{eqnarray}
Denoting $\frac{1}{({\changeHK{M}}+\tau) \changeHK{\Upsilon_{nc}}}$ as $\gamma\changeHK{_{nc}}$, we obtain the lower bound of the claim. 

\changeHK{
Next, we present the proof for the upper bound part by referring \cite{Wang_arXiv_2016}. 
$\tilde{\mathcal{H}}^k$ is defined as 
\begin{eqnarray}
\tilde{\mathcal{H}}^{k}_u  =  (\text{id} - \rho_k y_k s_{k}^{\flat})^{\flat} \check{\mathcal{H}}^{k}_{u-1}(\text{id} - \rho_k y_k s_{k}^{\flat})  + \rho_k s_k s_k^{\flat},  
\end{eqnarray}
where 
%\begin{eqnarray}
$\check{\mathcal{H}}^{k}_{u-1} = \mathcal{T}_{\eta_k} \circ \tilde{\mathcal{H}}^{k}_{u-1} \circ \mathcal{T}^{-1}_{\eta_k}$, and $\rho_k = 1/\langle y_k, s_k \rangle$ \cite{Huang_SIOPT_2015}. 
Therefore, the coordinate representation of $\tilde{\mathcal{H}}^{k}_u$ is 
\begin{eqnarray}
\hat{\tilde{\mathcal{H}}}^{k}_u  
& = &   (\text{id} - \rho_k y_k s_{k}^{T})^{T} \hat{\check{\mathcal{H}}}^{k}_{u-1}(\text{id} - \rho_k y_k s_{k}^{T})  + \rho_k s_k s_k^{T}, \nonumber \\
& = & \hat{\check{\mathcal{H}}}^{k}_{u-1} - \rho_k (\hat{\check{\mathcal{H}}}^{k}_{u-1} y_k s_k^T + s_k y^T \hat{\check{\mathcal{H}}}^{k}_{u-1}) + \rho_k s_k s_k^T + \rho_k^2 s_k y_k^T \hat{\check{\mathcal{H}}}^{k}_{u-1} y_k s_k^T.
\end{eqnarray}
Here, noticing below from the fact $\mathcal{T}$ is isometric in Assumption \ref{Assump:1}.4, 
\begin{eqnarray}
\| \check{\mathcal{H}}^{k}_{u-1} \| & = & \| \mathcal{T}_{\eta_k} \circ \tilde{\mathcal{H}}^{k}_{u-1} \circ \mathcal{T}^{-1}_{\eta_k}\| 
\ \ = \ \ \| \tilde{\mathcal{H}}^{k}_{u-1} \|,
\end{eqnarray}
we obtain below from (\ref{Eq:s2_ys_bound}) and (\ref{Eq:y2_ys_bound}), 
\begin{eqnarray}
\| \hat{\tilde{\mathcal{H}}}^k_u \| 
& \leq & \| \hat{\tilde{\mathcal{H}}}^k_{u-1} \|  
+ \frac{2 \| \hat{\tilde{\mathcal{H}}}^k_{u-1} \| \| y_k \|  \| s_k \|}{ \langle s_k, y_k \rangle} 
+ \frac{\| s_k \|^2}{\langle s_k, y_k \rangle}
+ \frac{\| s_k \|^2}{\langle s_k, y_k \rangle} 
\frac{\| \hat{\tilde{\mathcal{H}}}^k_{u-1} \|  \| y_k \|^2 }{ \langle s_k, y_k \rangle} \nonumber \\
& = & \| \hat{\tilde{\mathcal{H}}}^k_{u-1} \|  
+ 2 \| \hat{\tilde{\mathcal{H}}}^k_{u-1}  \| 
\left[
\frac{ \| y_k \|^2}{ \langle s_k, y_k \rangle} \cdot \frac{ \| s_k \|^2}{ \langle s_k, y_k \rangle} 
\right]^{1/2}
+ \frac{\| s_k \|^2}{\langle s_k, y_k \rangle}
+ \frac{\| s_k \|^2}{\langle s_k, y_k \rangle} 
\frac{\| \hat{\tilde{\mathcal{H}}}^k_{u-1} \|  \| y_k \|^2}{ \langle s_k, y_k \rangle} \nonumber \\
& = & 
\left(
1 + \frac{b_4}{\epsilon}
\right)^2 \| \hat{\tilde{\mathcal{H}}}^k_{u-1}  \|  + \frac{1}{\epsilon}\nonumber \\
& = & 
q \| \hat{\tilde{\mathcal{H}}}^k_{u-1}  \|  + \frac{1}{\epsilon},
\end{eqnarray}
where we denote $(1 + b_4/\epsilon)^2$ as $q$ for simplicity. Because we consider the definition $\hat{\tilde{\mathcal{B}}}^k_0= {\rm id}/\chi_k$\changeHSS{,} where, as a common choice in L-BFGS in the Euclidean, $\chi_k = \frac{\langle s_k,y_k \rangle}{\langle y_k, y_k \rangle}$, we obtain $\| \hat{\tilde{\mathcal{H}}}^k_0 \|$ as
\begin{eqnarray}
\| \hat{\tilde{\mathcal{H}}}^k_0 \|  &=&  \| \chi_k \| \ \ = \ \ \frac{ \langle s_k, y_k \rangle}{\| y_k\|^2}  \ \ \leq \ \  \frac{1}{\epsilon}, \nonumber
\end{eqnarray}
where the last inequality uses (\ref{Eq:y_y_s_y_lower}) in Lemma \ref{Lemma:y_y_s_y_lower}. Then, it follows that
\begin{eqnarray}
\| \hat{\tilde{\mathcal{H}}}^k_1 \|  & \leq &  q \| \hat{\tilde{\mathcal{H}}}^k_0 \| + \frac{1}{\epsilon} \ \leq \ \frac{1}{\epsilon}(q + 1). \nonumber
\end{eqnarray}
By recurrence relation, we calculate $\| \hat{\tilde{\mathcal{H}}}^k_u \|$ from $\| \hat{\tilde{\mathcal{H}}}^k_1 \|$ as 
\begin{eqnarray}
	\| \hat{\tilde{\mathcal{H}}}^k_u \|  
	& \leq  & \left(\frac{1}{\epsilon}(q + 1) - \frac{1}{\epsilon(1-q)}\right) q^{u-1} + \frac{1}{\epsilon(1-q)} \nonumber\\
	& = & \frac{q^{u+1}-1}{\epsilon (q-1)}.\nonumber
\end{eqnarray}
Consequently, \changeHK{plugging} $u= \tau$ and considering $\hat{\tilde{\mathcal{H}}}^k_\tau = \hat{\tilde{\mathcal{H}}}^k$ and 
$\lambda_{\rm max}(\hat{\tilde{\mathcal{H}}}^k_\tau) = \| \hat{\tilde{\mathcal{H}}}^k_\tau \| $, we obtain below;
\begin{eqnarray}
	\label{Eq:H_uppebound_nc}
	\tilde{\mathcal{H}}^k & \preceq & \frac{(1 + b_4/\epsilon\changeHK{)}^{2(\tau+1)}-1}{\epsilon ((1 + b_4/\epsilon)^2-1)} {\rm id} \quad \quad\changeHK{(=\Gamma_{nc}\changeHSS{{\rm id}})}.
\end{eqnarray}
Denoting the upper bound $\frac{(1 + b_4/\epsilon\changeHK{)}^{2(\tau+1)}-1}{\epsilon ((1 + b_4/\epsilon)^2-1)}$ as $\Gamma_{nc}$, we obtain the upper bound part of the claim. 
This completes the proof.  
}
\end{proof}

Finally we present Proposition \ref{AppProposition:HessianOperatorBoundsNonConvex}.
\noindent
\begin{Prop}[Bounds of $\mathcal{H}\changeHKK{_t^k}$ \changeHK{on non-convex functions}]
\label{AppProposition:HessianOperatorBoundsNonConvex}
Consider the operator $\changeHSS{{\mathcal{H}}_t^k}\changeHK{:= \mathcal{T}_{\tilde{\eta}_t^k} \circ \tilde{\mathcal{H}}^k \circ (\mathcal{T}_{\tilde{\eta}_t^k})^{-1}}$. If Assumption \ref{Assump:1} holds, the range of eigenvalues of $\mathcal{H}^k_t$ is bounded by $\gamma\changeHK{_{nc}}$ and $\Gamma\changeHK{_{nc}}$\ for all $k \geq 1, t \geq 1$, i.e., 
\begin{eqnarray}
\label{Eq:HessianOperatorBoundsNonConvex}
 \gamma\changeHK{_{nc}} {\rm id} \ \preceq \ \mathcal{H}^k_t \ \preceq \ \Gamma\changeHK{_{nc}}{\rm id}\changeHSS{,}
\end{eqnarray}
\changeHSS{where $\gamma\changeHK{_{nc}}$ and $\Gamma\changeHK{_{nc}}$ are some positive constants.}
\end{Prop}
\begin{proof}
Considering $\mathcal{H}^k_t := \mathcal{T}_{\tilde{\eta}_t^k} \circ \tilde{\mathcal{H}}^k \circ (\mathcal{T}_{\tilde{\eta}_t^k})^{-1}$, \changeHK{where $\tilde{\eta}_t^k =R^{-1}_{\tilde{w}^{k}}(w_{t}^{k})$,} since $\mathcal{T}_{\tilde{\eta}_t^k}$ is a liner transformation operator, we can \changeHK{conclude that} the eigenvalues of $\mathcal{H}^k_t$ and $\tilde{\mathcal{H}}^k$ are identical. Actually, \changeHS{let hat expressions be representation matrices with some bases of $T_{w^k_t}\mathcal{M}$ and $T_{\tilde{w}^k}\mathcal{M}$}, we have the relation below;
\begin{eqnarray}
{\rm det}(\lambda {\rm id} - \changeHS{\hat{\mathcal{H}}}^k_t) 
& = & {\rm det}(\lambda {\rm id} - \changeHS{\hat{\mathcal{T}}}_{\tilde{\eta}_t^k} \changeHS{\hat{\tilde{\mathcal{H}}}}^k (\changeHS{\hat{\mathcal{T}}}_{\tilde{\eta}_t^k})^{-1}) \nonumber\\
& = & {\rm det}(\changeHS{\hat{\mathcal{T}}}_{S_{\tilde{\eta}_t^k}} (\lambda {\rm id} - \changeHS{\hat{\tilde{\mathcal{H}}}}^k) (\changeHS{\hat{\mathcal{T}}}_{\tilde{\eta}_t^k})^{-1})\nonumber\\
& = & {\rm det}(\changeHS{\hat{\mathcal{T}}}_{S_{\tilde{\eta}_t^k}}) {\rm det}(\lambda {\rm id} - \changeHS{\hat{\tilde{\mathcal{H}}}}^k) {\rm det}((\changeHS{\hat{\mathcal{T}}}_{\tilde{\eta}_t^k})^{-1})\nonumber\\
& = & {\rm det}(\changeHS{\hat{\mathcal{T}}}_{S_{\tilde{\eta}_t^k}}) {\rm det}(\lambda {\rm id} - \changeHS{\hat{\tilde{\mathcal{H}}}}^k) {\rm det}(\changeHS{\hat{\mathcal{T}}}_{\tilde{\eta}_t^k})^{-1}\nonumber\\
& = & {\rm det}(\lambda {\rm id} - \changeHS{\hat{\tilde{\mathcal{H}}}}^k).\nonumber
\end{eqnarray}
Therefore, Lemma \ref{Lem:BoundsOfHkNonConvex} directly yields the claim. This completes the proof. 
\end{proof}

\subsection{Proof of \changeHK{global convergence analysis} (Theorem \ref{Thm:GlobalConvAnalysisNonConvex})}
\label{Apd:GlobalConvAnalysisFinal}

This subsection shows \changeHK{the global convergence analysis of the proposed R-SQN-VR.} This analysis partially extends the expectation\changeHSS{-}based analysis of SGD in the Euclidean \changeHSS{space} \cite{Bottou_arXiv_2016} into the proposed algorithm. 

\subsubsection{Essential lemmas}
We first obtain the following lemma from (\ref{Eq:HessianUpperBoundFunc}) in \changeHK{Lemma \ref{Lemma:HessianUpperBoundFunc}}. Subsequently, $\mathbb{E}_{i_t^k}[f(w^k_{t+1})]$ is a meaningful quantity because $w^k_{t+1}$ depends on $i_t^k$ through the update in {Algorithm 1}.

\begin{Lem}
\label{lemma:Lipschitz}
Under Lemma \changeHK{\ref{Lemma:HessianUpperBoundFunc}}, the iterates of Algorithm 1 satisfy the following inequality for all $k \in \mathbb{N}$:
\begin{eqnarray}
	\label{Eq:Lipschitz_expectation}
	\mathbb{E}_{i_t^k}[f(w^k_{t+1})] - f(w^k_{t}) \leq - \alpha^k_t \langle \gradf(w^k_t), \mathbb{E}_{i_t^k}[\mathcal{H}_t^k \xi_t^k] \rangle_{w^k_t}  + \frac{1}{2} (\alpha^k_t)^2 \Lambda \mathbb{E}_{i_t^k} [ \| \mathcal{H}_t^k \xi_t^k\|^2_{w^k_t} ].\ \ \ \ 
\end{eqnarray}
\end{Lem}

\begin{proof}
When $w^k_{t+1} = R_{w^k_t}(-\alpha^k_t \mathcal{H}_t^k \xi_t^k)$, substituting $-\mathcal{H}_t^k \xi_t^k$ into $\eta_k$, the iterates generated by Algorithm 1 satisfy from (\ref{Eq:HessianUpperBoundFunc}) in Lemma \changeHSS{\ref{Lemma:HessianUpperBoundFunc}}
\begin{eqnarray}
	\label{Eq:Lipschitz_expectation_proof}
	 f(w^k_{t+1}) - f(w^k_{t}) 
	 & \leq & \langle \gradf(w^k_t), -\alpha^k_t \mathcal{H}_t^k \xi_t^k \rangle_{w^k_t} + \frac{1}{2} \Lambda \| -\alpha^k_t \mathcal{H}_t^k \xi_t^k \|^2_{w^k_t}	\nonumber \\
	\label{Eq:HessianUpperBound_Derivation}
	& = & - \alpha^k_t \langle \gradf(w^k_t), \mathcal{H}_t^k \xi_t^k\rangle_{w^k_t} + \frac{1}{2} (\alpha^k_t)^2 \Lambda \| \mathcal{H}_t^k \xi_t^k \|^2_{w^k_t}.
\end{eqnarray}
Taking expectations in the inequalities above with respect to the distribution of $i_t^k$, and noting that $w^k_{t+1}$, but not $w^k_t$, depends on $i_t^k$, we obtain the desired bound. 
\end{proof}

This lemma shows that, regardless of how Algorithm 1 arrived at $w^k_t$, the expected decrease in the objective function yielded by the $k$-th step is bounded above by a quantity involving: (i) the {\it expected directional derivative} of $f$ at $w^k_t$ along $- \mathcal{H}_t^k \xi_t^k$ and (ii) the {\it second moment} of $\mathcal{H}_t^k \xi_t^k$.

Next, we derive the following lemma;
\begin{Lem}
\label{Lemma:SeqAveFunc}
Under Assumptions \ref{Assump:1} and \ref{Assump:2}, the sequence of average function $f(w^k_t)$ satisfies
\begin{eqnarray}
	\label{Eq:SeqAveFunc}
	\mathbb{E}[f(w^k_{t+1})] 	& \leq & f(w^k_{t}) -  \alpha^k_t \gamma\changeHK{_{nc}} \| \gradf(w^k_{t}) \|_{w^k_t}^2  + \frac{9\Lambda(\alpha^k_t)^2 \changeHK{\Gamma^2_{nc}}  S^2}{2}.
\end{eqnarray}
\end{Lem}
\begin{proof}

Taking expectation (\ref{Eq:Lipschitz_expectation}) in Lemma \ref{lemma:Lipschitz} with regard to $w^k_t$ considering that $\mathcal{H}_t^k$ is deterministic when $w^k_t$ is given, we write
\begin{eqnarray}
	\label{Eq:}
	&& \mathbb{E}[f(w^k_{t+1})] \nonumber\\ 
	& \leq & f(w^k_{t}) -  \alpha^k_t\langle \gradf(w^k_{t}), \mathcal{H}_t^k \mathbb{E}_{i_t^k}[\xi_t^k] \rangle_{w^k_{t}} + \frac{(\alpha^k_t)^2 \Lambda}{2} \mathbb{E}_{i_t^k}[\| \mathcal{H}_t^k \xi_t^k \|_{w^k_t}^2]\nonumber \\
	& \leq & f(w^k_{t}) -  \alpha^k_t\langle \gradf(w^k_{t}), \mathcal{H}_t^k \gradf(w^k_t) \rangle_{w^k_{t}}  + \frac{(\alpha^k_t)^2 \Lambda}{2} \mathbb{E}_{i_t^k}[\| \mathcal{H}_t^k \xi_t^k \|_{w^k_t}^2]\nonumber \\	
	\label{Eq:Kankei1}
	& \leq & f(w^k_{t}) -  \alpha^k_t\langle \gradf(w^k_{t}), \mathcal{H}_t^k \gradf(w^k_t) \rangle_{w^k_{t}} + \frac{(\alpha^k_t)^2 \Lambda}{2} \changeHS{ \mathbb{E}_{i_t^k}[ \changeHK{\Gamma^2_{nc}} \|\xi_t^k \|_{w^k_t}^2]} 	\\
	%& \leq & f(w^k_{t}) -  \alpha^k_t\langle \gradf(w^k_{t}), \mathcal{H}_t^k \gradf(w^k_t) \rangle_{w^k_{t}} + \frac{(\alpha^k_t)^2 \Lambda S^2}{2} \changeHK{\Gamma^2_{nc}},\nonumber \\
	& \leq & f(w^k_{t}) -  \alpha^k_t \gamma\changeHK{_{nc}} \| \gradf(w^k_{t}) \|_{w^k_t}^2  + \frac{9\Lambda(\alpha^k_t)^2 \changeHK{\Gamma^2_{nc}}  S^2}{2},\nonumber 	
	\end{eqnarray}
where the second inequality is obtained from $\mathbb{E}_{i_t^k}[\xi_t^k]=\gradf(w^k_t)$ because $\xi_t^k$ is an unbiased estimate of $\gradf(w^k_t)$. The last inequality comes from Assumption \ref{Assump:2} \changeHK{since}
\begin{eqnarray}
\label{Appd_Eq:AbstvalueModStograd}
\|\xi_t^k \|_{w_{t}^k} & = & \| \gradf_{i_t^k}(w_{t}^{k}) -  \mathcal{T}_{\tilde{\eta}_{t}^k} \left(\gradf_{i_t^k}(\tilde{w}^{k}) \right) + \mathcal{T}_{\tilde{\eta}_{t}^k}\left(\gradf(\tilde{w}^{k})\right) \|_{w_{t}^k} \nonumber \\
& \le & S + S + S = 3S,
\end{eqnarray}
where $\tilde{\eta}_t^k \in T_{\tilde{w}^{k}} \mathcal{M}$ satisfies $R_{\tilde{w}^{k}}(\tilde{\eta}_t^k) = w^k_t$.
This completes the proof. 
\end{proof}

\begin{Prop}
\label{Thm:ConvergenceWithDiminishingStepsize}
\changeHK{Under Assumptions \ref{Assump:1} and \ref{Assump:2}}, suppose that Algorithm 1 is run with a step-size sequence satisfying Assumption \ref{Assump:2}. 
\changeHK{Then, we have}
%Then, with $A_K:= \sum_{k=1}^{\changeHK{K}}\sum_{t=1}^{\changeHK{m_k}} \alpha^k_t$, 
%
\begin{eqnarray}
	\label{Eq:ConvergenceWithDiminishingStepsize-1}
	\mathbb{E} \left[ \sum_{k=1}^{\changeHK{K}}\sum_{t=1}^{\changeHK{m_k}} \alpha^k_t \| \gradf(w^k_t) \|^2_{w^k_t} \right] & < & \infty\changeHK{.} 
	%	\label{Eq:ConvergenceWithDiminishingStepsize-2}
%	and\ \ therefore\ \ \ \ \mathbb{E} \left[ \frac{1}{A_{K}} \sum_{k=1}^{\changeHK{K}}\sum_{t=1}^{\changeHK{m_k}}  \alpha^k_t \| \gradf(w^k_t) \|^2_{w^k_t} \right] & \xrightarrow{K\rightarrow\infty} 0 &
\end{eqnarray}
\end{Prop}

\begin{proof}
Taking the total expectation of (\ref{Eq:SeqAveFunc}) in Lemma \ref{Lemma:SeqAveFunc} yields 
\begin{eqnarray}
	\label{Eq:ExpectationDiff}	
	\mathbb{E}[f(w^k_{t+1})] - \mathbb{E}[f(w^k_{t})] 
	& \leq & -  \alpha^k_t \gamma\changeHK{_{nc}} \mathbb{E}[\| \gradf(w^k_t) \|^2_{w^k_t}] + \frac{9\Lambda(\alpha^k_t)^2 \changeHK{\Gamma^2_{nc}}  S^2}{2}.\nonumber
\end{eqnarray}

Summing both sides of this inequality for 
$\{w^1_1, \ldots, w^1_{m_1}, \ldots, 
w^{K-1}_1, \ldots, w^{K-1}_{m_{K-1}},
w^{K}_1, \ldots, w^{K}_{m_{K}}\} $ gives
\begin{eqnarray}
	%\label{Eq:}	
	f_{\rm inf} - f(w^1_{1}) & \leq &  \mathbb{E}[f(w^k_{t+1})] -  f(w^1_{1})   \nonumber\\
	& \leq & - \gamma\changeHK{_{nc}} \sum_{k=1}^{\changeHK{K}}\sum_{t=1}^{\changeHK{m_k}}  \alpha^k_t \mathbb{E}[\| \gradf(w^k_t) \|^2_{w^k_t}]	 + \frac{9\Lambda \changeHK{\Gamma^2_{nc}}  S^2}{2} \sum_{k=1}^{\changeHK{K}}\sum_{t=1}^{\changeHK{m_k}}  (\alpha^k_t)^2.\nonumber
\end{eqnarray}

Dividing by $\gamma\changeHK{_{nc}}$ and rearranging the terms, we obtain
\begin{eqnarray}
	%\label{Eq:}	
	\sum_{k=1}^{\changeHK{K}}\sum_{t=1}^{\changeHK{m_k}}  \alpha^k_t \mathbb{E}[\| \gradf(w^k_t) \|^2_{w^k_t}]	& \leq & \frac{( f(w^1_{1}) - f_{\rm inf})}{\gamma\changeHK{_{nc}}} +  \frac{9\Lambda \changeHK{\Gamma^2_{nc}}  S^2}{2\gamma\changeHK{_{nc}}} \sum_{k=1}^{\changeHK{K}}\sum_{t=1}^{m_k}  (\alpha^k_t)^2. \nonumber
\end{eqnarray}

The second condition \changeHK{of the decaying step-size sequence in Assumption \ref{Assump:2}, i.e., $\sum (\alpha^k_t)^2 < \infty$,} implies that the right-hand side of this inequality converges to a {\it finite limit} when $K$ increases.
%, proving (\ref{Eq:ConvergenceWithDiminishingStepsize-1}). 
%Then, (\ref{Eq:ConvergenceWithDiminishingStepsize-2}) follows since the first condition \changeHK{of the decaying step sizes in} Assumption \ref{Assump:2}\changeHK{, i.e., $\sum \alpha^k_t=\infty$,} ensures that $A_K \rightarrow \infty$ as $K \rightarrow \infty$.
This completes the proof.
\end{proof}

%\changeHK{Proposition} \ref{Thm:ConvergenceWithDiminishingStepsize} states about a {\it weighted sum-of-squares} and a {\it weighted average of squared} gradients of $f$. In particular, (\ref{Eq:ConvergenceWithDiminishingStepsize-2}) concludes that the weighted average norm of the squared gradients converges to zero even if the gradient are noisy. But, the fact only specifies a property of a weighted average is only of minor importance since one can still conclude {\it the expected gradient norms cannot asymptotically stay far from zero}.

%
Then, we obtain the following \changeHK{proposition} by taking (\ref{Eq:ConvergenceWithDiminishingStepsize-1}) into account with the first condition of Assumption \ref{Assump:2}.

\begin{Prop}
\label{Thm:GlobalConvergenceAnalysisFinal}
\changeHK{Under Assumptions \ref{Assump:1} and \ref{Assump:2}}, suppose that Algorithm 1 is run with a step-size sequence satisfying Assumption \ref{Assump:2}. Then, \changeHK{we have}
\begin{eqnarray}
	\label{Eq:GlobalConvergenceAnalysisFinal}
	\liminf_{k \rightarrow \infty} \mathbb{E} [\| \gradf(w^k_t) \|^2_{w^k_t}] & = & 0.
\end{eqnarray}
\end{Prop}
\begin{proof}
\changeHSS{The proof is by} contradiction. Assume that (\ref{Eq:GlobalConvergenceAnalysisFinal}) does not hold. Then, there exists $\delta > 0$ such that $\mathbb{E}[ \| \gradf(w^k_t) \|_{w^k_t}^{\changeHS{2}}] > \delta$ for all $k$ sufficiently large, say, $k> N$. We have 
\begin{eqnarray}
\displaystyle \mathbb{E} \left[ \sum_{k=\changeHSS{1}}^{\infty}\sum_{t=1}^{\changeHK{m_k}}  \alpha^k_t \| \gradf(w^k_t)\|_{w^k_t}^2 \right] \ \ \geq\ \ \sum_{k=N}^{\infty} \sum_{t=1}^{\changeHK{m_k}}  \alpha^k_t \mathbb{E}[ \| \gradf(w^k_t)\|_{w^k_t}^2 ] \ \ >\ \  \delta \sum_{k=N}^{\infty} \sum_{t=1}^{\changeHK{m_k}}\alpha^k_t \ \ =\ \  \infty. \nonumber
\end{eqnarray}
This contradicts (\ref{Eq:ConvergenceWithDiminishingStepsize-1}).
\end{proof}

%A ``lim inf'' result of this type should be familiar to those knowledgeable of the nonlinear optimization literature. This intuition is 
\changeHK{This implies} that, for the \changeHKK{R-SQN-VR} with \changeHK{decaying step-sizes sequence}, the expected gradient norms cannot stay bounded away from zero. 

\subsubsection{Main proof of Theorem \ref{Thm:GlobalConvAnalysisNonConvex}}
\label{ApdSub:GlobalConvAnalysisFinal}

%\begin{Thm}
\noindent
{\bf Theorem. \ref{Thm:GlobalConvAnalysisNonConvex}.} {\it \changeHK{Let $\mathcal{M}$ be a Riemannian manifold and $w^* \in \mathcal{M}$ be a non-degenerate local minimizer of $f$.} Consider Algorithm 1 and suppose Assumptions \ref{Assump:1} and \ref{Assump:2}, and that the mapping $w \changeHS{\mapsto} \| \gradf(w) \|_w^2$ has the positive real number that the largest eigenvalue of its Riemannian Hessian is bounded for all $w \in \mathcal{M}$. Then, we have
\begin{eqnarray*}
	\label{Eq:ConvergenceWithDiminishingStepsize2}
	\lim_{k \rightarrow \infty} \mathbb{E} [\| \gradf(w_t^k) \|^2_{w_t^k}] & = & 0.
\end{eqnarray*}
%\end{Thm}
}

\begin{proof}
We define $h(w)$ as $h(w):=\| \gradf(w) \|^2_{w^k_t}$ and 
let $\changeHK{\Lambda_h}$ \changeHK{be} the absolute value of the eigenvalue with the largest magnitude of \changeHS{the Hessian of} $h$. 
Then, from \changeHS{Taylor's theorem}, we obtain
\begin{eqnarray*}
	h(w^k_{t+1}) - h(w^k_{t}) & \leq & - \changeHS{2} \alpha^k_t \langle \gradh(w^k_t), \changeHS{\hessf(w_t^k)[\mathcal{H}_t^k \xi_t^k]} \rangle_{w^k_t} + \changeHK{\frac{1}{2}}(\alpha^k_t)^2 \changeHK{\Lambda_h} \| \mathcal{H}_t^k \xi_t^k \|^2_{w^k_t}.
\end{eqnarray*}

Taking the expectation with respect to the distribution of $i_t^k$, we obtain below;
\begin{eqnarray*}
	& & \mathbb{E}_{i_t^k}[h(w^k_{t+1})] - h(w^k_{t}) \\
	& \leq & - \changeHS{2} \alpha^k_t \langle \changeHK{\gradf}(w^k_t), \mathbb{E}_{i_t^k}[\changeHS{\hessf(w_t^k)[\mathcal{H}_t^k \xi^k_t]}] \rangle_{w^k_t} + \changeHK{\frac{1}{2}}(\alpha^k_t)^2 \changeHK{\Lambda_h} \mathbb{E}_{i_t^k}[ \| \mathcal{H}_t^k \xi^k_t \|^2_{w^k_t} ] \nonumber \\
	& = & - 2 \alpha^k_t \langle \changeHS{\gradf(w^k_t)}, \changeHS{\hessf(w_t^k)[\mathcal{H}_t^k  \mathbb{E}_{i_t^k}[\xi^k_t]]} \rangle_{w^k_t}  +  \changeHK{\frac{1}{2}}(\alpha^k_t)^2 \changeHK{\Lambda_h} \mathbb{E}_{i_t^k}[ \| \mathcal{H}_t^k \xi^k_t \|^2_{w^k_t} ] \nonumber\\
	& \leq & 2 \alpha^k_t \| \gradf(w^k_t) \|_{w^k_t} \| \changeHS{\hessf(w_t^k)[\mathcal{H}_t^k \gradf(w_t^k)]}\|_{w^k_t}+ \changeHK{\frac{1}{2}}(\alpha^k_t)^2 \changeHK{\Lambda_h} \changeHS{\changeHK{\Gamma^2_{nc}}} \mathbb{E}_{i_t^k}[ \| \xi^k_t \|^2_{w^k_t} ]  \nonumber\\
	& \leq & \changeHS{2} \alpha^k_t \Lambda \Gamma{\changeHK{_{nc}}} \| \gradf(w^k_t) \|^2_{w^k_t} + \changeHK{\frac{\changeHS{9}}{2}} (\alpha^k_t)^2 \changeHK{\Lambda_h} S^2 \changeHK{\Gamma^2_{nc}},
	%
	%& \overset{(\ref{Assump:gradient_second_moment})}{\leq} & 2 \alpha^k_t \mu_G \Lambda \| \gradf(w^k_t) \|^2_{w^k_t} + \frac{1}{2} (\alpha^k_t)^2 \changeHK{\changeHK{\Lambda_h}} (M+M_G\| \gradf(w_w) \|^2_{w^k_t} ),\nonumber		
\end{eqnarray*}
where the \changeHS{last} inequality comes from \changeHS{$\|\hessf(w)[\mathcal{H}_t^k \gradf(w^k_t)]\|_{w^k_t} \leq \Lambda \|\mathcal{H}_t^k \gradf(w^k_t) \|_{w^k_t} \leq \Lambda \Gamma{\changeHK{_{nc}}} \| \gradf(w_t^k) \|_{w_t^k}$}.

Taking the total expectation simply yields
\begin{eqnarray}
	\label{Eq:Disff_grad_expectation}
	\mathbb{E}[h(w^k_{t+1})] - \mathbb{E}[h(w^k_{t})] & \leq &  \changeHS{2} \alpha^k_t \Lambda \Gamma{\changeHK{_{nc}}}   \mathbb{E}[\|\gradf(w^k_t) \|^2_{w^k_t}]  + \changeHK{\frac{\changeHS{9}}{2}} (\alpha^k_t)^2 \changeHK{\Lambda_h} S^2 \changeHK{\Gamma^2_{nc}}.	
\end{eqnarray}

Recall that \changeHK{Proposition} \ref{Thm:ConvergenceWithDiminishingStepsize} establishes that the first component of this bound is the term of a convergent sum. The second component of this bound is also the term of a convergent sum since $\sum_{k=1}^\infty \sum_{t=1}^{\changeHK{m_k}}(\alpha^k_t)^2$ converges. This means that again the result of \changeHK{Proposition} \ref{Thm:ConvergenceWithDiminishingStepsize} can be applied. Therefore, the right-hand side of (\ref{Eq:Disff_grad_expectation}) is the term of a convergent sum. Let us now define 
%
%\begin{eqnarray}
	$S_K^+  =  \sum_{k=1}^{\changeHK{K}}\sum_{t=1}^{\changeHK{m_k}} \max(0, \mathbb{E}[h(w^k_{t+1})] - \mathbb{E}[h(w^k_{t})]),$ and
	$S_K^-  =  \sum_{k=1}^{\changeHK{K}}\sum_{t=1}^{\changeHK{m_k}} \max(0, \mathbb{E}[h(w^k_{t})] - \mathbb{E}[h(w^k_{t+1})]).$
%\end{eqnarray}

Since the bound (\ref{Eq:Disff_grad_expectation}) is positive and forms a convergent sum, the nondecreasing sequence $S_K^+$ is upper bounded and therefore converges. Since, for any $K \in \mathbb{N}$, one has $\mathbb{E}[h(w_K)]= h(w_1)+ S_K^{+} - S_K^- \geq 0$, the nondecreasing sequence $S_K^-$ is upper bounded and therefore also converges. Therefore $\mathbb{E}[h(w_K)]$ converges. 
Consequently, this implies that this limit must be zero from \changeHK{Proposition} \ref{Thm:GlobalConvergenceAnalysisFinal}. This completes the proof.
\end{proof}

\subsection{\changeHK{Proof of \changeHK{global convergence rate analysis} (Theorem \ref{Thm:GlobalConvRateAnalysisNonConvex})}}
\changeHK{The global convergence rate analysis on  non-convex functions in the Euclidean SVRG is proposed in \cite{Reddi_ICML_2016_s}.
Its further extensions into the stochastic L-BFGS setting and the Riemannian setting are proposed in \cite{Wang_arXiv_2016} and \cite{Zhang_NIPS_2016}, respectively. The proof in this subsection mainly follows that in \cite{Reddi_ICML_2016_s} by integrating its two extensions in \cite{Wang_arXiv_2016,Zhang_NIPS_2016}.}
\changeHK{Besides that, the special and careful treatments for \changeHKK{the} retraction and \changeHKK{the} vector transport operations are particularly taken in the proof.} \changeHKK{The results for the exponential mapping and the parallel translation are given in the corresponding corollaries as a special case.}

\subsubsection{\changeHK{Preliminary lemmas}}
\changeHK{
We first present some essential lemmas. 
\begin{Lem}[Lemma 6 in \cite{Zhang_COLT_2016}]
\label{LemZhang}
If $a$, $b$, and $c$ are the side lengths of a geodesic triangle in an Alexandrov space with curvature lower-bounded by $\kappa$, and $A$ is the angle between sides $b$ and $c$, then
\begin{equation*}
a^2 \le \frac{\sqrt{|\kappa|}c}{\tanh(\sqrt{|\kappa|}c)}b^2 + c^2 - 2bc\cos(A).
\end{equation*}
\end{Lem}
}

\begin{Lem}[In the proof of Lemma 3.9 in \cite{Huang_SIOPT_2015}] \label{Lem:pseudo_Lipschitz}
Under Assumptions \ref{Assump:1}.1 and \ref{Assump:1}.2, there exists a constant $\beta>0$ such that
\begin{eqnarray}
     \label{pseudo_Lipschitz}
	\| P_{\gamma}^{w \leftarrow z} ({\rm grad} f(z)) - {\rm grad} f(w) \|_w & \leq & \beta {\rm dist}(z,w),
\end{eqnarray}
where $w$ and $z$ are in $\Theta$ in Assumption \ref{Assump:1}.2 and $\gamma$ is a curve $\gamma(t):=R_z(\tau\eta)$ for $\eta \in T_z \mathcal{M}$ defined by a retraction $R$ on $\mathcal{M}$.
$P_{\gamma}^{w \leftarrow z}(\cdot)$ is a parallel translation operator along the curve $\gamma$ from $z$ to $w$.
\end{Lem}
Note that the curve $\gamma$ in this lemma is not necessarily the geodesic. The relation \eqref{pseudo_Lipschitz} is a generalization of the Lipschitz continuity condition. \changeHK{\changeHK{In addition}, we specifically use $\beta_0$ when the curve is geodesic.} 
\begin{Lem}[Lemma 3.5 in \cite{Huang_SIOPT_2015}] 
\label{Lemma:VecParaDiff}
Let $\changeHK{\mathcal{T}} \in C^{0}$ be a vector transport associated with the same retraction $R$ as that of the parallel translation $P \in C^{\infty}$. Under Assumption \ref{Assump:1}.5, for any $\bar{w}\in\mathcal{M}$ there exists a constant $\theta>0$ and a neighborhood $\mathcal{U}$ of $\bar{w}$ such that for all $w, z \in \mathcal{U}$,
\begin{eqnarray}
	\label{Eq:VecParaDiff}
	\|\mathcal{T}_\eta \xi-P_{\eta}\xi\|_z & \le & \theta \| \xi\|_w \| \eta\|_w,
\end{eqnarray}
where $\xi, \eta \in T_w\mathcal{M}$ and $R_w(\eta)=z$.
\end{Lem}

Modifying slightly Lemma 3 in \cite{Huang_MATHPRO_2015}, we obtain the following lemma.
\begin{Lem}[Lemma 3 in \cite{Huang_MATHPRO_2015}]
\label{Lemma:retraction_dist}
Let $\mathcal{M}$ be a Riemannian manifold endowed with retraction $R$ and let $\bar{w} \in \mathcal{M}$.
Then there exist $\tau_1 > 0$, $\tau_2 > 0$ and $\delta_{\tau_1, \tau_2}$ such that for all $w$ in a sufficiently small neighborhood of $\bar{w}$ and all $\xi \in T_w \mathcal{M}$ with $\|\xi\|_w \le \delta_{\tau_1, \tau_2}$, the inequalities
\begin{eqnarray}
\label{Eq:tau1_tau2}
\tau_1 {\rm dist} (w, R_w(\xi))\  \le & \|\xi\|_w & \le \ \tau_2 {\rm dist} (w, R_w(\xi))
\end{eqnarray}
hold.
\end{Lem}

\subsubsection{\changeHK{Essential propositions}}
\changeHK{
\changeHKK{This subsection first presents an essential lemma about the bound of $\mathbb{E}_{i^k_t}[\| \xi_t^k\|_{w^k_{t}}^2]$, where the vector transport is carefully handled to give the lemma. \changeHKK{Next}, an important proposition \ref{AppProp:LocalConvergenceNonConvex} is \changeHK{presented} by extending \cite{Reddi_ICML_2016_s,Wang_arXiv_2016,Zhang_NIPS_2016}. It should be noted that we carefully treat the difference between the exponential case and the retraction case for Proposition \ref{AppProp:LocalConvergenceNonConvex}.}
\begin{Lem}
\label{Lem:BoundExiNew}
Suppose Assumptions 
\ref{Assump:1}.1, 
\ref{Assump:1}.2, 
\changeHSS{\ref{Assump:1}.3,} 
\ref{Assump:1}.4, 
\ref{Assump:1}.5, and
\ref{Assump:1}.7, 
which guarantee Lemmas \ref{Lem:pseudo_Lipschitz}, \ref{Lemma:VecParaDiff}, and \ref{Lemma:retraction_dist} for $\bar{w} = w^*$.
Let $\beta>0$ be a constant such that
\begin{eqnarray}
	\| P_{\gamma}^{w \leftarrow z} ({\rm grad} f_n(z)) - {\rm grad} f_n(w) \|_{w} \leq \beta {\rm dist}(z,w),\qquad w, z \in \Theta,\ n = 1,2,\dots, N.
\end{eqnarray}
The existence of such $\beta$ is guaranteed by Lemma \ref{Lem:pseudo_Lipschitz}. Then, the upper bound of the variance of $\mathbb{E}_{i^k_t}[\| \xi_t^k\|_{w^k_{t}}^2] $ is given by	
\begin{eqnarray}
\label{Eq:BoundExiNew}	
	\mathbb{E}_{i^k_t}[\| \xi_t^k\|_{w^k_{t}}^2] 
	& \leq & 4(\beta^2+ \tau_2^2C^2 \theta^2) ({\rm dist}(w^{k}_{t}, \tilde{w}^k))^2 + 2 \| \gradf(w_{t}^{k}) \|_{w^k_{t}}^2.		
\end{eqnarray}
\end{Lem}
\begin{proof}
The proof is partially similar to that of Lemma 5.8 in \cite{Sato_arXiv_2017}. We first consider
%$\mathbb{E}_{i^k_t}[\| \xi_t^k\|_{w^k_{t}}^2]  =  \mathbb{E}_{i^k_t}[\|\gradf_{i_t^k}(w_{t}^{k}) -  \mathcal{T}_{\tilde{\eta}_t^k} (\gradf_{i_t^k}(\tilde{w}^{k}))  -  \mathcal{T}_{\tilde{\eta}_t^k} (\gradf(\tilde{w}^{k}))\|_{w^k_{t}}^2] $as 
\begin{eqnarray}
\label{Eq:DefExitk}
	\mathbb{E}_{i^k_t}[\| \xi_t^k\|_{w^k_{t}}^2] & = & \mathbb{E}_{i^k_t}[\|\gradf_{i_t^k}(w_{t}^{k}) -  \mathcal{T}_{\tilde{\eta}_t^k} (\gradf_{i_t^k}(\tilde{w}^{k}))  +  \mathcal{T}_{\tilde{\eta}_t^k} (\gradf(\tilde{w}^{k}))\|_{w^k_{t}}^2].
	%& = & \gradf(w_{t}^{k}) -  \mathcal{T}_{\tilde{\eta}_t^k} (\gradf(\tilde{w}^{k}))-  \mathcal{T}_{\tilde{\eta}_t^k} (\gradf(\tilde{w}^{k}))\|_{w^k_{t}}^2] 
\end{eqnarray}
%%
%Here,  we address 
%
The first and second terms in (\ref{Eq:DefExitk}) is 
\begin{eqnarray}
\label{Eq:}
	\mathbb{E}_{i^k_t}[\gradf_{i_t^k}(w_{t}^{k}) -  \mathcal{T}_{\tilde{\eta}_t^k} (\gradf_{i_t^k}(\tilde{w}^{k}))] 
	  &=&  \gradf(w_{t}^{k}) -  \mathcal{T}_{\tilde{\eta}_t^k} (\gradf(\tilde{w}^{k}))\changeHSS{,} \nonumber
\end{eqnarray}
\changeHSS{which is equivalent to}
\begin{eqnarray}
	\mathcal{T}_{\tilde{\eta}_t^k} (\gradf(\tilde{w}^{k})) 
	&=& \gradf(w_{t}^{k})  - 
	\mathbb{E}_{i^k_t}[\gradf_{i_t^k}(w_{t}^{k}) -  \mathcal{T}_{\tilde{\eta}_t^k} (\gradf_{i_t^k}(\tilde{w}^{k}))]. \nonumber 
\end{eqnarray}
Plugging this into the third term of (\ref{Eq:DefExitk}) yields
\begin{eqnarray}
\label{Eq:BoundExi}	
	\mathbb{E}_{i^k_t}[\| \xi_t^k\|_{w^k_{t}}^2] 
	& = & \mathbb{E}_{i^k_t}[\|\gradf_{i_t^k}(w_{t}^{k}) -  \mathcal{T}_{\tilde{\eta}_t^k} (\gradf_{i_t^k}(\tilde{w}^{k})) \nonumber \\
	&&  - \mathbb{E}_{i^k_t}[\gradf_{i_t^k}(w_{t}^{k}) -  \mathcal{T}_{\tilde{\eta}_t^k} (\gradf_{i_t^k}(\tilde{w}^{k}))] + \gradf(w_{t}^{k}) \|_{w^k_{t}}^2 ] \nonumber \\
	& \leq & 2 \mathbb{E}_{i^k_t}[\|\gradf_{i_t^k}(w_{t}^{k}) -  \mathcal{T}_{\tilde{\eta}_t^k} (\gradf_{i_t^k}(\tilde{w}^{k})) \nonumber \\
	 && - \mathbb{E}_{i^k_t}[\gradf_{i_t^k}(w_{t}^{k}) -  \mathcal{T}_{\tilde{\eta}_t^k} (\gradf_{i_t^k}(\tilde{w}^{k}))\changeHSS{]} \|_{w^k_{t}}^2]  + 2 \| \gradf(w_{t}^{k}) \|_{w^k_{t}}^2 \nonumber \\
	& \leq & 2 \mathbb{E}_{i^k_t}[\|\gradf_{i_t^k}(w_{t}^{k}) -  \mathcal{T}_{\tilde{\eta}_t^k} (\gradf_{i_t^k}(\tilde{w}^{k}))  \|_{w^k_{t}}^2] 
	 + 2 \| \gradf(w_{t}^{k}) \|_{w^k_{t}}^2, \nonumber \\	
\end{eqnarray}
where the second inequality comes from the fact \changeHSS{that}, for arbitrary random vector $z$ on arbitrary tangent space, $\mathbb{E}[ \| z - \mathbb{E}[z]\|^2 = \mathbb{E}[ \| z \|^2 ]- \| \mathbb{E}[z]\|^2 \leq \mathbb{E}[ \| z \|^2]$.
Now, the first term in the right-hand side is upper-bounded by the distance between $\tilde{w}^k$ and $w^k_{t}$ as
\begin{eqnarray}
\label{Eq:}	
	&& \mathbb{E}_{i^k_t}[\|\gradf_{i_t^k}(w_{t}^{k}) -  \mathcal{T}_{\tilde{\eta}_t^k} (\gradf_{i_t^k}(\tilde{w}^{k}))  \|_{w^k_{t}}^2]  \nonumber \\
& = &\mathbb{E}_{i_t^k}\left[ \| \gradf_{i_t^k}(w_{t}^{k}) - P^{w_{t}^{k} \leftarrow \tilde{w}^{k}}(\gradf_{i_t^k}(\tilde{w}^{k})) \right. \nonumber \\
&& \left. + P^{w_{t}^{k} \leftarrow \tilde{w}^{k}}(\gradf_{i_t^k}(\tilde{w}^{k}) - \mathcal{T}_{\tilde{\eta}_t^k}(\gradf_{i_t^k}(\tilde{w}^{k})) \|_{\changeHS{w_{t}^k}}^2 \right]  \nonumber\\
& \leq & 2\mathbb{E}_{i_t^k}\left[ \| \gradf_{i_t^k}(w_{t}^{k}) - P^{w_{t}^{k} \leftarrow \tilde{w}^{k}}(\gradf_{i_t^k}(\tilde{w}^{k}))  \|_{\changeHS{w_{t}^k}}^2 \right] \nonumber\\
&&+ 2\mathbb{E}_{i_t^k}\left[ \| P^{w_{t}^{k} \leftarrow \tilde{w}^{k}}(\gradf_{i_t^k}(\tilde{w}^{k})) - \mathcal{T}_{\tilde{\eta}_t^k}(\gradf_{i_t^k}(\tilde{w}^{k})) \|_{w_{t}^k}^2 \right]  \nonumber\\
& \overset{\scriptsize (\ref{Eq:VecParaDiff})}{\leq} & 
2\beta^2({\rm dist}(w^{k}_{t}, \tilde{w}^{k}))^2 + 2\mathbb{E}_{i^k_t}[ \theta^2 \| \tilde{\eta}_t^k \|_{w_{t}^k}^2 \|\gradf_{i_t^k}(w_{t}^k)\|_{w_{t}^k}^2]\nonumber\\
& =& 
2\beta^2({\rm dist}(w^{k}_{t}, \tilde{w}^{k}))^2 + 2\theta^2 \| \tilde{\eta}_t^k \|_{w_{t}^k}^2 \mathbb{E}_{i^k_t}[ \|\gradf_{i_t^k}(w_{t}^k)\|_{w_{t}^k}^2]\nonumber\\
& \leq & 2\beta^2\changeHS{({\rm dist}(w^{k}_{t}, \tilde{w}^{k}))}^2+2C^2\theta^2\| \tilde{\eta}_t^k\|^2_{w_{t}^k} \nonumber \\
& \overset{\scriptsize (\ref{Eq:tau1_tau2})}{\leq} & 2(\beta^2+\tau_2^2 C^2 \theta^2) ({\rm dist}(w^{k}_{t}, \tilde{w}^k))^2,
\end{eqnarray}
where the first inequality uses $\| a +b \| \leq 2 \|a\|^2+ \changeHKK{2} \|b\|^2$ for vector $a$ and $b$, and the second inequality uses (\ref{Eq:VecParaDiff}) in Lemma \ref{Lemma:VecParaDiff}. The third inequality uses Assumption {\ref{Assump:1}.7}, and the last inequality uses (\ref{Eq:tau1_tau2}) in Lemma \ref{Lemma:retraction_dist}. 
Substituting this into (\ref{Eq:BoundExi}) yields the claimed statement. 
%
%\begin{eqnarray}
%\label{Eq:BoundExiNew}	
%	\mathbb{E}_{i^k_t}[\| \xi_t^k\|_{w^k_{t}}^2] 
%	& \leq & 4(\beta^2+ \tau_2^2C^2 \theta^2) \mathbb{E}[({\rm dist}(w^{k}_{t}, \tilde{w}^k))^2] + 2 \mathbb{E}[\| \gradf(w_{t}^{k}) \|_{w^k_{t}}^2 ].		
%\end{eqnarray}
This complete the proof. 
\end{proof}
\changeHKK{
We obtain the counterpart result of Lemma \ref{Lem:BoundExiNew} for the parallel translation and the exponential mapping. 
\begin{Cor}
\label{Cor:BoundExiNewExpParallel}
Suppose Assumptions 
\ref{Assump:1}.1, 
\ref{Assump:1}.2, 
\changeHSS{\ref{Assump:1}.3,} 
\ref{Assump:1}.4, and
\ref{Assump:1}.5
which guarantee Lemmas \ref{Lem:pseudo_Lipschitz} for $\bar{w} = w^*$. 
Consider $\mathcal{T} = P$ and $R={\rm Exp}$, i.e., the parallel translation and the exponential mapping case.
Let $\beta>0$ be a constant such that
\begin{eqnarray*}
	\| P_{\gamma}^{w \leftarrow z} ({\rm grad} f_n(z)) - {\rm grad} f_n(w) \|_{w} &\leq& \beta {\rm dist}(z,w),\qquad w, z \in \Theta,\ n = 1,2,\dots, N.
\end{eqnarray*}
The existence of such $\beta_0$ is guaranteed by Lemma \ref{Lem:pseudo_Lipschitz}. Then, the upper bound of the variance of $\mathbb{E}_{i^k_t}[\| \xi_t^k\|_{w^k_{t}}^2] $ is given by	
\begin{eqnarray*}
%\label{Eq:BoundExiNew}	
	\mathbb{E}_{i^k_t}[\| \xi_t^k\|_{w^k_{t}}^2] 
	& \leq & 2\beta_0^2 ({\rm dist}(w^{k}_{t}, \tilde{w}^k))^2 + 2 \| \gradf(w_{t}^{k}) \|_{w^k_{t}}^2.		
\end{eqnarray*}
\end{Cor}
}

\begin{Prop}
\label{AppProp:LocalConvergenceNonConvex}
Let $\mathcal{M}$ be a Riemannian manifold and $w^* \in \mathcal{M}$ be a non-degenerate local minimizer of $f$ (i.e., ${\rm grad} f(w^*)=0$ and the Hessian ${\rm Hess}f(w^*)$ of $f$ at $w^*$ is positive definite). Suppose Assumption \ref{Assump:1} holds. \changeHSS{Assume also that $w^k_t$ and $w^k_{t+1}$ are sufficiently close to each other such that $\langle R_{w^k_t}^{-1}(w^k_{t+1}), {\rm Exp}_{w^k_t}^{-1}(\tilde{w}^k)\rangle_{w^k_t} \le \langle {\rm Exp}_{w^k_t}^{-1}(w^k_{t+1}), {\rm Exp}_{w^k_t}^{-1}(\tilde{w}^k)\rangle_{w^k_t}/\phi$ for some positive constant $\changeHK{\phi}$.} Let the constants  $\theta$ \changeHSS{be} in (\ref{Ineq:2}), $\tau_2$ in (\ref{Ineq:3}), and $\beta$ and $C$ in (\ref{Ineq:4}).  
\changeHSS{Let} $\Lambda$ \changeHSS{be} the constant in Assumption \ref{Assump:1}.3, and $\gamma\changeHK{_{nc}}$ and $\Gamma{\changeHK{_{nc}}}$ in (\ref{Eq:HessianOperatorBoundsNonConvexConvex}).
For $c\changeHK{^k_t}, c\changeHK{^k_{t+1}}, \nu_t > 0$, we set 
\begin{eqnarray}
	\label{Eq:DefCt}
	c_t^k \quad \changeHK{=}\quad c\changeHK{^k_{t+1}} (1 + \changeHSS{\phi}\alpha^k_t \nu_t + 4\zeta  (\alpha^k_t)^2 \frac{\Gamma^2_{nc}}{\tau_1^2} (\beta^2+ \tau_2^2C^2 \theta^2)) + 2(\alpha^k_t)^2 \Lambda \changeHK{\Gamma^2_{nc}} (\beta^2+ \tau_2^2C^2 \theta^2).
\end{eqnarray}	
We also define
\begin{eqnarray}
	\label{Eq:DefDelta}
	\Delta_t & := & \alpha^k_t  \left( \gamma\changeHK{_{nc}} - \frac{\changeHSS{\phi}c\changeHK{^k_{t+1}}\changeHK{\Gamma^2_{nc}} }{\nu_t}  - \alpha_t^k \Lambda \changeHK{\Gamma^2_{nc}} -  2 c\changeHK{^k_{t+1}} \zeta  \alpha^k_t \frac{\Gamma^2_{nc}}{\tau_1^2}  \right).
\end{eqnarray}	
Let $\alpha_t^k, \nu_t$ and $c\changeHK{^k_{t+1}}$ be defined such that $\Delta_t > 0$.
It then follows that for any sequence $\{\tilde{w}^k_t\}$ generated by Algorithm 1  with option II and with a fixed step-size 
$\alpha_t^k:=\alpha$ and $m_k:=m$ converging to $w^*$, the expected squared norm of the Riemannian gradient, $\gradf (w_t^{k})$, satisfies the following bound as
\begin{eqnarray}
	\label{AppEq:LocalConvergenceNonConvex}
	\mathbb{E}[\| \gradf (w_t^{k}) \|^2] & \leq & \frac{V_t^k - V_{t+1}^k}{\Delta_t},
\end{eqnarray}
where $V_t^k  :=  \mathbb{E}[f(w_t^k) + c\changeHK{^k_{t}} ( {\rm dist}(\tilde{w}^k, w^k_t))^2]$ for $0 \leq k \leq K-1$.
\end{Prop}
\begin{proof}
We first obtain the following from (\ref{Eq:Kankei1}) in Lemma \ref{Lemma:SeqAveFunc} as
\begin{eqnarray}
	\mathbb{E}[f(w^k_{t+1})] 
	& \leq & f(w^k_{t}) -   \alpha^k_t\langle \gradf(w^k_{t}), \mathcal{H}_t^k \gradf(w^k_t) \rangle_{w^k_{t}}  + \frac{(\alpha^k_t)^2 \Lambda}{2} \changeHS{ \mathbb{E}_{i_t^k}[ \changeHK{\Gamma^2_{nc}} \|\xi_t^k \|_{w^k_t}^2]} \nonumber \\
	\label{Eq:Expfwt1k}	
	& \leq & f(w^k_{t}) -  \alpha^k_t  \gamma\changeHK{_{nc}} \| \gradf(w^k_{t}) \|^2_{w^k_{t}} + \frac{(\alpha^k_t)^2 \Lambda \changeHK{\Gamma^2_{nc}}}{2} \mathbb{E}_{i_t^k}[  \|\xi_t^k \|_{w^k_t}^2].
\end{eqnarray}
\newline
Next, we bound the expected squared distance between $\tilde{w}^k$ and $w^k_{t+1}$, i.e., $\mathbb{E}[({\rm dist}(\tilde{w}^{k},w^k_{t+1}))^2]$, from Lemma \ref{LemZhang} as
\begin{eqnarray}
\label{Eq:distwtildewt1k}
	&& \mathbb{E}[({\rm dist}(\tilde{w}^{k},w^k_{t+1}))^2] \nonumber \\
	& \leq &  \mathbb{E}[\zeta({\rm dist}(w^k_{t},w^k_{t+1}))^2 + ({\rm dist}(w^k_{t}, \tilde{w}^{k}))^2 
	- 2 \langle  {\rm Exp}_{w^k_{t}}^{-1}(w^k_{t+1}), {\rm Exp}_{w^k_{t}}^{-1}(\tilde{w}^{k}) \rangle_{w^k_{t}}]\nonumber \\
	& \leq &  \mathbb{E}[\zeta({\rm dist}(w^k_{t},w^k_{t+1}))^2 + ({\rm dist}(w^k_{t}, \tilde{w}^{k}))^2 
	- 2\changeHSS{\phi} \langle  \changeHSS{R}_{w^k_{t}}^{-1}(w^k_{t+1}), {\rm Exp}_{w^k_{t}}^{-1}(\tilde{w}^{k}) \rangle_{w^k_{t}}]\nonumber \\
	& \overset{\scriptsize (\ref{Eq:tau1_tau2}) }{\leq} &  \mathbb{E} \left[\zeta \frac{\| - \alpha^k_t \mathcal{H}_t^k \xi_t^k\|_{w^k_{t}}^2}{\tau_1^2} \right] + ({\rm dist}(w^k_{t}, \tilde{w}^{k}))^2 
	- 2 \changeHSS{\phi}\langle  - \alpha^k_t \mathcal{H}_t^k \xi_t^k, {\rm Exp}_{w^k_{t}}^{-1}(\tilde{w}^{k}) \rangle_{w^k_{t}}]\nonumber \\	
	& \leq &  \zeta (\alpha^k_t)^2  \frac{\Gamma^2_{nc}}{\tau_1^2}   \mathbb{E}[\mathbb{E}_{i^k_t}[[\| \xi_t^k\|_{w^k_{t}}^2]] +\mathbb{E}[ ({\rm dist}(w^k_{t}, \tilde{w}^{k}))^2 \nonumber \\	
	&&  + \changeHSS{\phi}\alpha^k_t \mathbb{E}\left[  \frac{1}{\nu_t} \changeHK{\Gamma^2_{nc}} \| \gradf(w^k_{t}) \|_{w^k_{t}}^2 + \nu_t \|  {\rm Exp}_{w^k_{t}}^{-1}(\tilde{w}^{k}) \|\changeHK{^2_{w^k_{t}}} \right] \nonumber \\
	& = &  \zeta  (\alpha^k_t)^2 \frac{\Gamma^2_{nc}}{\tau_1^2}  \mathbb{E}[\mathbb{E}_{i^k_t}[[\| \xi_t^k\|_{w^k_{t}}^2]] +\mathbb{E}[ ({\rm dist}(w^k_{t}, \tilde{w}^{k}))^2] \nonumber \\	
	&&  + \changeHSS{\phi}\alpha^k_t \mathbb{E}\left[  \frac{1}{\nu_t} \changeHK{\Gamma^2_{nc}} \| \gradf(w^k_{t}) \|_{w^k_{t}}^2\right] + \changeHSS{\phi}\alpha^k_t \mathbb{E}\left[  \nu_t \|  {\rm Exp}_{w^k_{t}}^{-1}(\tilde{w}^{k}) \|\changeHK{^2_{w^k_{t}}}\right]  \nonumber \\	
	& = &  \zeta  (\alpha^k_t)^2 \frac{\Gamma^2_{nc}}{\tau_1^2} \mathbb{E}[\mathbb{E}_{i^k_t}[[\| \xi_t^k\|_{w^k_{t}}^2]]  + \frac{\changeHSS{\phi}\alpha^k_t\changeHK{\Gamma^2_{nc}} }{\nu_t}  \mathbb{E}[  \| \gradf(w^k_{t}) \|_{w^k_{t}}^2] + 
	(1 + \changeHSS{\phi}\alpha^k_t \nu_t)\mathbb{E} [ ({\rm dist}(w^k_{t}, \tilde{w}^{k}))^2], \nonumber \\		
\end{eqnarray}
where the second inequality uses (\ref{Eq:tau1_tau2}) in Lemma \ref{Lemma:retraction_dist} and $\mathbb{E}_{i_t^k}[-\mathcal{H}^k_t\xi_t^k]=-\mathcal{H}_t^k \gradf(w^k_{t})$. The third inequality uses the relation \note{ $2\langle a, b\rangle \leq \frac{1}{\nu_t} \| a\|^2 + \nu_t \| b \|^2$}. 
\newline
Now, we introduce the following function defined as 
\begin{eqnarray}
\label{Eq:LyapunovFunction}	
		V_t^k & := & \mathbb{E}[f(w_t^k) + c^{\changeHSS{k}}_{t} ( {\rm dist}(\tilde{w}^k, w^k_t))^2].
\end{eqnarray}
This function measures how far the current parameter $w^k_t$ is from $\tilde{w}^k$ and the objective function value. 
Then, $V_{t+1}^k$ is calculated from Lemma \ref{Lem:BoundExiNew} as 
\begin{eqnarray}
\label{Eq:LyapunovFunction1}	
	&&V_{t+1}^k   \nonumber \\
	& = & \mathbb{E}[f(w_{t+1}^k) + c^{\changeHSS{k}}_{t+1} ( {\rm dist}(\tilde{w}^k, w^k_{t+1}))^2] \nonumber \\
	& \overset{\scriptsize{(\ref{Eq:Expfwt1k}), (\ref{Eq:distwtildewt1k})}}{\leq} & \mathbb{E}[ f(w^k_{t}) -  \alpha^k_t  \gamma\changeHK{_{nc}} \| \gradf(w^k_{t}) \|^2_{w^k_{t}} + \frac{(\alpha^k_t)^2 \Lambda \changeHK{\Gamma^2_{nc}}}{2} \mathbb{E}_{i_t^k}[  \|\xi_t^k \|_{w^k_t}^2]  \nonumber \\
	&& + c^{\changeHSS{k}}_{t+1} \left[ 
	 \zeta  (\alpha^k_t)^2 \frac{\Gamma^2_{nc}}{\tau_1^2} \mathbb{E}[\mathbb{E}_{i^k_t}[[\| \xi_t^k\|_{w^k_{t}}^2]]  + \frac{\changeHSS{\phi}\alpha^k_t\changeHK{\Gamma^2_{nc}} }{\nu_t}  \mathbb{E}[  \| \gradf(w^k_{t}) \|_{w^k_{t}}^2] \right. \nonumber \\
	&& \hspace*{1cm}\left. +(1 + \changeHSS{\phi}\alpha^k_t \nu_t)\mathbb{E} [ ({\rm dist}(w^k_{t}, \tilde{w}^{k}))^2]
	 \right] \nonumber \\
	& = & \mathbb{E} \left[ f(w^k_{t}) -  \alpha^k_t  \left( \gamma\changeHK{_{nc}} - \frac{\changeHSS{\phi}c^{\changeHSS{k}}_{t+1}\changeHK{\Gamma^2_{nc}} }{\nu_t} \right)\| \gradf(w^k_{t}) \|^2_{w^k_{t}} \right] + c^{\changeHSS{k}}_{t+1} (1 + \changeHSS{\phi}\alpha^k_t \nu_t)\mathbb{E} [ ({\rm dist}(w^k_{t}, \tilde{w}^{k}))^2] \nonumber \\
	&& + \left(\frac{(\alpha^k_t)^2 \Lambda \changeHK{\Gamma^2_{nc}}}{2} + c^{\changeHSS{k}}_{t+1} \zeta  (\alpha^k_t)^2 \frac{\Gamma^2_{nc}}{\tau_1^2} \right)\mathbb{E}[\mathbb{E}_{i^k_t}[[  \|\xi_t^k \|_{w^k_t}^2]] \nonumber \\	 
	& \overset{\scriptsize (\ref{Eq:BoundExiNew})}{\leq} & \mathbb{E} \left[ f(w^k_{t}) -  \alpha^k_t  \left( \gamma\changeHK{_{nc}} - \frac{\changeHSS{\phi}c^{\changeHSS{k}}_{t+1}\changeHK{\Gamma^2_{nc}} }{\nu_t} \right)\| \gradf(w^k_{t}) \|^2_{w^k_{t}} \right] + c^{\changeHSS{k}}_{t+1} (1 + \changeHSS{\phi}\alpha^k_t \nu_t)\mathbb{E} [ ({\rm dist}(w^k_{t}, \tilde{w}^{k}))^2] \nonumber \\
	&& + \left(\frac{(\alpha^k_t)^2 \Lambda \changeHK{\Gamma^2_{nc}}}{2} + c^{\changeHSS{k}}_{t+1} \zeta  (\alpha^k_t)^2 \frac{\Gamma^2_{nc}}{\tau_1^2}  \right)\nonumber \\
	&& \hspace*{1cm} \changeHSS{\times}\left( 4(\beta^2+ \tau_2^2C^2 \theta^2) \mathbb{E}[({\rm dist}(w^{k}_{t}, \tilde{w}^k))^2] + 2 \mathbb{E}[\| \gradf(w_{t}^{k}) \|_{w^k_{t}}^2 ]. \right)\nonumber \\	 	
%	&= & \mathbb{E} \left[ f(w^k_{t}) -  \alpha^k_t  \left( \Gamma{\changeHK{_{nc}}} - \frac{c_{t+1}\changeHK{\Gamma^2_{nc}} }{\nu_t}  - \alpha_t^k \Lambda \changeHK{\Gamma^2_{nc}} -  2 c_{t+1} \zeta  \alpha^k_t \changeHK{\Gamma^2_{nc}} \right)\| \gradf(w^k_{t}) \|^2_{w^k_{t}} \right] \nonumber \\
%	&& + (c_{t+1} (1 + \alpha^k_t \nu_t) + \left(\frac{(\alpha^k_t)^2 \Lambda \changeHK{\Gamma^2_{nc}}}{2} + c_{t+1} \zeta  (\alpha^k_t)^2 \changeHK{\Gamma^2_{nc}} \right) 4(\beta^2+ \tau_2^2C^2 \theta^2))) 
%	\mathbb{E}_{i^k_t} [ ({\rm dist}(w^k_{t}, \tilde{w}^{k}))^2] \nonumber \\
	&= & \mathbb{E} \left[ f(w^k_{t}) -  \alpha^k_t  \left(\gamma\changeHK{_{nc}}- \frac{\changeHSS{\phi}c^{\changeHSS{k}}_{t+1}\changeHK{\Gamma^2_{nc}} }{\nu_t}  - \alpha_t^k \Lambda \changeHK{\Gamma^2_{nc}} -  2 c^{\changeHSS{k}}_{t+1} \zeta  \alpha^k_t \frac{\Gamma^2_{nc}}{\tau_1^2}  \right)\| \gradf(w^k_{t}) \|^2_{w^k_{t}} \right] \nonumber \\
	&& + (c^{\changeHSS{k}}_{t+1} (1 + \changeHSS{\phi}\alpha^k_t \nu_t + 4\zeta  (\alpha^k_t)^2 \frac{\Gamma^2_{nc}}{\tau_1^2}  (\beta^2+ \tau_2^2C^2 \theta^2)) + 2(\alpha^k_t)^2 \Lambda \changeHK{\Gamma^2_{nc}} (\beta^2+ \tau_2^2C^2 \theta^2)) \nonumber \\
	&&  \hspace*{1cm} 	\changeHSS{\times}\mathbb{E} [ ({\rm dist}(w^k_{t}, \tilde{w}^{k}))^2] \nonumber \\	
	&= & V\changeHK{_t^k} - \alpha^k_t  \left( \gamma\changeHK{_{nc}} - \frac{\changeHSS{\phi}c^{\changeHSS{k}}_{t+1}\changeHK{\Gamma^2_{nc}} }{\nu}  - \alpha_t^k \Lambda \changeHK{\Gamma^2_{nc}} -  2 c^{\changeHSS{k}}_{t+1} \zeta  \alpha^k_t \frac{\Gamma^2_{nc}}{\tau_1^2}  \right)  \mathbb{E}[ \| \gradf(w^k_{t}) \|^2_{w^k_{t}} ] \nonumber \\
	&= & V\changeHK{_t^k} - \Delta_t  \mathbb{E}[ \| \gradf(w^k_{t}) \|^2_{w^k_{t}}]. 
\end{eqnarray}
Rearranging the above yields the claim. This completes the proof. 
\end{proof}
}

\changeHK{
We obtain the counterpart result of Proposition \ref{AppProp:LocalConvergenceNonConvex} for the parallel translation and the exponential mapping. 
\begin{Cor}
\label{AppCor:GlobalConvergenceNonConvexRetractionVector}
Let $\mathcal{M}$ be a Riemannian manifold and $w^* \in \mathcal{M}$ be a non-degenerate local minimizer of $f$. 
Consider Algorithm 1 with option II  and with $\mathcal{T} = P$ and $R={\rm Exp}$, i.e., the parallel translation and the exponential mapping case.
Suppose Assumption \ref{Assump:1} holds. Let the constants  $\theta$ \changeHSS{be} in (\ref{Ineq:2}), $\tau_2$ in (\ref{Ineq:3}), and $\beta_0$, and $C$ in (\ref{Ineq:4}).  
$\Lambda$ is the constant in Assumption \ref{Assump:1}.3, and $\gamma\changeHK{_{nc}}$ and $\Gamma{\changeHK{_{nc}}}$ are the constants $\gamma$ and $\Gamma$ in (\ref{Eq:HessianOperatorBoundsNonConvexConvex}).
For $c\changeHK{^k_t}, c\changeHK{^k_{t+1}}, \nu_t > 0$, we set 
\begin{eqnarray}
	%\label{Eq:DefCt}
	c_t^k &\changeHK{=}& c\changeHK{^k_{t+1}} (1 + \nu_t \alpha_t^k+ 2\zeta  (\alpha^k_t)^2 \changeHK{\Gamma^2_{nc}} \beta_0^2) + 2(\alpha^k_t)^2 \Lambda \changeHK{\Gamma^2_{nc}} \beta_0^2.
\end{eqnarray}	
We also define
\begin{eqnarray}
	%\label{Eq:DefDelta}
	\Delta_t & := & \alpha^k_t  \left( \gamma\changeHK{_{nc}} - \frac{c\changeHK{^k_{t+1}}\changeHK{\Gamma^2_{nc}} }{\nu_t}  - \alpha_t^k \Lambda \changeHK{\Gamma^2_{nc}} -  2 c\changeHKK{^k_{t+1}} \zeta  \alpha^k_t \changeHK{\Gamma^2_{nc}} \right).
\end{eqnarray}	
Let $\alpha_t^k, \nu_t$ and $c\changeHK{^k_{t+1}}$ be defined such that $\Delta_t > 0$.
It then follows that, for any sequence $\{\tilde{w}^k_t\}$ generated by Algorithm 1 with a fixed step-size 
$\alpha_t^k:=\alpha$ and $m_k:=m$ converging to $w^*$, the expected squared Riemannian gradient satisfies the following bound as
\begin{eqnarray}
	\label{AppEq:LocalConvergenceNonConvex}
	\mathbb{E}[\| \gradf (w_t^{k}) \|^2] & \leq & \frac{V_t^k - V_{t+1}^k}{\Delta_t},
\end{eqnarray}
where $V_t^k  :=  \mathbb{E}[f(w_t^k) + c\changeHK{^k_{t}} ( {\rm dist}(\tilde{w}^k, w^k_t))^2]$ for $0 \leq k \leq K-1$.
\end{Cor}
}

\changeHK{
The following proposition is very similar to Theorem 2 in \cite{Reddi_ICML_2016_s}.  
\begin{Prop}[Theorem 2 in \cite{Reddi_ICML_2016_s}]
\label{AppProp:LocalConvergenceNonConvex2}
Let $\mathcal{M}$ be a Riemannian manifold and $w^* \in \mathcal{M}$ be a non-degenerate local minimizer of $f$. Consider Algorithm 1 with option II and IV, and suppose Assumption \ref{Assump:1} holds. 
Let the constants $\theta$ \changeHSS{be} in (\ref{Ineq:2}), $\tau_1$ and $\tau_2$ in (\ref{Ineq:3}), and $\beta$, and $C$ in (\ref{Ineq:4}).  
$\Lambda$ is the constant in Assumption \ref{Assump:1}.3, and $\gamma\changeHK{_{nc}}$ and $\Gamma{\changeHK{_{nc}}}$ are the constants $\gamma$ and $\Gamma$ in (\ref{Eq:HessianOperatorBoundsNonConvexConvex}).
Let $c_m=0$, $\alpha_t^k=\alpha > 0, \nu_t=\nu > 0$, and $c\changeHK{^k_t}$ is defined as (\ref{Eq:DefCt}) such that $\Delta_t$ defined in (\ref{Eq:DefDelta}) satisfies $\Delta_t > 0$ for $0 \leq t \leq m-1$. Define $\delta_t := {\rm min}_t \Delta_t$. 
Let $T$ be $mK$. It then follows that, for the output $w_{\rm sol}$ of Algorithm 1, we have 
\begin{eqnarray}
	\label{Eq:LocalConvergenceNonConvex2}
	\mathbb{E}[\| \gradf (w_{\rm sol}) \|^2] & \leq & \frac{f(w^0) - f(w^*)}{T \delta_t}.
\end{eqnarray}
\end{Prop}
\begin{proof}
Because the proof is identical to those in \cite{Reddi_ICML_2016_s, Wang_arXiv_2016, Zhang_NIPS_2016}, we omit it. The complete proof is therein.
The sketch of the proof is as follows; we first telescoping the sum of (\ref{AppEq:LocalConvergenceNonConvex}) from $t=0$ to $t=m-1$ by introducing $\delta_t$, then estimate its upper bound from the difference between $V_0^s$ and $V_0^m$ defined in (\ref{Eq:LyapunovFunction}). After showing that this difference is equivalent to the expected difference between $f(\tilde{w}^k)$ and $f(\tilde{w}^{k+1})$, summing up from $k=0$ to $k=K-1$, we obtain the desired claim. 
\end{proof}
}

\subsubsection{\changeHK{Main proof of Theorem \ref{Thm:GlobalConvRateAnalysisNonConvex}}}
\changeHK{
We finally present the proof of Theorem \ref{Thm:GlobalConvRateAnalysisNonConvex} based on the extensions of results in \cite{Reddi_ICML_2016_s,Wang_arXiv_2016,Zhang_NIPS_2016}.
}

\noindent
\changeHK{
{\bf Theorem. \ref{Thm:GlobalConvRateAnalysisNonConvex}.} 
{\it 
Let $\mathcal{M}$ be a Riemannian manifold and $w^* \in \mathcal{M}$ be a non-degenerate local minimizer of $f$. 
Consider Algorithm 1 with option II and IV, and suppose Assumption \ref{Assump:1} holds.
Let the constants  $\theta$ in (\ref{Ineq:2}), $\tau_1$ and $\tau_2$ in (\ref{Ineq:3}), and $\beta$, and $C$ in (\ref{Ineq:4}).  
$\Lambda$ is the constant in Assumption \ref{Assump:1}.3, and $\gamma\changeHK{_{nc}}$ and $\Gamma{\changeHK{_{nc}}}$ are the constants $\gamma$ and $\Gamma$ in (\ref{Eq:HessianOperatorBoundsNonConvexConvex}).
Set $\nu = \frac{\sqrt{\beta^2+\tau^2_2 C^2 \theta^2}\Gamma_{nc}}{N^{a_1/2}\tau_1} \zeta^{1-a_2}$ and 
$\alpha_t^k=\alpha = \frac{\mu_0 \tau_1}{\sqrt{\beta^2+\tau^2_2 C^2 \theta^2} N^{a_1} \Gamma_{nc} \zeta^{a_2}}$, where $0< a_1<1$, and $0 < a_2 <  \note{2}$. 
Given sufficiently small $\mu_0 \in (0,1)$, suppose that $\varrho > 0$ is chosen such that 
\begin{eqnarray}
	\label{Eq:condition_varrho}
	\frac{\sqrt{\beta^2+\tau^2_2 C^2 \theta^2}}{\Lambda \Gamma_{nc}} \gamma\changeHK{_{nc}} \left(1 - \frac{\varrho \Gamma_{nc}}{\mu_0 \gamma\changeHK{_{nc}}} \right) 
	& > & \frac{2\changeHK{\phi}\mu_0 (e-1) \changeHK{\tau_1}}{\zeta^{2-a_2}} + \frac{\mu_0\changeHK{\tau_1}}{N^{a_1}\zeta^{a_2}} + \frac{4\mu_0^2(e-1)}{N^{\frac{3a_1}{2}}\zeta^{a_2} \tau_1}
\end{eqnarray}
holds. 
Set $m = \lfloor  \frac{N^{3a_1/2}}{5 \changeHK{\phi_1} \mu_0 \zeta^{1-2a_2}} \rfloor$ and $T=m K$. Then, we have
\begin{eqnarray}
	\label{AppEq:GlobalConvRateAnalysisNonConvex}
	\mathbb{E}[\| \gradf (w_a) \|^2] & \leq & \frac{\sqrt{\beta^2+\tau^2_2 C^2 \theta^2} N^{a_1} \zeta^{a_2} [f(w^0) - f(w^*)]}{T \varrho}.
\end{eqnarray}
}
\begin{proof}
From (\ref{Eq:DefDelta}) in Proposition \ref{AppProp:LocalConvergenceNonConvex}, we need to consider the upper bound of $c_t^k$ defined in (\ref{Eq:DefCt}). 
To this end, the upper bound of $c\changeHK{^k_0}$ is first derived. Denoting, for simplicity, 
$\varphi = \changeHK{\phi}\alpha \nu + 4\zeta  \alpha^2 \frac{\Gamma_{nc}^2}{\tau_1^2} (\beta^2+ \tau_2^2C^2 \theta^2)$ and $\omega=\tau_2C \theta$, we first consider the bound of $\varphi$ as 
\begin{eqnarray}
	\varphi  
	& = & \changeHK{\phi}\alpha \nu + 4\zeta  \alpha^2 \frac{\Gamma_{nc}^2}{\tau_1^2}  (\beta^2+ \omega^2)  \nonumber \\
	& = & \frac{\changeHK{\phi}\mu_0 \zeta^{1-2a_2}}{N^{\frac{3a_1}{2}}}  + \frac{4 \mu_0^2 \zeta^{1-2a_2} }{N^{2a_1}} \nonumber \\
	& = & \frac{\changeHK{\phi}\mu_0 \zeta^{1-2a_2}}{N^{\frac{3a_1}{2}}} \left(1+ \frac{4 \mu_0 }{\changeHK{\phi}N^{\frac{1}{2a_1}}}\right). \nonumber
\end{eqnarray}
Consequently, we obtain the bound of $\varphi$ as 
%\begin{eqnarray}
	$\varphi  \in \left(\frac{\changeHK{\phi}\mu_0 \zeta^{1-2a_2}}{N^{3a_1/2}} , 5 \frac{\changeHK{\phi}\mu_0 \zeta^{1-2a_2}}{N^{3a_1/2}}  \right)$.
%\end{eqnarray}
\newline
Then, we consider the recurrence relation 
$c\changeHK{^k_t}=c\changeHK{^2_{t+1}}(1+\varphi) + 2\alpha^2 \Lambda \frac{\Gamma_{nc}^2}{\tau_1^2} (\beta^2+\omega^2)$ as
\begin{eqnarray}
	\label{Eq:C0_upperbound}
	c\changeHK{^k_0} & = & \displaystyle{ \left(c\changeHK{^k_m}-\frac{2\alpha^2 \Lambda \Gamma_{nc}^2 (\beta^2+\omega^2)}{-\varphi \tau_1^2} \right) (1+\varphi)^m + \frac{2\alpha^2 \Lambda \Gamma_{nc}^2 (\beta^2+\omega^2)}{-\varphi \tau_1^2}} \nonumber \\
	& = & 2\alpha^2 \Lambda \frac{\Gamma_{nc}^2}{\tau_1^2} (\beta^2+\omega^2) \frac{(1+\varphi)^m - 1}{\varphi} \nonumber \\
	& = & 2\frac{\mu_0^2 \Lambda}{N^{2a_1} \zeta^{2a_2}}\frac{(1+\varphi)^m - 1}{\varphi} \nonumber \\
	& \leq & 2\frac{\mu_0 \Lambda}{N^{\frac{a_1}{2}}\zeta}  ((1+\varphi)^m - 1) \nonumber \\
	& \leq & 2\frac{\mu_0 \Lambda}{N^{\frac{a_1}{2}}\zeta}  (e - 1),
\end{eqnarray}
where the second equality uses $c\changeHK{^k_m} = 0$, and the first inequality uses the lower bound of $\varphi$ derived above. Regarding the last inequality, because $m = \lfloor  \frac{N^{3a_1/2}}{5 \changeHK{\phi}\mu_0 \zeta^{1-2a_2}} \rfloor$, $\varphi \leq 1/m$. In addition, noting that $\lim_{r \rightarrow \infty} (1+1/r)^r = e$ for $r>0$ where $e$ is the Euler's number, we used the relation $(1+\varphi)^m < e$. 
\newline
Now, we attempt to estimate the lower bound of $\delta_t$, i.e., ${\rm min}_t \Delta_t$. 
\begin{eqnarray}
	\label{Eq:MInDelta}
	\delta_t & = & \min_t \Delta_t \nonumber \\
			& = & \min_t  \alpha  \left( \gamma\changeHK{_{nc}} - \frac{\changeHK{\phi}c\changeHK{^k_{t+1}}\Gamma^2_{nc} }{\nu}  - \alpha \Lambda \Gamma^2_{nc} -  2 c\changeHK{^k_{t+1}} \zeta  \alpha \frac{\Gamma^2_{nc}}{\tau_1^2} \right)\nonumber \\
			& \geq & \alpha  \left( \gamma\changeHK{_{nc}} - \frac{\changeHK{\phi}c\changeHK{^k_{0}}\Gamma^2_{nc} }{\nu}  - \alpha \Lambda \Gamma^2_{nc} -  2 \changeHK{c^k_{0}} \zeta  \alpha \frac{\Gamma^2_{nc}}{\tau_1^2}\right)\nonumber \\
			& \overset{\scriptsize  (\ref{Eq:C0_upperbound})}{\geq} & \alpha   \left( \gamma\changeHK{_{nc}} - \frac{2\changeHK{\phi}\mu_0 (e-1)}{\zeta^{2-a_2}} \frac{\Lambda}{\sqrt{\beta^2+\omega^2}}\Gamma_{nc} \changeHK{\tau_1}
			- \frac{\mu_0}{N^{a_1}\zeta^{a_2}} \frac{\Lambda}{\sqrt{\beta^2+\omega^2}}\Gamma_{nc} \changeHK{\tau_1} \right. \nonumber \\
			&&\hspace*{0.5cm} \left.	-  \frac{4\mu_0^2(e-1)}{N^{\frac{3a_1}{2}}\zeta^{a_2}} \frac{\Lambda}{\sqrt{\beta^2+\omega^2}}\frac{\Gamma_{nc}}{\tau_1}
			 \right)\nonumber \\
			& \overset{\scriptsize  (\ref{Eq:condition_varrho})}{\geq} & \frac{\varrho}{\sqrt{\beta^2+\omega^2} N^{a_1} \zeta^{a_2}},
\end{eqnarray}
where the second inequality uses (\ref{Eq:C0_upperbound}) for the second and the fourth terms. Substituting (\ref{Eq:MInDelta}) into (\ref{Eq:LocalConvergenceNonConvex2}) in Proposition \ref{AppProp:LocalConvergenceNonConvex2} completes the proof. 
\end{proof}
}

\changeHK{
\begin{Cor} Suppose the same assumptions and conditions as those of Theorem \ref{Thm:GlobalConvRateAnalysisNonConvex}. Then, the total number of gradient evaluations in Algorithm 1 is  $\mathcal{O}(N^{a_1}/\epsilon)$ to obtain an $\epsilon$-solution. 
\end{Cor}
\begin{proof}
The total number of gradient evaluations is equals to $(N+m)K$. Comparing the left term in (\ref{AppEq:GlobalConvRateAnalysisNonConvex}) with $\epsilon$, we obtain $K \approx \mathcal{O}(N^{a_1}/(m\epsilon))$. Additionally, $m = \lfloor  \frac{N^{3a_1/2}}{5 \mu_0 \zeta^{1-2a_2}} \rfloor \approx \mathcal{O}(N^{3a_1/2})$. Consequently, it results in that the total number of gradient evaluations is $\mathcal{O}(N^{a_1}/\epsilon)$.
\end{proof}
The obtained complexity is the same as that of R-SVRG \cite{Zhang_NIPS_2016} in terms of the total number of samples, $N$. 
}

\changeHK{
We obtain the corresponding result of Theorem \ref{Thm:GlobalConvRateAnalysisNonConvex} when the parallel translation and the exponential mapping are used. 
\begin{Cor}
\label{AppCo:GlobalConvergenceNonConvexRetractionVector}
Let $\mathcal{M}$ be a Riemannian manifold and $w^* \in \mathcal{M}$ be a non-degenerate local minimizer of $f$. 
Consider Algorithm 1 with option II and IV and with $\mathcal{T} = P$ and $R={\rm Exp}$, i.e., the parallel translation and the exponential mapping case.
Let the constant $\beta_0$ be $\beta$ in (\ref{Ineq:4}).  
$\Lambda$ is the constant in Assumption \ref{Assump:1}.3, and $\gamma_{nc}$ and $\Gamma{_{nc}}$ are the constants $\gamma$ and $\Gamma$ in (\ref{Eq:HessianOperatorBoundsNonConvexConvex}).
Set $\nu = \frac{\beta_0\Gamma_{nc} \zeta^{1-2a_2}}{N^{a_1/2}}$ and 
$\alpha_t^k=\alpha = \frac{\mu_0}{\beta_0 N^{a_1} \Gamma_{nc} \zeta^{a_2}}$, where $0< a_1<1$, and $0 < a_2 <  \note{2}$. 
Given sufficiently small $\mu_0 \in (0,1)$, suppose that $\varrho > 0$ is chosen such that 
\begin{eqnarray*}
	\frac{\beta_0}{\Lambda \Gamma_{nc}} \gamma\changeHK{_{nc}} \left(1 - \frac{\varrho \Gamma_{nc}}{\mu_0 \gamma_{nc}} \right) 
	& > &  \frac{\mu_0 (e-1)}{\zeta^{2-a_2}} + \frac{\mu_0}{N^{a_1}\zeta^{a_2}} + \frac{\changeHKK{2}\mu_0^2(e-1)}{N^{3a_1/2}\zeta^{a_2}}
\end{eqnarray*}
holds. 
Set $m = \lfloor  \frac{N^{3a_1/2}}{3 \mu_0 \zeta^{1-2a_2}} \rfloor$, and $T=m K$. Then, we have
\begin{eqnarray*}
	\mathbb{E}[\| \gradf (w_a) \|^2] & \leq & \frac{\beta_0 N^{a_1} \zeta^{a_2} [f(w^0) - f(w^*)]}{T \varrho}.
\end{eqnarray*}
\end{Cor}
\changeHKK{
\begin{proof}
The proof is similar to that of Theorem \ref{Thm:GlobalConvRateAnalysisNonConvex}. Therefore, we omit it. However, it should be noted that we follow Corollary \ref{Cor:BoundExiNewExpParallel} to bound $\mathbb{E}_{i^k_t}[\| \xi_t^k\|_{w^k_{t}}^2]$ instead of using Lemma \ref{Lem:BoundExiNew}.
\end{proof}
}
}

\clearpage
\clearpage
\section{\changeHK{\changeHK{Proof} of convergence analysis on retraction-convex functions}}

\changeHK{This subsection presents a local convergence rate analysis in neighborhood of a local minimum for retraction-convex functions. This {\it local} setting is very common and standard in manifold optimization.}

\subsection{Preliminary lemmas}

This subsection first states some essential lemmas. Since $f$ is strongly retraction-convex on $\changeHK{\Theta}$ by Assumption \ref{Assump:3}, there exist constants $0 < \lambda$ such that $\lambda \leq \frac{d^2 \changeHK{f(R_{\changeHK{w^k_{t}}}(t \eta_k/\|\eta_k\|\changeHK{_{w^k_{t}}}))}}{dt^2}$ for all $t \in [0, \alpha^k_t \| \eta_k \|\changeHK{_{w^k_{t}}}]$. From Taylor's theorem, we obtain below;
\begin{Lem}[In Lemma 3.2 in {\cite{Huang_SIOPT_2015}}]
\label{Lemma:HessianBoundFunc}
\changeHK{Under} Assumptions \ref{Assump:1}.1, \ref{Assump:1}.2, and \changeHK{Assumption \ref{Assump:3}}, there exists $\lambda$ such that
\begin{eqnarray}
\label{Eq:HessianLowerBoundFunc}
f(w^k_{t+1}) - f(w^k_{t}) & \geq & \langle \gradf (w^k_{t}), \alpha^k_t \eta_k \rangle_{w^k_{t}} + \frac{1}{2} \lambda (\alpha^k_t \| \eta_k \|\changeHK{_{w^k_{t}}})^2.
\end{eqnarray}
\end{Lem}
\begin{proof}
%From Taylor's theorem, we have
%\begin{eqnarray}
%\label{}
%f(w^k_{t+1}) - f(w^k_{t}) 
%	%& = & m_k(\alpha^k_t \| \eta_k \|) - m_k(0) \nonumber \\
%	& = & \changeHK{f(R_{\changeHK{w^k_{t}}}(\alpha_t^k \eta_k)) - f(R_{\changeHK{w^k_{t}}}(0))} \nonumber \\
%	& = & \frac{d \changeHK{f(R_{\changeHK{w^k_{t}}}(0))}}{dt}\alpha^k_t \| \eta_k \|\changeHK{_{w^k_{t}}}  + \frac{1}{2} \frac{d^2 \changeHK{f(R_{\changeHK{w^k_{t}}}(p \eta_k/\|\eta_k\|\changeHK{_{w^k_{t}}}))}}{dt^2} (\alpha^k_t \| \eta_k \|\changeHK{_{w^k_{t}}} )^2 \nonumber \\
%	& = & \langle \gradf (w^k_{t}), \alpha^k_t \eta_k \rangle_{w^k_{t}}  + \frac{1}{2} \frac{d^2 \changeHK{f(R_{\changeHK{w^k_{t}}}(p \eta_k/\|\eta_k\|\changeHK{_{w^k_{t}}}))}}{dt^2} (\alpha^k_t \| \eta_k \|\changeHK{_{w^k_{t}}} )^2 \nonumber \\	
%	& \geq & \langle \gradf (w^k_{t}), \alpha^k_t \eta_k \rangle_{w^k_{t}} + \frac{1}{2} \lambda (\alpha^k_t \| \eta_k \|\changeHK{_{w^k_{t}}})^2,\nonumber 
%\end{eqnarray}
From Taylor's theorem, we have
\begin{eqnarray}
\label{}
& & f(w^k_{t+1}) - f(w^k_{t})  \nonumber \\
	%& = & m_k(\alpha^k_t \| \eta_k \|) - m_k(0) \nonumber \\
	& = & \changeHK{f(R_{\changeHK{w_t^k}}(\alpha_t^k \eta_k)) - f(R_{\changeHK{w_t^k}}(0))} \nonumber \\
	& = & \changeHSS{\frac{d}{d\tau}} \changeHK{f(R_{\changeHK{w_t^k}}(\changeHSS{\tau \eta_k/\|\eta_k\|_{w^k_t}}))}\changeHSS{\Big|_{\tau=0}\cdot} \alpha^k_t \| \eta_k \|\changeHK{_{w^k_{t}}}  + \frac{1}{2} \changeHSS{\frac{d^2}{d\tau^2}} \changeHK{f(R_{\changeHK{w_t^k}}(\changeHSS{\tau} \eta_k/\|\eta_k\|\changeHK{_{w^k_{t}}}))}\changeHSS{\Big|_{\tau=p}}\cdot (\alpha^k_t \| \eta_k \|\changeHK{_{w^k_{t}}} )^2 \nonumber \\
	& = & \langle \gradf (w^k_{t}), \alpha^k_t \eta_k \rangle_{w^k_{t}}  + \frac{1}{2} \changeHSS{\frac{d^2}{d\tau^2}} \changeHK{f(R_{\changeHK{w_t^k}}(\changeHSS{\tau} \eta_k/\|\eta_k\|\changeHK{_{w^k_{t}}}))}\changeHSS{\Big|_{\tau=p}}\cdot (\alpha^k_t \| \eta_k \|\changeHK{_{w^k_{t}}} )^2 \nonumber \\	
	& \geq & \langle \gradf (w^k_{t}), \alpha^k_t \eta_k \rangle_{w^k_{t}} + \frac{1}{2} \changeHK{\lambda} (\alpha^k_t \| \eta_k \|\changeHK{_{w^k_{t}}})^2,\nonumber 
\end{eqnarray}

where $0 \leq p \leq \alpha^k_t \| \eta_k \|\changeHK{_{w^k_{t}}}$, and the inequality uses Assumption \ref{Assump:3}. This yields (\ref{Eq:HessianLowerBoundFunc}). 
This completes the proof. 
\end{proof}

%\begin{Lem}[{\cite[Lemma 3.3]{Huang_SIOPT_2015}}]
\begin{Lem}[Lemma 3.3 in {\cite{Huang_SIOPT_2015}}]
\label{Lemma:g_s_y_g_s_s}
\changeHK{Under} Assumptions \ref{Assump:1}.1, \ref{Assump:1}.2, \ref{Assump:1}.3, \ref{Assump:1}.4, \changeHK{\ref{Assump:1}.6}, \changeHK{and Assumption \ref{Assump:3}}, there exist two constants $0 < \lambda < \Lambda$ such that
\begin{eqnarray}
\label{Eq:g_s_y_g_s_s}
\lambda \ \leq \ \frac{\langle s_k, y_k \rangle}{\langle s_k, s_k \rangle} \ \leq \  \Lambda 
\end{eqnarray}
for all $k$.
\end{Lem}
\begin{proof}
From Lemma \ref{Lemma:HessianBoundFunc}, \changeHK{the proof of this lemma is given, but we omit it}. \changeHK{The reader can see} the \changeHK{complete} proof in Lemma 3.3 in {\cite{Huang_SIOPT_2015}}.
\end{proof}

\begin{Lem}[Lemma 3.9 in {\cite{Huang_SIOPT_2015}}]
\label{Lemma:y_y_s_y}
Suppose Assumption \ref{Assump:1} \changeHK{and Assumption \ref{Assump:3}} hold, there exists a constant $0 < \changeHK{\Upsilon_{c}}$  for all $k$ such that
\begin{eqnarray}
\label{Eq:y_y_s_y}
%\changeHK{\upsilon} \ \leq \  
\frac{\langle y_k, y_k \rangle}{\langle s_k, y_k \rangle} \  \leq \ \changeHK{\Upsilon_{c}}.
\end{eqnarray}
\end{Lem}
\begin{proof}
(Lemma 3.9 in {\cite{Huang_SIOPT_2015}}) The \changeHK{complete} proof is given in {\cite{Huang_SIOPT_2015}}. It is also stated in 
Lemma \ref{Lemma:y_y_s_y_non_convex} for completeness. 
(\ref{Eq:y_s_bound_new}) yields $\| y_k \| \leq b_4 \| s_k \|$ derived in Lemma \ref{Lemma:y_y_s_y_non_convex}, where $b_4 > 0$. Therefore, by Lemma \ref{Lemma:g_s_y_g_s_s}, we have 
\begin{eqnarray}
	\frac{\langle y_k, y_k \rangle }{\langle s_k, y_k \rangle} &\leq& \frac{\langle \changeHK{y_k}, y_k \rangle}{\lambda \langle s_k, s_k \rangle} 
	\ \ \leq \ \ \frac{b^2_4}{\lambda} \quad \quad\changeHK{(=\Upsilon_{c})}.
\end{eqnarray}	
\changeHK{Denoting $b_4^2/\lambda$ as $\Upsilon_{c}$,} this completes the proof. 
\end{proof}

\subsection{Eigenvalue bounds of $\mathcal{H}_t^k$ on retraction-convex functions}

Now, we attempt to bound ${\rm trace}(\hat{\tilde{\mathcal{B}}})$ and ${\rm det}(\hat{\tilde{\mathcal{B}}})$ in order to bound the eigenvalues of $\tilde{\mathcal{H}}^k$, where a {\it hat} denotes the coordinate expression of the operator. The bound of ${\rm trace}(\hat{\tilde{\mathcal{B}}})$  \changeHK{is identical to that of the non-convex case in Lemma \ref{Lemma:TraceBound}. Therefore, we concentrate on the bound of ${\rm det}(\hat{\tilde{\mathcal{B}}})$.} As the same as Lemma \ref{Lemma:TraceBound}, the proof follows stochastic L-BFGS methods in the Euclidean space, e.g., \cite{Mokhtari_JMLR_2015_s,Byrd_SIOPT_2016}. \changeHK{Similarly to Section \ref{AppSec:BoundOfHNonConvex}, it should be noted that ${\rm trace}(\hat{\tilde{\mathcal{B}}})$ and ${\rm det}(\hat{\tilde{\mathcal{B}}})$ \changeHK{do not depend on} the chosen basis.}
\begin{Lem}[Bounds of trace and determinant of $\tilde{\mathcal{B}}^k$]
\label{Lemma:TraceDetBound}
Consider the recursion of $\tilde{\mathcal{B}}^k_u$  \changeHK{defined in (\ref{Eq:HessianOperatorUpdate})}. 
If Assumption \ref{Assump:1} \changeHK{and Assumption \ref{Assump:3}} hold, the ${\rm trace}(\hat{\tilde{\mathcal{B}}}^k)$ in a coordinate expression of $\tilde{\mathcal{B}}^k$ is uniformly upper bounded for all $k \geq 1$,
\begin{eqnarray}
	\label{Eq:TraceBound}
	{\rm trace}(\changeHK{\hat{\tilde{\mathcal{B}}}^k}) & \leq & (\changeHK{M}+\tau) \changeHK{\Upsilon_{c}}.
\end{eqnarray}
\changeHK{where $M$ is the dimension of$\mathcal{M}$.} \changeHK{Similarly}, if Assumption \ref{Assump:1} \changeHK{and Assumption \ref{Assump:3}} hold, the ${\rm det}(\hat{\tilde{\mathcal{B}}}^k)$ in a coordinate expression of $\tilde{\mathcal{B}}^k$ is uniformly lower bounded for all $k$,
\begin{eqnarray}
	\label{Eq:det_bound}
	{\rm det}(\changeHK{\hat{\tilde{\mathcal{B}}}^{k}}) & \geq &  
	\changeHK{\upsilon}^{\changeHK{M}} \left[\frac{\lambda}{({\changeHK{M}}+\tau) \changeHK{\Upsilon_{c}}}\right]^{\tau}.
\end{eqnarray}
Here, a hat expression represents the coordinate expression of an operator.
\end{Lem}

\begin{proof}
\changeHS{The proof can be completed parallel to the Euclidean case~\cite{Moritz_AISTATS_2016_s}.} \changeHK{As mentioned, the proof for the bound of ${\rm trace}(\hat{\tilde{\mathcal{B}}})$ is given in Lemma \ref{Lemma:TraceBound}, we address only ${\rm det}(\hat{\tilde{\mathcal{B}}}^k)$.} 

\changeHK{Because} $\changeHK{\mathcal{T}}$ is an isometry vector transport, $\mathcal{T}_{\eta_k}$ is invertible for all $k$\changeHK{. Accordingly, ${\rm det}(\hat{\tilde{\mathcal{B}}}^k)$ can be reformulated as} 
\begin{eqnarray}
	\label{Eq:DetEquivalency}	
	{\rm det}(\hat{\check{\mathcal{B}}}^k) & = &{\rm det}(\hat{\mathcal{T}}_{\eta_k} \hat{\tilde{\mathcal{B}}}^k \hat{\mathcal{T}}^{-1}_{\eta_k}) = {\rm det}(\hat{\tilde{\mathcal{B}}}^k).
	\end{eqnarray}
We consider the determinant lower bound of ${\rm det}(\hat{\tilde{\mathcal{B}}}_{k,\tau})$ from (\ref{Eq:HessianOperatorUpdate}) as
\begin{eqnarray}
	\label{Eq:derivation_det_trans}
	{\rm det}(\hat{\tilde{\mathcal{B}}}\changeHK{^k_{u+1}}) 
	& = & {\rm det}(\hat{\check{\mathcal{B}}}^k_u)
	{\rm det}\left(\mat{I}  
	- \frac{ s_{k-\tau + u} (\hat{\check{\mathcal{B}}}^k_u s_{k-\tau + u})^T}
	{\langle \hat{\check{\mathcal{B}}}^k_u s_{k-\tau + u},s_{k-\tau + u} \rangle}
	+ \frac{(\hat{\check{\mathcal{B}}}^k_u)^{-1} y_{k-\tau + u} y^T_{k-\tau + u}}
	{\langle y_{k-\tau + u},s_{k-\tau + u} \rangle}\right) \nonumber \\
	& = & {\rm det}(\hat{\check{\mathcal{B}}}^k_u ) 
	{\rm det}\left( \frac{(\hat{\check{\mathcal{B}}}^k_u s_{k-\tau + u})^T}{\langle \hat{\check{\mathcal{B}}}^k_u s_{k-\tau + u},s_{k-\tau + u} \rangle} (\hat{\check{\mathcal{B}}}^k_u)^{-1} y_{k-\tau + u}\right)\nonumber \\
	& = & {\rm det}(\hat{\check{\mathcal{B}}}^k_u ) 
	\frac{\langle s_{k-\tau + u},y_{k-\tau + u} \rangle}{\langle \hat{\check{\mathcal{B}}}^k_u s_{k-\tau + u},s_{k-\tau + u} \rangle} \nonumber \\
	& = & {\rm det}(\hat{\check{\mathcal{B}}}^k_u ) 
	\frac{\langle s_{k-\tau + u},y_{k-\tau + u})}{\|s_{k-\tau + u} \|^2} \frac{\|s_{k-\tau + u} \|^2}{\langle \hat{\check{\mathcal{B}}}^k_u s_{k-\tau + u},s_{k-\tau + u} \rangle} \nonumber \\	
	& \geq & {\rm det}(\hat{\check{\mathcal{B}}}^k_u) 
	\frac{\lambda}{\lambda_{\changeHK{\rm max}}(\hat{\mathcal{B}}^k_u)} \nonumber \\
	& \geq & {\rm det}(\hat{\check{\mathcal{B}}}^k_u ) 
	\frac{\lambda}{{\rm trace}(\hat{\mathcal{B}}^k_u)} \nonumber \\			
	& \geq &  {\rm det}(\hat{\check{\mathcal{B}}}^k_u ) 
	\frac{\lambda}{({\changeHK{M}}+\tau) \changeHK{\Upsilon_{c}}}.
\end{eqnarray}

\changeHK{Regarding the second equality, we obtain it from the formula} ${\rm det}(\mat{I}+\vec{u}_1\vec{v}_1^T+\vec{u}_2\vec{v}_2^T)=(1+\vec{u}_1^T\vec{v}_1)(1+\vec{u}_2^T\vec{v}_2)-(\vec{u}^T_1\vec{v}_1)(\vec{u}^T_2\vec{v}_2)$ by setting 
$\vec{u}_1=-s_{k-\tau + u}$, 
$\vec{v}_1 = \hat{\mathcal{B}}^k_u s_{k-\tau + u}/\langle \hat{\mathcal{B}}^k_u s_{k-\tau + u},s_{k-\tau + u} \rangle$, 
$\vec{u}_2 = (\hat{\mathcal{B}^k_u)}^{-1} y_{k-\tau + u}$, and
$\vec{v}_2 = y_{k-\tau + u}/\langle s_{k-\tau + u},y_{k-\tau + u} \rangle$. The first inequality follows from (\ref{Eq:g_s_y_g_s_s}) in Lemma \ref{Lemma:g_s_y_g_s_s} and the fact $\langle \hat{\mathcal{B}}^k_u s_{k-\tau + u},s_{k-\tau + u} \rangle \leq \lambda_{\changeHK{\rm max}}(\hat{\mathcal{B}}^k_u) \|s_{k-\tau + u} \|^2$. \changeHK{Actually, we use} the fact \changeHK{the trace of a positive definite matrix bounds its maximal eigenvalue for the second inequality}. The \changeHK{last} inequality follows (\ref{Eq:B_k_u}). Then, applying (\ref{Eq:DetEquivalency}), (\ref{Eq:derivation_det_trans}) turns \changeHK{to be} 
\begin{eqnarray}
	\label{Eq:derivation_det_trans2}
	{\rm det}(\hat{\tilde{\mathcal{B}}}^k_{u+1}) 
	& \geq &  {\rm det}(\hat{\tilde{\mathcal{B}}}^k_u ) 
	\frac{\lambda}{({\changeHK{M}}+\tau) \changeHK{\Upsilon_{c}}}.
\end{eqnarray}

Applying (\ref{Eq:derivation_det_trans2}) recursively \changeHK{from} $u=0$ \changeHK{to} $u=\tau-1$\changeHK{,} %and further observing that $u \leq \tau$ in all of the resulting factors, 
\changeHK{we obtain} that 
\begin{eqnarray}
	%\label{Eq:derivation_det_trans}
	{\rm det}(\hat{\tilde{\mathcal{B}}}\changeHK{^k_\tau}\changeHK{)} & \geq &  
	\left[\frac{\lambda}{({\changeHK{M}}+\tau) \changeHK{\Upsilon_{c}}}\right]^{\tau}{\rm det}(\hat{\tilde{\mathcal{B}}}^k_0 ).\nonumber
\end{eqnarray}

To bound the determinant of $\hat{\tilde{\mathcal{B}}}\changeHK{^k_0}$, considering $\hat{\tilde{\mathcal{B}}}\changeHK{^k_0}={\rm id}/\changeHK{\chi_k}$ as above and Lemma \ref{Lemma:y_y_s_y_lower}, we can rewrite for $\changeHK{k} \geq 1$ as
\begin{eqnarray}
	\label{}
	{\rm det}(\hat{\tilde{\mathcal{B}}}\changeHK{^k_0}) & = & {\rm det}\left(\frac{\mat{I}}{\changeHK{\chi_k}} \right) \quad = \quad \frac{1}{\changeHK{\chi_k}^{\changeHK{M}}} \quad=\quad \left( \frac{\langle y_{k},y_{k} \rangle}{\langle s_{k},y_{k} \rangle} \right)^{\changeHK{M}} \quad\geq \quad\changeHK{\upsilon}^{\changeHK{M}},\nonumber
\end{eqnarray}
\changeHK{where $\changeHK{\upsilon}$ is defined in Lemma \ref{Lemma:y_y_s_y_lower}.} 
Consequently, we obtain as
\begin{eqnarray}
	%\label{Eq:det_bound}
	{\rm det}(\hat{\tilde{\mathcal{B}}}\changeHK{^k_\tau}) & \geq &  
	\changeHK{\upsilon}^{\changeHK{M}}\left[\frac{\lambda}{({\changeHK{M}}+\tau) \changeHK{\Upsilon_{c}}}\right]^{\tau}.\nonumber
\end{eqnarray}
Thus, this yields (\ref{Eq:det_bound}), and these complete the proof. 
\end{proof}

Now we prove the main lemma for Proposition \ref{AppProposition:HessianOperatorBoundsConvex}.

\begin{Lem}
\label{Lem:BoundsOfHk}
If Assumption \ref{Assump:1} \changeHK{and Assumption \ref{Assump:3}} hold, the eigenvalues of $\tilde{\mathcal{H}}^k$ is bounded by $\gamma{\changeHK{_{c}}}$ and $\Gamma{\changeHK{_{c}}}$ \changeHSS{with $0< \gamma\changeHK{_{c}} < \Gamma{\changeHK{_{c}}} < \infty$} for all $k \geq 1$ \changeHK{as} 
\begin{eqnarray}
	%\label{Eq:InverseHessianOperatorBounds}
	\gamma\changeHK{_{c}} {\rm id} \ \preceq \ \tilde{\mathcal{H}}^k \ \preceq \ \Gamma{\changeHK{_{c}}} {\rm id}. \nonumber
\end{eqnarray}
\end{Lem}

\begin{proof}
%Lemma \ref{Lemma:TraceDetBound} states that the trace and determinants of the operator $\hat{\tilde{\mathcal{B}}}\changeHK{^k}=\hat{\tilde{\mathcal{B}}}^k_u$ are bounded for all $k \geq 1$. 
\changeHS{The proof is obtained as parallel to the Euclidean case~\cite{Mokhtari_JMLR_2015_s}.}
\changeHK{The lower part is identical to the proof that is given in Lemma \ref{Lemma:TraceBound}. Regarding the upper bound,} \changeHK{because the} determinant of a matrix is the product of its eigenvalues, the lower bound in (\ref{Eq:det_bound}) \changeHK{bounds} the product of the eigenvalues of $\hat{\tilde{\mathcal{B}}}\changeHK{^k}$ \changeHK{from below}. This means that $\prod_{i=1}^{\changeHK{M}} \lambda_i \ge \frac{\lambda^{\tau} \changeHK{\upsilon}^{\changeHK{M}}}{[({\changeHK{M}}+\tau) \changeHK{\Upsilon_{c}}]^{\tau}}$\changeHS{.} Thus, \changeHK{we have below} for any given eigenvalue of $\hat{\tilde{\mathcal{B}}}^k$, say $\lambda_j$\changeHS{,}
\begin{eqnarray}
	\lambda_j & \ge & \frac{1}{\prod_{k=1, k\neq j}^{\changeHK{M}} \lambda_k} \cdot \frac{\lambda^{\tau} \changeHK{\upsilon}^{\changeHK{M}}}{[({\changeHK{M}}+\tau) \changeHK{\Upsilon_{c}}]^{\tau}}.  
\end{eqnarray}
Considering that $({\changeHK{M}}+\tau) \changeHK{\Upsilon_{c}}$ is \changeHS{an} upper bound for the eigenvalues of $\hat{\tilde{\mathcal{B}}}^k$, $[({\changeHK{M}}+\tau) \changeHK{\Upsilon_{c}}]^{{\changeHK{M}}-1}$ \changeHK{gives the upper bound of} the product of the $({\changeHK{M}}-1)$ eigenvalues $\prod_{k=1, k\neq j}^{\changeHK{M}} \lambda_k$. 

As a result, we obtain that any eigenvalues of $\hat{\mathcal{B}}^k$ is lower bounded as 
\begin{eqnarray}
	\label{Eq:B_t_bound}
	\lambda_j & \ge & \frac{1}{[({\changeHK{M}}+\tau) \changeHK{\Upsilon_{c}}]^{{\changeHK{M}}-1}} \cdot \frac{\lambda^{\tau} \changeHK{\upsilon}^{\changeHK{M}}}{[({\changeHK{M}}+\tau) \changeHK{\Upsilon_{c}}]^{\tau}}  \ = \  \frac{\lambda^{\tau} \changeHK{\upsilon}^{\changeHK{M}}}{[({\changeHK{M}}+\tau) \changeHK{\Upsilon_{c}}]^{{\changeHK{M}}+\tau-1}} .  
\end{eqnarray}

Consequently, we finally obtain $ \frac{\lambda^{\tau} \changeHK{\upsilon}^{\changeHK{M}}}{[({\changeHK{M}}+\tau) \changeHK{\Upsilon_{c}}]^{{\changeHK{M}}+\tau-1}} \changeHK{\rm id} \changeHS{\preceq} \tilde{\mathcal{B}}^k$.

Now, we obtain the claim. The bounds in \changeHK{(\ref{Eq:TraceBound})} and (\ref{Eq:B_t_bound}) imply that their inverses are bounds \changeHK{for} the eigenvalues of $\hat{\tilde{\mathcal{H}}}^k=(\hat{\tilde{\mathcal{B}}}^k)^{-1}$ as
\begin{eqnarray}
	\label{Eq:H_t_bound_c}
	\quad\changeHK{(\gamma_c\changeHK{{\rm id}=})} \quad \frac{1}{({\changeHK{M}}+\tau) \changeHK{\Upsilon_{c}}} {\rm id}\  \changeHS{\preceq} & \tilde{\mathcal{H}}^k & \changeHS{\preceq} \  \frac{[({\changeHK{M}}+\tau) \changeHK{\Upsilon_{c}}]^{{\changeHK{M}}+\tau-1}}{\lambda^{\tau} \changeHK{\upsilon}^{\changeHK{M}}}  {\rm id} \quad \changeHK{(=\Gamma_{c}\changeHK{{\rm id}})}.
\end{eqnarray}
Denoting $\frac{1}{({\changeHK{M}}+\tau) \changeHK{\Upsilon_{c}}}$ as \changeHK{$\gamma_c$ as in Lemma \ref{Lem:BoundsOfHkNonConvex}}, and $\frac{[({\changeHK{M}}+\tau) \changeHK{\Upsilon_{c}}]^{{\changeHK{M}}+\tau-1}}{\lambda^{\tau} \changeHK{\upsilon}^{\changeHK{M}}}$ as $\Gamma{\changeHK{_{c}}}$, we obtain the claim. This completes the proof. 
\end{proof}

Finally we give Proposition \ref{AppProposition:HessianOperatorBoundsConvex} \changeHK{for retraction-convex functions}. 

\begin{Prop}[Bounds of $\mathcal{H}\changeHKK{^k_t}$ \changeHK{for retraction-convex functions}]
\label{AppProposition:HessianOperatorBoundsConvex}
Consider the operator $\changeHK{\check{\mathcal{H}}^k}\changeHK{:= \mathcal{T}_{\tilde{\eta}_t^k} \circ \tilde{\mathcal{H}}^k \circ (\mathcal{T}_{\tilde{\eta}_t^k})^{-1}}$. Define the constant $0< \gamma\changeHK{_{c}} < \Gamma{\changeHK{_{c}}} < \infty$. If Assumption \ref{Assump:1} holds, the range of eigenvalues of $\mathcal{H}^k_t$ is bounded by $\gamma{\changeHK{_{c}}}$ and $\Gamma{\changeHK{_{c}}}$ for all $k \geq 1, t \geq 1$, i.e., 
\begin{eqnarray}
 \gamma\changeHK{_{c}} {\rm id} \ \preceq \ \mathcal{H}^k_t \ \preceq \ \Gamma{\changeHK{_{c}}}{\rm id}.
\end{eqnarray}
\end{Prop}
\begin{proof}
\changeHK{The proof is identical to that of Proposition \ref{AppProposition:HessianOperatorBoundsNonConvex}}.
\end{proof}

\changeHK{
\begin{Rmk}
We discuss the obtained bounds of $\mathcal{H}^k_t$ by comparing the retraction-convex case in Proposition \ref{AppProposition:HessianOperatorBoundsConvex} with the non-convex case in Proposition \ref{AppProposition:HessianOperatorBoundsNonConvex}. 
The lower bound of $\mathcal{H}^k_t$ in the convex case is $\gamma_c=1/((M+\tau)\Upsilon_{c})=\lambda/((M+\tau)b_4^2) $ from (\ref{Eq:y_y_s_y}) and (\ref{Eq:H_t_bound_c}). 
The non-convex case is $\gamma_{nc}=1/((M+\tau)\Upsilon_{nc})=\epsilon/((M+\tau)b_4^2) $ from (\ref{Eq:y2_ys_bound}) and (\ref{Eq:H_t_bound}). 
In terms of $\epsilon$ and $\lambda$, $\gamma_c$ is $\mathcal{O}(\lambda)$ and $\gamma_{nc}$ is $\mathcal{O}(\epsilon)$. Assuming $\lambda$ of the strongly retraction-convex functions much larger than $\epsilon$ because it is generally set to very small value \cite{Li_SIOPT_2001}, we conclude $\gamma_c > \gamma_{nc}$. 
Meanwhile, the upper bound of $\mathcal{H}^k_t$ in the convex case is 
$\Gamma_c=\frac{[(M+\tau) \Upsilon_c]^{M+\tau-1}}{\lambda^{\tau} \upsilon^M}
=\frac{[(M+\tau)b_4^2]^{M+\tau-1}}{\lambda^{M+2\tau-1} \epsilon^M}$ from (\ref{Eq:y_y_s_y}) and (\ref{Eq:H_t_bound_c}). 
The non-convex case is $\Gamma_{nc}= \frac{(1 + b_4/\epsilon)^{2(\tau+1)-1}}{\epsilon ((1 + b_4/\epsilon)^2-1)}$ from (\ref{Eq:H_uppebound_nc}). 
With respect to $\epsilon$ and $\lambda$, $\Gamma_c$ is $\mathcal{O}(1/(\lambda^{M+2\tau-1}) \epsilon^M)$ and $\Gamma_{nc}$ is $\mathcal{O}(1/\epsilon^{2\tau})$. 
Similarly to the lower bound mentioned above, we conclude $\Gamma_c < \Gamma_{nc}$. Consequently, the range of the bounds of $\mathcal{H}^k_t$ on strongly retraction-convex functions is smaller than that on non-convex functions. 
\end{Rmk}
}

%\subsection{\changeHK{Global convergence analysis on retraction-convex functions}}
%%
%\changeHK{
%The convergence analysis on retraction-convex functions is the same as the case in the retraction non-convex case except the upper bound of $\mathcal{H}^k_t$, i.e., $\Gamma{\changeHK{_{c}}}$, instead of $\Gamma{\changeHK{_{nc}}}$.
%}
%

%
%
%
%
%
\subsection{Proof of \changeHK{local convergence rate analysis} (Theorem \ref{Thm:LocalConvergenceConvex})}
\label{Apd:LocalConvergence}

This subsection first introduce\changeHK{s} some essential lemmas. Then, the main proof of Theorem \ref{Thm:LocalConvergenceConvex} is given. This section also derives at the end a corollary about the analysis when \changeHK{the} using exponential mapping and \changeHK{the} parallel translation that are special cases of \changeHK{the} retraction and \changeHK{the} vector transport. 
\subsubsection{Essential lemmas}

We first introduce a property of the Karcher mean on a general Riemannian manifold.
\begin{Lem}[Lemma C.2 in \cite{Sato_arXiv_2017}] 
\label{AppenLem:KarcherMeanDistance}
Let $w_1,\dots,w_m$ be points on a Riemannian manifold $\mathcal{M}$ and let $w$ be the Karcher mean of the $m$ points.
For an arbitrary point $p$ on $\mathcal{M}$, we have
\begin{eqnarray*}
({\rm dist}(p,w))^2 & \le &  \frac{4}{m}\sum_{i=1}^m({\rm dist}(p,w_i))^2.
\end{eqnarray*}
\end{Lem}

We now \changeHK{bound} the variance of $\xi^k_t$ as follows.

\begin{Lem}[Lemma 5.8 in \cite{Sato_arXiv_2017}]
\label{AppenLem:UpperBoundVariance}
Suppose Assumptions 
\ref{Assump:1}.1, 
\ref{Assump:1}.2, 
\ref{Assump:1}.4,  
\ref{Assump:1}.5, and \ref{Assump:1}.7, 
which guarantee Lemmas \ref{Lem:pseudo_Lipschitz}, \ref{Lemma:VecParaDiff}, and \ref{Lemma:retraction_dist} for $\bar{w} = w^*$.
Let $\beta>0$ be a constant such that
\begin{eqnarray*}
	\| P_{\gamma}^{w \leftarrow z} ({\rm grad} f_n(z)) - {\rm grad} f_n(w) \|_{w} & \leq & \beta {\rm dist}(z,w),\qquad w, z \in \Theta,\ n = 1,2,\dots, N.
\end{eqnarray*}
The existence of such $\beta$ is guaranteed by Lemma \ref{Lem:pseudo_Lipschitz}.
The upper bound of the variance of $\xi_t^k$ is given by
\begin{eqnarray}
\label{Append_Eq:UpperBoundVariance}
	\mathbb{E}_{i_t^k}[\| \xi_t^k \|_{w_{t}^k}^2] &\leq &
4(\beta^2+\tau_2^2C^2\theta^2)(7({\rm dist}(w_{t}^k,w^*))^2 + 4({\rm dist}(\tilde{w}^{k},w^*))^2),
\end{eqnarray}
where the constant $\theta$ corresponds to that in Lemma \ref{Lemma:VecParaDiff}, 
%\note{$C$ is the upper bound of $\|\gradf_n(w)\|_w,\ n=1, 2, \dots, N$ for $w \in \Theta$}, 
\changeHK{$C$ is the constant of Assumption \ref{Assump:1}}, and $\tau_2 > 0$ appears in \eqref{Eq:tau1_tau2}.
\end{Lem}

We also have the following corollary of the previous lemma with the case $R={\rm Exp}$ and $\mathcal{T} = P$.
\begin{Cor}[Corollary 5.1 in \cite{Sato_arXiv_2017}]
\label{AppenCor:UpperBoundVariance}
Consider Algorithm 1 with $\mathcal{T} = P$ and $R={\rm Exp}$, i.e., the parallel translation and the exponential mapping case.
When each $\gradf_n$ is $\beta_0$-Lipschitz continuously differentiable, the upper bound of the variance of $\xi_t^k$ is given by
\begin{eqnarray}
	\mathbb{E}_{i_t^k}[\| \xi_t^k \|^2_{w_t^k}] &\leq &
\beta_0^2 (14({\rm dist}(w_{t}^k,w^*))^2 + 8{\rm dist}(\tilde{w}^{k},w^*))^2  ).
\end{eqnarray}
\end{Cor}

Next, we show the lemma that finds a lower bound for $\| \gradf(w^k_t)\|_{w^k_t}$ \changeHK{with respect to} the error $f(w^k_t) - f(w^*)$\changeHK{, which} is a standard derivation in the Euclidean space. See, e.g., \cite{Boyd_ConvexOpt_2004}. We extend this into manifolds. 

\begin{Lem}
\label{Lemma:GradientLowerBound}
\changeHS{Let $w \in \mathcal{M}$ and $z$ be in a totally retractive neighborhood of $w$. It holds that} 
\begin{eqnarray}
	\label{Eq:GradientLowerBound}
	2 \lambda (f(w) -f(z) ) & \leq & \| \gradf(w) \|_w^2.
\end{eqnarray} 
\end{Lem}

\begin{proof}
\changeHS{Let $\zeta =R^{-1}_w(z)$.}
Using (\ref{Eq:HessianLowerBoundFunc}) in Lemma \ref{Lemma:HessianBoundFunc}, which is equivalent to the strong convexity of $\gradf$, 
we obtain
\begin{eqnarray}
\label{ApdEq:LocalRate_F_expect_3}
f(z) 
& \ge & f(w) + \langle \gradf(w), \zeta\rangle_{w} + \frac{\lambda}{2} \|\zeta\|_{w}^2 \nonumber \\
& \ge & f(w) + \min_{\xi \changeHS{\in T_w \mathcal{M}}} \left( \langle \gradf(w), \xi \rangle_{w} + \frac{\lambda}{2} \|\xi\|_{w}^2 \right)\nonumber \\
& \ge & f(w) - \frac{1}{2\lambda} \| \gradf(w) \|_w^2\changeHS{.} \nonumber 
\end{eqnarray}
Rearranging this inequality completes the proof. 
\end{proof}

\subsubsection{Main proof of Theorem \ref{Thm:LocalConvergenceConvex}}

%\begin{Thm}
%\label{Append_Thm:LocalConvergenceConvex}
\noindent
{\bf Theorem \ref{Thm:LocalConvergenceConvex}.} {\it
Let $\mathcal{M}$ be a Riemannian manifold and $w^* \in \mathcal{M}$ be a non-degenerate local minimizer of $f$ (i.e., ${\rm grad} f(w^*)=0$ and the Hessian ${\rm Hess}f(w^*)$ of $f$ at $w^*$ is positive definite). Suppose \changeHK{Assumptions \ref{Assump:1} and \ref{Assump:3} hold}.
Let the constants $\beta, \theta$, and $C$ \changeHS{be} in Lemma \ref{AppenLem:UpperBoundVariance}, and $\tau_1$ \changeHK{and $\tau_2$} \changeHS{be} in Lemma \ref{Lemma:retraction_dist}. 
\changeHK{$\Lambda$ and $\lambda$ are the constants in Lemmas \ref{Lemma:HessianUpperBoundFunc} and \ref{Lemma:HessianBoundFunc}, respectively}.
$\gamma\changeHK{_{c}}$ and $\Gamma{\changeHK{_{c}}}$ are the constants in \changeHK{Proposition} \ref{AppProposition:HessianOperatorBoundsConvex}.
Let $\alpha$ be a positive number satisfying $\lambda \tau^2_1 > 2\alpha(  \lambda^2 \tau^2_1 - \changeHS{14 \alpha \Lambda \changeHK{\Gamma^2_{c}}} (\beta^2+\tau_2^2 C^2 \theta^2))$ \changeHS{and $\gamma\changeHK{_{c}} \lambda^2 \tau^2_1 > 14 \alpha \Lambda \changeHK{\Gamma^2_{c}} (\beta^2+\tau_2^2 C^2\theta^2)$}.
It then follows that for any sequence $\{\tilde{w}^k\}$ generated by Algorithm 1 \changeHK{with Option II} under a fixed step-size $\alpha_t^k:=\alpha$ and $m_k:=m$ converging to $w^*$, there exists $\changeHK{0<K_{th}<K}$ such that for all $k>\changeHK{K_{th}}$,
\begin{eqnarray}
\mathbb{E}[ ({\rm dist}(\tilde{w}^{k+1},w^*))^2] & \le & 
\frac{2(\Lambda \tau^2_2 + 16 m \alpha^2 \Lambda \changeHK{\Gamma^2_{c}}(\beta^2+\tau_2^2C^2 \theta^2)}{m\alpha( \gamma\changeHK{_{c}} \lambda^2 \tau^2_1 - 14 \alpha \Lambda \changeHK{\Gamma^2_{c}} (\beta^2+\tau_2^2 C^2\theta^2)) }   \mathbb{E}[({\rm dist}(\tilde{w}^{k-1},w^*))^2\changeHS{]}. \nonumber
\end{eqnarray}
}
%
%\end{Thm}

%%%%%%%%%%%%%%%%%
\begin{proof}
Using (\ref{Eq:HessianUpperBoundFunc}) in Lemma \ref{Lemma:HessianUpperBoundFunc}, which is equivalent to the Lipschitz continuity of $\gradf$ from Assumptions \ref{Assump:1}, we obtain
\begin{eqnarray}
f(w^k_{t+1}) - f(w^k_{t}) & \leq & \langle \gradf (w^k_t), -\alpha \mathcal{H}_t^k \xi_t^k \rangle_{w^k_t} + \frac{1}{2} \changeHS{\Lambda} (-\alpha \|\mathcal{H}_t^k \xi^k_t\|_{w^k_t})^2.\nonumber
\end{eqnarray}

Taking expectation with regard to $i_t^k$, this becomes
\begin{eqnarray}
\label{ApdEq:LocalRate_F_expect_1}
\mathbb{E}_{i_t^k}[f(w^k_{t+1})] -  f(w^k_{t})  & \le &  \mathbb{E}_{i_t^k}[\langle \gradf (w^k_t), -\alpha \mathcal{H}_t^k \xi_t^k \rangle_{w^k_t} + \frac{1}{2} \alpha ^2 \Lambda \|\mathcal{H}_t^k \xi_t^k\|_{w^k_t}^2] \nonumber \\
& \le &     -\alpha \langle \gradf (w^k_t),  \mathbb{E}_{i_t^k}[\mathcal{H}_t^k \xi_t^k] \rangle_{w^k_t} + \frac{1}{2} \alpha^2 \Lambda \mathbb{E}_{i_t^k}[\|\mathcal{H}_t^k \xi_t^k\|_{w^k_t}^2]\nonumber \\
&  \le &    -\alpha \langle \gradf (w^k_t), \mathcal{H}_t^k \gradf (w^k_t)\rangle_{w^k_t}  + \frac{1}{2} \alpha^2 \Lambda \mathbb{E}_{i_t^k}[\|\mathcal{H}_t^k \xi_t^k\|^2]\nonumber \\
&  \le &  -\alpha \gamma\changeHK{_{c}} \| \gradf (w^k_t)\|^2_{w^k_{t}} + \frac{1}{2} \alpha^2 \Lambda \changeHK{\Gamma^2_{c}} \mathbb{E}_{i_t^k}[\| \xi_t^k\|^2_{w^k_t}].
\end{eqnarray}
where the third inequality used the fact that $\mathbb{E}_{i_t^k}[\changeHS{\mathcal{H}_t^k} \xi_t^k]=\mathcal{H}_t^k \gradf (w^k_t)$. The last inequality used the bound of $\mathcal{H}_t^k$ in \changeHK{Proposition} \ref{AppProposition:HessianOperatorBoundsConvex}.
%We then use Lemma \ref{Lemma:GradientLowerBound} to bound the second term on (\ref{ApdEq:LocalRate_F_expect_1}) to get 

From Lemma \ref{Lemma:GradientLowerBound}, (\ref{ApdEq:LocalRate_F_expect_1}) yields
\begin{eqnarray}
\label{ApdEq:LocalRate_F_expect_4}
\mathbb{E}_{i_t^k}[f(w^k_{t+1})] - f(w^k_{t}) 
&  \le & - 2 \alpha \gamma\changeHK{_{c}} \lambda (f(w^k_t) - f(w^*)) + \frac{1}{2} \alpha^2 \Lambda \changeHK{\Gamma^2_{c}} \mathbb{E}_{i_t^k}[\| \xi_t^k\|^2_{w^k_t}].
\end{eqnarray}

Using (\ref{Eq:HessianLowerBoundFunc}) in Lemma \ref{Lemma:HessianBoundFunc} with $\gradf(w^*)=0$, and using 
Lemma \ref{Lemma:retraction_dist}, we obtain 
\begin{eqnarray}
\label{ApdEq:Bound_F_diff}
f(w^k_{t}) - f(w^*)  &\ge & \frac{\lambda}{2} \| R^{-1}_{w^*}(w^k_{t})\|^2_{w^*} \ \ge \ \frac{\lambda \tau_1^2}{2}  ({\rm dist}(w^k_{t}, w^*))^2.
\end{eqnarray}
 
Plugging (\ref{ApdEq:Bound_F_diff}) and the bound of $\mathbb{E}_{i_t^k}[\| \xi_k^k \|^2$ in (\ref{Append_Eq:UpperBoundVariance})  in Lemma \ref{AppenLem:UpperBoundVariance} into (\ref{ApdEq:LocalRate_F_expect_4}) yields 
\begin{eqnarray}
\label{ApdEq:LocalRate_F_expect_5}
\mathbb{E}_{i_t^k}[f(w^k_{t+1})] - f(w^k_{t}) 
&  \le & - \alpha \gamma\changeHK{_{c}} \lambda^2 \tau_1^2 ({\rm dist}(w^k_{t}, w^*))^2 + \frac{1}{2} \alpha^2 \Lambda \changeHK{\Gamma^2_{c}} \mathbb{E}_{i_t^k}[\| \xi_t^k\|^2_{w^k_t}] \nonumber \\
& \le & - \alpha \gamma\changeHK{_{c}} \lambda^2 \tau_1^2 ({\rm dist}(w^k_{t}, w^*))^2\nonumber \\
&&+  \frac{1}{2} \alpha^2 \Lambda \changeHK{\Gamma^2_{c}} 
\{ 
4(\beta^2+\tau_2^2C^2\theta^2)(7({\rm dist}(w_{t}^k,w^*))^2 + 4({\rm dist}(\tilde{w}^{k},w^*))^2)
\} \nonumber \\
& \le & (- \alpha \gamma\changeHK{_{c}} \lambda^2 \tau_1^2 + 14\alpha^2 \Lambda \changeHK{\Gamma^2_{c}} 
(\beta^2+\tau_2^2 C^2\theta^2))({\rm dist}(w_{t}^k,w^*))^2 \nonumber \\
&&+  8  \alpha^2 \Lambda \changeHK{\Gamma^2_{c}}  (\beta^2+\tau_2^2C^2\theta^2) 
({\rm dist}(\tilde{w}^{k},w^*))^2.\nonumber
\end{eqnarray}

Taking expectations over all random variables, \changeHK{we obtain below by further} summing over $t=0, \ldots, \changeHS{m} - 1$ of the inner loop on $k$-th epoch
\begin{eqnarray}
\label{ApdEq:LocalRate_F_expect_6}
\mathbb{E}[f(w^k_{\changeHS{m}}) - f(w^k_{0})] 
&  \le & - (\alpha \gamma\changeHK{_{c}} \lambda^2 \tau_1^2 - 14 \alpha^2 \Lambda \changeHK{\Gamma^2_{c}} 
(\beta^2+\tau_2^2 C^2\theta^2))\sum_{t=0}^{\changeHS{m}-1} \mathbb{E}[ ({\rm dist}(w_{t}^k,w^*))^2] \nonumber \\
&&+  8m \alpha^2 \Lambda \changeHK{\Gamma^2_{c}}(\beta^2+\tau_2^2C^2\theta^2)   \mathbb{E}[
({\rm dist}(\tilde{w}^{k},w^*))^2].
\end{eqnarray}

Here, considering the difference with the solution $w^*$ in terms of the cost function value, we obtain 
\begin{eqnarray}
\label{ApdEq:}
\mathbb{E}[f(w^k_{\changeHS{m}}) - f(w^k_{0})] & = & \mathbb{E}[f(w^k_{\changeHS{m}}) - f(w^*) - (f(w^k_{0})- f(w^*))] \nonumber \\
& \ge & \frac{1}{2} \mathbb{E}[ \lambda \tau^2_1 ({\rm dist}(w_{m}^k,w^*))^2 - \Lambda \tau^2_2  ({\rm dist}(w_{0}^k,w^*))^2].\nonumber
\end{eqnarray}

Plugging the above into (\ref{ApdEq:LocalRate_F_expect_6}) yields
\begin{eqnarray}
\label{ApdEq:LocalRate_F_expect_7}
&&\mathbb{E}[ \lambda \tau^2_1 ({\rm dist}(w_{m}^k,w^*))^2 - \Lambda \tau^2_2  ({\rm dist}(w_{0}^k,w^*))^2]\nonumber \\ 
&  \le & - \changeHS{2}\alpha ( \gamma\changeHK{_{c}} \lambda^2 \tau^2_1 - 14 \alpha \Lambda \changeHK{\Gamma^2_{c}} 
(\beta^2+\tau_2^2 C^2\theta^2))\sum_{t=0}^{\changeHS{m}-1} \mathbb{E}[ ({\rm dist}(w_{t}^k,w^*))^2] \nonumber \\
&&+ 16 m \alpha^2 \Lambda \changeHK{\Gamma^2_{c}}(\beta^2+\tau_2^2C^2\theta^2)   \mathbb{E}[
({\rm dist}(\tilde{w}^{k},w^*))^2].\nonumber
\end{eqnarray}

Rearranging this gives
\begin{eqnarray}
\label{ApdEq:LocalRate_F_expect_8}
&& 2\alpha( \gamma\changeHK{_{c}} \lambda^2 \tau^2_1 - 14 \alpha \Lambda \changeHK{\Gamma^2_{c}} 
(\beta^2+\tau_2^2 C^2\theta^2))\sum_{t=0}^{\changeHS{m}-1} \mathbb{E}[ ({\rm dist}(w_{t}^k,w^*))^2] \nonumber \\ 
& \le & \mathbb{E}[ \Lambda \tau^2_2  ({\rm dist}(w_{0}^k,w^*))^2] - \mathbb{E}[ \lambda \tau^2_1 ({\rm dist}(w_{m}^k,w^*))^2 ]\nonumber \\ 
&& + 16 m \alpha^2 \Lambda \changeHK{\Gamma^2_{c}}(\beta^2+\tau_2^2C^2\theta^2)   \mathbb{E}[
({\rm dist}(\tilde{w}^{k},w^*))^2].
\end{eqnarray}

Now, addressing option I in Algorithm 1, which uses $\tilde{w}^{k+1}=g_{m_k}(w_{1}^k,\ldots,w_{\changeHS{m}}^k)$, we derive below from Lemma \ref{AppenLem:KarcherMeanDistance} as
\begin{eqnarray}
\label{ApdEq:LocalRate_F_expect_9}
&&\frac{m} {4}2\alpha(\gamma\changeHK{_{c}} \lambda^2 \tau^2_1 - 14 \alpha \Lambda \changeHK{\Gamma^2_{c}} (\beta^2+\tau_2^2 C^2\theta^2)) 
\mathbb{E}[ ({\rm dist}(\tilde{w}^{k+1},w^*))^2] \nonumber \\ 
& \le & 2\alpha( \gamma\changeHK{_{c}} \lambda^2 \tau^2_1 - 14 \alpha \Lambda \changeHK{\Gamma^2_{c}} (\beta^2+\tau_2^2 C^2\theta^2)) \nonumber\\
& & 
\hspace*{1cm}\times \mathbb{E} \biggl[ \sum_{t=\changeHS{0}}^{\changeHS{m-1}}  ({\rm dist}(w_{t}^k,w^*))^2 + ({\rm dist}(w_{\changeHS{m}}^k,w^*))^2 -  ({\rm dist}(w_{0}^k,w^*))^2 \biggr] \nonumber \\ 
& \overset{(\ref{ApdEq:LocalRate_F_expect_8})}{\le} &\mathbb{E}[ \Lambda \tau^2_2  ({\rm dist}(w_{0}^k,w^*))^2] - \mathbb{E}[ \lambda \tau^2_1 ({\rm dist}(w_{m}^k,w^*))^2] \nonumber \\
&& + 16 m \alpha^2 \Lambda \changeHK{\Gamma^2_{c}}(\beta^2+\tau_2^2C^2\theta^2)   \mathbb{E}[({\rm dist}(\tilde{w}^{k},w^*))^2\changeHS{]}  \nonumber \\
&&+ 2\alpha( \gamma\changeHK{_{c}} \lambda^2 \tau^2_1 - 14 \alpha \Lambda \changeHK{\Gamma^2_{c}} (\beta^2+\tau_2^2 C^2\theta^2)) \mathbb{E}[({\rm dist}(w_{\changeHS{m}}^k,w^*))^2 -  ({\rm dist}(w_{0}^k,w^*))^2] \nonumber \\ 
& \le & (\Lambda \tau^2_2  - 2\alpha( \gamma\changeHK{_{c}} \lambda^2 \tau^2_1 - 14 \alpha \Lambda \changeHK{\Gamma^2_{c}} (\beta^2+\tau_2^2 C^2\theta^2)) ) \mathbb{E}[({\rm dist}(w_{0}^k,w^*))^2] \nonumber \\
& & 
+ 16 m \alpha^2 \Lambda \changeHK{\Gamma^2_{c}}(\beta^2+\tau_2^2C^2\theta^2)   \mathbb{E}[({\rm dist}(\tilde{w}^{k},w^*))^2) \nonumber \\
&&- (\lambda \tau^2_1 - 2\alpha( \gamma\changeHK{_{c}} \lambda^2 \tau^2_1 - 14 \alpha \Lambda \changeHK{\Gamma^2_{c}} (\beta^2+\tau_2^2 C^2\theta^2))) \mathbb{E}[({\rm dist}(w_{\changeHS{m}}^k,w^*))^2 ]. \nonumber
\end{eqnarray}

\changeHS{Combining the relation $\Lambda \tau^2_2 > \lambda \tau^2_1$ and the assumption $\lambda \tau^2_1 > 2\alpha( \gamma\changeHK{_{c}} \lambda^2 \tau^2_1 - 14 \alpha) \Lambda \changeHK{\Gamma^2_{c}} (\beta^2+\tau_2^2 C^2\theta^2))$}, since $w_{0}^k= \tilde{w}^{k}$, we obtain
\begin{eqnarray}
\label{ApdEq:LocalRate_F_expect_10}
&&\frac{m}{4}2\alpha( \gamma\changeHK{_{c}} \lambda^2 \tau^2_1 - 14 \alpha \Lambda \changeHK{\Gamma^2_{c}} (\beta^2+\tau_2^2 C^2\theta^2)) 
\mathbb{E}[ ({\rm dist}(\tilde{w}^{k+1},w^*))^2] \nonumber \\ 
& \le & 
 (\Lambda \tau^2_2 + 16 m \alpha^2 \Lambda \changeHK{\Gamma^2_{c}}(\beta^2+\tau_2^2C^2\theta^2)\changeHS{)}   \mathbb{E}[({\rm dist}(\tilde{w}^{k},w^*))^2\changeHS{]}.\nonumber
\end{eqnarray}

Finally, we obtain
\begin{eqnarray}
\label{ApdEq:LocalRate_F_expect_11}
\mathbb{E}[ ({\rm dist}(\tilde{w}^{k+1},w^*))^2] \le  
\frac{2(\Lambda \tau^2_2 + 16 m \alpha^2 \Lambda \changeHK{\Gamma^2_{c}}(\beta^2+\tau_2^2C^2\theta^2)}
{m\alpha( \gamma\changeHK{_{c}} \lambda^2 \tau^2_1 - 14 \alpha \Lambda \changeHK{\Gamma^2_{c}} (\beta^2+\tau_2^2 C^2\theta^2)) }   \mathbb{E}[({\rm dist}(\tilde{w}^{k},w^*))^2).
\end{eqnarray}
This completes the proof. 
\end{proof}

\begin{Rmk}
From the proof of Lemma \ref{Lemma:y_y_s_y}, if we \changeHK{adopt} the parallel translation \changeHK{as} the vector transport, i.e., $\mathcal{T} = P$, the first two terms in (\ref{Eq:y_s_bound_remark}) are equal to zero, and $\changeHK{\Upsilon_{c}}$ in (\ref{Eq:y_y_s_y}) gets smaller than that of the case of vector transport. This leads to a smaller $\Gamma{\changeHK{_{c}}}$ and a larger $\gamma\changeHK{_{c}}$ in Proposition \ref{AppProposition:HessianOperatorBoundsConvex}. Then, the smaller $\Gamma{\changeHK{_{c}}}$ and the larger $\gamma\changeHK{_{c}}$ leads to a smaller coefficient in (\ref{ApdEq:LocalRate_F_expect_11}) of Theorem \ref{Thm:LocalConvergenceConvex}. Consequently, the parallel translation can result in a faster local convergence rate.
\end{Rmk}

We obtain the following corollary of the previous theorem with the case $R={\rm Exp}$ and $\mathcal{T} = P$.
\begin{Cor} 
\label{Appd_Cor:}
Consider Algorithm 1 with $\mathcal{T} = P$ and $R={\rm Exp}$, i.e., the parallel translation and the exponential mapping case. Let $\mathcal{M}$ be a Riemannian manifold and $w^* \in \mathcal{M}$ be a non-degenerate local minimizer of $f$ (i.e., ${\rm grad} f(w^*)=0$ and the Hessian ${\rm Hess}f(w^*)$ of $f$ at $w^*$ is positive definite). Suppose \changeHK{Assumptions \ref{Assump:1} and \ref{Assump:3}} hold.
Let the constants $\theta$, and $C$ in Lemma \ref{AppenLem:UpperBoundVariance}. 
$\beta_0$ is the constant in Corollary \ref{AppenCor:UpperBoundVariance}.
\changeHK{$\Lambda$ and $\lambda$ are the constants in Lemmas \ref{Lemma:HessianUpperBoundFunc} and \ref{Lemma:HessianBoundFunc}, respectively}.
$\gamma\changeHK{_{c}}$ and $\Gamma{\changeHK{_{c}}}$ are the constants in Proposition \ref{AppProposition:HessianOperatorBoundsConvex}.
Let $\alpha$ be a positive number satisfying $\lambda > 2\alpha(\gamma\changeHK{_{c}}\lambda^2 - 7 \alpha \Lambda \changeHK{\Gamma^2_{c}} \beta_0^2)$ \changeHS{and $\gamma\changeHK{_{c}} \lambda^2 \tau^2_1 > 14 \alpha \Lambda \changeHK{\Gamma^2_{c}} (\beta^2+\tau_2^2 C^2\theta^2)$}. 
It then follows that for any sequence $\{\tilde{w}^k\}$ generated by Algorithm 1 with \changeHK{Option II under} a fixed step-size $\alpha_t^k:=\alpha$ and $m_k:=m$ converging to $w^*$, there exists $\changeHK{0<K_{th}<K}$ such that for all $k>\changeHK{K_{th}}$,
\begin{eqnarray}
\mathbb{E}[ ({\rm dist}(\tilde{w}^{k+1},w^*))^2] & \le & 
\frac{2(\Lambda  + 8 m \alpha^2 \Lambda   \changeHK{\Gamma^2_{c}} \beta_0^2)}
{m\alpha(\gamma\changeHK{_{c}} \lambda^2  - 7 \alpha \Lambda  \changeHK{\Gamma^2_{c}} \beta_0^2) }  \mathbb{E}[({\rm dist}(\tilde{w}^{k},w^*))^2) 
\end{eqnarray}
\end{Cor}
\begin{proof}
The proof is given similarly \changeHK{to} Theorem \ref{Thm:LocalConvergenceConvex}. We use Corollary \ref{AppenCor:UpperBoundVariance}, and also set as $\theta=0$ in Lemma \ref{Lemma:VecParaDiff}, and as $\tau_1=\tau_2=1$ in Lemma \ref{Lemma:retraction_dist}. 
\end{proof}

\clearpage	
\section{Additional numerical experiments}

\changeHS{In this section, we show additional numerical experiments which do not appear in the main text.}

\subsection{Matrix completion problem on synthetic datasets}

\subsubsection{Additional results}

This section shows the results of six problem instances. Due to the page limitations, we only show the loss on a test set $\Phi$, which is different from the training set $\Omega$. The loss on the test set demonstrates the convergence speed to a good prediction accuracy of missing entries. 

\noindent
{\bf Case MC-S1:} We first show the results of the comparison when the number of samples $N=5000$, the dimension $d=200$, the memory size $L=10$, the oversampling ratio (OS) is $8$, and the condition number (CN) is $50$. We also add Gaussian noise $\sigma=10^{-10}$. 
Figures \ref{Addfig:MC_Synthetic_MC_S1} show the results of 4 runs except the result shown in the main text, which corresponds to "run 1." They show superior performances than other algorithms. 

\noindent
{\bf Case MC-S2: influence on low sampling.} We look into problem instances from scarcely sampled data, e.g. OS is $4$. Other conditions are the same as {\bf Case MC-S1}. 
From Figures \ref{Addfig:MC_Synthetic_MC_S2}, we can find that the proposed algorithm gives much better and stabler performances against other algorithms. 

\noindent
{\bf Case MC-S3: influence on ill-conditioning.} We consider the problem instances with higher condition number (CN) $100$. Other conditions are the same as {\bf Case MC-S1}.  
Figures \ref{Addfig:MC_Synthetic_MC_S3} show the superior performances of the proposed algorithm against other algorithms. 

\noindent
{\bf Case MC-S4: influence on higher noise.} We consider noisy problem instances, where $\sigma=10^{-6}$. Other conditions are the same as {\bf Case MC-S1}.  
Figures \ref{Addfig:MC_Synthetic_MC_S4} show that the convergent MSE values are much higher than the other cases. Then, we can see the superior performance of the proposed R-SQN-VR against other algorithms. 

\noindent
{\bf Case MC-S5: influence on higher rank.} We consider problem instances with higher rank, where $r=10$. Other conditions are the same as {\bf Case MC-S1}.  
From Figures \ref{Addfig:MC_Synthetic_MC_S5}, the proposed R-SQN-VR still shows the superior performances against other algorithms. Grouse indicates the faster decrease of the MSE at the begging of the iterations. However, the convergent MSE values are much higher than those of others.

\begin{figure*}[htbp]
\begin{center}
\begin{minipage}{0.490\hsize}
\begin{center}
\includegraphics[width=\hsize]{results_pdf/mc/synthetic/r_5_m_10_os_8_cn_50_n_10/mc_test_MSE_N5000_d200_r5_run_1.pdf}\\
\vspace*{-2.0cm}
{(a) run 2}
\end{center}
\end{minipage}
%\vspace*{0.1cm}
%%%%%%
\begin{minipage}{0.490\hsize}
\begin{center}
\includegraphics[width=\hsize]{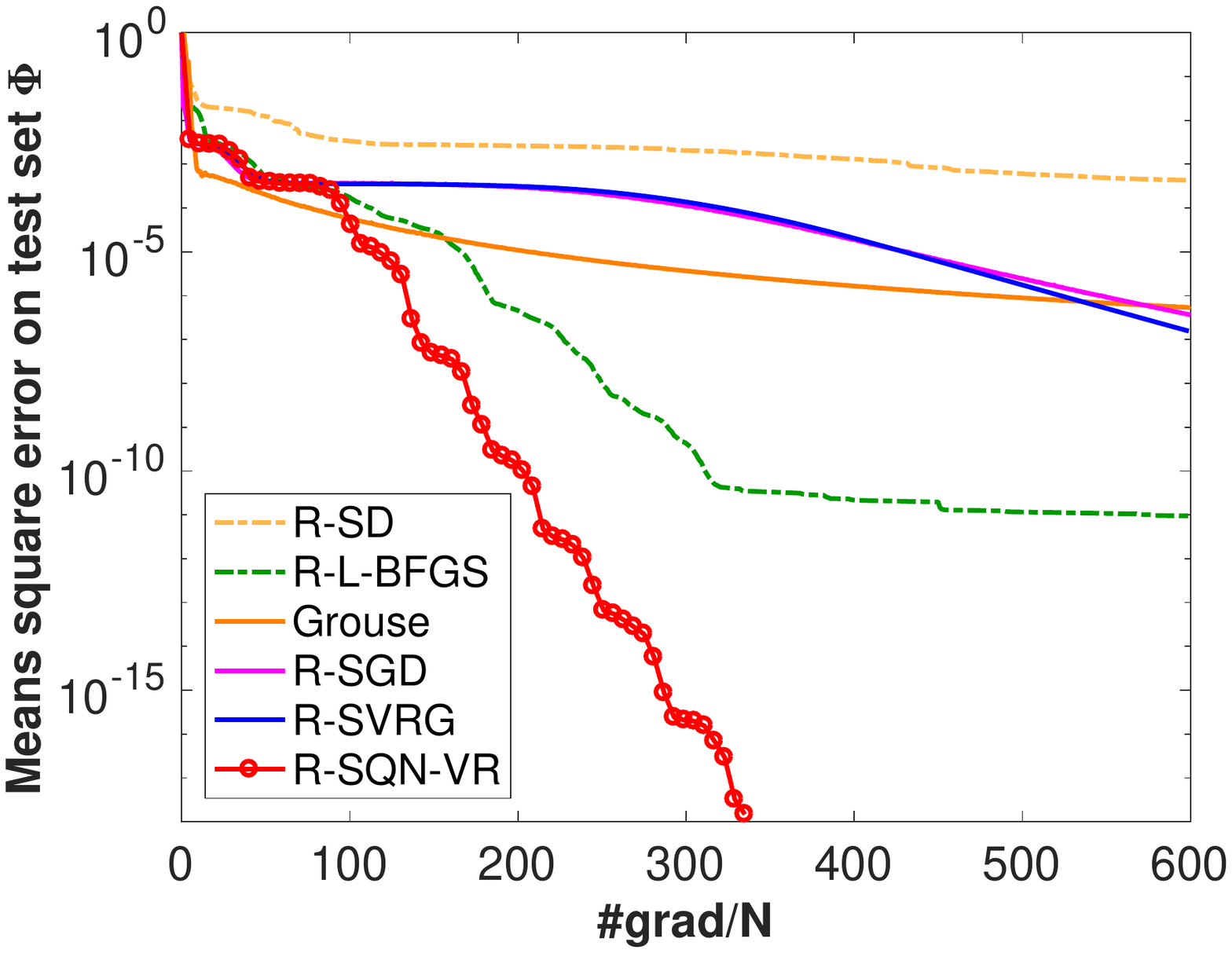}\\
\vspace*{-2.0cm}
{ (b) run 3}
\end{center}
\end{minipage}\\
\vspace*{-1.5cm}

\begin{minipage}{0.490\hsize}
\begin{center}
\includegraphics[width=\hsize]{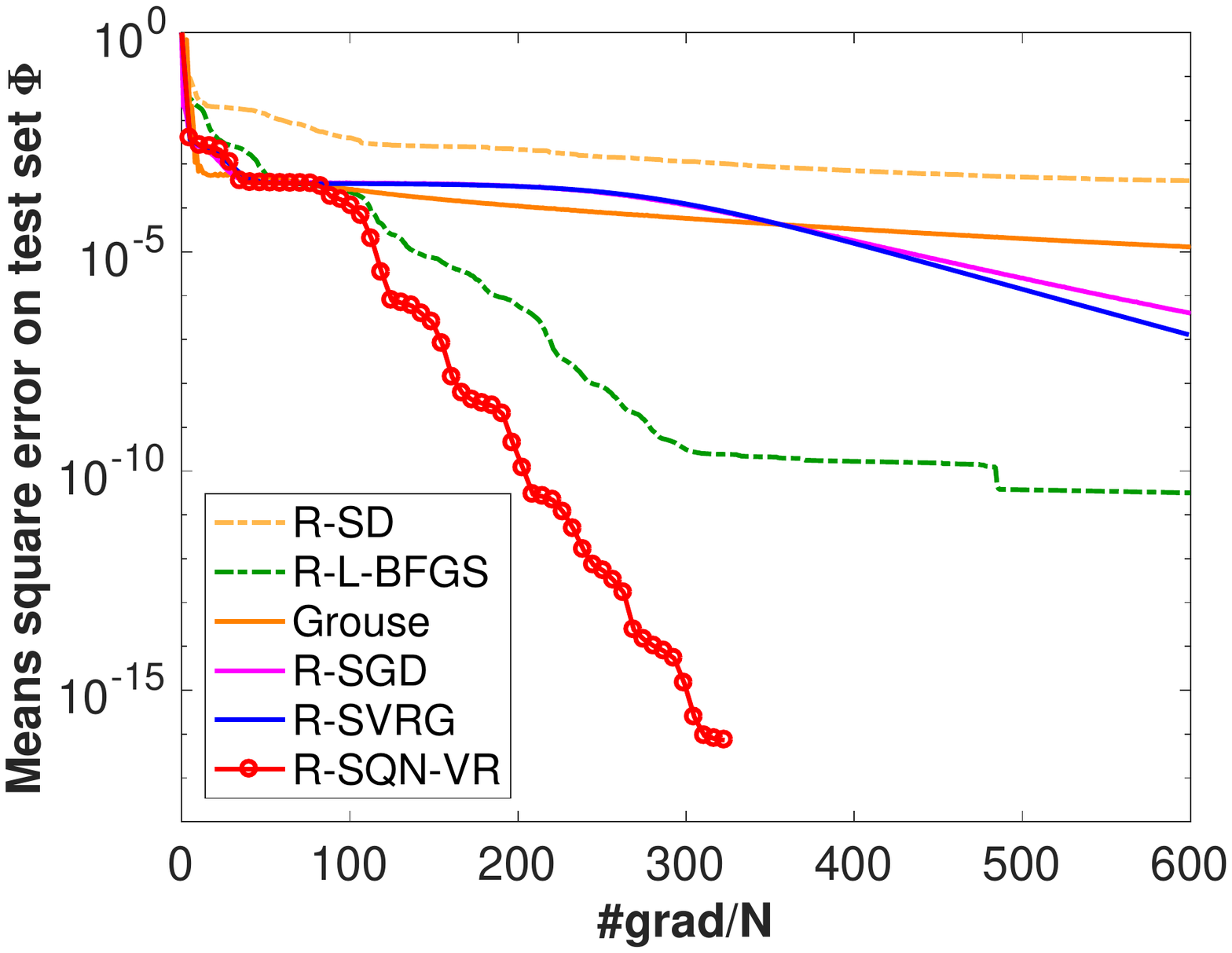}\\
\vspace*{-2.0cm}
{(c) run 4}
\end{center}
\end{minipage}
\vspace*{0.3cm}
\begin{minipage}{0.490\hsize}
\begin{center}
\includegraphics[width=\hsize]{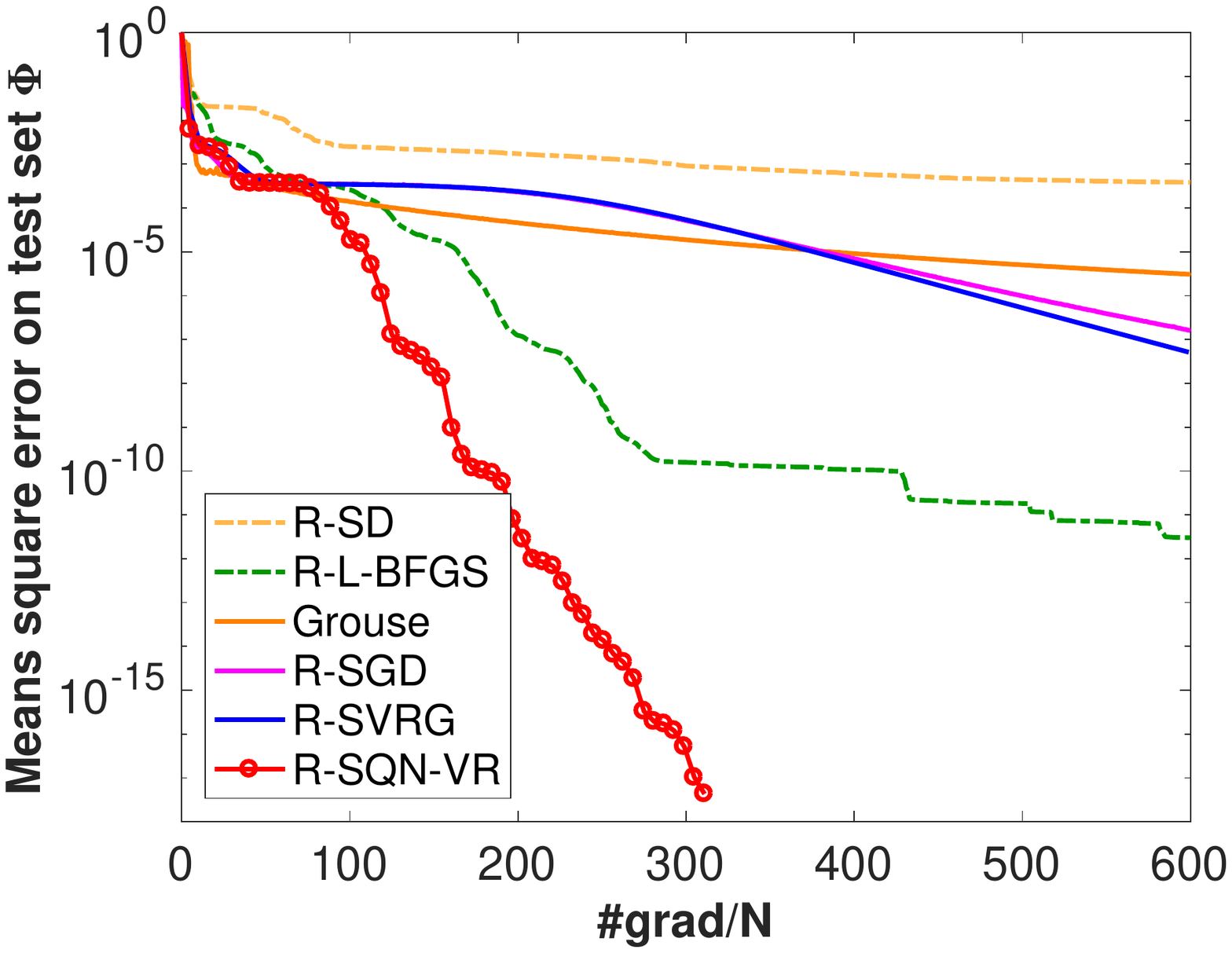}\\
\vspace*{-2.0cm}
{(d) run 5}
\end{center}
\end{minipage}
\caption{Performance evaluations on low-rank MC problem ({\bf Case MC-S1: baseline.}).}
\label{Addfig:MC_Synthetic_MC_S1}
\end{center}
\end{figure*}

\begin{figure*}[htbp]
\begin{center}
\begin{minipage}{0.490\hsize}
\begin{center}
\includegraphics[width=\hsize]{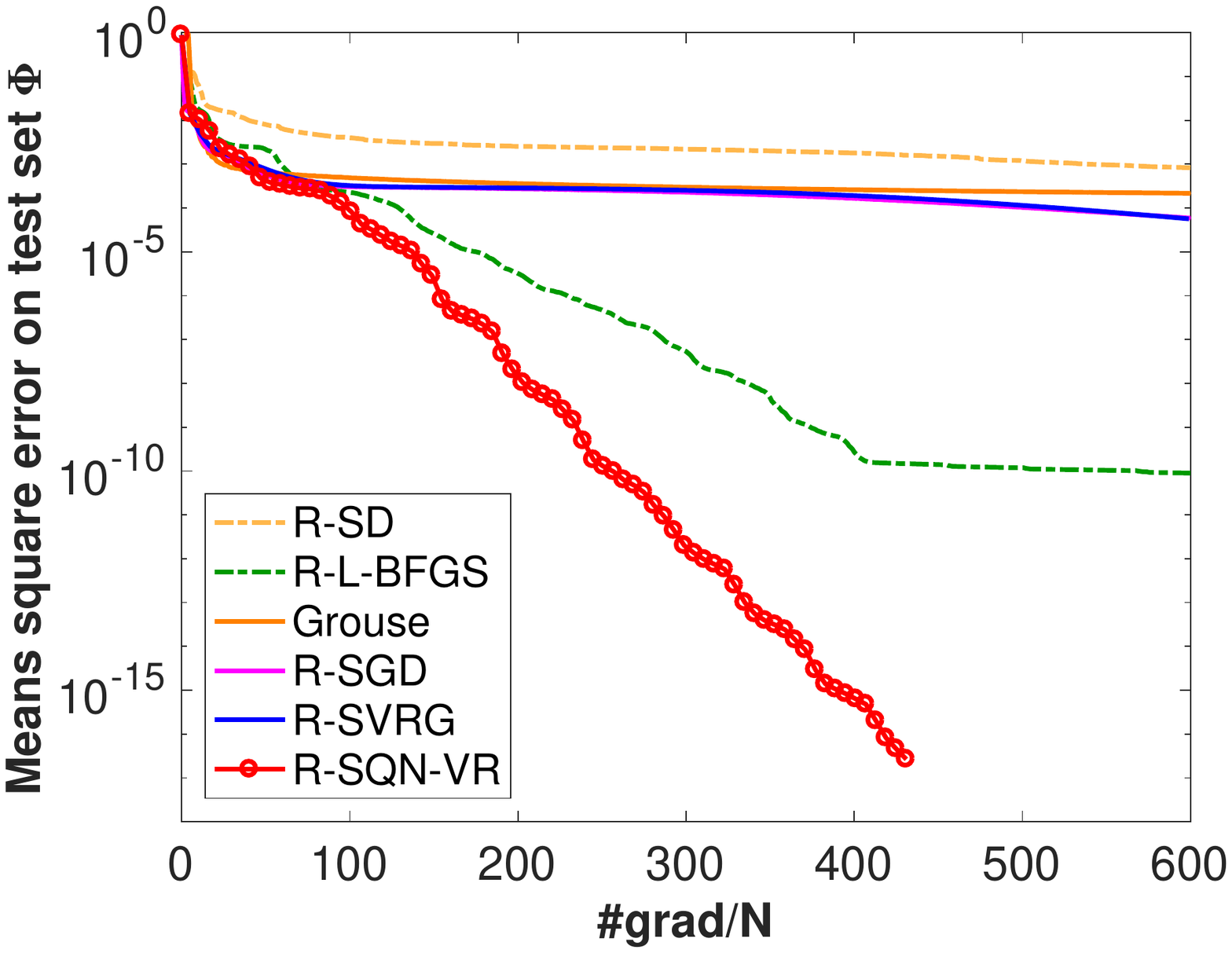}\\
\vspace*{-2.0cm}
{run 2}
\end{center}
\end{minipage}
%\vspace*{0.1cm}
%%%%%%
\begin{minipage}{0.490\hsize}
\begin{center}
\includegraphics[width=\hsize]{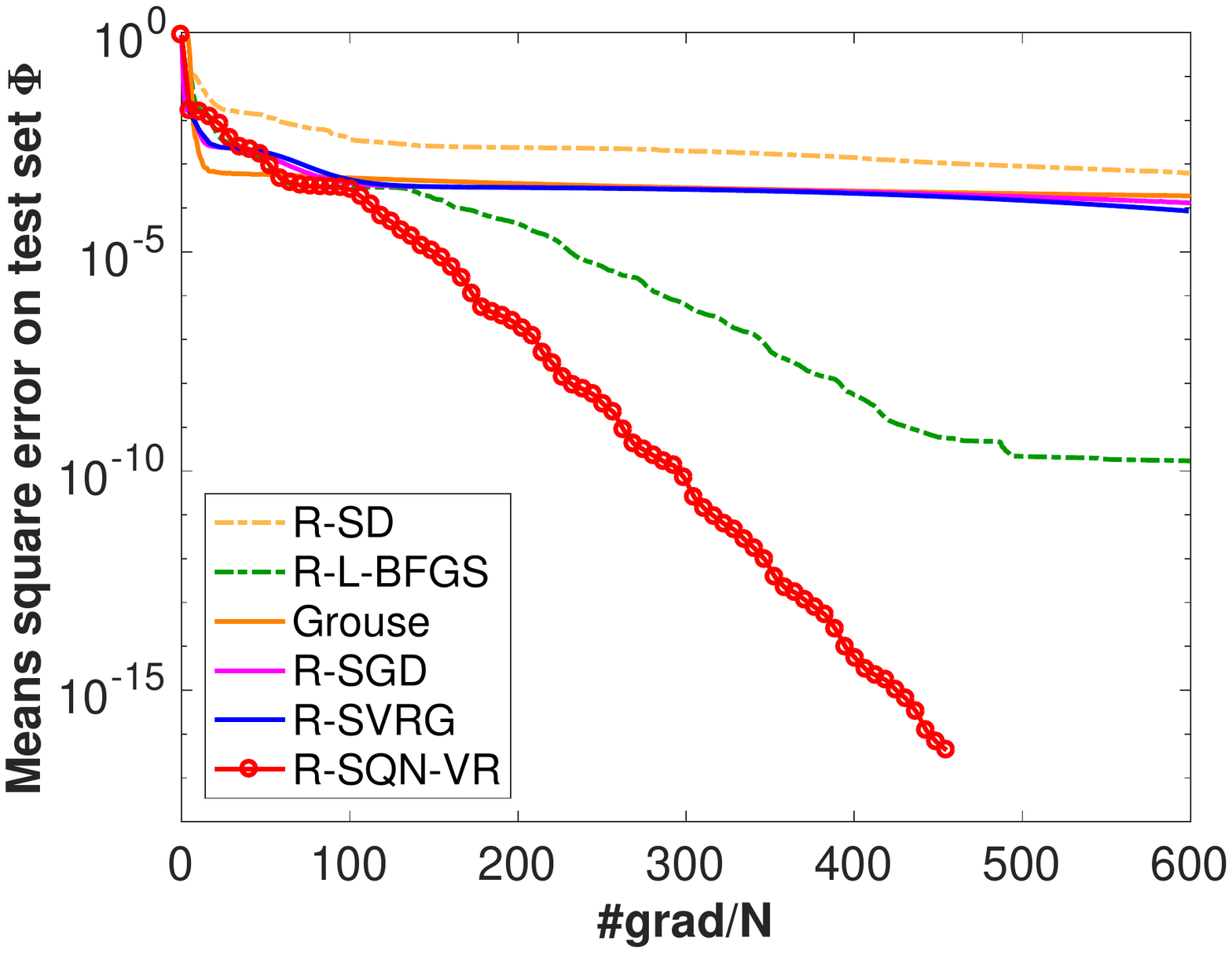}\\
\vspace*{-2.0cm}
{(b) run 3}
\end{center}
\end{minipage}\\
\vspace*{-1.5cm}

\begin{minipage}{0.490\hsize}
\begin{center}
\includegraphics[width=\hsize]{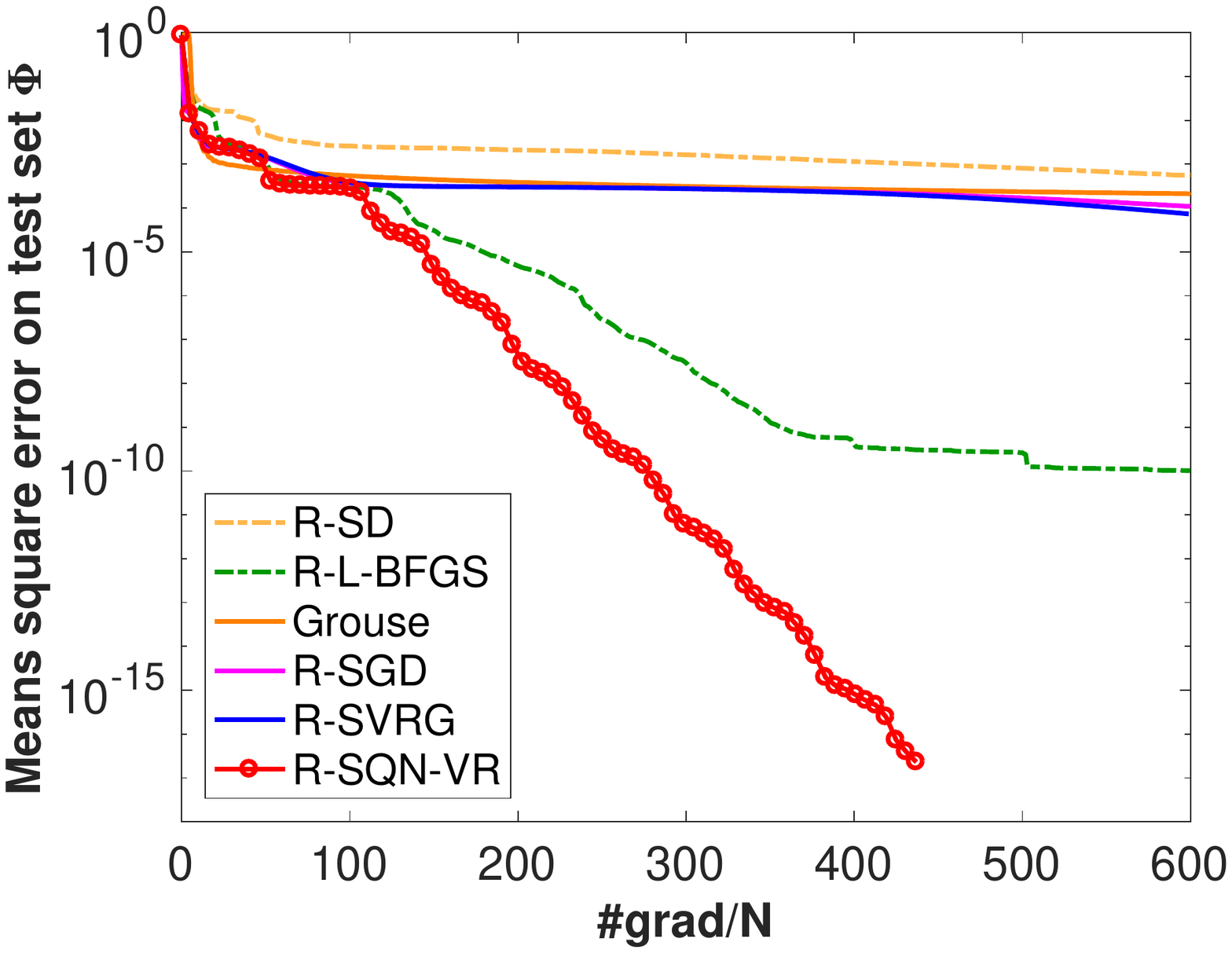}\\
\vspace*{-2.0cm}
{ (c) run 4}
\end{center}
\end{minipage}
\vspace*{0.3cm}
\begin{minipage}{0.490\hsize}
\begin{center}
\includegraphics[width=\hsize]{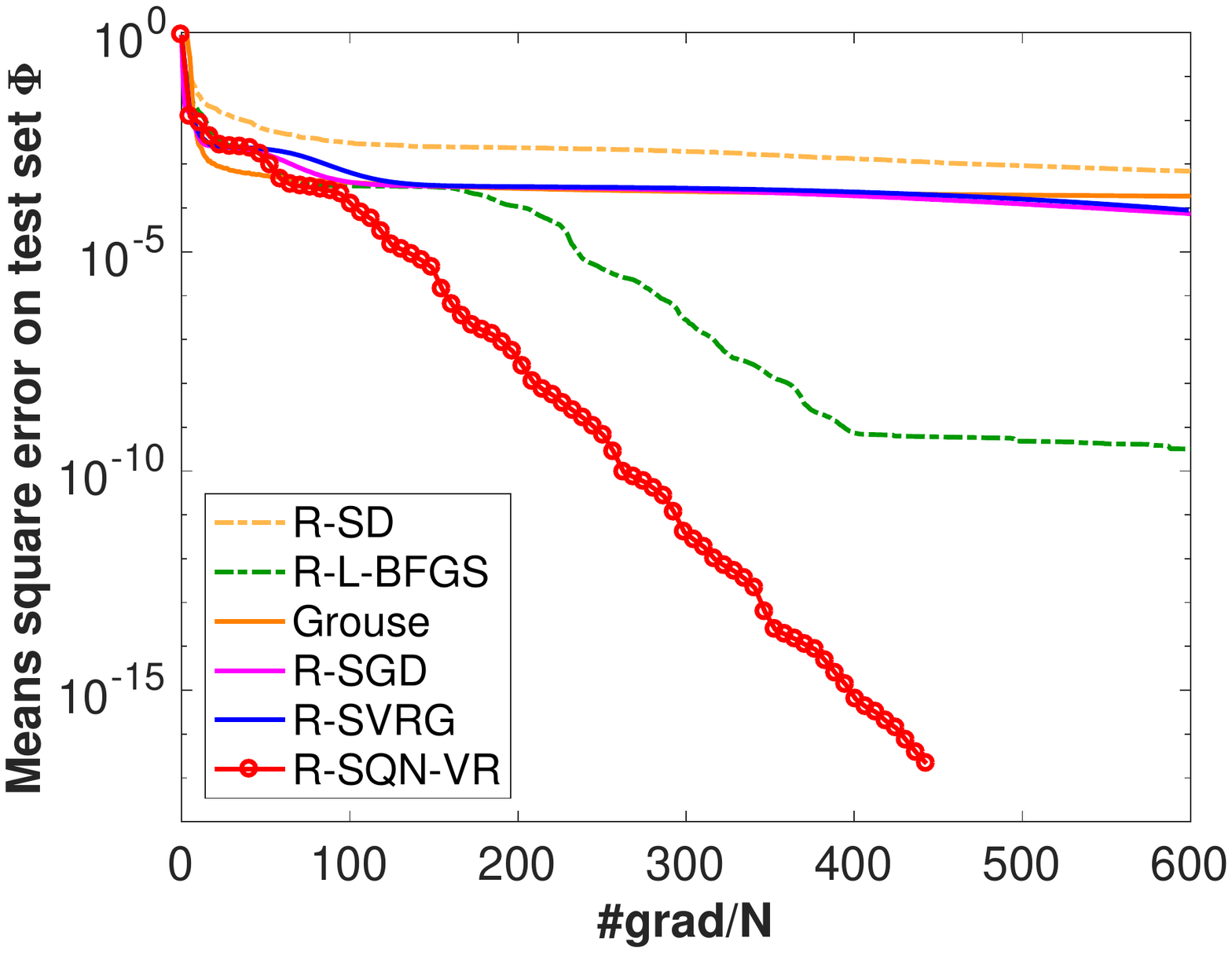}\\
\vspace*{-2.0cm}
{(d) run 5}
\end{center}
\end{minipage}
\caption{Performance evaluations on low-rank MC problem ({\bf Case MC-S2: low sampling.}).}
\label{Addfig:MC_Synthetic_MC_S2}
\end{center}
\end{figure*}

\begin{figure*}[htbp]
\begin{center}
\begin{minipage}{0.490\hsize}
\begin{center}
\includegraphics[width=\hsize]{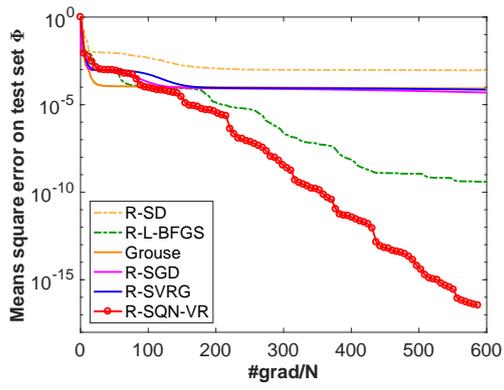}\\
\vspace*{-2.0cm}
{(a) run 2}
\end{center}
\end{minipage}
%\vspace*{0.1cm}
%%%%%%
\begin{minipage}{0.490\hsize}
\begin{center}
\includegraphics[width=\hsize]{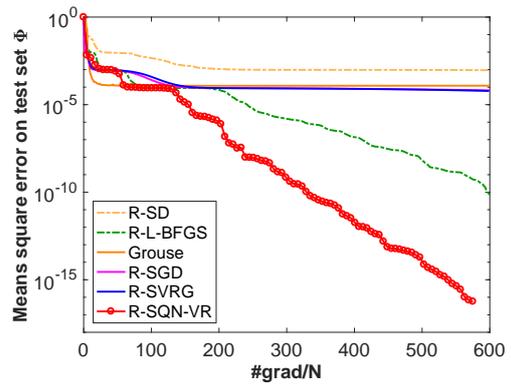}\\
\vspace*{-2.0cm}
{(b) run 3}
\end{center}
\end{minipage}\\
\vspace*{-1.5cm}

\begin{minipage}{0.490\hsize}
\begin{center}
\includegraphics[width=\hsize]{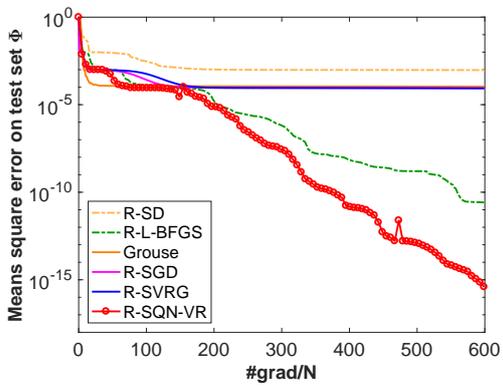}\\
\vspace*{-2.0cm}
{(c) run 4}
\end{center}
\end{minipage}
\vspace*{0.3cm}
\begin{minipage}{0.490\hsize}
\begin{center}
\includegraphics[width=\hsize]{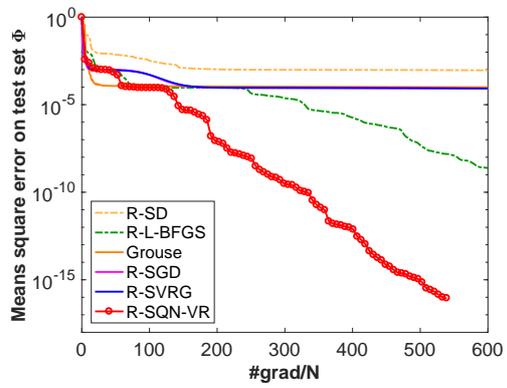}\\
\vspace*{-2.0cm}
{(d) run 5}
\end{center}
\end{minipage}
\caption{Performance evaluations on low-rank MC problem ({\bf Case MC-S3: ill-conditioning.}).}
\label{Addfig:MC_Synthetic_MC_S3}
\end{center}
\end{figure*}

\begin{figure*}[htbp]
\begin{center}
\begin{minipage}{0.490\hsize}
\begin{center}
\includegraphics[width=\hsize]{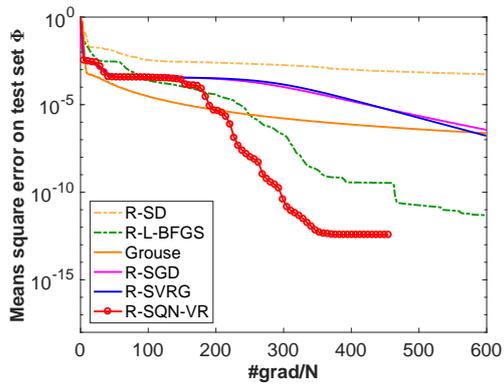}\\
\vspace*{-2.0cm}
{(a) run 2}
\end{center}
\end{minipage}
%\vspace*{0.1cm}
%%%%%%
\begin{minipage}{0.490\hsize}
\begin{center}
\includegraphics[width=\hsize]{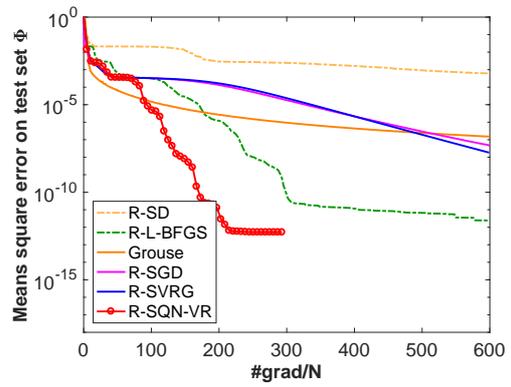}\\
\vspace*{-2.0cm}
{(b) run 3}
\end{center}
\end{minipage}\\
\vspace*{-1.5cm}

\begin{minipage}{0.490\hsize}
\begin{center}
\includegraphics[width=\hsize]{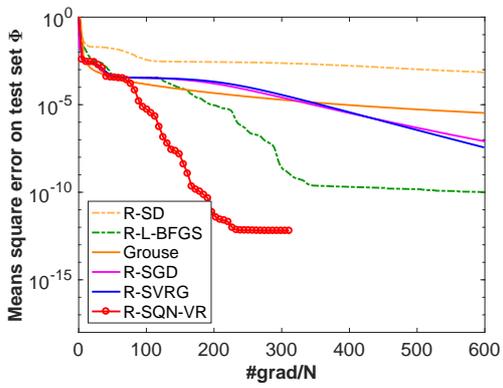}\\
\vspace*{-2.0cm}
{(c) run 4}
\end{center}
\end{minipage}
\vspace*{0.3cm}
\begin{minipage}{0.490\hsize}
\begin{center}
\includegraphics[width=\hsize]{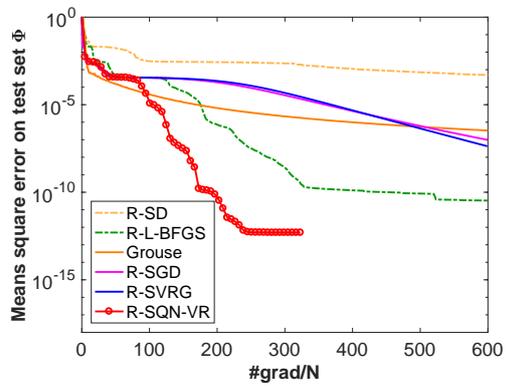}\\
\vspace*{-2.0cm}
{(d) run 5}
\end{center}
\end{minipage}
\caption{Performance evaluations on low-rank MC problem ({\bf Case MC-S4: noisy data.}).}
\label{Addfig:MC_Synthetic_MC_S4}
\end{center}
\end{figure*}

\begin{figure*}[htbp]
\begin{center}
\begin{minipage}{0.490\hsize}
\begin{center}
\includegraphics[width=\hsize]{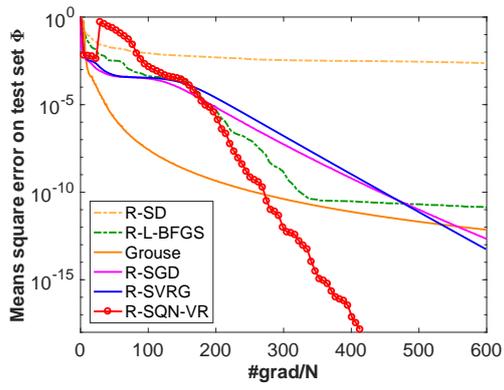}\\
\vspace*{-2.0cm}
{(a) run 2}
\end{center}
\end{minipage}
%\vspace*{0.1cm}
%%%%%%
\begin{minipage}{0.490\hsize}
\begin{center}
\includegraphics[width=\hsize]{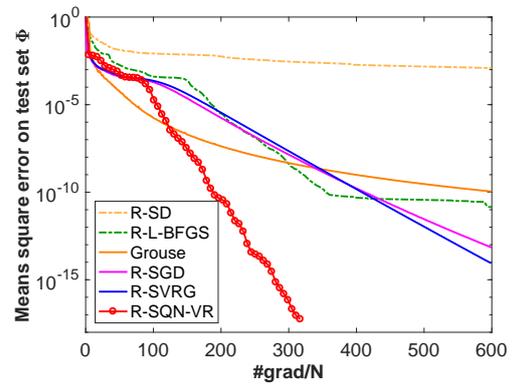}\\
\vspace*{-2.0cm}
{(b) run 3}
\end{center}
\end{minipage}\\
\vspace*{-1.5cm}

\begin{minipage}{0.490\hsize}
\begin{center}
\includegraphics[width=\hsize]{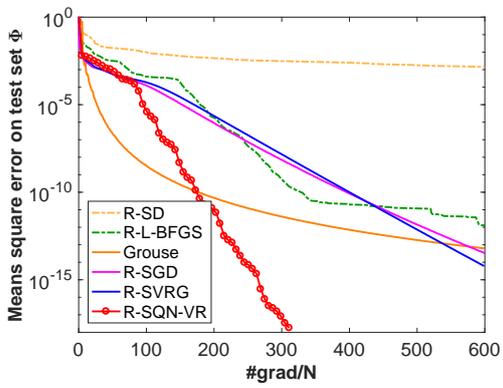}\\
\vspace*{-2.0cm}
{(c) run 4}
\end{center}
\end{minipage}
\vspace*{0.3cm}
\begin{minipage}{0.490\hsize}
\begin{center}
\includegraphics[width=\hsize]{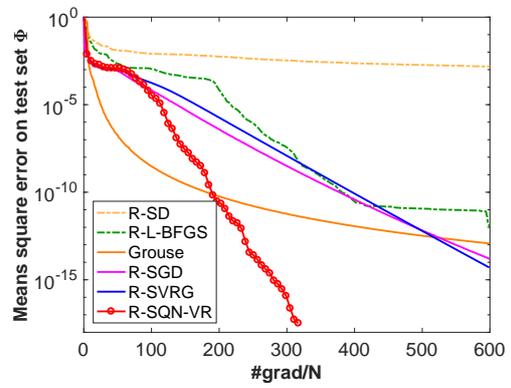}\\
\vspace*{-2.0cm}
{(d) run 5}
\end{center}
\end{minipage}
\caption{Performance evaluations on low-rank MC problem ({\bf Case MC-S5: higher rank.}).}
\label{Addfig:MC_Synthetic_MC_S5}
\end{center}
\end{figure*}

\clearpage
\subsubsection{Processing time experiments}

The results in terms of the processing time is presented. 

\changeHKK{
{\bf Case MC-S7: Comparison in terms of processing time.} Because one major concern of second-order algorithms is, in general, higher computational processing load than first-order algorithms, we  additionally show the results in terms of the processing time. This evaluation addresses only R-SGD, R-SVRG and R-SQN-VR because the code structures of them are similar whereas the batch-based algorithms, i.e., R-SD and R-L-BFGS, have completely different implementations.  Figures \ref{Addfig:MC_Synthetic_time} (a)-(e) show the results of the relationship between test MSE and the processing time [sec]. 
From the figures, as expected, R-SGD gains much faster speed in comparison with the results in terms of iteration than other algorithms. However, it should be noted that R-SGD suffers from the problem that it heavily decreases the convergence speed around the solution as reported in the literature. Comparing R-SQN-VR with R-SVRG, R-SQN-VR still gives better performance although R-SQN-VR requires one more additional vector transport of a gradient in each inner iteration and $L$ vector transports of the curvature pairs at every outer epoch than R-SVRG does. 
Overall, R-SQN-VR outperforms R-SGD and R-SVRG in terms of the processing time. Consequently, we also have confirmed the effectiveness of the proposed R-SQN-VR from the viewpoint of processing time. 
}

\begin{figure*}[htbp]
\begin{center}
\begin{minipage}{0.320\hsize}
\begin{center}
\includegraphics[width=\hsize]{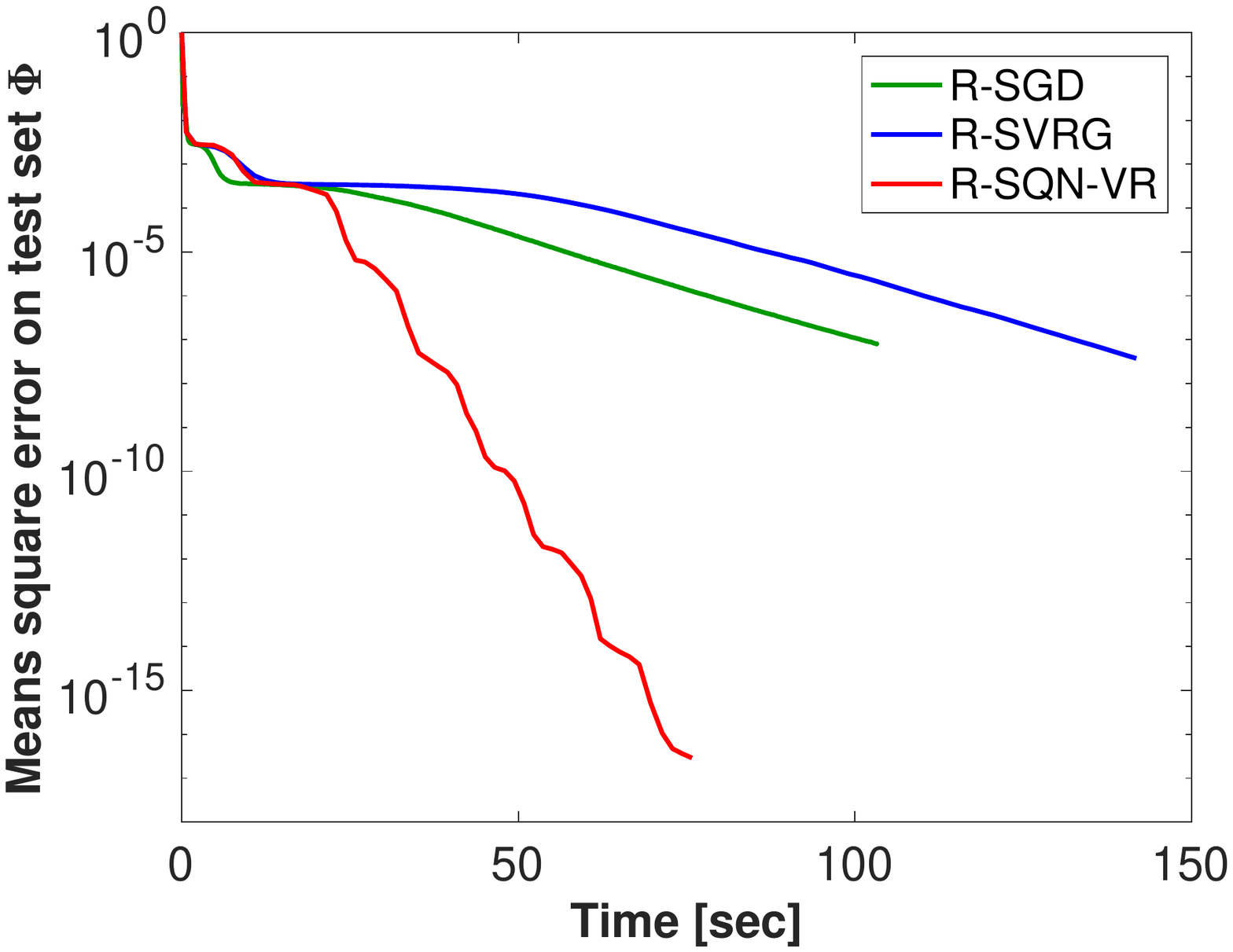}\\
\vspace*{-1.5cm}
{\small (a) {\bf Case MC-S1: \\baseline.}}
\end{center}
\end{minipage}
%\vspace*{0.1cm}
%%%%%%
\begin{minipage}{0.320\hsize}
\begin{center}
\includegraphics[width=\hsize]{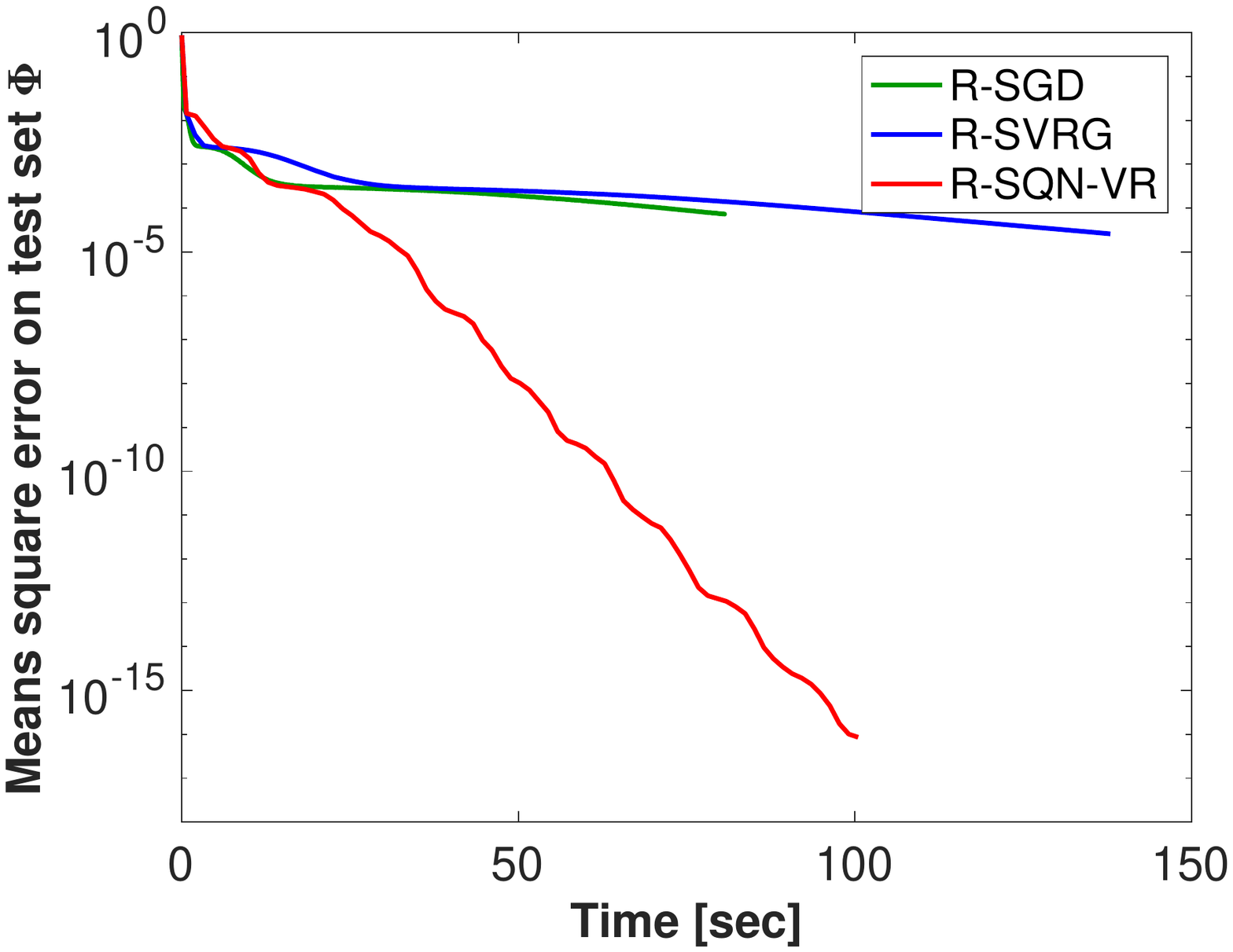}\\
\vspace*{-1.5cm}
{\small (b) {\bf Case MC-S2: \\low sampling.}}
\end{center}
\end{minipage}
%\vspace*{0.1cm}
\begin{minipage}{0.320\hsize}
\begin{center}
\includegraphics[width=\hsize]{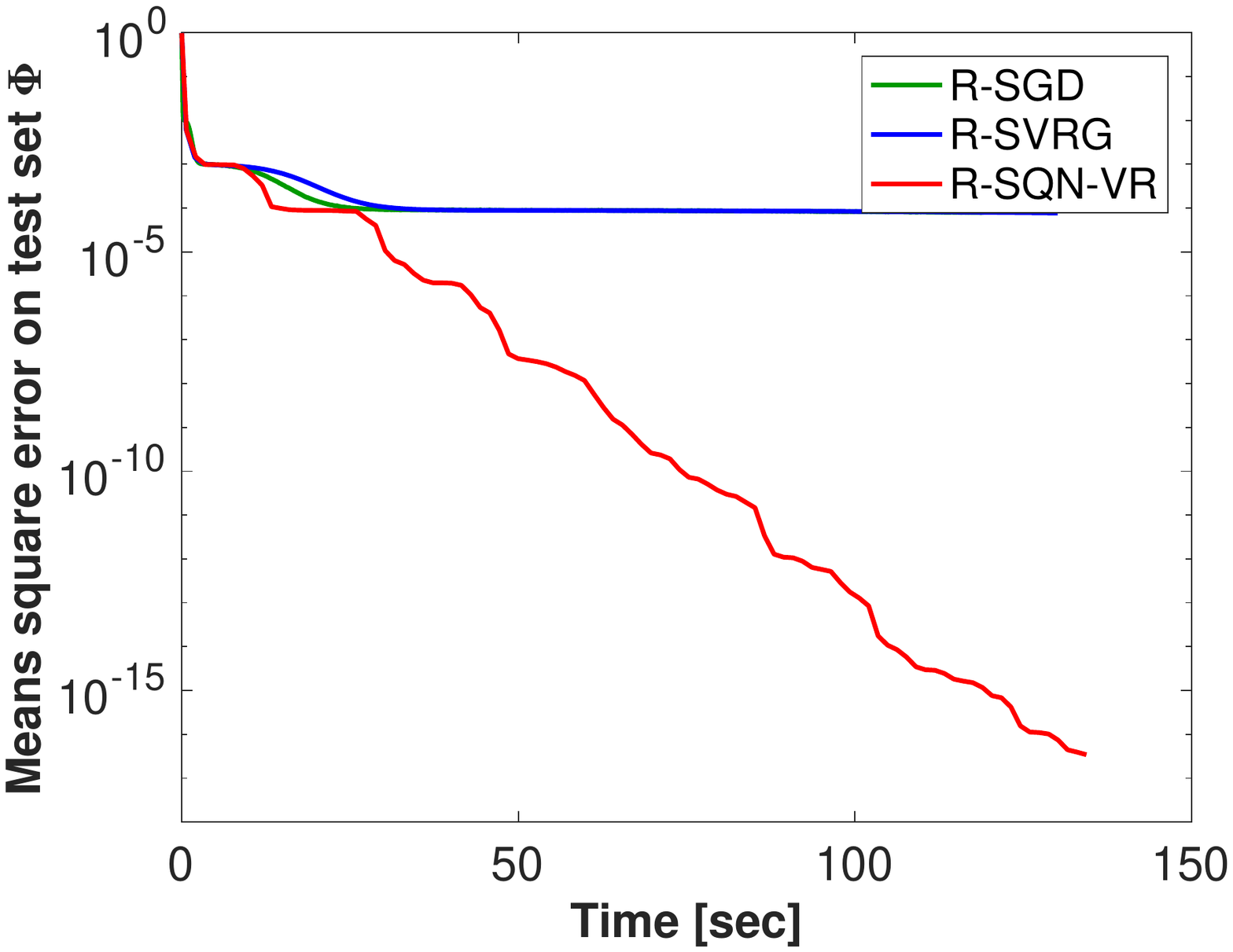}\\
\vspace*{-1.5cm}
{\small(c) {\bf Case MC-S3: \\ill-conditioning.}}
\end{center}
\end{minipage}\\
\vspace*{-0.5cm}
\begin{minipage}{0.320\hsize}
\begin{center}
\includegraphics[width=\hsize]{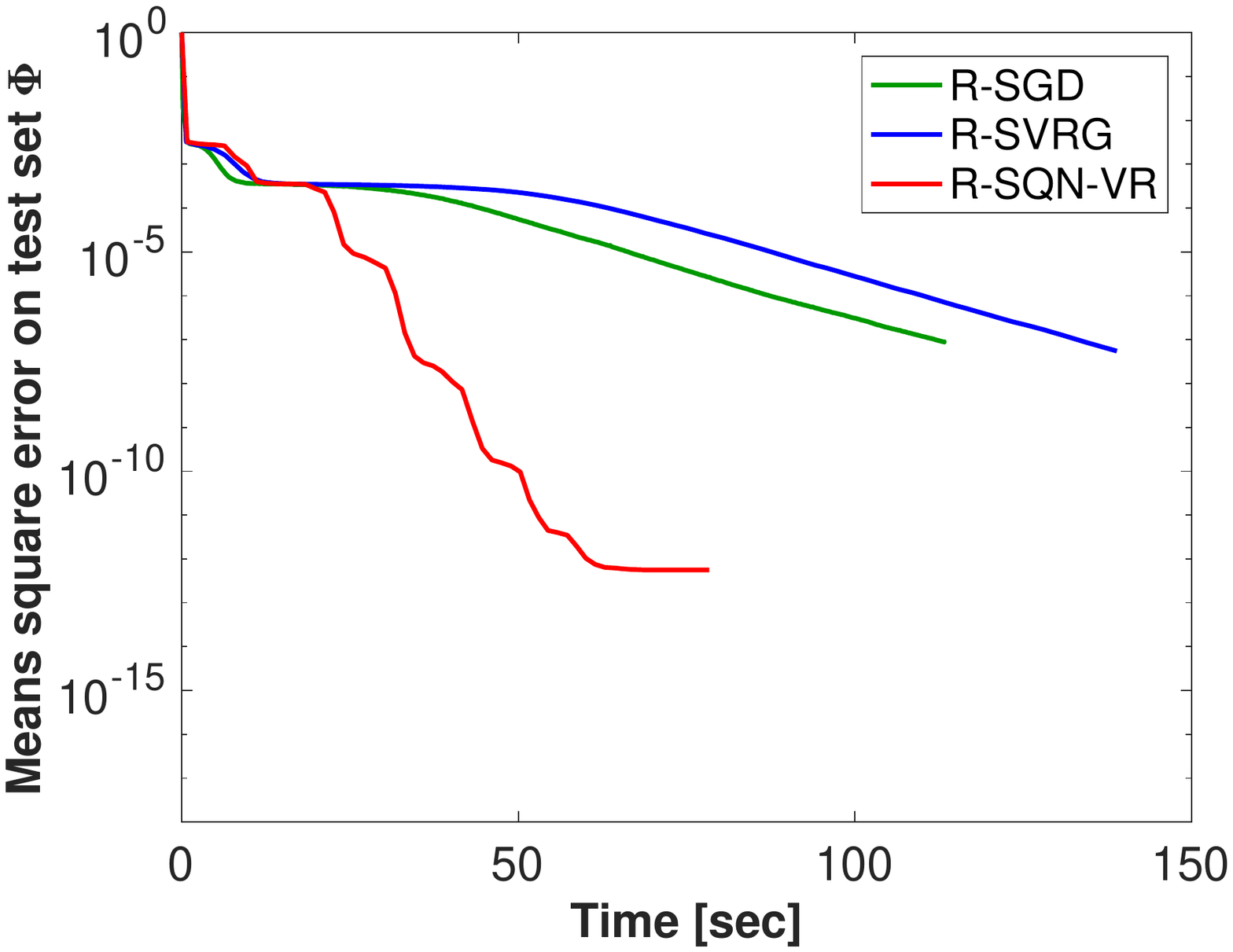}\\
\vspace*{-1.5cm}
{\small (d) {\bf Case MC-S4: \\noisy data.}}
\end{center}
\end{minipage}
%\vspace*{0.1cm}
%%%%%%%%%%
\begin{minipage}{0.320\hsize}
\begin{center}
\includegraphics[width=\hsize]{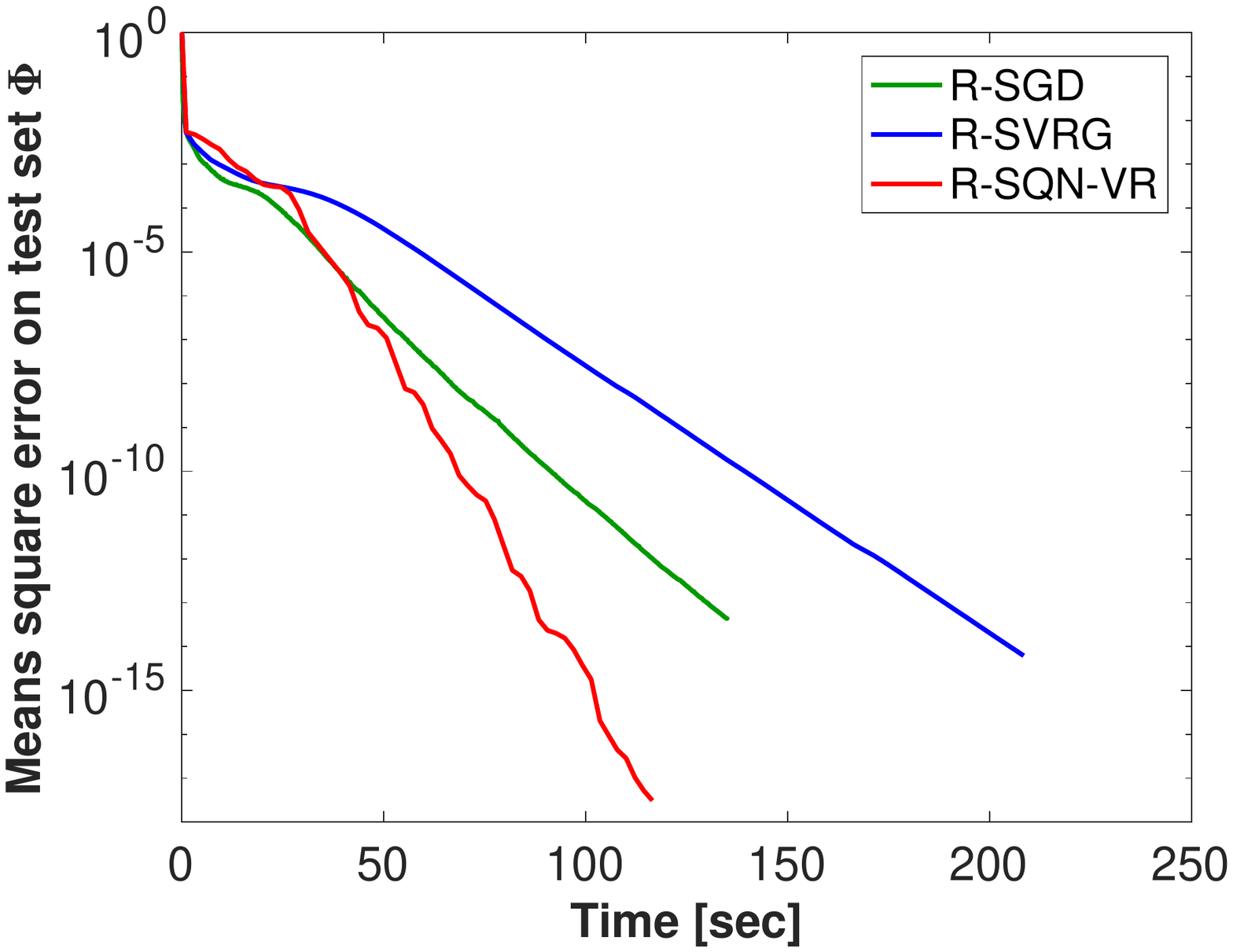}\\
\vspace*{-1.5cm}
{\small (e) {\bf Case MC-S5: \\higher rank.}}
\end{center}
\end{minipage}
\caption{Performance evaluations on low-rank MC problem ({\bf Case MC-S7}).}
\label{Addfig:MC_Synthetic_time}
\end{center}
\end{figure*}

\clearpage
Finally, Figure \ref{Addfig:MC_Synthetic_time_diff_L} shows the results when the memory size of $L$ is changed in R-SQN-VR. Comparing the results with Figure \ref{fig:PerformanceEvalatons} (h), the lower size cases improved their results very slightly, but we do not observe a big advantage of lower memory sizes in terms of processing load. From these results of both the convergence speed and the processing load, we cannot conclude which size of L is the best. This should be left to a future research topic. 
\begin{figure}[htbp]
\vspace*{-3.5cm}
\begin{center}
\includegraphics[width=0.75\hsize]{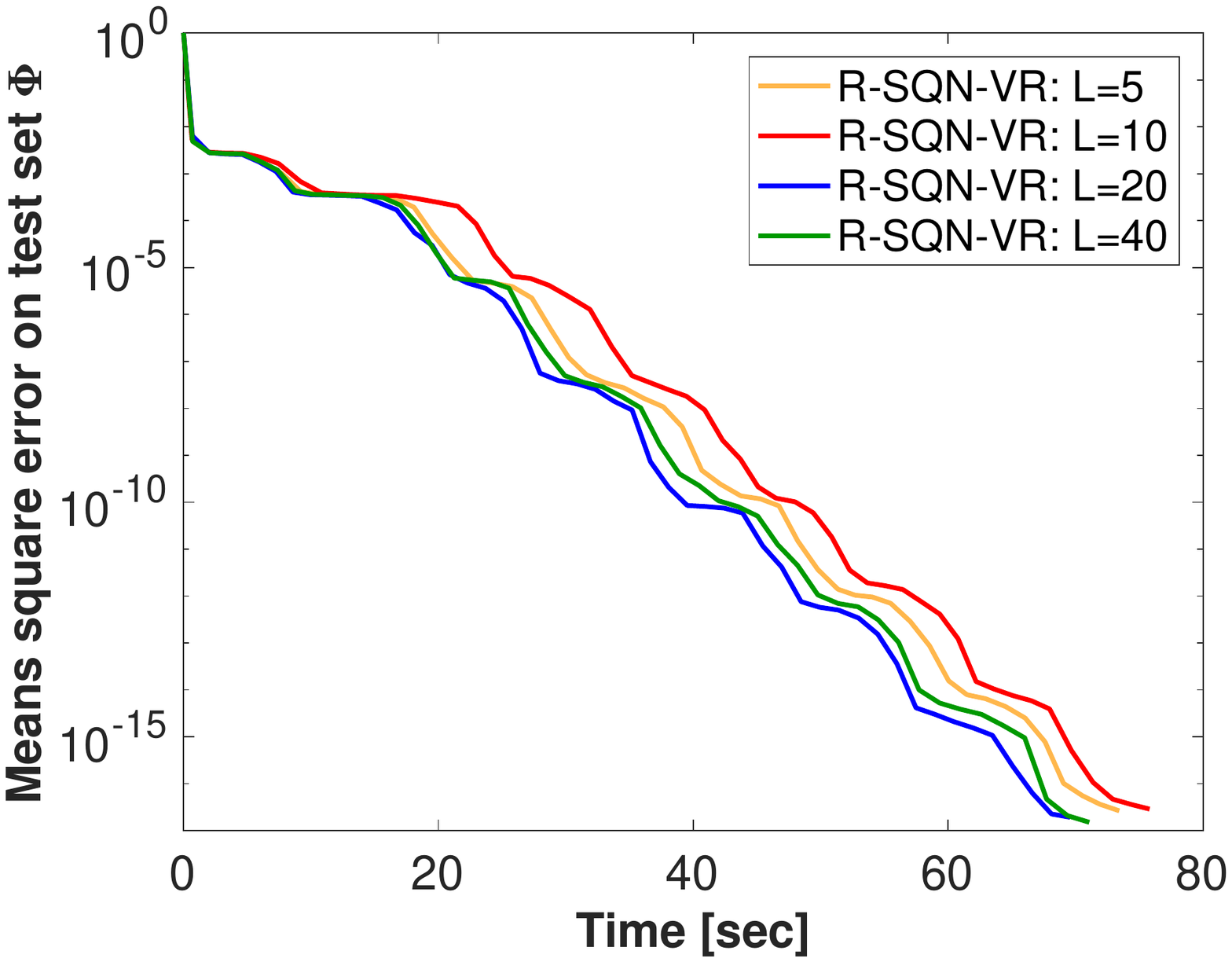}\\
\vspace*{-3.5cm}
\caption{Performance evaluations on low-rank MC problem (processing time) ({\bf Case MC-S6: different memory sizes}).}
\label{Addfig:MC_Synthetic_time_diff_L}
\end{center}
\end{figure}

%\clearpage

\subsection{Matrix completion problem on MovieLens 1M dataset}

Figures \ref{Addfig:MC_MovieLens_rank10} and \ref{Addfig:MC_MovieLens_rank20} show the results of the cases of $r=10$ ({\bf MC-R1: lower rank}) and $r=20$ ({\bf MC-R2: higher rank}). They show the convergence plots of the training error on $\Omega$ and the test error on $\Phi$ for all the five runs when rank $r=10$ and $r=20$, respectively. They show that the proposed R-SQN-VR give good performances on other algorithms in all runs.

\begin{figure*}[htbp]
\vspace*{-1.5cm}
\begin{center}
\begin{minipage}{0.320\hsize}
\begin{center}
\includegraphics[width=1.1\hsize]{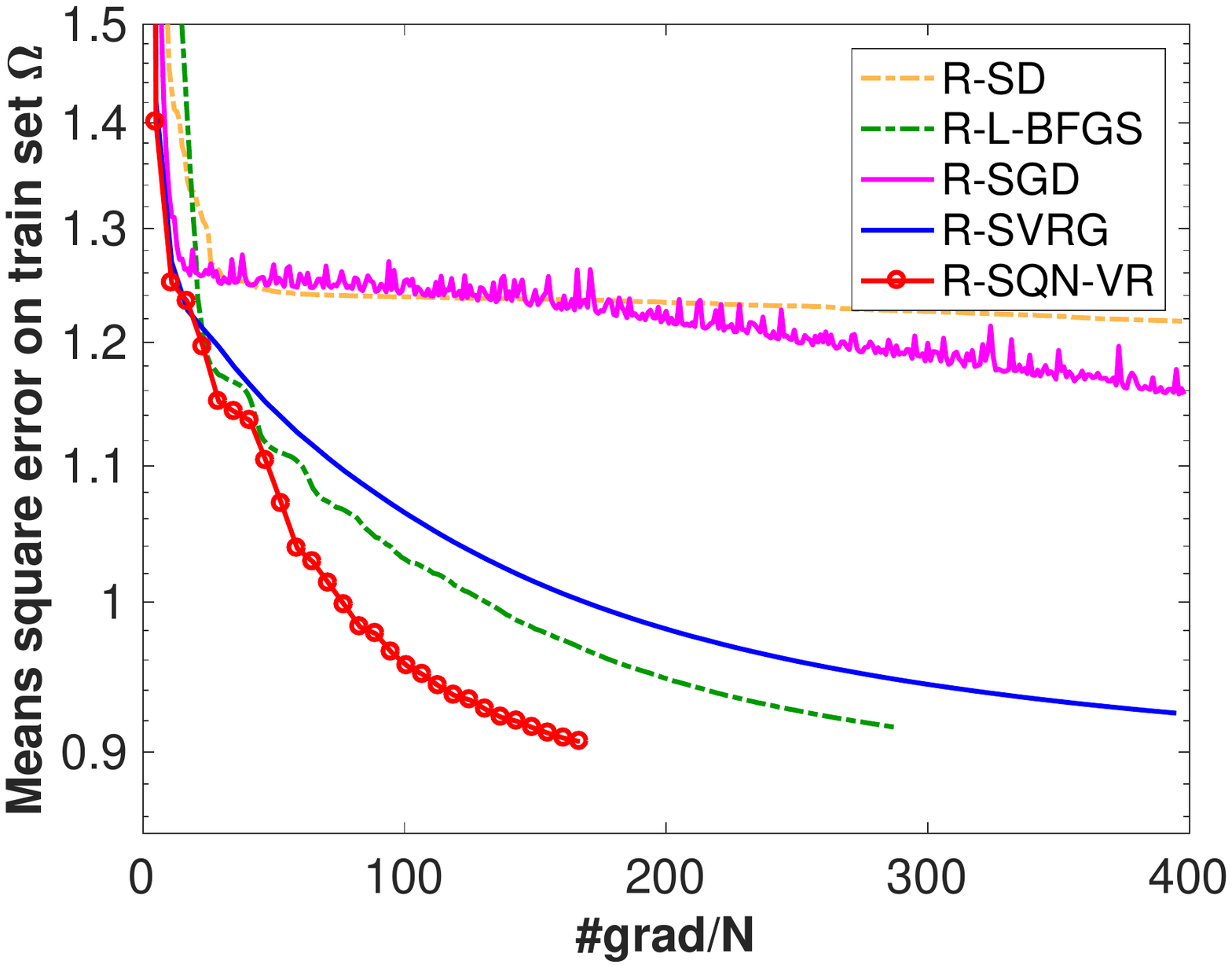}\\
\vspace*{-1.5cm}
{(a-1) run 1}
\end{center}
\end{minipage}
\begin{minipage}{0.320\hsize}
\begin{center}
\includegraphics[width=1.1\hsize]{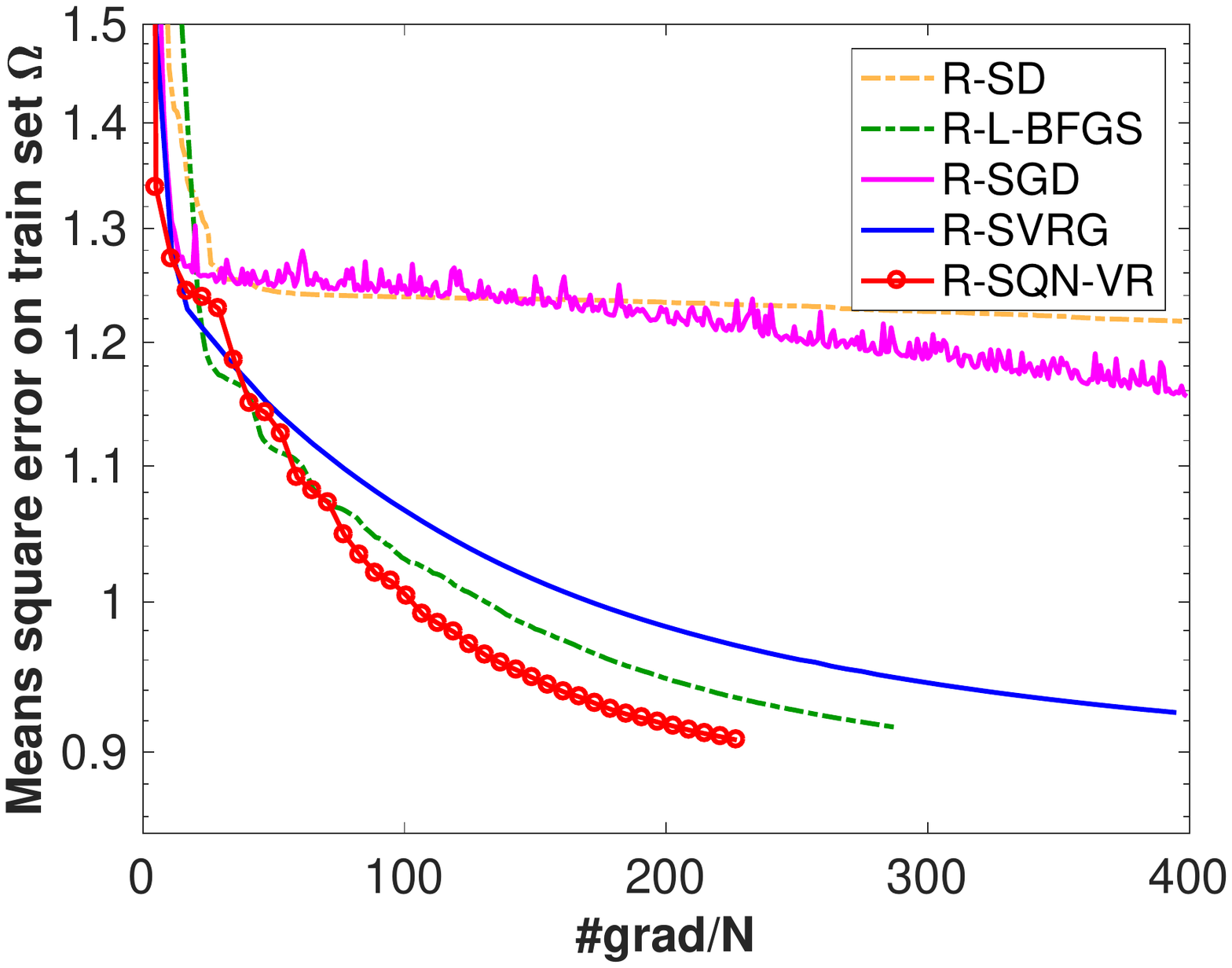}\\
\vspace*{-1.5cm}
{(a-2) run2}
\end{center}
\end{minipage}
\begin{minipage}{0.320\hsize}
\begin{center}
\includegraphics[width=1.1\hsize]{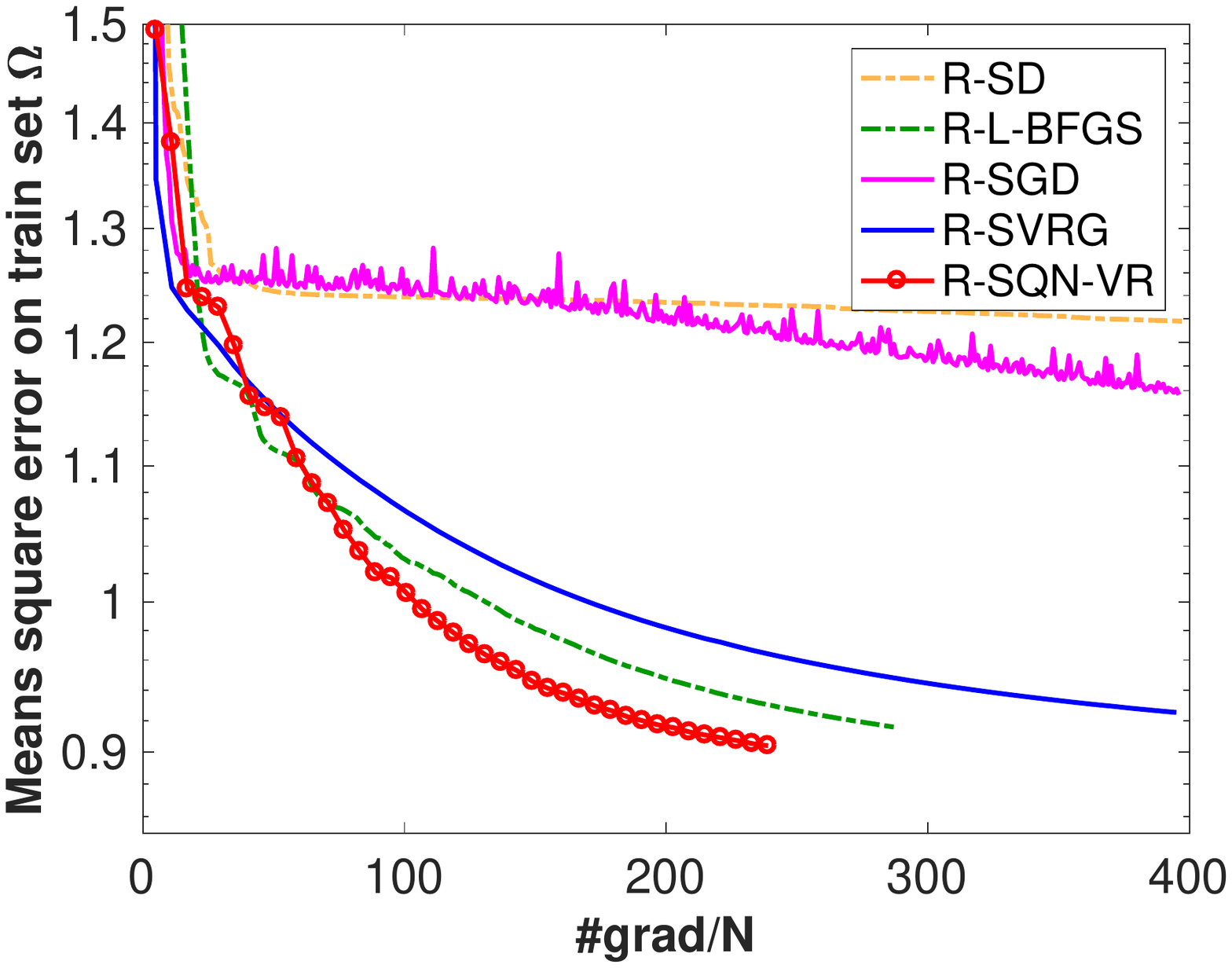}\\
\vspace*{-1.5cm}
{(a-3) run 3}
\end{center}
\end{minipage}\\
\vspace*{-1cm}

%%%%%%
\begin{minipage}{0.320\hsize}
\begin{center}
\includegraphics[width=1.1\hsize]{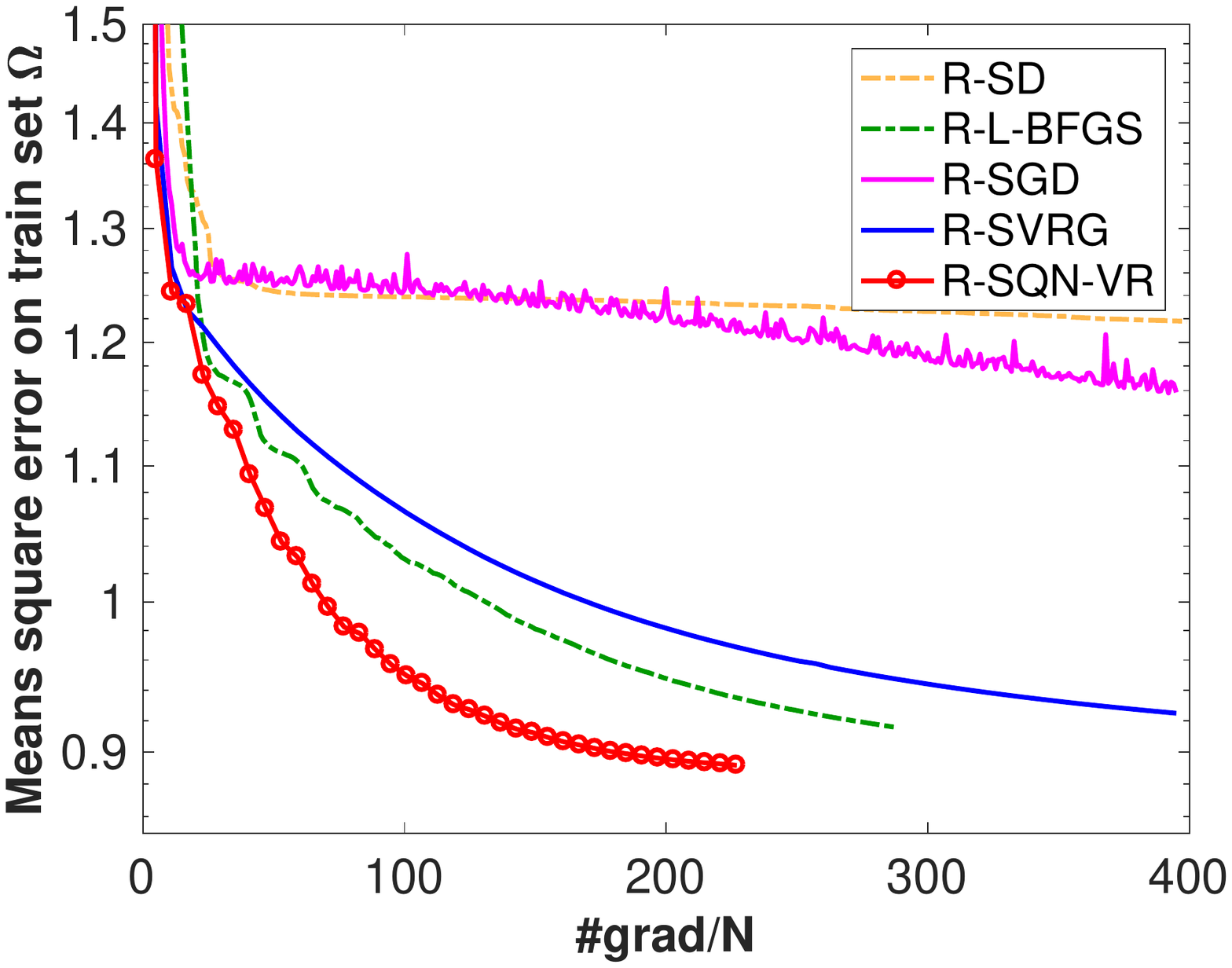}\\
\vspace*{-1.5cm}
{(a-4) run 4}
\end{center}
\end{minipage}
\begin{minipage}{0.320\hsize}
\begin{center}
\includegraphics[width=1.1\hsize]{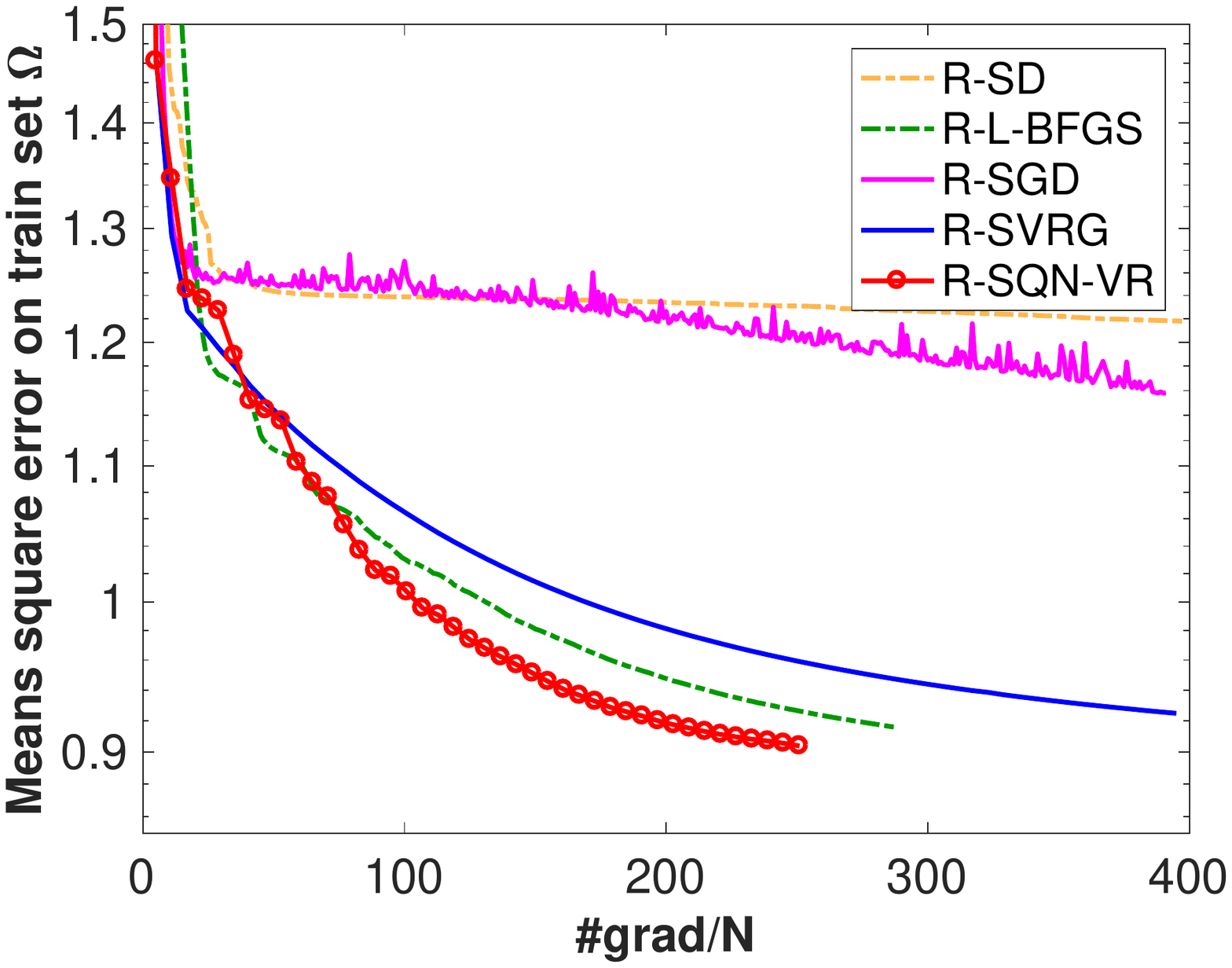}\\
\vspace*{-1.5cm}
{(a-5) run 5}
\end{center}
\end{minipage}
\vspace*{0.3cm}

{\bf(a) MSE on train set $\Omega$}
\vspace*{-1.0cm}

\begin{minipage}{0.320\hsize}
\begin{center}
\includegraphics[width=1.1\hsize]{results_pdf/mc/movielens/rank_10/final/mc_movielens_test_MSE_N3952_d6040_r10-run5.pdf}\\
\vspace*{-1.5cm}
{(b-1) run 1}
\end{center}
\end{minipage}
\begin{minipage}{0.320\hsize}
\begin{center}
\includegraphics[width=1.1\hsize]{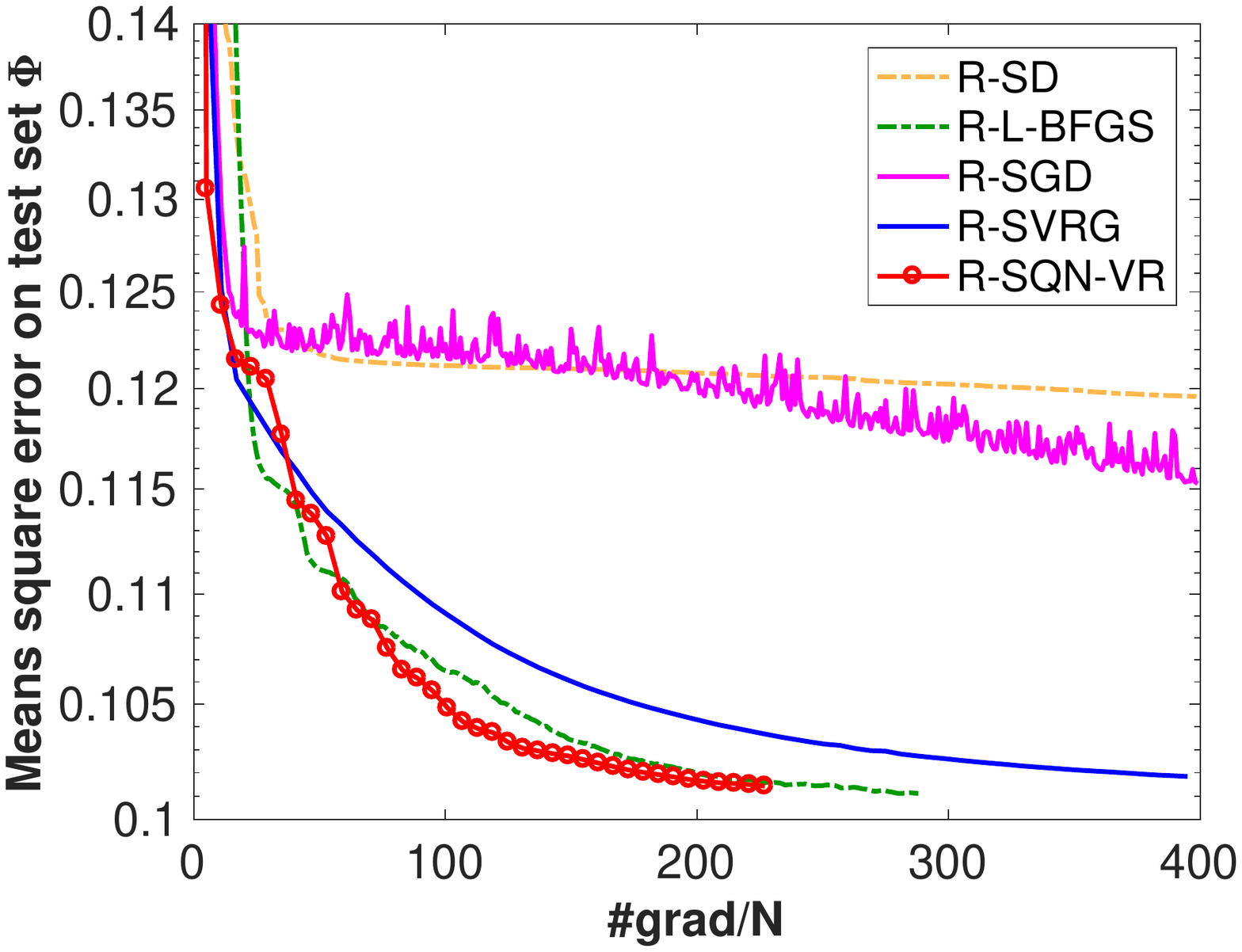}\\
\vspace*{-1.5cm}
{(b-2) run2}
\end{center}
\end{minipage}
\begin{minipage}{0.320\hsize}
\begin{center}
\includegraphics[width=1.1\hsize]{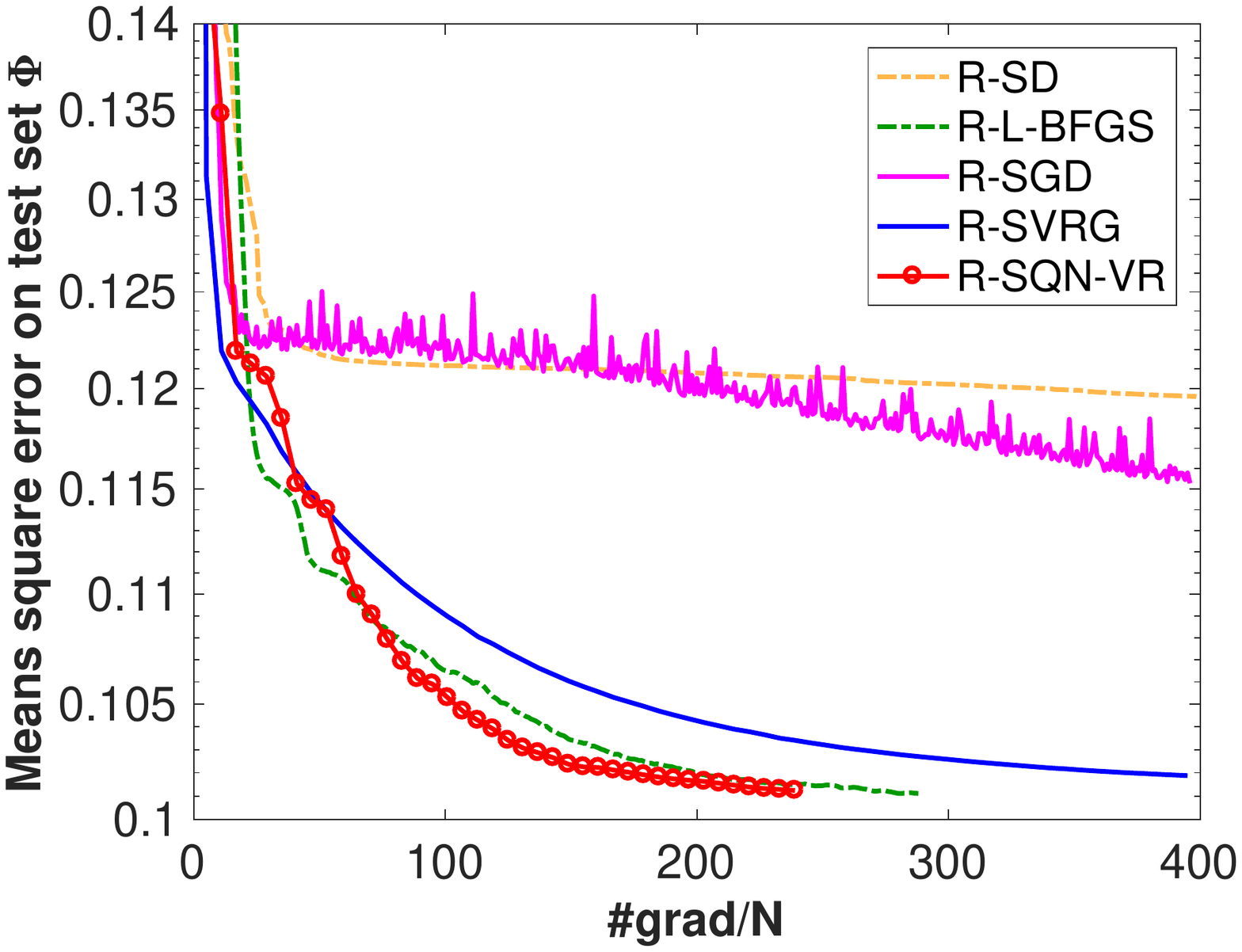}\\
\vspace*{-1.5cm}
{(b-3) run 3}
\end{center}
\end{minipage}\\
\vspace*{-1cm}

%%%%%%
\begin{minipage}{0.320\hsize}
\begin{center}
\includegraphics[width=1.1\hsize]{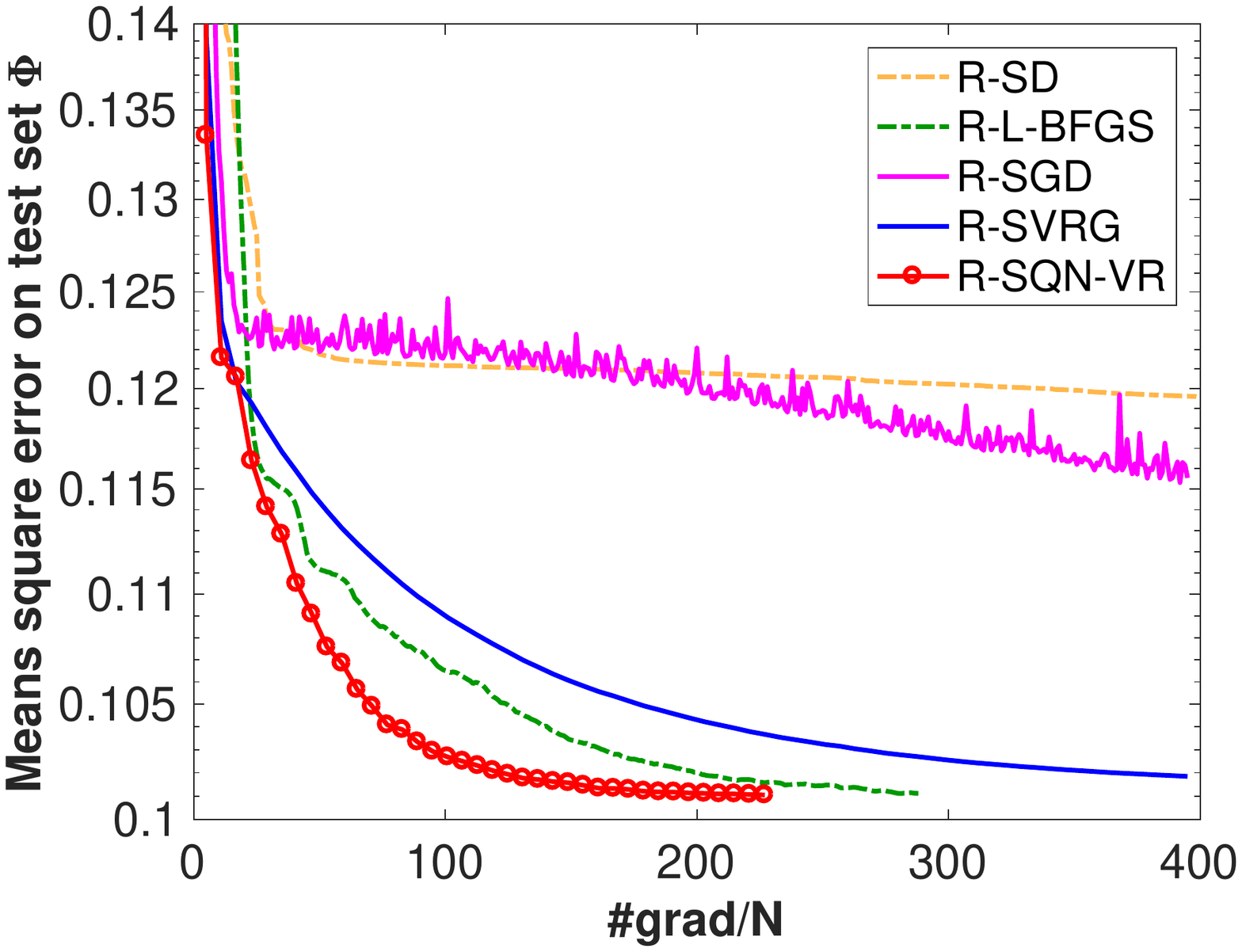}\\
\vspace*{-1.5cm}
{(b-4) run 4}
\end{center}
\end{minipage}
\begin{minipage}{0.320\hsize}
\begin{center}
\includegraphics[width=1.1\hsize]{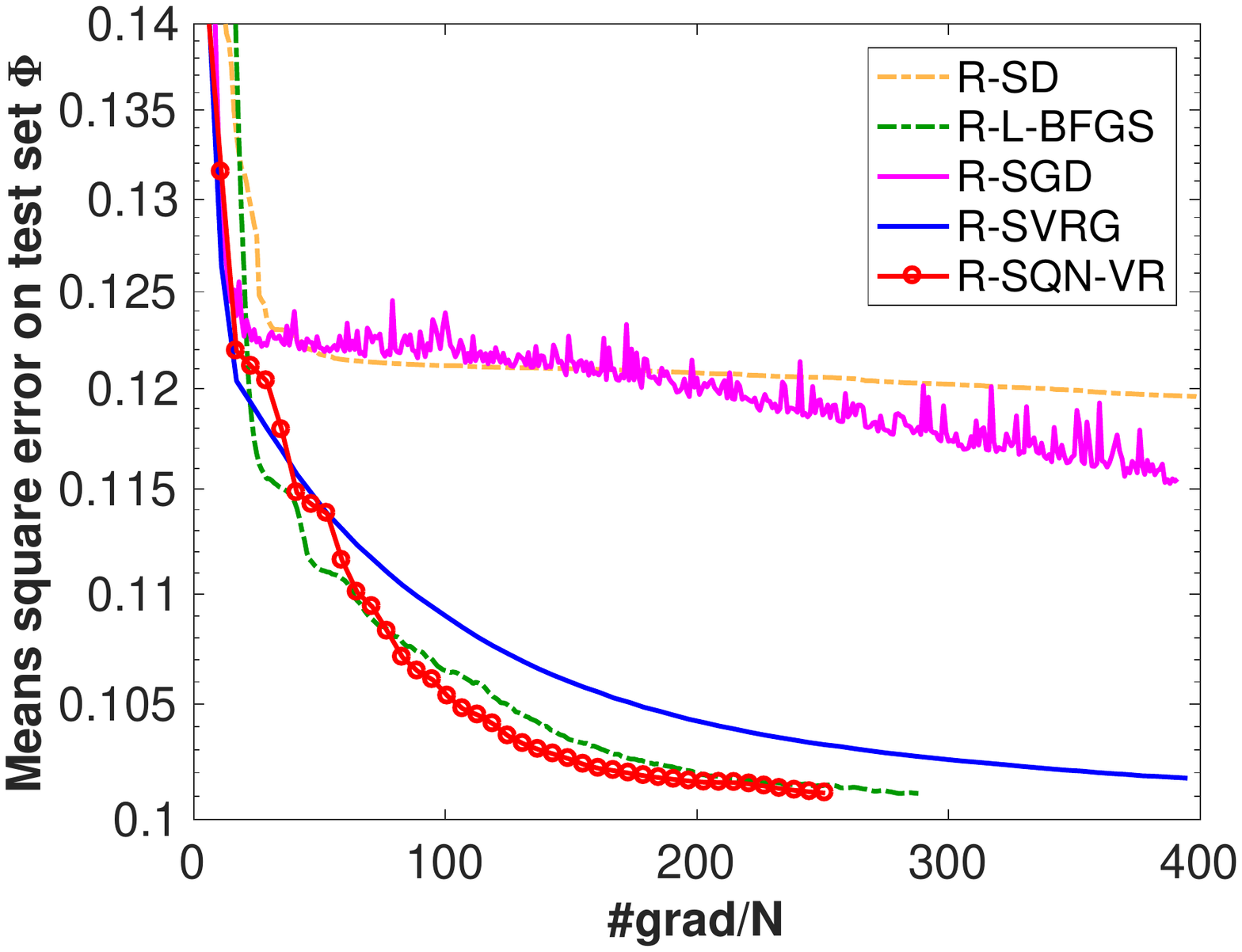}\\
\vspace*{-1.5cm}
{(b-5) run 5}
\end{center}
\end{minipage}
\vspace*{0.3cm}

{\bf (b) MSE on test set $\Phi$}

\caption{Performance evaluations on low-rank MC problem ({\bf MC-R1: lower rank}).}
\label{Addfig:MC_MovieLens_rank10}
\end{center}
\end{figure*}

\begin{figure*}[htbp]
\vspace*{-1.5cm}
\begin{center}
\begin{minipage}{0.320\hsize}
\begin{center}
\includegraphics[width=1.1\hsize]{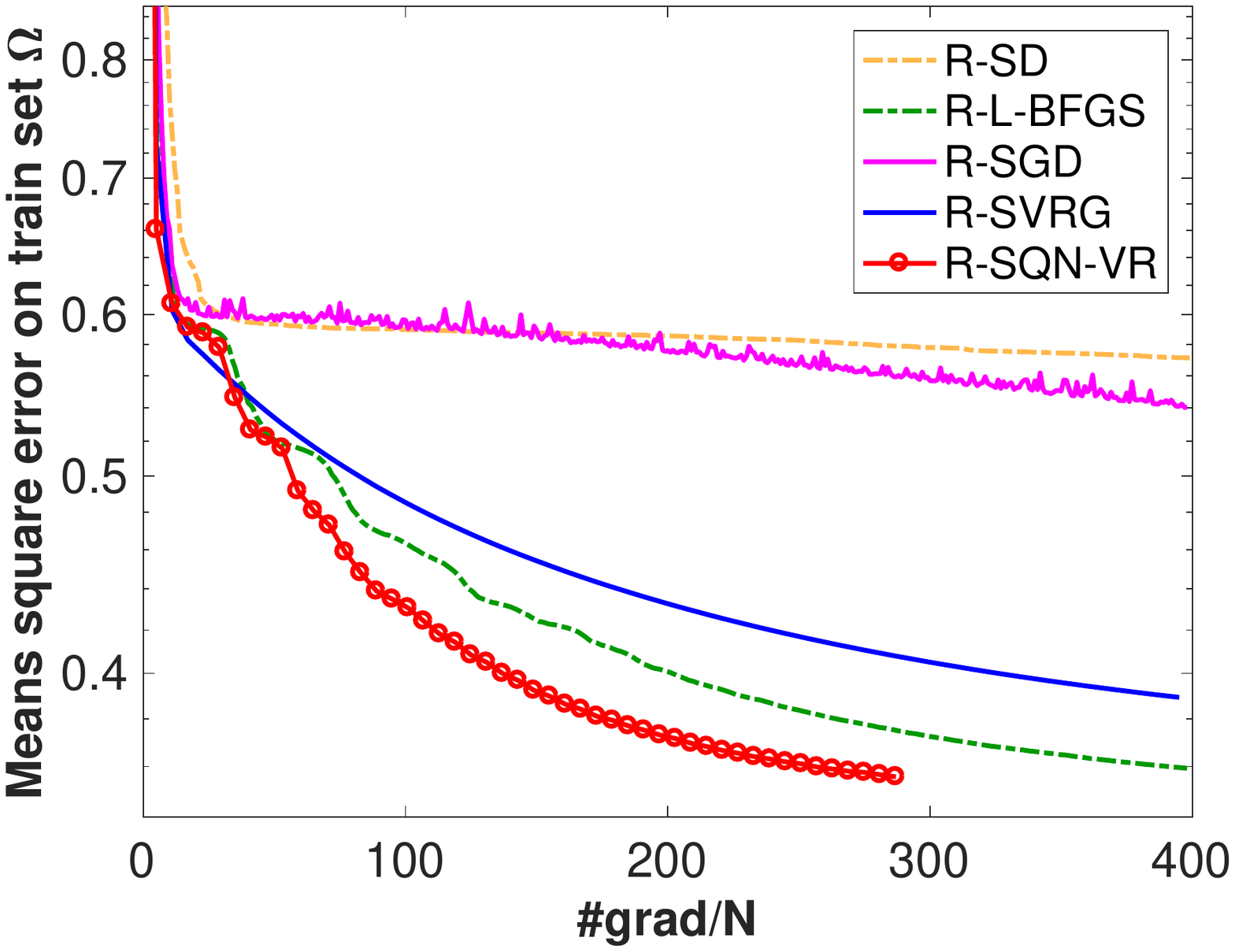}\\
\vspace*{-1.5cm}
{(a-1) run 1}
\end{center}
\end{minipage}
\begin{minipage}{0.320\hsize}
\begin{center}
\includegraphics[width=1.1\hsize]{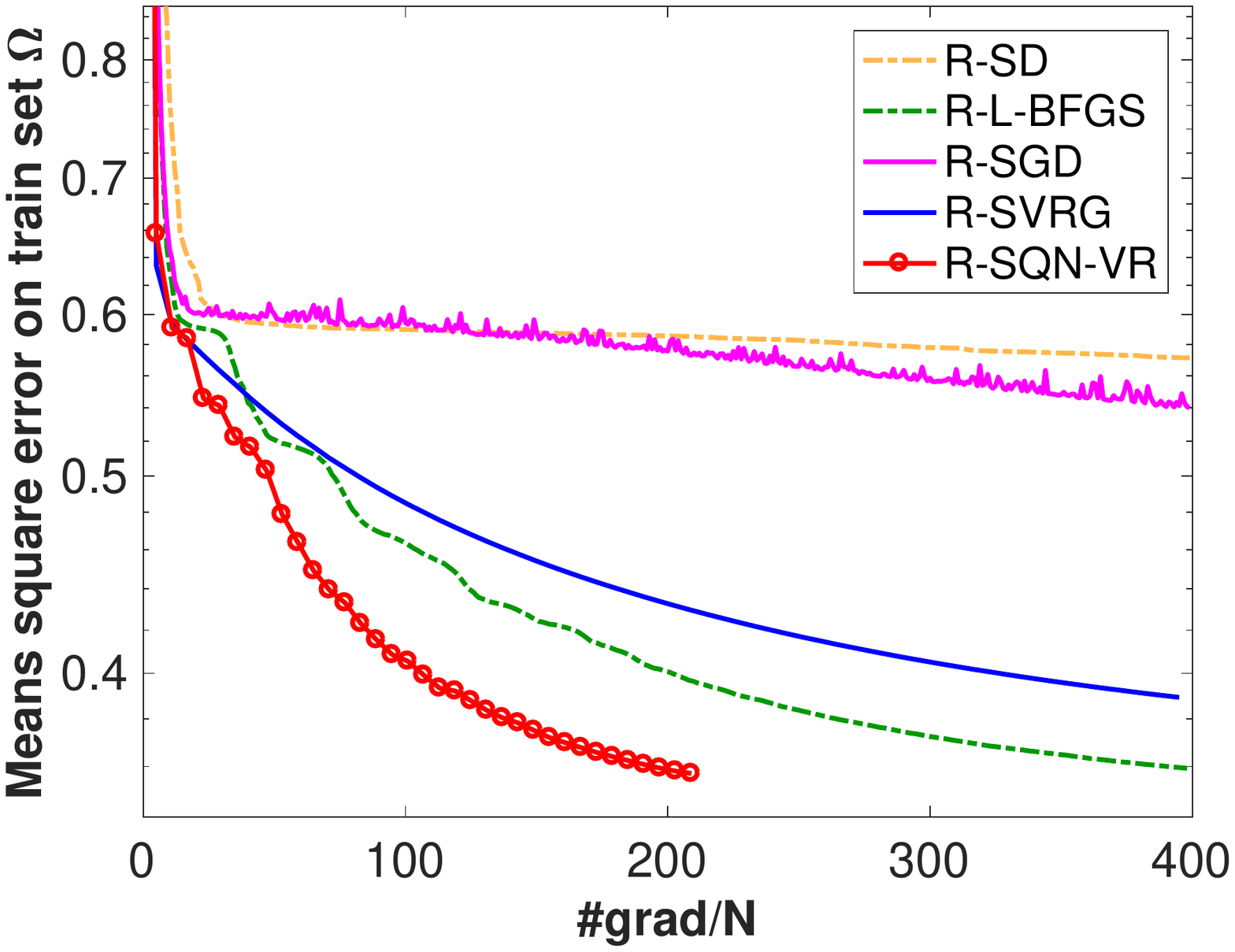}\\
\vspace*{-1.5cm}
{(a-2) run2}
\end{center}
\end{minipage}
\begin{minipage}{0.320\hsize}
\begin{center}
\includegraphics[width=1.1\hsize]{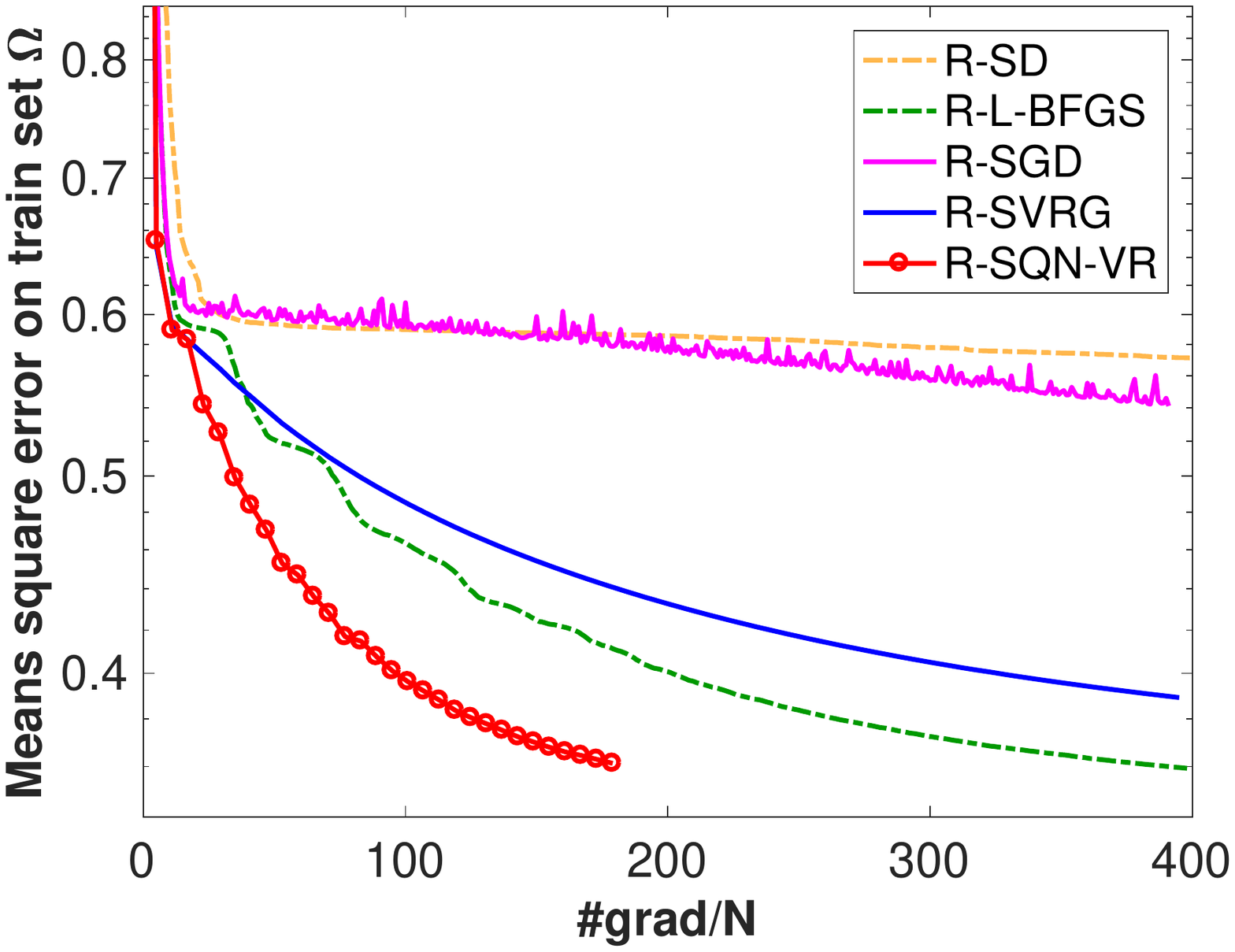}\\
\vspace*{-1.5cm}
{(a-3) run 3}
\end{center}
\end{minipage}\\
\vspace*{-1cm}

%%%%%%
\begin{minipage}{0.320\hsize}
\begin{center}
\includegraphics[width=1.1\hsize]{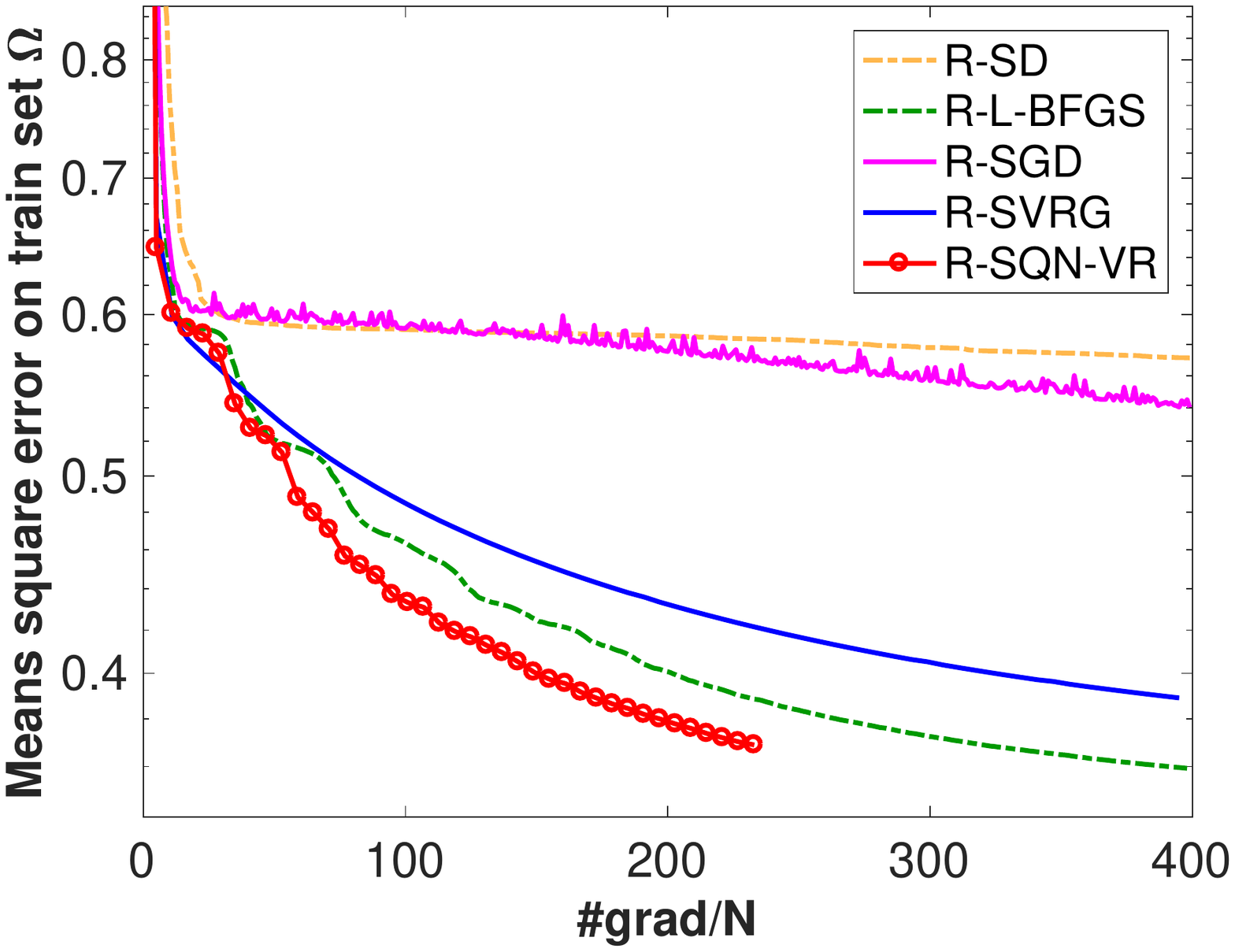}\\
\vspace*{-1.5cm}
{(a-4) run 4}
\end{center}
\end{minipage}
\begin{minipage}{0.320\hsize}
\begin{center}
\includegraphics[width=1.1\hsize]{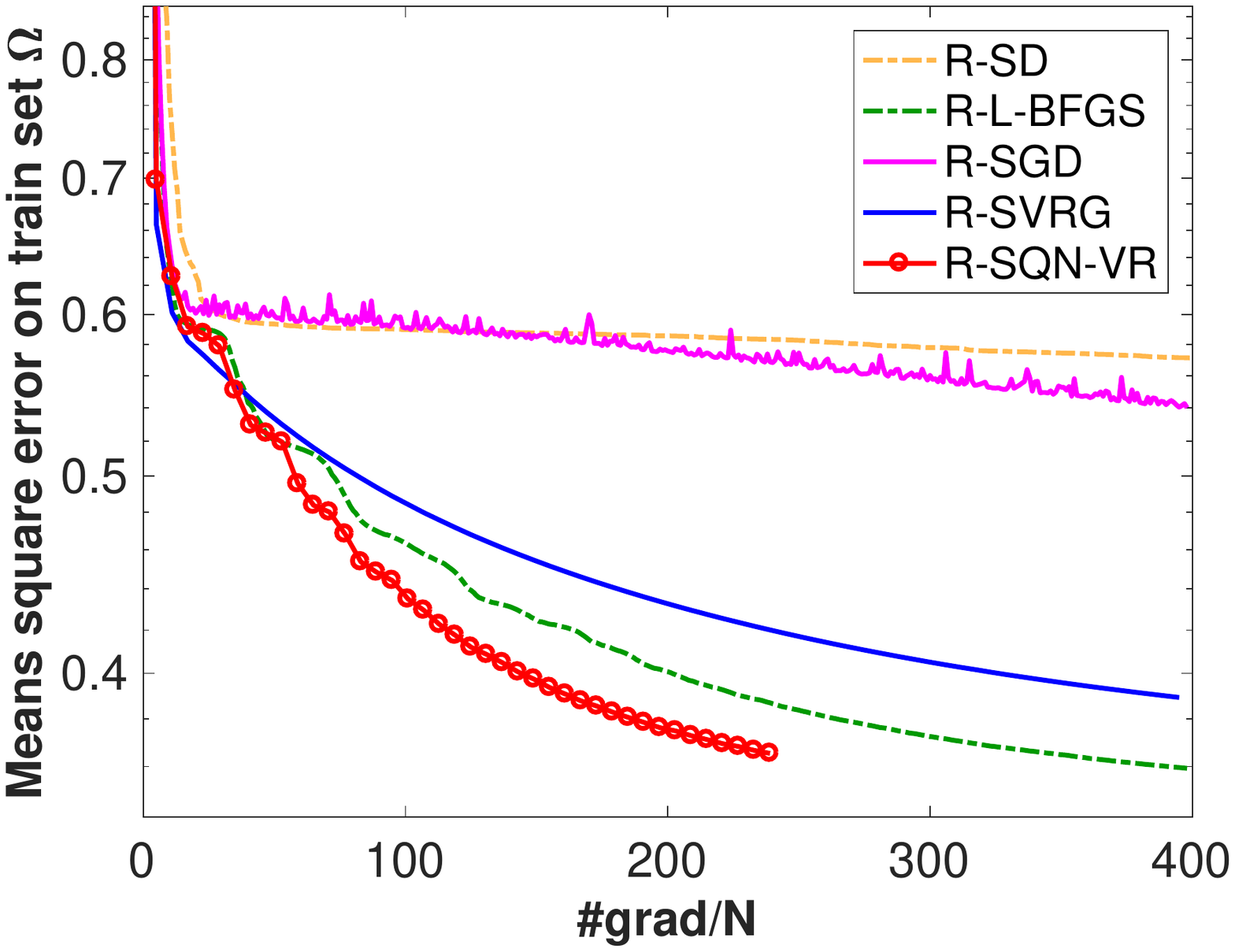}\\
\vspace*{-1.5cm}
{(a-5) run 5}
\end{center}
\end{minipage}
\vspace*{0.3cm}

{\bf(a) MSE on train set $\Omega$}
\vspace*{-1.0cm}

\begin{minipage}{0.320\hsize}
\begin{center}
\includegraphics[width=1.1\hsize]{results_pdf/mc/movielens/rank_20/mc_movielens_train_MSE_N3952_d6040_r20-run5.pdf}\\
\vspace*{-1.5cm}
{(b-1) run 1}
\end{center}
\end{minipage}
\begin{minipage}{0.320\hsize}
\begin{center}
\includegraphics[width=1.1\hsize]{results_pdf/mc/movielens/rank_20/mc_movielens_train_MSE_N3952_d6040_r20-run1.pdf}\\
\vspace*{-1.5cm}
{(b-2) run2}
\end{center}
\end{minipage}
\begin{minipage}{0.320\hsize}
\begin{center}
\includegraphics[width=1.1\hsize]{results_pdf/mc/movielens/rank_20/mc_movielens_train_MSE_N3952_d6040_r20-run2.pdf}\\
\vspace*{-1.5cm}
{(b-3) run 3}
\end{center}
\end{minipage}\\
\vspace*{-1cm}

%%%%%%
\begin{minipage}{0.320\hsize}
\begin{center}
\includegraphics[width=1.1\hsize]{results_pdf/mc/movielens/rank_20/mc_movielens_train_MSE_N3952_d6040_r20-run6.pdf}\\
\vspace*{-1.5cm}
{(b-4) run 4}
\end{center}
\end{minipage}
\begin{minipage}{0.320\hsize}
\begin{center}
\includegraphics[width=1.1\hsize]{results_pdf/mc/movielens/rank_20/mc_movielens_train_MSE_N3952_d6040_r20-run4.pdf}\\
\vspace*{-1.5cm}
{(b-5) run 5}
\end{center}
\end{minipage}
\vspace*{0.3cm}

{\bf(b) MSE on test set $\Phi$}

\caption{Performance evaluations on low-rank MC problem ({\bf MC-R2: higher rank}).}
\label{Addfig:MC_MovieLens_rank20}
\end{center}
\end{figure*}

\end{document}